\documentclass{article} 
\pdfoutput=1
\usepackage[T1]{fontenc}
\usepackage{geometry}
\usepackage{array}
\usepackage{float}
\usepackage{booktabs}
\usepackage{mathtools}
\usepackage{url}
\usepackage{multirow}
\usepackage{dsfont}
\usepackage{amsmath}
\usepackage{amsthm}
\usepackage{amssymb}
\usepackage{graphicx}
\usepackage{tablefootnote}
\usepackage{setspace}

\makeatletter

\providecommand{\tabularnewline}{\\}
\floatstyle{ruled}
\newfloat{algorithm}{tbp}{loa}
\providecommand{\algorithmname}{Algorithm}
\floatname{algorithm}{\protect\algorithmname}

\theoremstyle{plain}
\newtheorem{thm}{\protect\theoremname}[section]
\theoremstyle{plain}
\newtheorem{cor}[thm]{\protect\corollaryname}
\theoremstyle{plain}
\newtheorem{lem}[thm]{\protect\lemmaname}
\ifx\proof\undefined
\newenvironment{proof}[1][\protect\proofname]{\par
	\normalfont\topsep6\p@\@plus6\p@\relax
	\trivlist
	\itemindent\parindent
	\item[\hskip\labelsep\scshape #1]\ignorespaces
}{%
	\endtrivlist\@endpefalse
}
\providecommand{\proofname}{Proof}
\fi
\theoremstyle{remark}
\newtheorem{rem}[thm]{\protect\remarkname}

\usepackage{hyperref}
\usepackage{thmtools}

\@ifundefined{date}{}{\date{}}

\@ifundefined{showcaptionsetup}{}{%
 \PassOptionsToPackage{caption=false}{subfig}}
\usepackage{subfig}
\makeatother

\providecommand{\corollaryname}{Corollary}
\providecommand{\lemmaname}{Lemma}
\providecommand{\remarkname}{Remark}
\providecommand{\theoremname}{Theorem}

\global\long\def\bH{\mathbb{\mathbf{\mathbf{H}}}}%
\global\long\def\bD{\mathbb{\mathbf{\mathbf{D}}}}%
\global\long\def\bE{\mathbb{\mathbf{\mathbf{E}}}}%
\global\long\def\hH{\mathbb{\widehat{\bH}}}%
\global\long\def\tH{\widetilde{\bH}}%
\global\long\def\E{\mathbb{\mathbb{E}}}%
\global\long\def\F{\mathcal{F}}%
\global\long\def\norm#1{\|#1\|}%
\global\long\def\R{\mathbb{R}}%
\global\long\def\algnameold{\text{META-STORM-SG}}%
\global\long\def\algnamenew{\text{META-STORM}}%
\global\long\def\algnamena{\text{META-STORM-NA}}%
\global\long\def\domxi{\mathcal{D}}%
\global\long\def\hG{\widehat{G}}%
\global\long\def\hS{\widehat{\sigma}}%
\global\long\def\rd{\hS}%
\global\long\def\rn{\hG}%
\global\long\def\hp{\widehat{p}}%
\global\long\def\xout{x_{\text{out}}}%

\title{META-STORM: Generalized Fully-Adaptive Variance Reduced SGD for Unbounded
Functions}

\author{ 
Zijian Liu\thanks{Equal contribution. New York University, Stern School of Business,
\texttt{{zl3067@nyu.edu}}. Work done while at Boston University.} \\ \and 
Ta Duy Nguyen\thanks{Equal contribution. Department of Computer Science, Boston University,
\texttt{{taduy@bu.edu}.}}  \\ \and 
Thien Hang Nguyen\thanks{Equal contribution. Khoury College of Computer and Information Science, Northeastern University,
\texttt{{nguyen.thien@northeastern.edu}.}}  \\ 
\end{tabular}\hfil\linebreak[4]\hfil
\begin{tabular}[t]{l} \ignorespaces
Alina Ene\thanks{Department of Computer Science, Boston University, \texttt{{aene@bu.edu}.}} \\ \and 
Huy L. Nguyen\thanks{Khoury College of Computer and Information Science, Northeastern University,
\texttt{{hu.nguyen@northeastern.edu}.}}}

\begin{document}

\maketitle
\begin{abstract}
We study the application of variance reduction (VR) techniques to
general non-convex stochastic optimization problems. In this setting,
the recent work STORM \cite{cutkosky2019momentum} overcomes the
drawback of having to compute gradients of ``mega-batches'' that
earlier VR methods rely on. There, STORM utilizes recursive momentum
to achieve the VR effect and is then later made fully adaptive in
STORM+ \cite{levy2021storm+}, where full-adaptivity removes the
requirement for obtaining certain problem-specific parameters such
as the smoothness of the objective and bounds on the variance and
norm of the stochastic gradients in order to set the step size. However,
STORM+ crucially relies on the assumption that the function values
are bounded, excluding a large class of useful functions. In this
work, we propose $\algnamenew$, a generalized framework of STORM+
that removes this bounded function values assumption while still attaining
the optimal convergence rate for non-convex optimization. $\algnamenew$
not only maintains full-adaptivity, removing the need to obtain problem
specific parameters, but also improves the convergence rate's dependency
on the problem parameters. Furthermore, $\algnamenew$ can utilize
a large range of parameter settings that subsumes previous methods
allowing for more flexibility in a wider range of settings. Finally,
we demonstrate the effectiveness of META-STORM through experiments
across common deep learning tasks. Our algorithm improves upon the
previous work STORM+ and is competitive with widely used algorithms
after the addition of per-coordinate update and exponential moving
average heuristics.
\end{abstract}

\section{Introduction }

In this paper, we consider the stochastic optimization problem in
the form 
\begin{align}
\min_{x\in\R^{d}}F(x) & :=\E_{\xi\sim\domxi}\left[f(x,\xi)\right],\label{eq:problem}
\end{align}
where $F:\R^{d}\to\R$ is possibly non-convex. We assume only access
to a first-order stochastic oracle via sample functions $f(x,\xi)$,
where $\xi$ comes from a distribution $\domxi$ representing the
randomness in the sampling process. Optimization problems of this
form are ubiquitous in machine learning and deep learning. Empirical
risk minimization (ERM) is one instance, where $F(x)$ is the loss
function that can be evaluated by a sample or a minibatch represented
by $\xi$.

An important advance in solving Problem (\ref{eq:problem}) is the
recent development of variance reduction (VR) techniques that improve
the convergence rate to critical points of vanilla SGD from $O(1/T^{1/4})$
to $O(1/T^{1/3})$ \cite{fang2018spider,li2021page} for the class
of \emph{mean-squared smooth} functions \cite{arjevani2019lower}.
In contrast to earlier VR algorithms which often require the computation
of the gradients over large batches, recent methods such as \cite{cutkosky2019momentum,levy2021storm+,huang2021super}
avoid this drawback by using a weighted average of past gradients,
often known as \emph{momentum}. When the weights are selected appropriately,
momentum reduces the error in the gradient estimates which improves
the convergence rate.

A different line of work on adaptive methods \cite{duchi2011adaptive, kingma2014adam},
some of which incorporate momentum techniques, have shown tremendous
success in practice. These adaptive methods remove the burden of obtaining
certain problem-specific parameters, such as smoothness, in order
to set the right step size to guarantee convergence. STORM+ \cite{levy2021storm+}
is the first algorithm to bridge the gap between \emph{fully-adaptive}
algorithms and VR methods, achieving the variance-reduced convergence
rate of $O(1/T^{1/3})$ while not requiring knowledge of any problem-specific
parameter. This is also the first work to demonstrate the interplay
between adaptive momentum and step sizes to adapt to the problem's
structure, while still achieving the VR rate. However, STORM+ relies
on a strong assumption that the function values are bounded, which
generally does not hold in practice. Moreover, the convergence rate
of STORM+ has high polynomial dependencies on the problem parameters,
compared to what can be achieved by appropriately configuring the
step sizes and momentum parameters given knowledge of the problem
parameters (see Section \ref{subsec:Non-adaptive-algorithm}).

\textbf{Our contributions:} In this work, we propose $\algnameold$
and $\algnamenew$, two flexible algorithmic frameworks that attain
the optimal variance-reduced convergence rate for general non-convex
objectives. Both of them generalize STORM+ by allowing a wider range
of parameter selection and removing the restrictive bounded function
value assumption while maintaining its desirable fully-adaptive property
-- eliminating the need to obtain any problem-specific parameter.
These have been enabled via our novel analysis framework that also
establishes a convergence rate with much better dependency on the
problem parameters. We present a comparison of $\algnamenew$ and
its sibling $\algnameold$ against recent VR methods in Table \ref{tab:Algs-Comparison}.
In the appendix, we propose another algorithm, $\text{META-STORM-NA}$,
with even less restrictive assumptions; however, with a tradeoff of
losing the adaptivity to the variance parameter. 

We complement our theoretical results with experiments across three
common tasks: image classification, masked language modeling, and
sentiment analysis. Our algorithms improve upon the previous work,
STORM+. Furthermore, the addition of heuristics such as exponential
moving average and per-coordinate updates improves our algorithms'
generalization performance. These versions of our algorithms are shown
to be competitive with widely used algorithms such as Adam and AdamW.

\begin{table}[t]
\begin{onehalfspace}
\caption{{\small{}Comparison of the convergence rate after $T$ iterations
under constant success probability. The assumptions and definitions
of the parameters referenced can be found in Section \ref{sec:preliminaries}.
Assumptions }\textbf{\small{}1}{\small{} and }\textbf{\small{}2}{\small{}
are used in all algorithms, thus we leave them out from the table.\label{tab:Algs-Comparison}}}

\end{onehalfspace}
\centering{}{\footnotesize{}}%
\begin{tabular}{>{\raggedright}m{0.16\textwidth}c>{\centering}p{0.5\textwidth}c}
\hline 
\noalign{\vskip\doublerulesep}
{\footnotesize{}Methods} & {\footnotesize{}Adaptive?} & {\footnotesize{}Convergence rate} & {\footnotesize{}Assumptions}\tabularnewline[4pt]
\hline 
\hline 
\noalign{\vskip\doublerulesep}
\multirow{2}{0.16\textwidth}{{\footnotesize{}STORM \cite{cutkosky2019momentum}}} & \multirow{2}{*}{{\footnotesize{}$\times$}} & {\footnotesize{}$O\left(\frac{\kappa^{1/2}+\kappa^{3/4}\hG^{-1/2}+\sigma+\hG\log^{3/4}T}{T^{1/2}}+\frac{\sigma^{1/3}}{T^{1/3}}\right)$} & \multirow{2}{*}{\textbf{\footnotesize{}3', 4}}\tabularnewline[4pt]
\noalign{\vskip\doublerulesep}
 &  & {\footnotesize{}$\kappa=O\left(\beta\left(F(x_{1})-F^{*}\right)\right)$} & \tabularnewline[4pt]
\hline 
\noalign{\vskip\doublerulesep}
\multirow{3}{0.16\textwidth}{{\footnotesize{}Super-ADAM \cite{huang2021super}}} & \multirow{3}{*}{{\footnotesize{}$\times$}} & {\footnotesize{}$O\left(\left(\kappa^{1/2}+\sigma\log T\right)\left(\frac{1}{T^{1/2}}+\frac{1}{T^{1/3}}\right)\right)$} & \multirow{3}{*}{\textbf{\footnotesize{}3'}}\tabularnewline[4pt]
\noalign{\vskip\doublerulesep}
 &  & {\footnotesize{}$\kappa=O\left(\beta\left(F(x_{1})-F^{*}\right)\right)$} & \tabularnewline[4pt]
\noalign{\vskip\doublerulesep}
 &  & {\footnotesize{}Does not adapt to $\sigma$} & \tabularnewline[4pt]
\hline 
\noalign{\vskip\doublerulesep}
\multirow{3}{0.16\textwidth}{{\footnotesize{}STORM+ \cite{levy2021storm+}}} & \multirow{3}{*}{{\footnotesize{}$\checkmark$}} & {\footnotesize{}$O\left(\frac{\kappa_{1}}{T^{1/2}}+\frac{\kappa_{2}\sigma^{1/3}}{T^{1/3}}\right)$} & \multirow{3}{*}{\textbf{\footnotesize{}3', 4, 6}}\tabularnewline[4pt]
\noalign{\vskip\doublerulesep}
 &  & {\footnotesize{}$\kappa_{1}=O\left(\beta^{9/4}+\hG^{5}+\beta^{3/2}\hG^{6}+B^{9/8}\right)$} & \tabularnewline[4pt]
\noalign{\vskip\doublerulesep}
 &  & {\footnotesize{}$\kappa_{2}=O\left(\beta^{3/2}+B^{3/4}\right)$} & \tabularnewline[4pt]
\hline 
\hline 
\noalign{\vskip\doublerulesep}
\multirow{3}{0.16\textwidth}{{\footnotesize{}$\algnameold$, $p=\frac{1}{2}$ (Ours)}\tablefootnote{{\footnotesize{}This bound holds when $\sigma^{2}>0$ and $T$ is
large enough.}}} & \multirow{3}{*}{{\footnotesize{}$\checkmark$}} & {\footnotesize{}$O\left(\left(\kappa_{1}+\kappa_{2}\log\left(1+\sigma^{2}T\right)\right)\left(\frac{1}{T^{1/2}}+\frac{\sigma^{1/3}}{T^{1/3}}\right)\right)$} & \multirow{3}{*}{\textbf{\footnotesize{}3}{\footnotesize{},}\textbf{\footnotesize{}
4}}\tabularnewline[4pt]
\noalign{\vskip\doublerulesep}
 &  & {\footnotesize{}$\kappa_{1}=O\left(F(x_{1})-F^{*}+\sigma^{2}+\hG^{2}+\kappa_{2}\log\kappa_{2}\right)$} & \tabularnewline[4pt]
\noalign{\vskip\doublerulesep}
 &  & {\footnotesize{}$\kappa_{2}=O((1+\hG^{2})\beta)$} & \tabularnewline[4pt]
\hline 
\noalign{\vskip\doublerulesep}
\multirow{3}{0.16\textwidth}{{\footnotesize{}$\algnamenew$, $p=\frac{1}{2}$ (Ours)}} & \multirow{3}{*}{{\footnotesize{}$\checkmark$}} & {\footnotesize{}$O\left(\left(\kappa_{1}+\kappa_{2}\log\left(1+\sigma^{2}T\right)\right)\left(\frac{1}{T^{1/2}}+\frac{\sigma^{1/3}}{T^{1/3}}\right)\right)$} & \multirow{3}{*}{\textbf{\footnotesize{}3}{\footnotesize{},}\textbf{\footnotesize{}
5}}\tabularnewline[4pt]
\noalign{\vskip\doublerulesep}
 &  & {\footnotesize{}$\kappa_{1}=O\left(F(x_{1})-F^{*}+\rd\sigma+\sigma^{2}+\rd^{3}+\kappa_{2}\log\kappa_{2}\right)$} & \tabularnewline[4pt]
\noalign{\vskip\doublerulesep}
 &  & {\footnotesize{}$\kappa_{2}=O((1+\rd^{3})\beta)$} & \tabularnewline[4pt]
\hline 
\end{tabular}{\footnotesize\par}
\end{table}

\subsection{Related work}

\textbf{Variance reduction methods for stochastic non-convex optimization:}
Variance reduction is introduced for non-convex optimization by \cite{allen2016variance,reddi2016stochastic}
in the context of finite sum optimization, achieving faster convergence
over the full gradient descent method. These methods are first improved
by \cite{lei2017non} and later by \cite{fang2018spider,li2021page}
both of which achieve an $O(1/T^{1/3})$ convergence rate, matching
the lower bounds in \cite{arjevani2019lower}. However, these earlier
methods periodically need to compute the full gradient (in the finite-sum
case) or a giant batch at a check point, which can be quite costly.
Shortly after, \cite{cutkosky2019momentum} and \cite{tran2019hybrid}
introduce a different approach that utilizes stochastic gradients
from previous time steps instead of computing the full gradient at
a checkpoints. These methods are framed as momentum-based methods
as they are similar to using a weighted average of the gradient estimates
to achieve the variance reduction. Recently, SUPER-ADAM \cite{huang2021super}
integrates STORM in a larger framework of adaptive algorithms, but
loses adaptivity to the variance parameter $\sigma$. At the same
time, STORM+ \cite{levy2021storm+} proposes a \emph{fully adaptive}
version of STORM, which our work builds upon.

\textbf{Adaptive methods for stochastic non-convex optimization: }Classical
methods, like SGD \cite{ghadimi2013stochastic}, typically require
the knowledge of problem parameters, such as the smoothness and the
variance of the stochastic gradients, to set the step sizes. In contrast,
adaptive methods \cite{duchi2011adaptive, tieleman2012lecture, kingma2014adam}
forgo this requirement: their step sizes only rely on the stochastic
gradients obtained by the algorithms. Although these adaptive methods
are originally designed for \emph{convex} optimization, they enjoy
great successes and popularity in highly \emph{non-convex} practical
applications such as training deep neural networks, often making them
the method of choice in practice. As a result, theoretical understanding
of adaptive methods for non-convex problems has received significant
attention in recent years. The works by \cite{ward2019adagrad,kavis2021high}
propose a convergence analysis of AdaGrad under various assumptions.
Among VR methods, STORM+ is the only fully adaptive algorithm that
does not require knowledge of any problem parameter. Our work builds
on and generalizes STORM+, removing the bounded function value assumption
while obtaining much better dependencies on the problem parameters.

\subsection{Problem definition and assumptions \label{sec:preliminaries}}

We study stochastic non-convex optimization problems for which the
objective function $F:\R^{d}\to\R$ that has form $F(x):=\E_{\xi\sim\domxi}\left[f(x,\xi)\right]$
and $f(\cdot,\xi)$ is a sampling function depending on a random variable
$\xi$ drawn from a distribution $\domxi$. We will omit the writing
of $\domxi$ in $\E_{\xi\sim\domxi}\left[f(x,\xi)\right]$ for simplicity
in the remaining paper. $\|\cdot\|$ represents $\|\cdot\|_{2}$ for
brevity. $\left[T\right]$ is defined as $\left\{ 1,2,\cdots,T\right\} $.

The analysis of our algorithms relies on the following assumptions
1--5:

\textbf{1. Lower bounded function value}: $F^{*}\coloneqq\inf_{x\in\R^{d}}F(x)>-\infty$.

\textbf{2. Unbiased estimator with bounded variance}: We assume to
have access to $\nabla f(x,\xi)$ satisfying $\E_{\xi}\left[\nabla f(x,\xi)\right]=\nabla F(x)$,
$\E_{\xi}\left[\norm{\nabla f(x,\xi)-\nabla F(x)}^{2}\right]\le\sigma^{2}$
for some $\sigma\geq0$.

\textbf{3. Averaged $\beta$-smoothness}: $\E_{\xi}\left[\norm{\nabla f(x,\xi)-\nabla f(y,\xi)}^{2}\right]\le\beta^{2}\norm{x-y}^{2},\forall x,y\in\R^{d}$.

\textbf{4. Bounded stochastic gradients:} $\|\nabla f(x,\xi)\|\leq\widehat{G},\forall x\in\R^{d},\xi\in\mathbf{support}(\domxi)$
for some $\hG\geq0$.

\textbf{5. Bounded stochastic gradient differences}: $\norm{\nabla f(x,\xi)-\nabla f(x,\xi')}\le2\hS,\forall x\in\R^{d},\xi,\xi'\in\mathbf{support}(\domxi)$
for some $\hS\geq0$.

Assumptions 1, 2 and 3 are standard in the VR setting \cite{arjevani2019lower}.
Assumption 5 is weaker than the assumptions made in the prior works
based on the STORM framework \cite{cutkosky2019momentum,levy2021storm+}.
These works assume that the stochastic gradients are bounded, i.e.,
Assumption 4. We note that assumption 4 implies that assumption 5
holds by replacing $\widehat{\sigma}$ by $\hG$, thus we only have
to consider $\widehat{\sigma}=O(\hG)$. To better understand assumption
5, we fix $\xi\in\mathbf{support}(\domxi)$ and consider another $\xi'\sim\domxi$,
then due to the convexity of $\|\cdot\|$, $\|\nabla f(x,\xi)-\nabla F(x)\|=\|\nabla f(x,\xi)-\E_{\xi'}\left[\nabla f(x,\xi')\right]\|\leq\E_{\xi'}\left[\|\nabla f(x,\xi)-\nabla f(x,\xi)'\|\right]\leq2\hS$.
This means assumption 5 implies a stronger version of assumption 2.
For this reason, we can consider $\sigma=O(\widehat{\sigma})$.

Additional assumptions made in the prior works \cite{cutkosky2019momentum,levy2021storm+, huang2021super}
include the following: 

\textbf{3'. Almost surely $\beta$-smooth}: $\norm{\nabla f(x,\xi)-\nabla f(y,\xi)}\leq\beta\|x-y\|,\forall x,y\in\R^{d},\xi\in\mathbf{support}(\domxi)$. 

\textbf{6. Bounded function values:} There exists $B\geq0$ such that
$\left|F(x)-F(y)\right|\leq B$  for all $x,y\in\R^{d}$.

We remark that 3' is strictly stronger than 3 and it is \textbf{NOT}
a standard assumption in \cite{arjevani2019lower}. Moreover, assumption
6, which plays a critical role in the analysis of \cite{levy2021storm+},
is relatively strong and cannot be always satisfied in non-convex
optimization. Our work removes these two restrictive assumptions and
also improves the dependency on the problem parameters. 

\section{Our algorithms\label{sec:Our-algorithm}}

\begin{figure*}[t]

\begin{minipage}[t]{0.49\columnwidth}%
\begin{algorithm}[H]
\caption{$\protect\algnameold$}
\label{alg:fully-adaptive-norm}

\textbf{Input:} Initial point $x_{1}\in\R^{d}$\vspace{2.5pt}

\textbf{Parameters: }$a_{0},b_{0},\eta,p\in[\frac{1}{4},\frac{1}{2}],p+2q=1$

Sample $\xi_{1}\sim\domxi,d_{1}=\nabla f(x_{1},\xi_{1})$

\textbf{for} $t=1,\cdots,T$ \textbf{do:}

$\quad$$a_{t+1}=\left(1+\sum_{i=1}^{t}\frac{\norm{\nabla f(x_{i},\xi_{i})}^{2}}{a_{0}^{2}}\right)^{-\frac{2}{3}}$

$\quad$$b_{t}=(b_{0}^{1/p}+\sum_{i=1}^{t}\|d_{i}\|^{2})^{p}/a_{t+1}^{q}$

$\quad$$x_{t+1}=x_{t}-\frac{\eta}{b_{t}}d_{t}$

$\quad$Sample $\xi_{t+1}\sim\domxi$

$\quad$$d_{t+1}=\nabla f(x_{t+1},\xi_{t+1})+(1-a_{t+1})(d_{t}-\nabla f(x_{t},\xi_{t+1}))$

\textbf{end for}

\textbf{Output $\xout=x_{t}$} where $t\sim\mathrm{Uniform}\left(\left[T\right]\right)$.
\end{algorithm}
\end{minipage}\enskip{}%
\begin{minipage}[t]{0.505\columnwidth}%
\begin{algorithm}[H]
\caption{$\protect\algnamenew$}
\label{alg:fully-adaptive-diff}

\textbf{Input:} Initial point $x_{1}\in\R^{d}$

\textbf{Parameters: }$a_{0},b_{0},\eta,p\in[\frac{3-\sqrt{7}}{2},\frac{1}{2}],p+2q=1$

Sample $\xi_{1}\sim\domxi,d_{1}=\nabla f(x_{1},\xi_{1}),a_{1}=1$ 

\textbf{for} $t=1,\cdots,T$ \textbf{do:}

$\quad$$b_{t}=(b_{0}^{1/p}+\sum_{i=1}^{t}\|d_{i}\|^{2})^{p}/a_{t}^{q}$

$\quad$$x_{t+1}=x_{t}-\frac{\eta}{b_{t}}d_{t}$

$\quad$Sample $\xi_{t+1}\sim\domxi$

$\quad$$a_{t+1}=\left(1+\sum_{i=1}^{t}\frac{\|\nabla f(x_{i},\xi_{i})-\nabla f(x_{i},\xi_{i+1})\|^{2}}{a_{0}^{2}}\right)^{-\frac{2}{3}}$

$\quad$$d_{t+1}=\nabla f(x_{t+1},\xi_{t+1})+(1-a_{t+1})(d_{t}-\nabla f(x_{t},\xi_{t+1}))$

\textbf{end for}

\textbf{Output $\xout=x_{t}$} where $t\sim\mathrm{Uniform}\left(\left[T\right]\right)$.
\end{algorithm}
\end{minipage}\end{figure*}

In this section, we introduce our two main algorithms, $\algnameold$
and $\algnamenew$, shown in Algorithm \ref{alg:fully-adaptive-norm}
and Algorithm \ref{alg:fully-adaptive-diff} respectively. Our algorithms
follow the generic framework of momentum-based variance-reduced SGD
put forward by STORM \cite{cutkosky2019momentum}. The STORM template
incorporates momentum and variance reduction as follows:
\begin{align}
d_{t} & =\underbrace{a_{t}\nabla f(x_{t},\xi_{t})+\left(1-a_{t}\right)d_{t-1}}_{\text{momentum}}+\underbrace{\left(1-a_{t}\right)\left(\nabla f(x_{t},\xi_{t})-\nabla f(x_{t-1},\xi_{t})\right)}_{\text{variance reduction}}\label{eq:update-1}\\
x_{t+1} & =x_{t}-\frac{\eta}{b_{t}}d_{t}.\label{eq:update-2}
\end{align}
The first variant, $\algnameold$, similar to prior works, uses the
gradient norms when setting $a_{t}$ and similarly, requires the strong
assumption on the boundedness of the stochastic gradients. The major
difference lies in the structure of the momentum parameters and the
step sizes and their relationship, which is further developed in the
second algorithm $\algnamenew$ so that assumption 4 can be relaxed
to assumption 5. We now highlight our key algorithmic contributions
and how they depart from prior works. 

A first point of departure is our use of stochastic gradient differences
when setting the momentum parameter $a_{t}$ in $\algnamenew$: prior
works set $a_{t}$ based on the stochastic gradients, while $\algnamenew$
sets $a_{t}$ based on the difference of two gradient estimators taken
at two different time step $\xi_{t-1}$ and $\xi_{t}$ at the same
point $x_{t-1}$. The gradient difference can be viewed as a proxy
for the variance $\sigma^{2}$, which allows us to require the mild
assumption 5 in the analysis. With this choice, our algorithm obtains
the best dependency on the problem parameters. On the other hand,
the coefficient $1-a_{t+1}$ in the update for $d_{t+1}$ now depends
on $\xi_{t+1}$, and addressing this correlation requires a more careful
analysis. The second point of departure is the setting of the step
sizes $b_{t}$ and their relationship to the momentum parameters $a_{t}$
in both $\algnameold$ and $\algnamenew$. We propose a general update
rule $b_{t}=(b_{0}^{1/p}+\sum_{i=1}^{t}\|d_{i}\|^{2})^{p}/a_{t}^{q}$
that allows for a broad range of choices for $p$ and $q$ that subsume
prior works. In practice, different problem domains may benefit from
different choices of $p$ and $q$. Our framework allows us to capture
prior works such as the STORM+ update $b_{t}=(\sum_{i=1}^{t}\|d_{i}\|^{2}/a_{i+1})^{1/3}$
using a different but related choice of momentum parameters and a
simpler update that uses only the current momentum value $a_{t}$
instead of all the previous momentum values $a_{i+1}$ with $i\leq t$.
We further motivate and provide intuition for our algorithmic choices
in Section \ref{sec:Analysis-outline}. We note that our algorithm
uses only the stochastic gradient information received, and it does
not require any knowledge of the problem parameters.

We provide an overview and intuition for our algorithm in Section
\ref{sec:Analysis-outline}, and give the complete analysis in the
appendix. Our analysis departs significantly from prior works such
as STORM+, and it allows us to forgo the bounded function value assumption
and improve the convergence rate's dependency on the problem parameters.
It remains an interesting open question to determine the best convergence
rate that can be achieved when the function values are bounded. 

We can further alleviate assumption 5 in another new algorithm, $\algnamena$
(Algorithm \ref{alg:fully-adaptive-na}), provided in Section \ref{sec:Algorithm-NA}
in the appendix. To the best of our knowledge, $\algnamena$ is the
only adaptive algorithm that enjoys the convergence rate $\widetilde{O}(1/T^{1/3})$
under only the weakest assumptions 1-3. It also allows a wide range
of choices for $p\in\left(0,\frac{1}{2}\right].$ However, the tradeoff
is that the algorithm does not adapt to the variance parameter $\sigma$.
For the detailed analysis, we refer readers to Section \ref{sec:Algorithm-NA}.

Finally, we show the convergence rate obtained by Algorithms \ref{alg:fully-adaptive-norm}
and \ref{alg:fully-adaptive-diff} in the following theorems. The
convergence rates for general $p$ are given in the appendix.
\begin{thm}
\label{thm:Main-SG-convergence-rate}Under the assumptions 1-4 in
Section \ref{sec:preliminaries}, with the choice $p=\frac{1}{2}$
and setting $a_{0}=b_{0}=\eta=1$ to simplify the final bound, $\algnameold$
ensures that
\begin{align*}
\E\left[\|\nabla F(\xout)\|^{\frac{2}{3}}\right] & =O\Big(\frac{W_{1}\mathds{1}\left[(\sigma^{2}T)^{1/3}\leq W_{1}\right]}{T^{1/3}}+\left(W_{2}+W_{3}\log^{\frac{2}{3}}\big(1+\sigma^{2}T\big)\right)\Big(\frac{1}{T^{1/3}}+\frac{\sigma^{2/9}}{T^{2/9}}\Big)\Big)
\end{align*}
where $W_{1}=O\big(F(x_{1})-F^{*}+\sigma^{2}+\rn^{2}+\beta\big(1+\rn^{2}\big)\log\big(\beta+\rn^{2}\beta\big)\big)$,
$W_{2}=O\big((F(x_{1})-F^{*})^{2/3}+\sigma^{4/3}+\rn^{4/3}+(1+\rn^{4/3})\beta^{2/3}\log^{2/3}\big(\beta+\rn^{2}\beta\big)\big)$
and $W_{3}=O\big((1+\rn^{4/3})\beta^{2/3}\big)$.
\end{thm}
We note that when $\sigma^{2}>0$ and $T$ is large enough, the effect
of $W_{1}$ can be eliminated. Combining Theorem \ref{thm:Main-SG-convergence-rate}
and Markov's inequality, we immediately have the following corollary.
\begin{cor}
\label{cor:Main-SG-prob-rate}Under the same setting in Theorem \ref{thm:Main-SG-convergence-rate},
additionally we assume $\sigma^{2}>0$ and $T$ is large enough, then
for any $0<\delta<1$, with probability $1-\delta$
\[
\|\nabla F(\xout)\|\leq O\Big(\frac{\kappa_{1}+\kappa_{2}\log\left(1+\sigma^{2}T\right)}{\delta^{3/2}}\Big(\frac{1}{T^{1/2}}+\frac{\sigma^{1/3}}{T^{1/3}}\Big)\Big)
\]
where $\kappa_{1}=O\big(F(x_{1})-F^{*}+\sigma^{2}+\hG^{2}+\kappa_{2}\log\kappa_{2}\big)$
and $\kappa_{2}=O\big(\big(1+\hG^{2}\big)\beta\big)$.
\end{cor}
\begin{thm}
\label{thm:Main-MS-convergence-rate}Under the assumptions 1--3 and
5 in Section \ref{sec:preliminaries}, with the choice $p=\frac{1}{2}$
and setting $a_{0}=b_{0}=\eta=1$ to simplify the final bound, $\text{\ensuremath{\algnamenew}}$
ensures that
\begin{align*}
\E\left[\|\nabla F(\xout)\|^{\frac{6}{7}}\right] & =O\Big(\left(Q_{1}+Q_{2}\log^{\frac{6}{7}}\left(1+\sigma^{2}T\right)\right)\Big(\frac{1}{T^{3/7}}+\frac{\sigma^{2/7}}{T^{2/7}}\Big)\Big)
\end{align*}
where $Q_{1}=O\big(\big(F(x_{1})-F^{*}\big)^{6/7}+\big(\rd\sigma\big)^{6/7}+\sigma^{12/7}+\rd^{18/7}+\big(1+\rd^{18/7}\big)\beta^{6/7}\log^{6/7}\big(\beta+\rd^{3}\beta\big)$
and $Q_{2}=O\big(\big(1+\rd^{18/7}\big)\beta^{6/7}\big)$.
\end{thm}
Combining Theorem \ref{thm:Main-MS-convergence-rate} and Markov's
inequality, we also have the following corollary.
\begin{cor}
\label{cor:Main-MS-prob-rate}Under the same setting in Theorem \ref{thm:Main-MS-convergence-rate},
then, for any $0<\delta<1$, with probability $1-\delta$
\[
\|\nabla F(\xout)\|\leq O\Big(\frac{\kappa_{1}+\kappa_{2}\log\left(1+\sigma^{2}T\right)}{\delta^{7/6}}\Big(\frac{1}{T^{1/2}}+\frac{\sigma^{1/3}}{T^{1/3}}\Big)\Big)
\]
where $\kappa_{1}=O\big(F(x_{1})-F^{*}+\rd\sigma+\sigma^{2}+\rd^{3}+\kappa_{2}\log\kappa_{2}\big)$
and $\kappa_{2}=O\big(\big(1+\rd^{3}\big)\beta\big)$.
\end{cor}
We emphasize that the aim of our analysis is to provide a convergence
in expectation or with constant probability. In particular, we state
Corollaries \ref{cor:Main-SG-prob-rate} and \ref{cor:Main-MS-prob-rate}
only to give a more intuitive way to see the dependency on the problem
parameters. To boost the success probability and achieve a $\log\frac{1}{\delta}$
dependency on the probability margin, a common approach is to perform
$\log\frac{1}{\delta}$ independent repetitions of the algorithms.

We briefly discuss the difference between the convergence rate of
the two algorithms. We note that these two rates cannot be compared
directly since assumption 4 is stronger than assumption 5. Additionally,
as pointed out in Section \ref{sec:preliminaries}, we have $\hS=O(\hG)$
and thus the term $O(\hS^{3})$ in Corollary \ref{cor:Main-MS-prob-rate}
is $O(\widehat{G}^{3})$, whereas Corollary \ref{cor:Main-SG-prob-rate}
has a $O(\hG^{2})$ term. To give an intuition why an extra higher
order term $W_{1}$ appears in Theorem \ref{thm:Main-SG-convergence-rate}
when $\sigma=0$ compared with Theorem \ref{thm:Main-MS-convergence-rate},
we note that when $\sigma=0$, $d_{t}$ in both algorithms degenerates
to $\nabla F(x_{t})$. However, the coefficient $a_{t+1}$ becomes
$1$ in $\algnamenew$ but does not in $\algnameold$. This discrepancy
leads to $b_{t}$ being larger in $\algnameold$ than in $\algnamenew$,
and moreover the META-STORM $b_{t}$ becomes exactly the same as the
stepsize used in AdaGrad. Due to the larger $b_{t}$ when $\sigma=0$,
it is reasonable to expect a slower convergence rate for $\algnameold$.
The appearance of the term $W_{1}$ reflects that. 

\section{Overview of main ideas and analysis\label{sec:Analysis-outline}}

In this section, we an overview of our novel analysis framework. We
first give a basic non-adaptive algorithm and its analysis to motivate
the algorithmic choices made by our adaptive algorithms. We then discuss
how to turn the non-adaptive algorithm into an adaptive one. Section
\ref{sec:Proof-sketch} in the appendix gives a proof sketch for Theorem
\ref{thm:Main-MS-convergence-rate} for the special case $p=\frac{1}{2}$
that illustrates the main ideas used in the analyses of all of our
algorithms. We give the complete analyses in the appendix.

\subsection{Non-adaptive algorithm\label{subsec:Non-adaptive-algorithm}}

As a warm-up towards our fully adaptive algorithms and their analysis,
we start with a basic non-adaptive algorithm and analysis that will
guide our algorithmic choices and provide intuition for our analysis.
The algorithm instantiates the STORM template using fixed choices
$a_{t}=a$ and $b_{t}/\eta=b$ for the momentum and step size. In
the following, we outline an analysis for the algorithm and derive
appropriate choices for the values $a$ and $b$. 

\textbf{Algorithm}: As noted above, the algorithm performs the following
updates:
\begin{align*}
x_{t+1}=x_{t}-\frac{1}{b}d_{t};\qquad d_{t+1} & =\nabla f(x_{t+1},\xi_{t+1})+(1-a)(d_{t}-\nabla f(x_{t},\xi_{t+1})).
\end{align*}
To make it simpler, we assume $d_{1}=\nabla F(x_{1})$. Alternatively,
one can use a standard mini-batch setting to set $d_{1}=\frac{1}{m}\sum_{i=1}^{m}\nabla f(x_{1};\xi_{i})$
with a proper $m$ leading to small variance as in previous non-adaptive
analysis \cite{fang2018spider,zhou2018stochastic,tran2019hybrid}.

\textbf{Key idea:} We start by introducing some convenient notation.
Let $\epsilon_{t}=d_{t}-\nabla F(x_{t})$ be the stochastic error
(in particular, $\epsilon_{1}=0$) and
\begin{align*}
\bH_{t}\coloneqq\sum_{i=1}^{t}\|\nabla F(x_{i})\|^{2}\qquad\bD_{t}\coloneqq\sum_{i=1}^{t}\|d_{i}\|^{2}\qquad\bE_{t}\coloneqq\sum_{i=1}^{t}\|\epsilon_{i}\|^{2}.
\end{align*}
First, to bound $\E\left[\|\nabla F(\xout)\|\right]$ where $\xout$
is an iterate chosen uniformly at random, it suffices to upper bound
$\E[\bH_{T}]$. Then, we can translate this term to a convergence
guarantee for $\E\left[\|\nabla F(\xout)\|\right]$. An important
intuition from STORM/STORM+ is the incorporation of VR in (\ref{eq:update-1}),
leading to a decrease over time of the error term $\epsilon_{t}$.
Thus, we can view $d_{t}$ as a proxy for $\nabla F(x_{t})$. It is
then natural to decompose $\bH_{T}$ in terms of $\bD_{T}$ and $\bE_{T}$.
By the definition of $\epsilon_{t}$, we can write $\bH_{T}\leq2\bD_{T}+2\bE_{T}.$
Therefore, to upper bound $\E[\bH_{T}]$, it suffices to upper bound
$\E[\bD_{T}]$ and $\E[\bE_{T}]$, which will be the essential steps
in the analysis framework.  A key insight is that $\E[\bD_{T}]$
and $\E[\bE_{T}]$ can be upper bounded in terms of each other, as
we now show. 

\textbf{Bounding $\bD_{T}$:} Starting from the function value analysis,
using smoothness, the update rule $x_{t+1}=x_{t}-\frac{1}{b}d_{t}$,
the definition of $\epsilon_{t}=d_{t}-\nabla F(x_{t})$, and Cauchy-Schwarz,
we obtain
\begin{align*}
F(x_{t+1})-F(x_{t}) & \leq\langle\nabla F(x_{t}),x_{t+1}-x_{t}\rangle+\frac{\beta}{2}\|x_{t+1}-x_{t}\|^{2}=-\frac{1}{b}\langle\nabla F(x_{t}),d_{t}\rangle+\frac{\beta}{2b^{2}}\|d_{t}\|^{2}\\
 & =-\frac{1}{b}\|d_{t}\|^{2}+\frac{1}{b}\langle\epsilon_{t},d_{t}\rangle+\frac{\beta}{2b^{2}}\|d_{t}\|^{2}\leq-\frac{1}{2b}\left\Vert d_{t}\right\Vert ^{2}+\frac{1}{2b}\left\Vert \epsilon_{t}\right\Vert ^{2}+\frac{\beta}{2b^{2}}\|d_{t}\|^{2}.
\end{align*}
Suppose that we choose $b$ so that $b\geq2\beta$, which ensures
that $\frac{\beta}{2b^{2}}\leq\frac{1}{4b}$. By rearranging the previous
inequality, summing up over all iterations, and taking expectation,
we obtain
\begin{equation}
\E\left[\bD_{T}\right]\leq4b\E\left(F(x_{1})-F(x_{T+1})\right)+2\E\left[\bE_{T}\right]\leq4b\left(F(x_{1})-F^{*}\right)+2\E\left[\bE_{T}\right].\label{eq:non-adaptive-D}
\end{equation}

\textbf{Bounding $\bE_{T}$:} By the standard calculation for the
stochastic error $\epsilon_{t}$ used in STORM, we have 
\begin{align*}
\E\left[\left\Vert \epsilon_{t+1}\right\Vert ^{2}\right] & \leq(1-a)^{2}\E\left[\left\Vert \epsilon_{t}\right\Vert ^{2}\right]+2(1-a)^{2}\frac{\beta^{2}}{b^{2}}\E\left[\left\Vert d_{t}\right\Vert ^{2}\right]+2a^{2}\sigma^{2}.
\end{align*}
Summing up over all iterations, rearranging, and using that $a\in\left[0,1\right]$
and $\epsilon_{1}=0$, we obtain
\begin{equation}
\E\left[\bE_{T}\right]\leq\frac{1}{1-(1-a)^{2}}\big(2(1-a)^{2}\frac{\beta^{2}}{b^{2}}\E\left[\bD_{T}\right]+2a^{2}\sigma^{2}T\big)\leq\frac{2\beta^{2}}{ab^{2}}\E\left[\bD_{T}\right]+2a\sigma^{2}T.\label{eq:non-adaptive-E}
\end{equation}
By combining inequalities (\ref{eq:non-adaptive-D}) and (\ref{eq:non-adaptive-E}),
we obtain
\begin{align}
\E\left[\bD_{T}\right] & \leq4b\left(F(x_{1})-F^{*}\right)+\frac{4\beta^{2}}{ab^{2}}\E\left[\bD_{T}\right]+4a\sigma^{2}T;\label{eq:non-adaptive-D-combined}\\
\E\left[\bE_{T}\right] & \leq\frac{8\beta^{2}}{ab}\left(F(x_{1})-F^{*}\right)+\frac{4\beta^{2}}{ab^{2}}\E\left[\bE_{T}\right]+2a\sigma^{2}T.\label{eq:non-adaptive-E-combined}
\end{align}

\textbf{Ideal non-adaptive choices for $a,b$: }Here, we set $a$
and $b$ to optimize the overall bound, and obtain choices that depend
on the problem parameters. In the next section, we build upon these
choices to obtain adaptive algorithms that use only the stochastic
gradient information received by the algorithm.

We observe that (\ref{eq:non-adaptive-D-combined}) and (\ref{eq:non-adaptive-E-combined})
bound $\E[\bD_{T}]$ and $\E[\bE_{T}]$ in terms of themselves, and
the coefficient on the right-hand side is $\frac{4\beta^{2}}{ab^{2}}$.
Suppose that we set $a$ so that this coefficient is $\frac{1}{2}$,
i.e., we set $a=\frac{8\beta^{2}}{b^{2}}$, so that $\frac{4\beta^{2}}{ab^{2}}=\frac{1}{2}$
(note that this requires setting $b\geq2\sqrt{2}\beta,$ so that $a\leq1$).
By plugging this choice into (\ref{eq:non-adaptive-D-combined}) and
(\ref{eq:non-adaptive-E-combined}), we obtain
\begin{align*}
\E\left[\bD_{T}\right],\E\left[\bE_{T}\right] & \leq O\Big(b\left(F(x_{1})-F^{*}\right)+\frac{\beta^{2}\sigma^{2}T}{b^{2}}\Big).
\end{align*}
The best choice for $b$ is the one that balances the two terms above:
$b=\Theta\big(\frac{\beta^{2}\sigma^{2}T}{F(x_{1})-F^{*}}\big)^{1/3}$.
Since we also need $b\geq\Omega(\beta)$, we can set $b$ to the sum
of the two. Hence, we obtain
\begin{align}
a & =\Theta\Big(\frac{\beta^{2}}{b^{2}}\Big)=\Theta\Big(\frac{1}{1+\left(\beta\left(F(x_{1})-F^{*}\right)\right)^{-2/3}\left(\sigma^{2}T\right)^{2/3}}\Big);\label{eq:non-adaptive-a}\\
b & =\Theta\big(\beta+\beta^{2/3}\big(F(x_{1})-F^{*}\big)^{-1/3}\big(\sigma^{2}T\big)^{1/3}\big);\label{eq:non-adaptive-b}\\
\E\left[\bD_{T}\right],\E\left[\bE_{T}\right],\E\left[\bH_{T}\right] & \leq O\big(\beta\big(F(x_{1})-F^{*}\big)+\big(\beta\big(F(x_{1})-F^{*}\big)\big)^{2/3}\big(\sigma^{2}T\big)^{1/3}\big).\label{eq:non-adaptive-H}
\end{align}

\subsection{Adaptive algorithm\label{subsec:Fully-adaptive-algorithms}}

In this section, we build on the non-adaptive algorithm and its analysis
from the previous section. We first motivate the algorithmic choices
made by our algorithm via a thought experiment where we pretend that
$\bH_{T},\bD_{T},\bE_{T}$ are deterministic quantities.

\textbf{Towards adaptive algorithms: }To develop an adaptive algorithm,
we would like to pick $a,b$ without an explicit dependence on the
problem parameters by using quantities that the algorithm can track.
We break this down by first considering choices that do not depend
on $\beta$, but on $\sigma$, and then removing the dependency on
$\sigma$. As a thought experiment, let us pretend that $\bH_{T},\bD_{T},\bE_{T}$
are deterministic quantities. A natural choice for $a$ that mirrors
the non-adaptive choice (\ref{eq:non-adaptive-a}) is $a=(1+\sigma^{2}T)^{-2/3}$.
Since we are pretending that $\bD_{T}$ is a deterministic quantity,
we can set $b$ by inspecting (\ref{eq:non-adaptive-E}):
\[
\bE_{T}\overset{\eqref{eq:non-adaptive-E}}{\leq}\frac{2\beta^{2}}{ab^{2}}\bD_{T}+2a\sigma^{2}T
\]
If we set $b=\bD_{T}^{1/2}/a^{1/4}$, we ensure that $\bD_{T}$ cancels
and we obtain the desired upper bound on $\bE_{T}$. More precisely,
by plugging in $a=(1+\sigma^{2}T)^{-2/3}$ and $b=\bD_{T}^{1/2}/a^{1/4}$
into (\ref{eq:non-adaptive-E}), we obtain
\begin{align*}
\bE_{T} & \overset{\eqref{eq:non-adaptive-E}}{\leq}\frac{2\beta^{2}}{a^{1/2}\bD_{T}}\bD_{T}+2a\sigma^{2}T\le O\big(\beta^{2}\left(1+\sigma^{2}T\right)^{1/3}+\left(1+\sigma^{2}T\right)^{1/3}\big)
\end{align*}
We now consider two cases for $\bD_{T}$. If $\bD_{T}\le16\beta^{2}(1+\sigma^{2}T)^{1/3}$,
the above inequality together with $\bH_{T}\leq2\bD_{T}+2\bE_{T}$
imply that $\bH_{T}\leq O((1+\beta^{2})(1+\sigma^{2}T)^{1/3})$. Otherwise,
we have $\bD_{T}\ge16\beta^{2}(1+\sigma^{2}T)^{1/3}$ and thus $ab^{2}\geq16\beta^{2}$.
Plugging into (\ref{eq:non-adaptive-D-combined}), we obtain
\[
\bD_{T}\overset{\eqref{eq:non-adaptive-D-combined}}{\leq}O\big(\sqrt{\bD_{T}}\left(1+\sigma^{2}T\right)^{1/6}\left(F(x_{1})-F(x^{*})\right)+\left(1+\sigma^{2}T\right)^{1/3}\big)
\]
which solves to $\bD_{T}\leq O((1+\sigma^{2}T)^{1/3}(F(x_{1})-F^{*})^{2})$.
We can again bound $\bH_{T}$ using $\bH_{T}\le2\bD_{T}+2\bE_{T}$.
In both cases, we have the bound
\[
\bH_{T}\leq O\big(\big(1+\beta^{2}+\left(F(x_{1})-F^{*}\right)^{2}\big)\left(1+\sigma^{2}T\right)^{1/3}\big)
\]
We now turn to removing the dependency on $\sigma^{2}T$ in $a$.
The algorithm can also track $\widetilde{\bH}_{T}:=\sum_{t=1}^{T}\|\nabla f(x_{t};\xi_{t})-\nabla f(x_{t};\xi_{t+1})\|^{2}$,
which can be viewed as a proxy for $\sigma^{2}T$. Replacing $\sigma^{2}T$
by this proxy and making $a$ and $b$ be time dependent give the
update rules employed by our algorithm in the special case $p=\frac{1}{2}$.
Our update rule for general $p$ follows from a similar thought experiment.

\textbf{Analysis}: Using a similar approach as in the non-adaptive
analysis, we can turn the above argument into a rigorous analysis.
In the appendix, we give the complete analysis as well as a proof
sketch in Section \ref{sec:Proof-sketch} that gives an overview of
our main analysis techniques.

\section{Experiments\label{sec:Experiments}}

We examine the empirical performance of our methods against the previous
work STORM+ \cite{levy2021storm+} and popular algorithms (Adam,
AdamW, AdaGrad, and SGD) on three tasks: (1) \textbf{Image classification}
with the CIFAR10 dataset \cite{krizhevsky2009cifar10} using ResNet18
\cite{ren2016resnet} models; (2)\textbf{ Masked language modeling}
via the BERT pretraining loss \cite{devlin2018bert} with the IMDB
dataset \cite{maas2011imdbdataset} using distill-BERT models \cite{sanh2019distilbert},
where we employ the standard cross entropy loss for MLM fine tuning
(with whole word masking and fixed test masks) with maximum length
128; and (3)\textbf{ Sentiment analysis} with the SST2 dataset \cite{sst2dataset}
via finetuning BERT models \cite{devlin2018bert}. We use the standard
train/validation split and run all algorithms for 4 epochs. 

 We use the default implementation of AdaGrad, Adam, AdamW, and SGD
from Pytorch. For STORM+, we follow the authors' original implementation.\footnote{Link to the code of STORM+: \url{https://github.com/LIONS-EPFL/storm-plus-code}.}
We give the complete implementation details and tables of hyperparameters
for all algorithms in Section \ref{subsec:Implementation-details-and-hyperperparams}
of the Appendix.

\textbf{Heuristics.} For our algorithms, we further examine whether
heuristics like exponential moving average (EMA) of the gradient sums
(or often called online moment estimation) and per-coordinate update
would be beneficial. This version with heuristics is further denoted
(H) in our results below. This is discussed in full details in Section
\ref{subsec:Heuristics-versions} of the Appendix.

\textbf{Results.} We perform our experiments on the standard train/test
splits of each dataset. We tune for the best learning rate across
a fixed grid for all algorithms and perform each run 5 times. For
readability, we omit error bars in the plot. Full plots with error
bars and tabular results with standard deviation as well as further
discussions are presented in Section \ref{subsec:Full-results-exp}
of the Appendix.\footnote{The reader should keep in mind that variance-reduced algorithms like
META-STORM require twice the amount of gradient queries, so the improvement
in performance that our algorithms exhibit does not come without a
cost. Additional plots and further discussions are available in Section
\ref{sec:Additional-Experimental-Details}.}

1. \textbf{CIFAR10} (Figure \ref{fig:cifar10-experimental-result}).
Overall, META-STORM-SG achieves the lowest training loss with META-STORM
and STORM+ coming in close. META-STORM with heuristics attains the
best test accuracy, with Adam coming in close.

2. \textbf{IMDB} (Figure \ref{fig:imdb-train-test-loss}). AdamW attains
the best training loss. However, META-STORM with heuristics achieve
the best test loss (with AdamW coming in close). META-STORM-SG and
the heuristic algorithms outperform STORM+ in both minimizing training
loss and test loss. 

3. \textbf{SST2} (Figure \ref{fig:sst2-exp-results}). META-STORM
with heuristics attain the best training loss and accuracy, above
Adam and AdamW. It also achieves the best validation accuracy out
of all the algorithms. Furthermore, non-heuristic META-STORM and META-STORM-SG
outperform STORM+. We remark that STORM+ appears to be rather unstable
for this task as some of the random runs do not converge to good stationary
points.

\begin{figure}[t]
\centering{}\vspace{-5pt}
\subfloat{\includegraphics[width=0.5\textwidth]{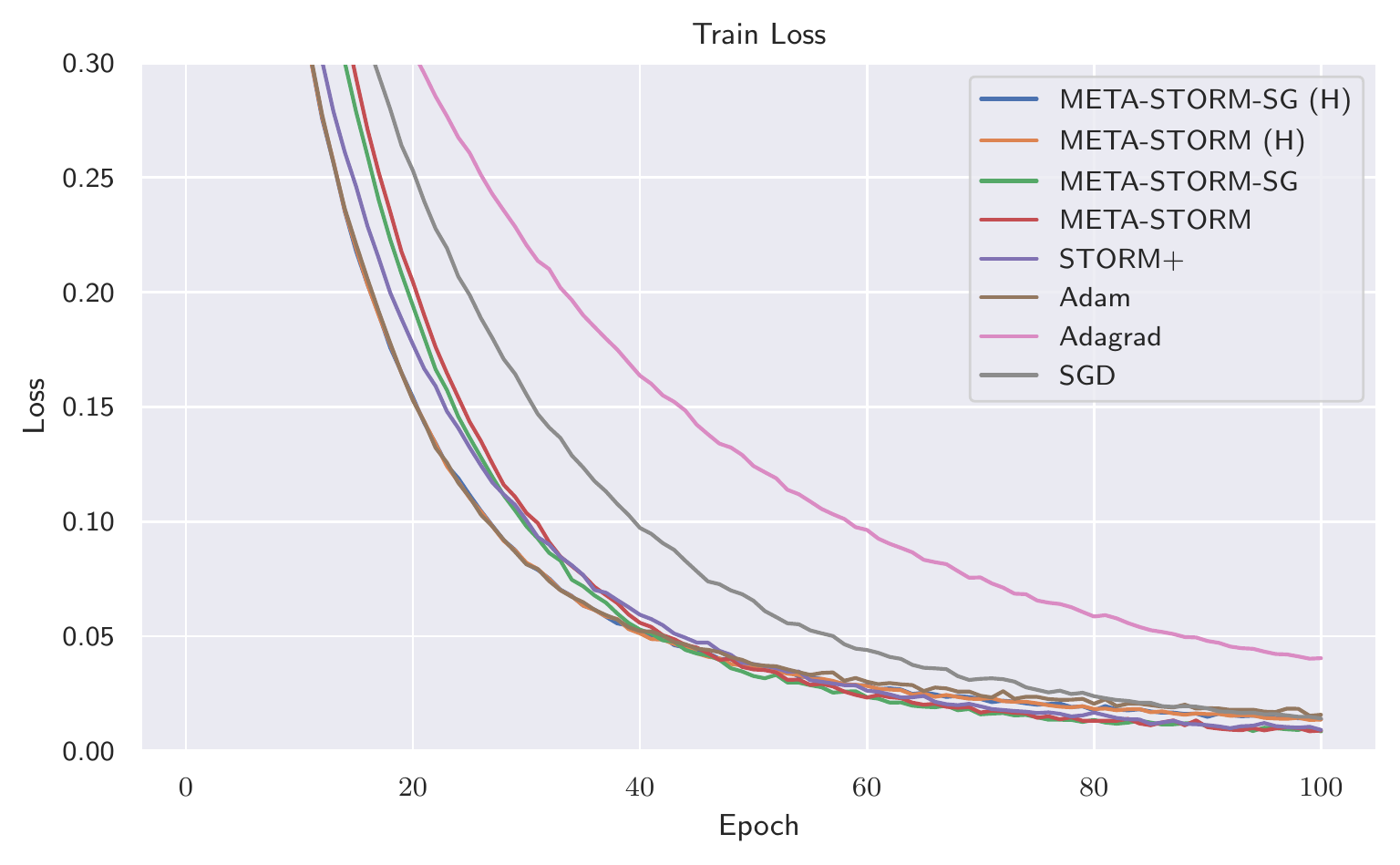}}\subfloat{\includegraphics[width=0.5\textwidth]{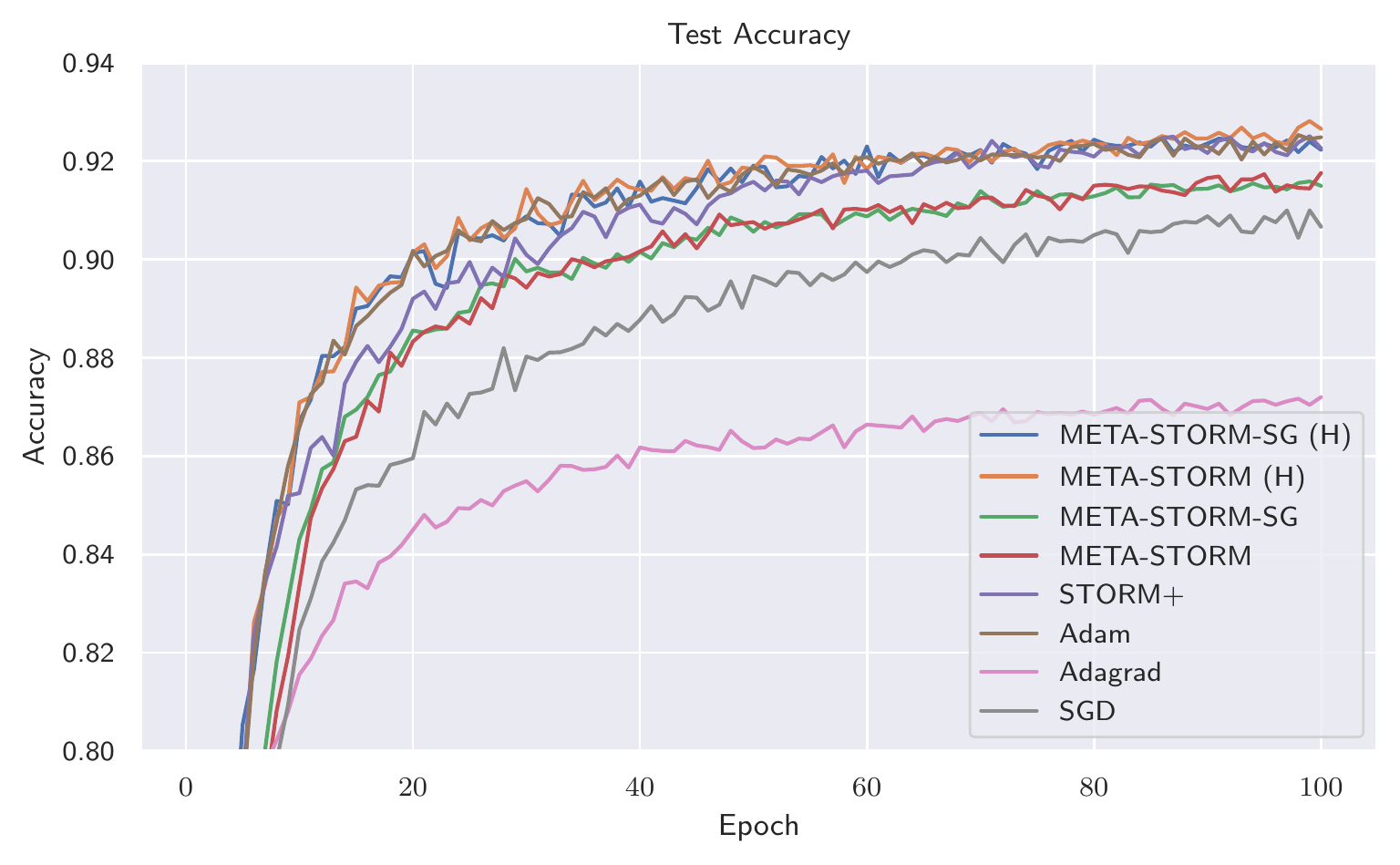}}\caption{\label{fig:cifar10-experimental-result}Training loss and test accuracy
on CIFAR10. (H) denotes the addition of heuristics.}
\vspace{-17.5pt}
\end{figure}
\begin{figure}[H]
\begin{centering}
\vspace{-5pt}
\par\end{centering}
\centering{}\subfloat{\includegraphics[width=0.5\textwidth]{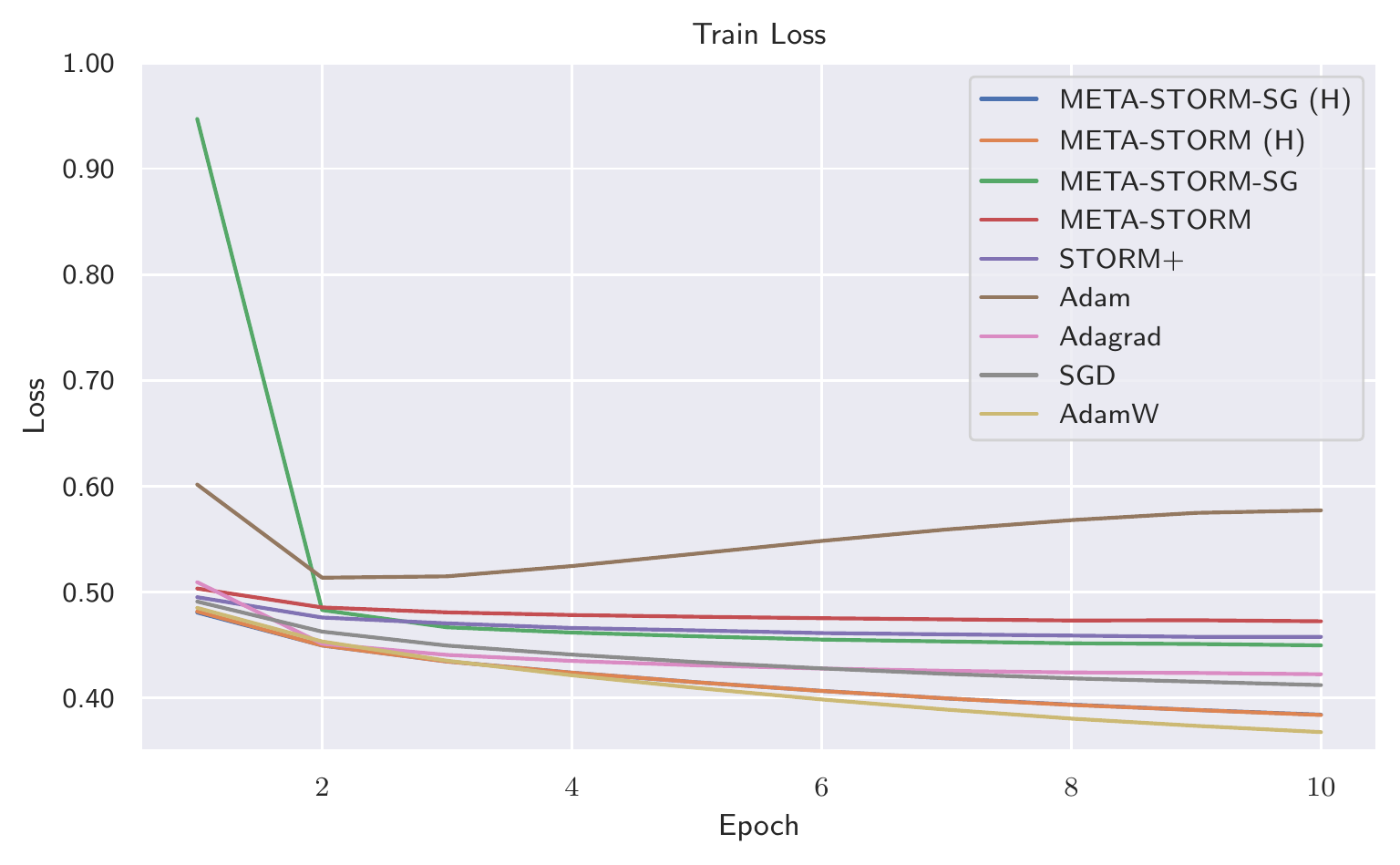}}\subfloat{\includegraphics[width=0.5\textwidth]{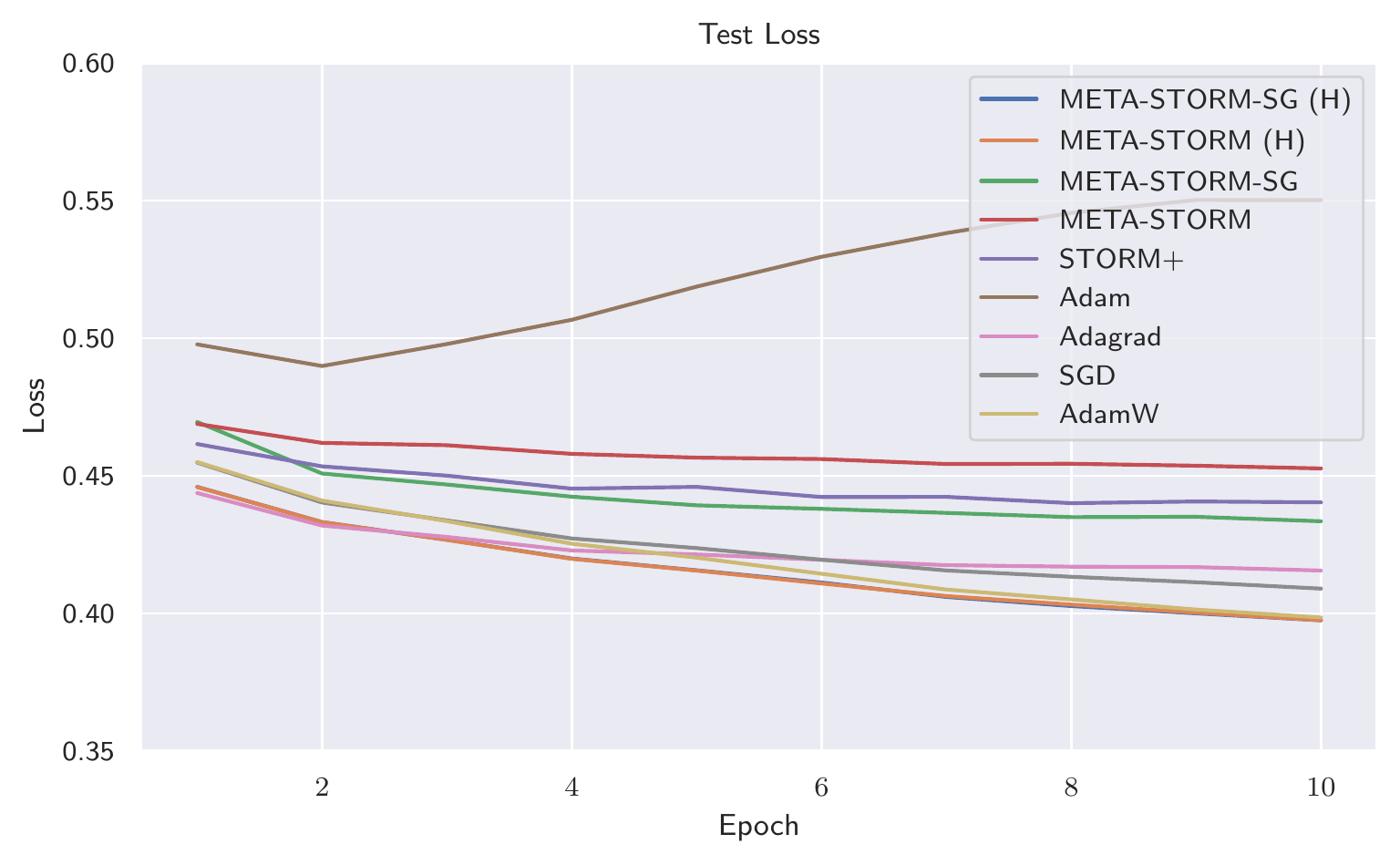}}\caption{\label{fig:imdb-train-test-loss}Training loss and test loss on IMDB.
(H) denotes the addition of heuristics.}
\vspace{-17.5pt}
\end{figure}
\begin{figure}[H]
\begin{centering}
\vspace{-5pt}
\par\end{centering}
\centering{}\subfloat{\includegraphics[width=0.5\textwidth]{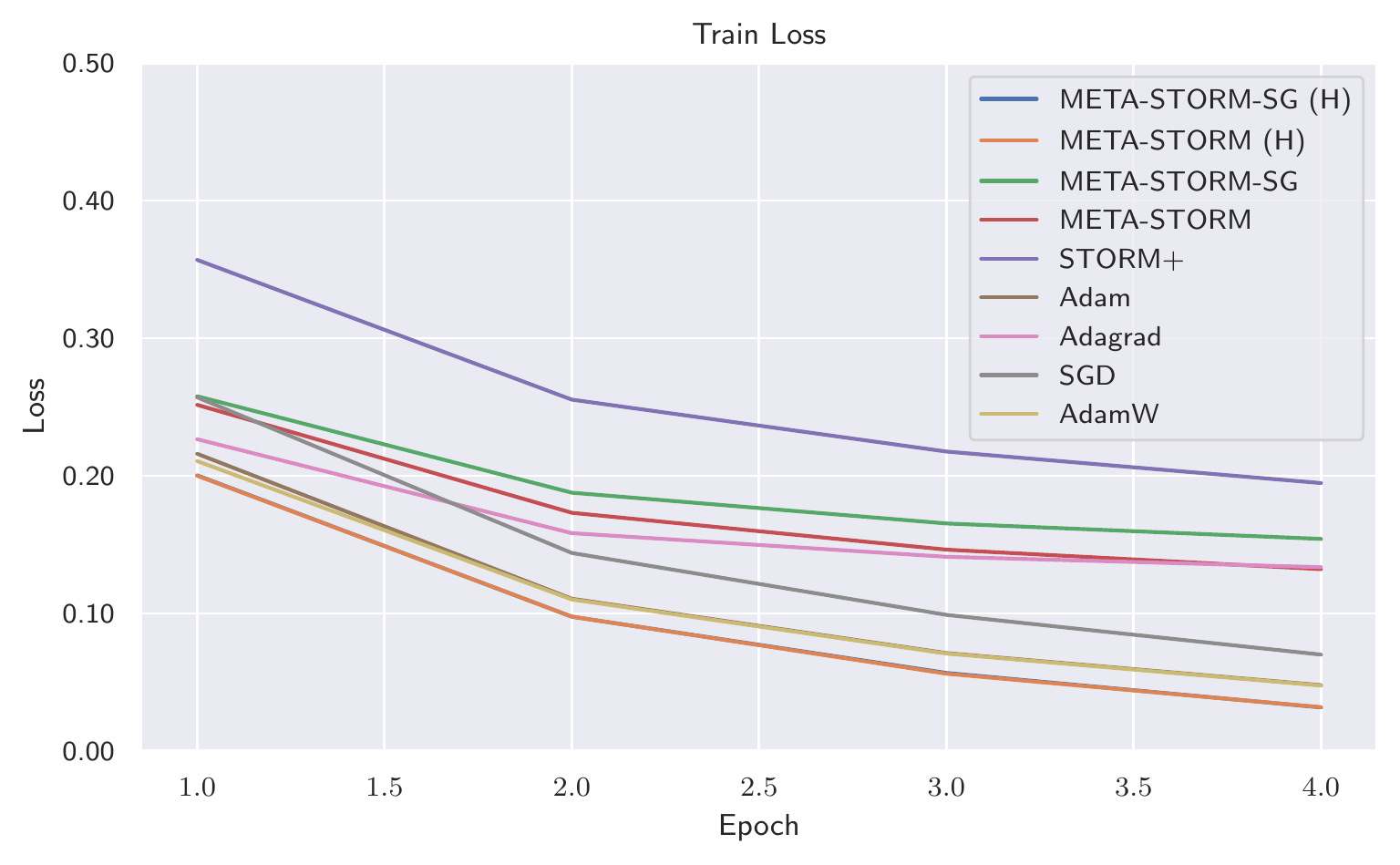}}\subfloat{\includegraphics[width=0.5\textwidth]{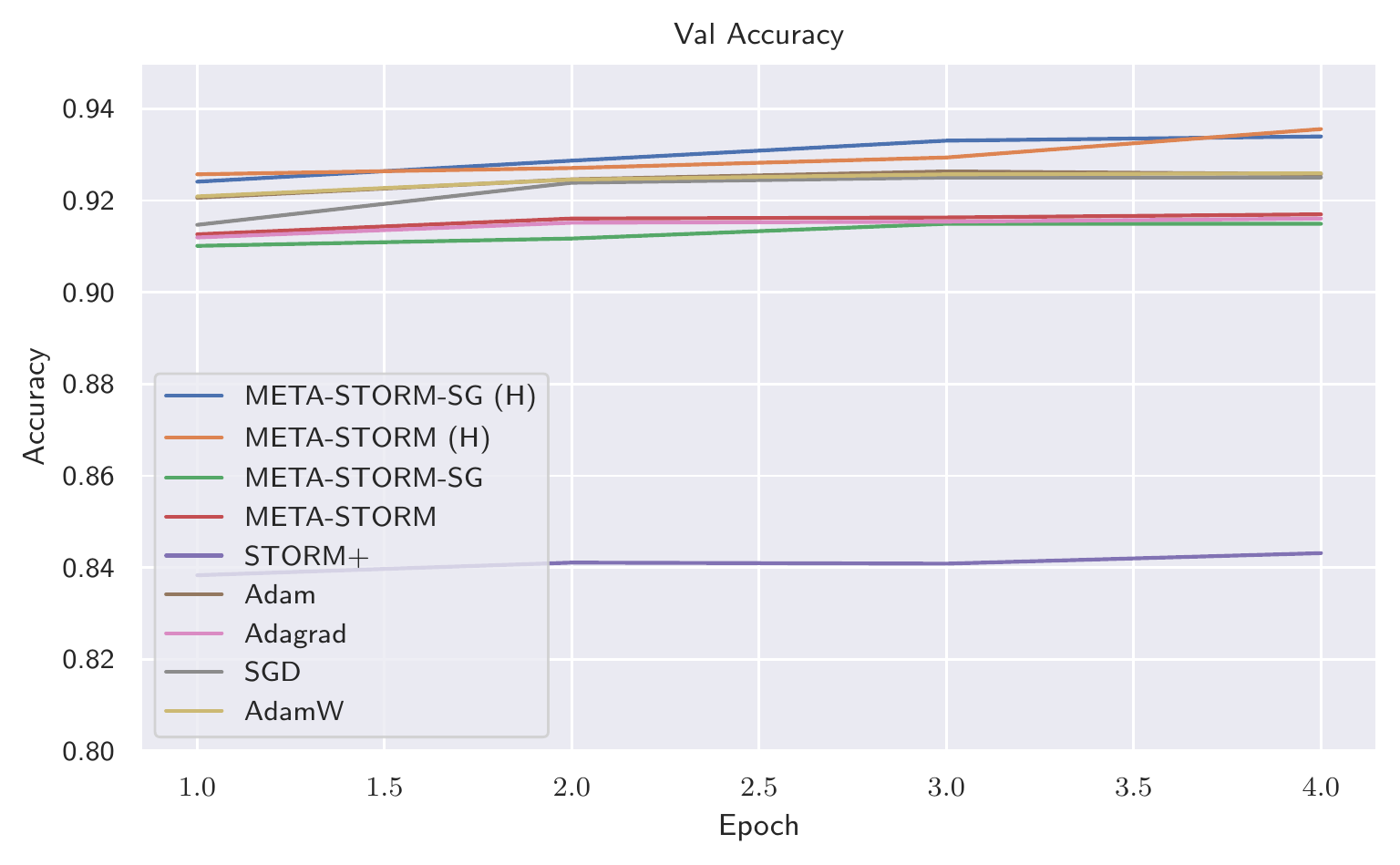}}\caption{\label{fig:sst2-exp-results}Training loss and validation accuracy
on SST2. (H) denotes the addition of heuristics.}
\vspace{-7.5pt}
\end{figure}

\section{Conclusion}

In this paper, we propose $\algnameold$ and $\algnamenew$, two fully-adaptive
momentum-based variance-reduced SGD frameworks that generalize upon
STORM+ and remove STORM+'s restrictive bounded function values assumption.
$\algnamenew$ and its sibling $\algnameold$ attain the optimal convergence
rate with better dependency on the problem parameters than previous
methods and allow for a wider range of configurations. Experiments
demonstrate our algorithms' effectiveness across common deep learning
tasks against the previous work STORM+, and when heuristics are further
added, achieve competitive performance against state-of-the-art algorithms.

\bibliographystyle{plain}
\bibliography{ref}

\paragraph{}

\newpage{}

\appendix

\section{Appendix outline}

The appendix is organized as follows. 
\begin{itemize}
\item Section \ref{sec:Additional-Experimental-Details} presents the full
implementation details for our algorithms and hyperparameters used.
This section also includes additional ablation studies and experiments.
\item Section \ref{sec:Notations-and-assumptions} introduces the notations
used in the analysis of our algorithms.
\item Section \ref{sec:Proof-sketch} presents the proof sketch of Theorem
\ref{thm:Main-MS-convergence-rate}.
\item Section \ref{sec:Basic-analysis} establishes some basic results that
are used in our full analysis.
\item Section \ref{sec:Algorithm-MS} gives the analysis of $\algnamenew$
for general $p$.
\item Section \ref{sec:Algorithm-SG} gives the analysis of $\algnameold$
for general $p$.
\item Section \ref{sec:Algorithm-NA} introduces $\algnamena$ and gives
the analysis for general $p$.
\item Section \ref{sec:Basic-inequalities} gives several basic inequalities
that are used in our analysis.
\end{itemize}

\section{Experimental details and additional experiments \label{sec:Additional-Experimental-Details}}

In this section, we present the complete implementation details along
with the full experimental setup. All of our experiments were conducted
on two NVIDIA RTX3090. 

\subsection{Implementation details and hyperparameter tuning\label{subsec:Implementation-details-and-hyperperparams}}

In this section, we present the full implementation details of the
heuristics version, parameter selection, and hyperparameter tuning
for all 3 datasets.

\subsubsection{Heuristics versions of $\protect\algnamenew$ and $\protect\algnameold$\label{subsec:Heuristics-versions}}

\begin{algorithm}[h]
\caption{Heuristic update of $\protect\algnamenew$ and $\protect\algnameold$.
\label{alg:Heuristic-metastorm}}
\begin{align*}
b_{t} & =\begin{cases}
\left(b_{0}^{1/p}+D_{t}\right)^{p}/a_{t}^{q} & \text{for }\algnamenew\text{ (H)}\\
\left(b_{0}^{1/p}+D_{t}\right)^{p}/a_{t+1}^{q} & \text{for }\algnameold\text{ (H)}
\end{cases}\\
a_{t+1} & =\left(1+G_{t}/a_{0}^{2}\right){}^{-2/3}\\
\text{ where }D_{t} & =\alpha D_{t-1}+(1-\alpha)d_{t}^{2}\\
\text{ }G_{t} & =\begin{cases}
\alpha G_{t-1}+(1-\alpha)\left(\nabla f(x_{t},\xi_{t})-\nabla f(x_{t},\xi_{t+1})\right)^{2} & \text{for }\algnamenew\text{ (H)}\\
\alpha G_{t-1}+(1-\alpha)\left(\nabla f(x_{t},\xi_{t})\right)^{2} & \text{for }\algnameold\text{ (H)}
\end{cases}
\end{align*}
\end{algorithm}

For our algorithms, we employ the common heuristic of using an exponential
moving average (EMA) scheme in the momentum and the step size. We
also perform a per-coordinate update instead of simply using the norm.
With this, our update rules for $x_{t+1}=x_{t}-\eta d_{t}/b_{t}$
becomes coordinate-wise division with the update rules as in Algorithm
\ref{alg:Heuristic-metastorm}, where all the operations between vectors
here are coordinate-wise multiplication, exponentiation, and division.
In our experiments, we set $\alpha=0.99$, $a_{0}=1$, $b_{0}=10^{-8}$
as selected by the criterion detailed next.

\subsubsection{Algorithm development and default parameters selection}

We develop our algorithm on MNIST and tune for $p,a_{0},$ and $b_{0}$.
For $a_{0}$, we tune on MNIST across a range of values from $1$
to $10^{8}$ and found that larger values of $a_{0}$ are helpful.
For $b_{0}$, we simply need a small number for numerical stability
so we pick $10^{-8}$. For the heuristic versions of our algorithms,
$a_{0}=1$ gives the best results. This might be due to the effects
of per-coordinate operations removing the need to scale down the gradient-accumulated
step-size. 

\paragraph{Effects of varying $p$.}

In Figures \ref{fig:fastrd-mnist-p-tuning} and \ref{fig:fastrn-mnist-p-tuning},
we show the training loss and test accuracy of different values of
$p$ of our algorithms on MNIST (with $a_{0}=10^{8}$ and $b_{0}=10^{-8}$).
For each configuration, we tune the base learning rate $\eta$ across
$\left\{ 10^{-3},10^{-2},10^{-1},1,10\right\} .$The results suggest
that the lower values of $p$ tend to perform better. While $p=1/3$
has comparable performance to the lowest setting of $p$, this choice
is somewhat analogous to STORM+. Hence, we select the lowest possible
value $p$ for our algorithms in the subsequent experiments (with
$p=0.20$ for $\algnamenew$ and $p=0.25$ for $\algnameold$).

\begin{figure}[h]
\begin{centering}
\subfloat{\includegraphics[width=0.5\textwidth]{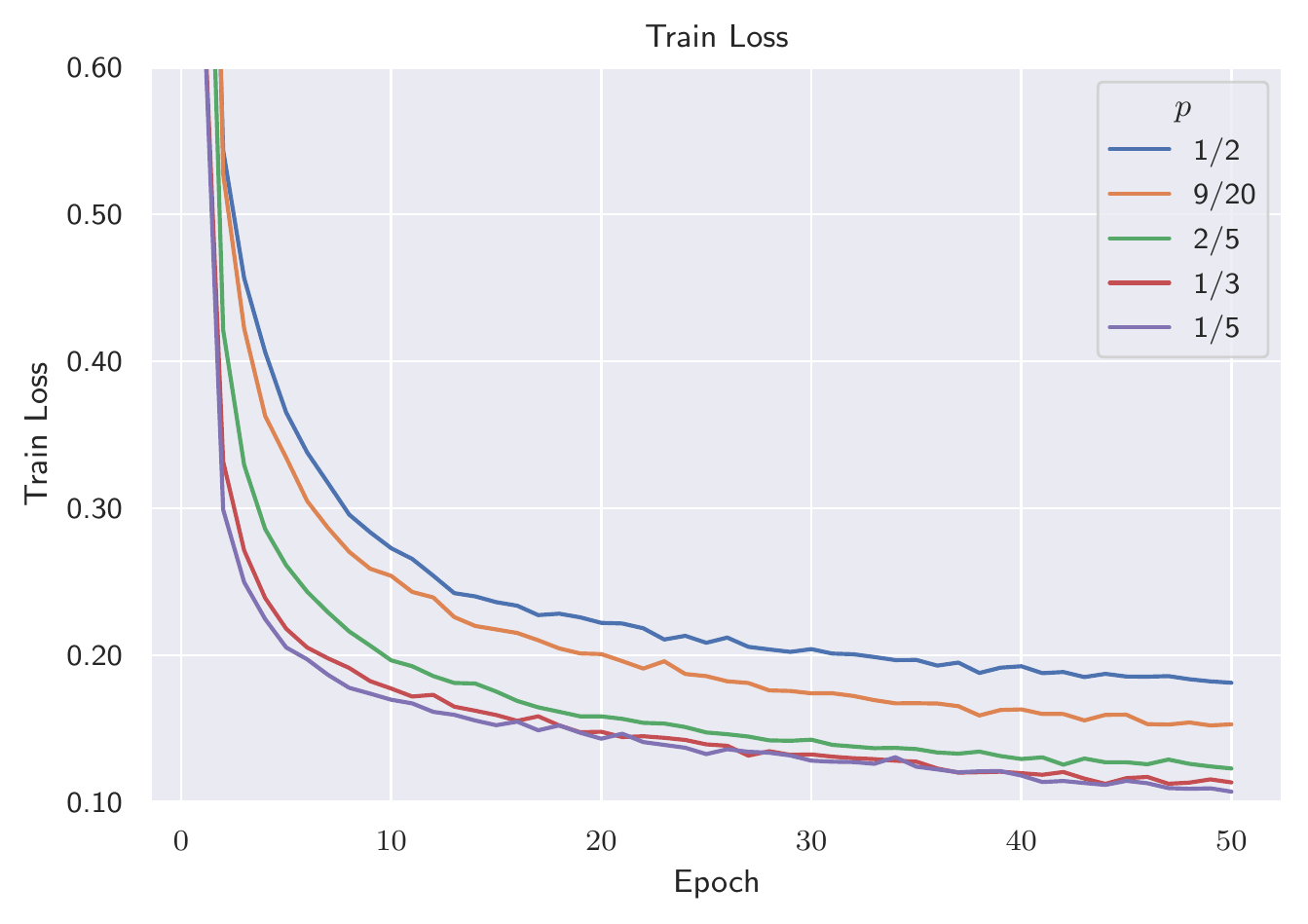}}\subfloat{\includegraphics[width=0.5\textwidth]{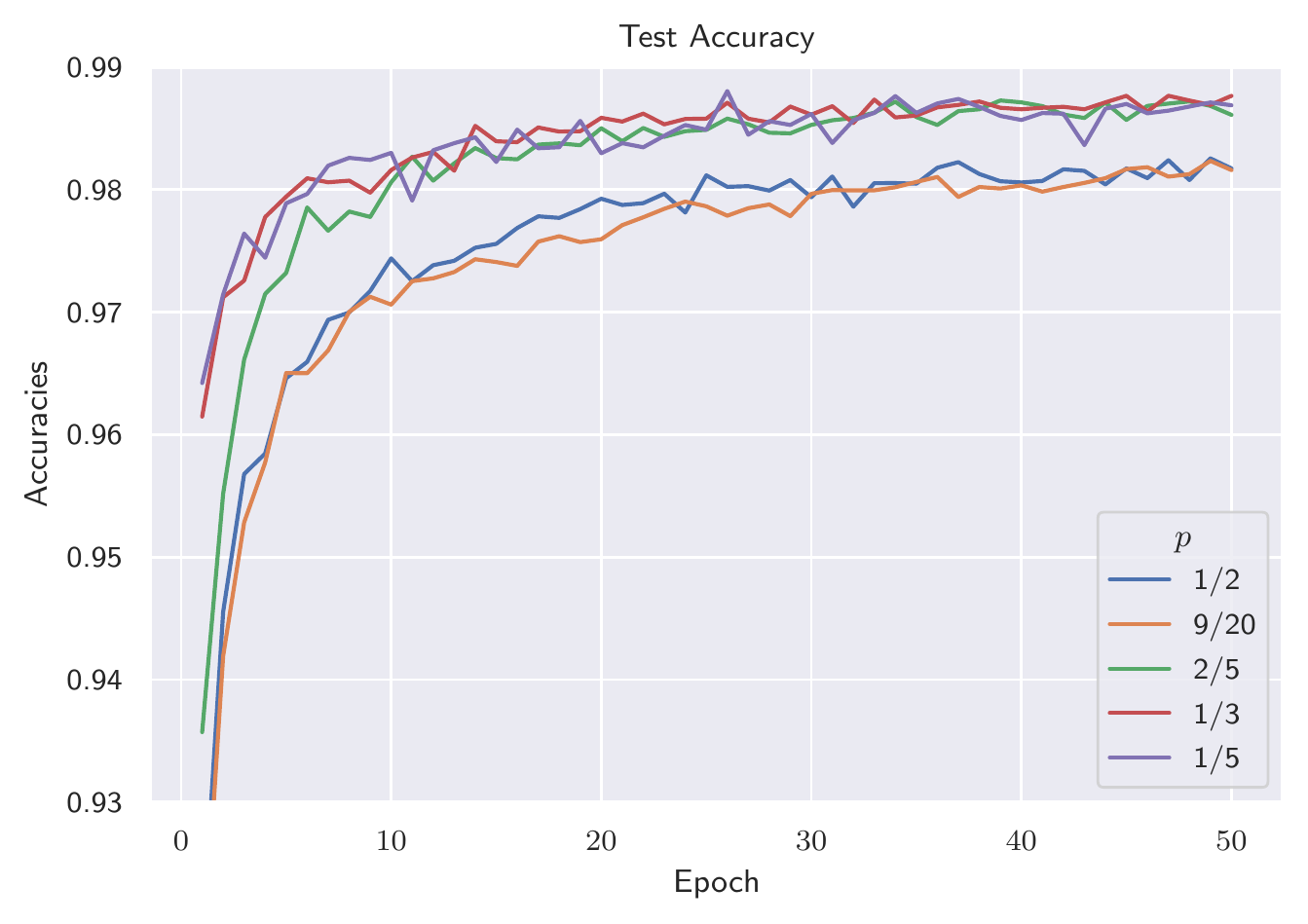}}\caption{\label{fig:fastrd-mnist-p-tuning}Training loss and test accuracy
for $\protect\algnamenew$ on MNIST for different $p$ values.}
\par\end{centering}
\vspace{-17.5pt}
\end{figure}
\begin{figure}[h]
\begin{centering}
\subfloat{\includegraphics[width=0.5\textwidth]{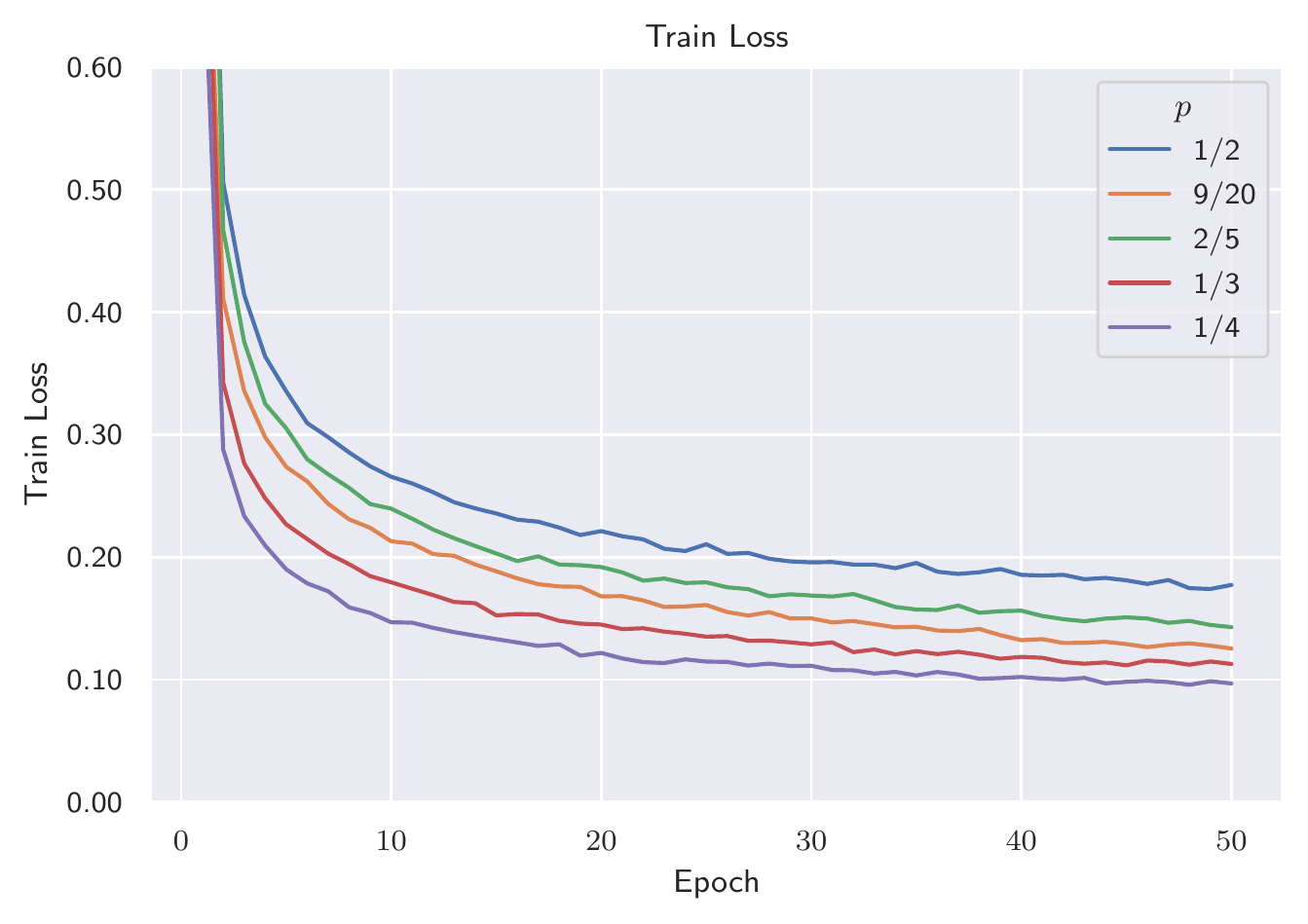}}\subfloat{\includegraphics[width=0.5\textwidth]{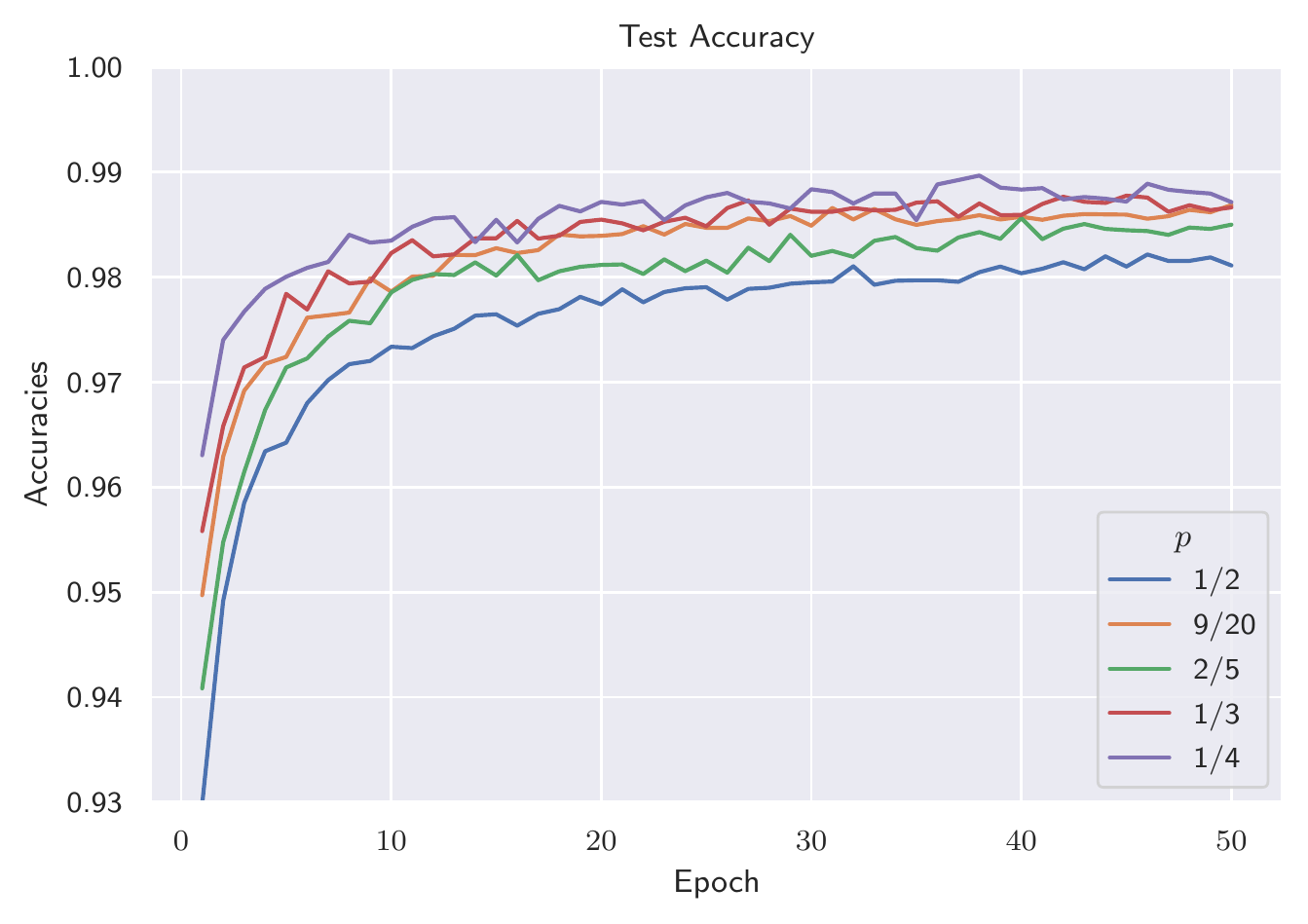}}
\par\end{centering}
\caption{\label{fig:fastrn-mnist-p-tuning}Training loss and test accuracy
for $\protect\algnameold$ on MNIST for different $p$ values.}
\end{figure}

For the heuristics versions of our algorithms, we perform the same
experiments and show the results in Figures \ref{fig:fastrd-h-mnist-p-tuning}
and \ref{fig:fastrn-h-mnist-p-tuning}. Since $p=0.50$ attains the
lowest training loss for both heuristics versions of our algorithms,
we select such value for all our experiments. 

\begin{figure}[t]
\centering{}\subfloat{\includegraphics[width=0.5\textwidth]{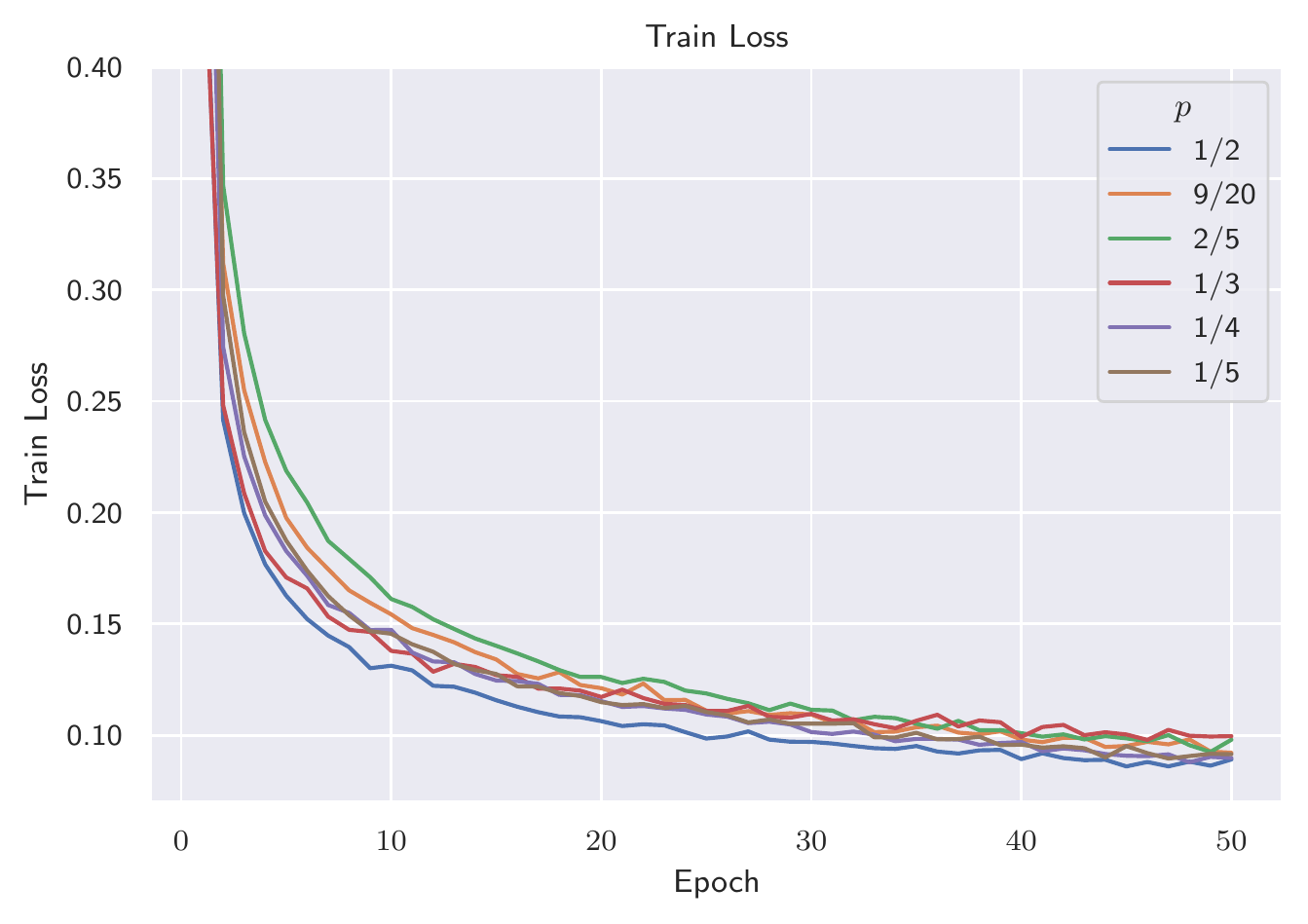}}\subfloat{\includegraphics[width=0.5\textwidth]{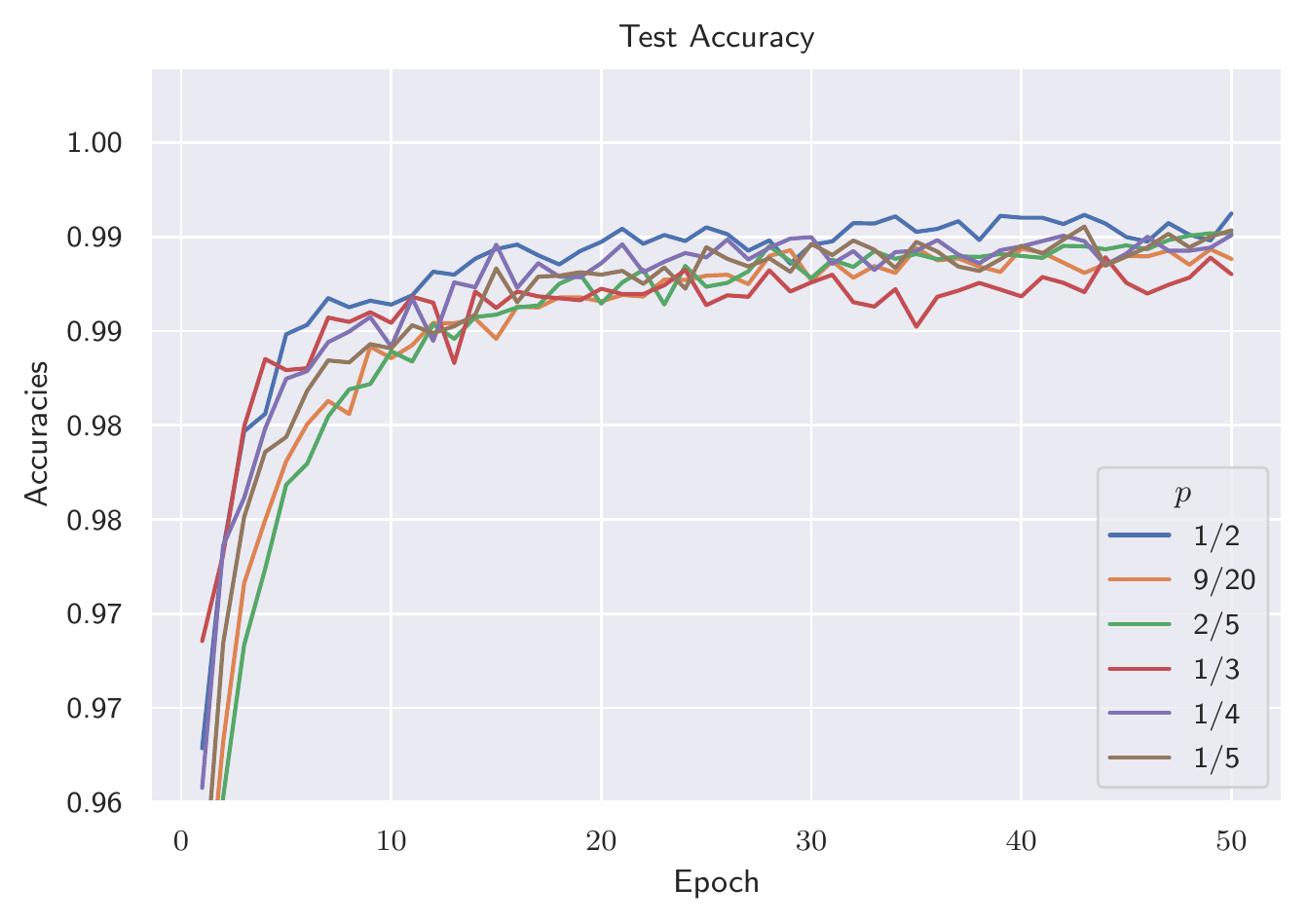}}\caption{\label{fig:fastrd-h-mnist-p-tuning}Training loss and test accuracy
for $\protect\algnamenew$ (H) on MNIST for different $p$ values.}
\vspace{-17.5pt}
\end{figure}
\begin{figure}[h]
\centering{}\subfloat{\includegraphics[width=0.5\textwidth]{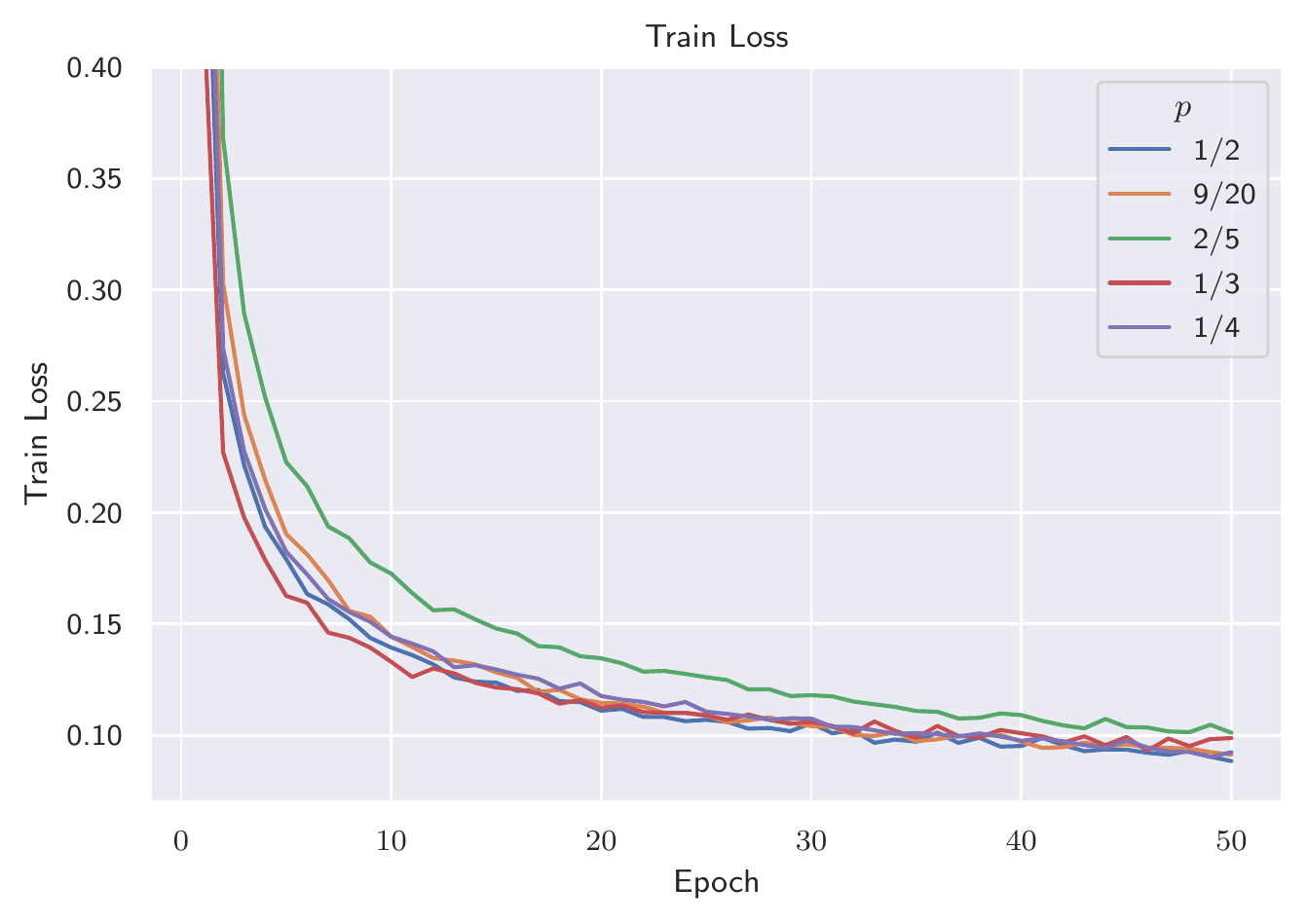}}\subfloat{\includegraphics[width=0.5\textwidth]{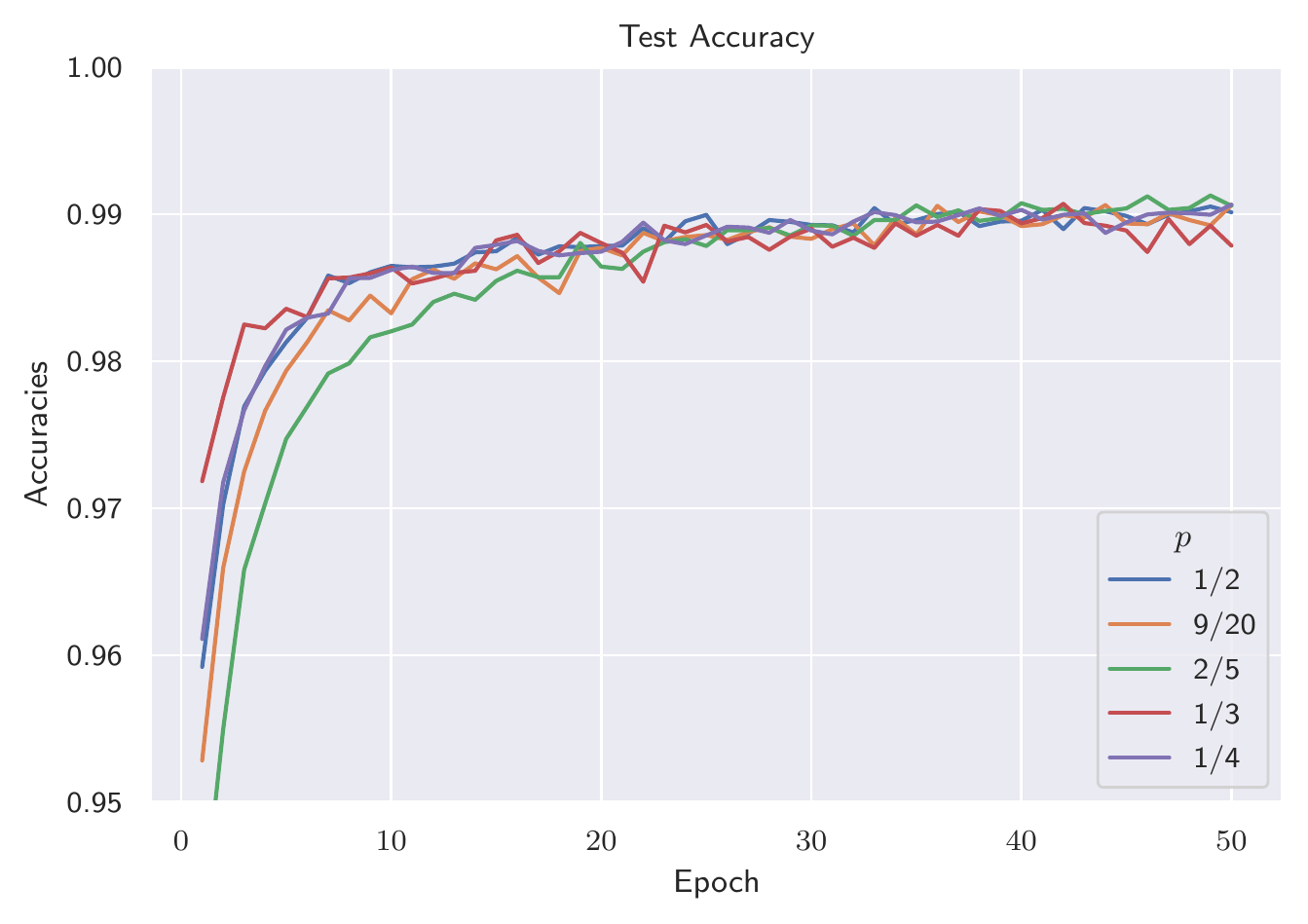}}\caption{\label{fig:fastrn-h-mnist-p-tuning}Training loss and test accuracy
for $\protect\algnameold$ (H) on MNIST for different $p$ values.}
\end{figure}

\paragraph{Default parameters.}
\begin{center}
\begin{table}[h]
\begin{centering}
\caption{Default parameters for META-STORM algorithms and STORM+. The version
with heuristics is denoted with an additional (H).}
\par\end{centering}
\centering{}%
\begin{tabular}{|c|c|c|c|}
\hline 
Algorithm & $p$ & $a_{0}$ & $b_{0}$\tabularnewline
\hline 
\hline 
$\algnamenew$ & $0.20$ & $10^{8}$ & $10^{-8}$\tabularnewline
\hline 
$\algnameold$ & $0.25$ & $10^{8}$ & $10^{-8}$\tabularnewline
\hline 
$\algnamenew$ (H) & $0.50$ & $1$ & $10^{-8}$\tabularnewline
\hline 
$\algnameold$ (H) & $0.50$ & $1$ & $10^{-8}$\tabularnewline
\hline 
STORM+ & N/A & \# of parameters & $1$\tabularnewline
\hline 
\end{tabular}
\end{table}
\par\end{center}

The discussion above leads to the choice of $a_{0}=10^{8}$ and $b_{0}=10^{-8}$
by default for our algorithms with $p=0.20$ for META-STORM and $p=0.25$
for META-STORM-SG on the benchmarks present in this section. For the
heuristic versions of META-STORM, we use $p=0.50,a_{0}=1,$ and $b_{0}=10^{-8}$
for our algorithm with heuristics. This version with heuristics is
further denoted (H) in our results below. For STORM+, we use the original
authors' implementation of setting $a_{0}$ to the number of parameters
of the model (which is roughly $10^{8}$ for ResNet18 for example).
For other baseline algorithms, we use the default parameters from
Pytorch implementation.

\paragraph{Hyperparameter tuning.}

For all algorithms, we tune only the learning rate while using the
default values for the other parameters for all algorithms. For STORM+,
the default $a_{0}$ is equal to the number of parameters of the model
and $b_{0}=1$. 

For learning rate tuning, we perform a grid search across values $\left\{ 10^{-5},10^{-4},10^{-3},10^{-2},10^{-1},1\right\} $
for CIFAR10 and IMDB and across values $\left\{ 10^{-5},2\times10^{-5},10^{-4},10^{-3},10^{-2},10^{-1},1\right\} $
for SST2 (due to $2\times10^{-5}$ being the default learning rate
for AdamW on SST2 and also more practical due to SST2 being a smaller
dataset). For Adam on IMDB, the learning rate in our grid search is
not small enough to converge, requiring additional tuning for decreasing
training loss. 

Table \ref{tab:Table-of-Hyperparameters} includes the selected learning
rate we used for each algorithm across the datasets. After obtaining
the best learning rate, we additionally run each algorithm across
5 different seeds to obtain error bars.
\begin{center}
\begin{table}[h]
\begin{centering}
\caption{Table of Hyperparameters. \label{tab:Table-of-Hyperparameters}}
\par\end{centering}
\centering{}%
\begin{tabular}{|c|c|c|c|}
\hline 
Algorithm & CIFAR10 & IMDB & SST2\tabularnewline
\hline 
\hline 
$\algnamenew$ & $1$ & $10^{-2}$ & $10^{-2}$\tabularnewline
\hline 
$\algnameold$ & $1$ & $10^{-1}$ & $10^{-2}$\tabularnewline
\hline 
$\algnamenew$ (H) & $10^{-3}$ & $10^{-4}$ & $2\cdot10^{-5}$\tabularnewline
\hline 
$\algnameold$ (H) & $10^{-3}$ & $10^{-4}$ & $2\cdot10^{-5}$\tabularnewline
\hline 
STORM+ & $0.1$ & $10^{-2}$ & $10^{-2}$\tabularnewline
\hline 
Adam & $10^{-3}$ & $10^{-6}$ & $10^{-5}$\tabularnewline
\hline 
AdamW & N/A & $10^{-4}$ & $10^{-5}$\tabularnewline
\hline 
Adagrad & $10^{-3}$ & $10^{-3}$ & $10^{-4}$\tabularnewline
\hline 
SGD & $10^{-3}$ & $10^{-2}$ & $10^{-3}$\tabularnewline
\hline 
\end{tabular}
\end{table}
\par\end{center}

\subsection{Full results for experiments in Section \ref{sec:Experiments} and
additional experiments \label{subsec:Full-results-exp}}

In this section, we show complete plots and tabular results along
with more detailed discussions for our experiments. The reader should
note that STORM-based methods require twice the amount of oracle access
over the baselines. The plots show average across 5 seeds along with
min/max bars. The tables show the average across 5 seeds across a
range of selected epochs and one standard deviation is included at
the last epoch. In the plots and tables below: (H) denotes the version
of the algorithm with the heuristics (EMA and per-coordinate update)
employed.

\subsubsection{CIFAR10: results and discussions\label{subsec:CIFAR10-full-results}}

Figure \ref{fig:cifar10-full-results} shows all 4 plots of the main
experiments in Section \ref{sec:Experiments} in Figure \ref{fig:cifar10-full-results}.
\begin{figure}[h]
\begin{centering}
\subfloat{\includegraphics[width=0.49\textwidth]{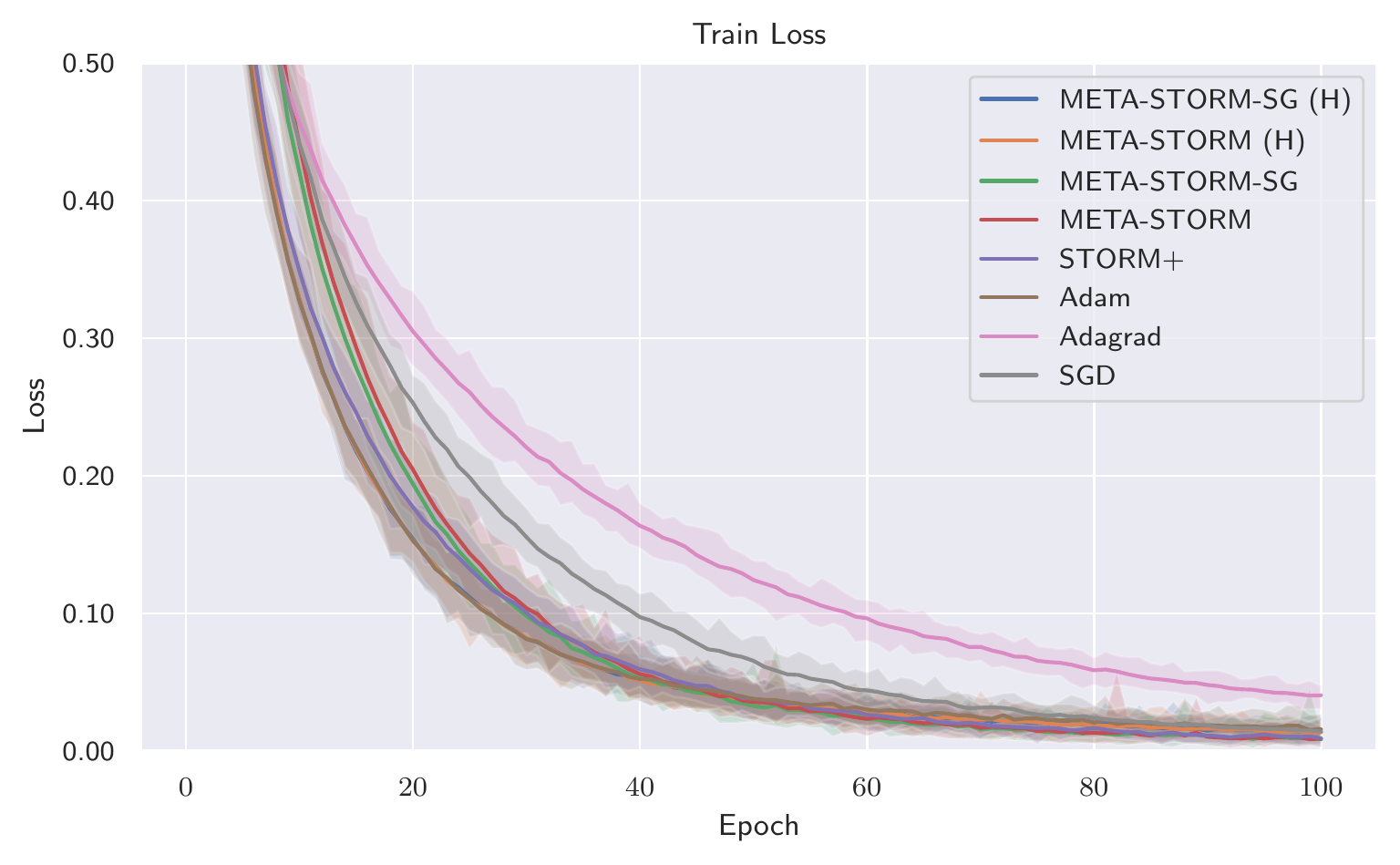}}\subfloat{\includegraphics[width=0.49\textwidth]{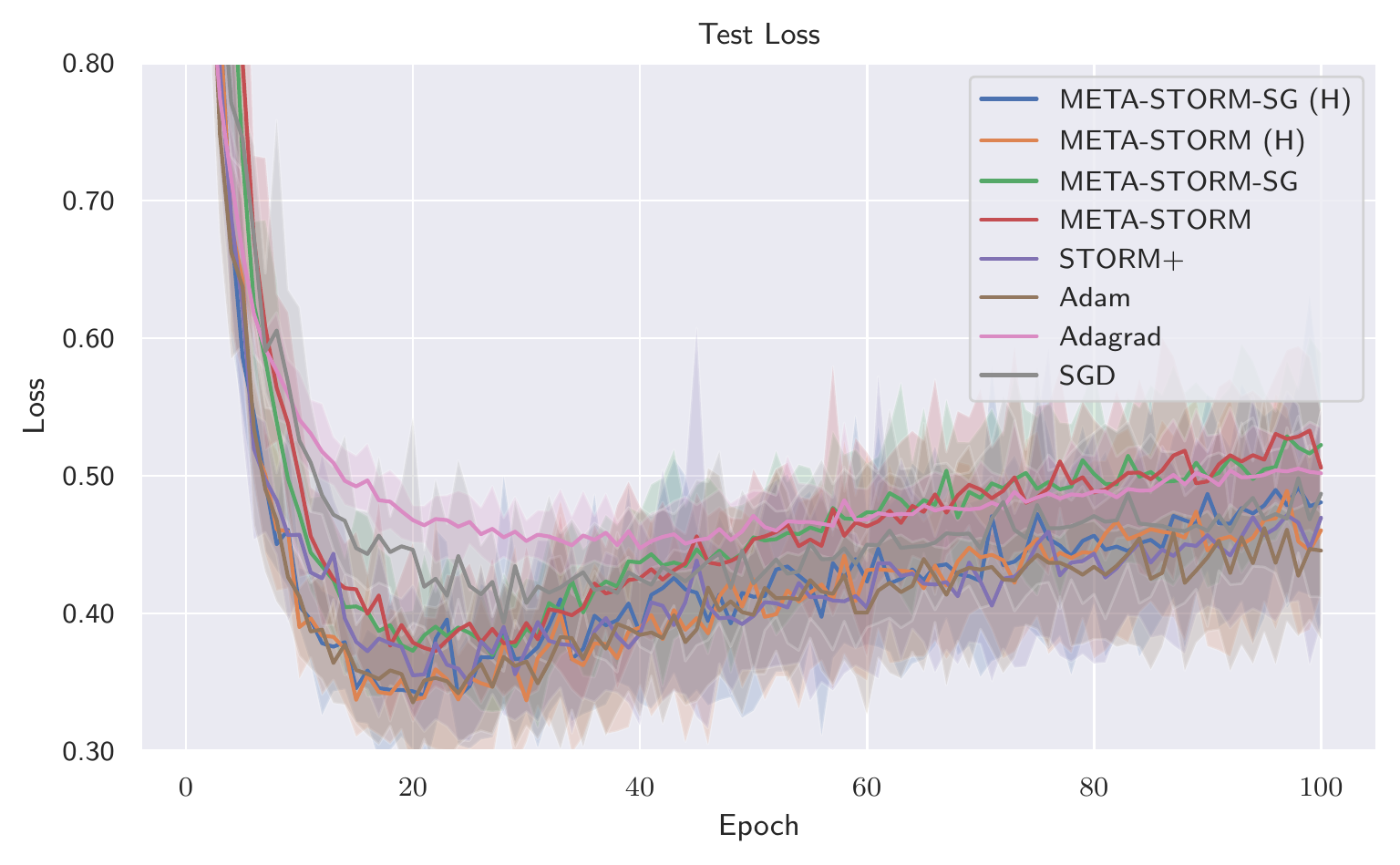}}\hfill{}
\par\end{centering}
\begin{centering}
\subfloat{\includegraphics[width=0.49\textwidth]{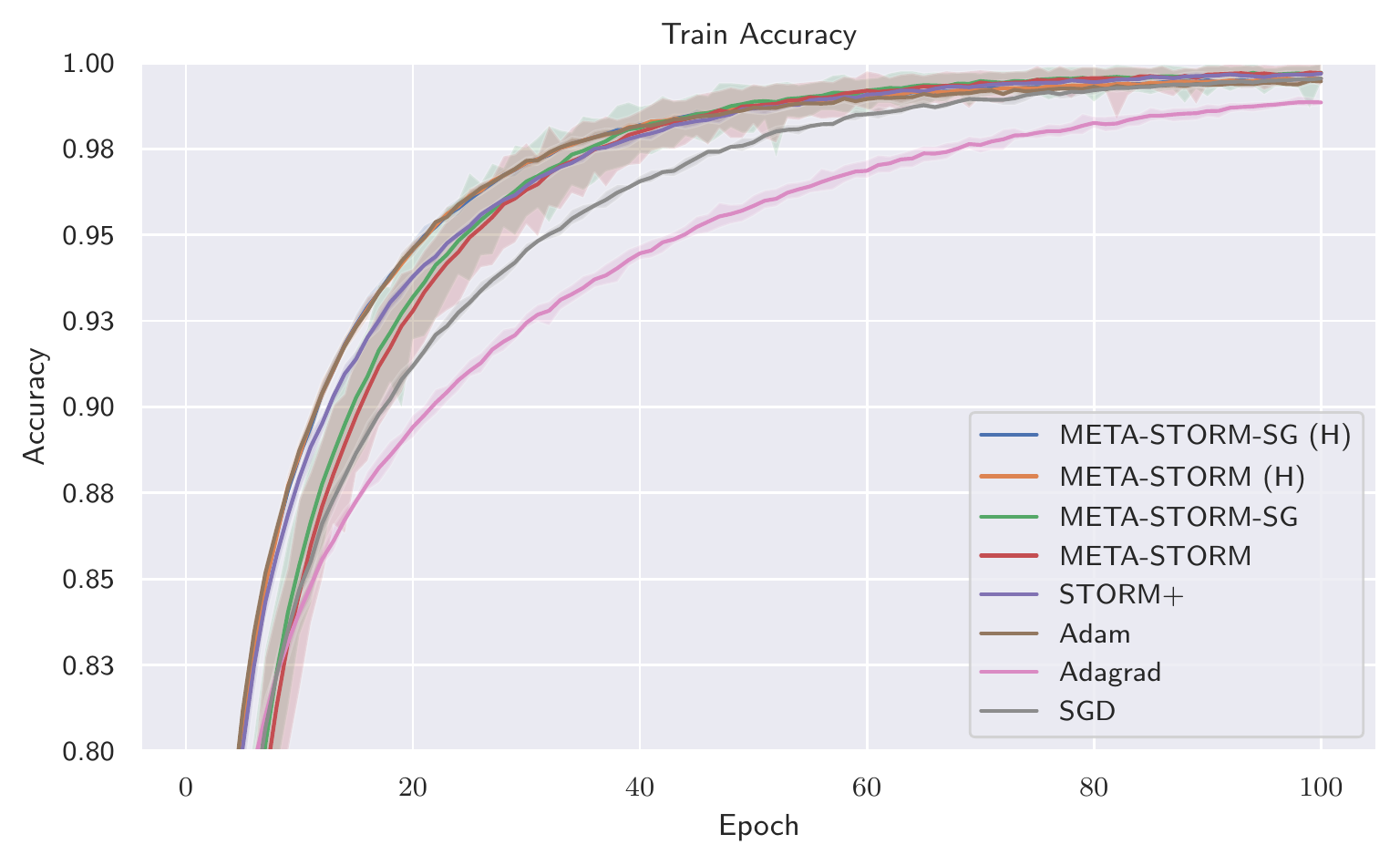}}\subfloat{\includegraphics[width=0.49\textwidth]{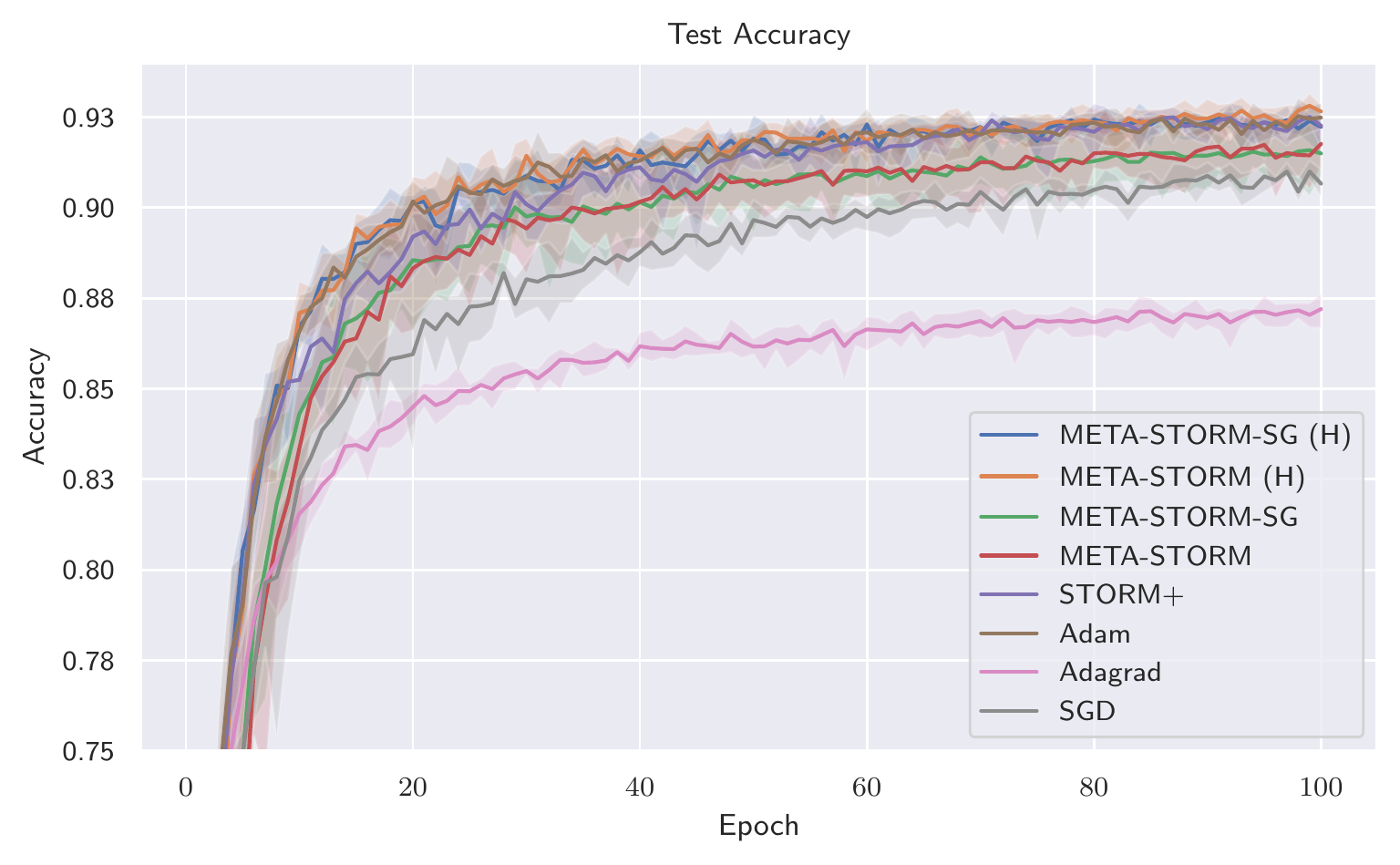}}\hfill{}
\par\end{centering}
\centering{}\caption{\label{fig:cifar10-full-results}Losses and accuracies on CIFAR10.}
\end{figure}

\paragraph{Tables. }

Tables \ref{tab:CIFAR10-train-loss} and \ref{tab:CIFAR10-train-acc}
show the training loss and accuracy for CIFAR10. Tables \ref{tab:CIFAR10-test-loss}
and \ref{tab:CIFAR10-test-acc} show the test loss and accuracy for
CIFAR10.
\begin{center}
{\small{}}
\begin{table}[H]
{\small{}\caption{CIFAR10 average training loss across 5 seeds for selected epochs.
Lowest loss is bolded per selected epoch. \label{tab:CIFAR10-train-loss}}
}{\small\par}

{\small{}}{\small\par}
\centering{}{\small{}}%
\begin{tabular}{lccccccccc}
\toprule 
{\small{}Algorithm} & {\small{}1} & {\small{}10} & {\small{}20} & {\small{}40} & {\small{}60} & {\small{}70} & {\small{}80} & {\small{}90} & {\small{}100}\tabularnewline
\midrule 
{\small{}MS-SG (H)} & {\small{}1.611} & {\small{}0.328} & {\small{}0.154} & {\small{}0.052} & {\small{}0.029} & {\small{}0.023} & {\small{}0.018} & {\small{}0.015} & {\small{}0.014$\pm$0.004}\tabularnewline
{\small{}MS (H)} & {\small{}1.618} & {\small{}0.329} & \textbf{\small{}0.153} & \textbf{\small{}0.051} & {\small{}0.028} & {\small{}0.023} & {\small{}0.018} & {\small{}0.016} & {\small{}0.014$\pm$0.003}\tabularnewline
{\small{}MS-SG} & {\small{}1.899} & {\small{}0.421} & {\small{}0.194} & {\small{}0.053} & \textbf{\small{}0.023} & \textbf{\small{}0.016} & {\small{}0.014} & {\small{}0.011} & \textbf{\small{}0.008}{\small{}$\pm$0.001}\tabularnewline
{\small{}MS} & {\small{}1.941} & {\small{}0.441} & {\small{}0.204} & {\small{}0.056} & \textbf{\small{}0.023} & {\small{}0.017} & \textbf{\small{}0.013} & \textbf{\small{}0.010} & {\small{}0.009$\pm$0.001}\tabularnewline
{\small{}STORM+} & {\small{}1.604} & {\small{}0.349} & {\small{}0.177} & {\small{}0.059} & {\small{}0.026} & {\small{}0.019} & {\small{}0.017} & {\small{}0.011} & {\small{}0.009$\pm$0.002}\tabularnewline
{\small{}Adam} & {\small{}1.452} & \textbf{\small{}0.327} & \textbf{\small{}0.153} & {\small{}0.052} & {\small{}0.030} & {\small{}0.024} & {\small{}0.021} & {\small{}0.019} & {\small{}0.016$\pm$0.003}\tabularnewline
{\small{}Adagrad} & \textbf{\small{}1.359} & {\small{}0.456} & {\small{}0.305} & {\small{}0.164} & {\small{}0.096} & {\small{}0.076} & {\small{}0.059} & {\small{}0.048} & {\small{}0.040$\pm$0.002}\tabularnewline
{\small{}SGD} & {\small{}1.561} & {\small{}0.441} & {\small{}0.253} & {\small{}0.097} & {\small{}0.044} & {\small{}0.031} & {\small{}0.024} & {\small{}0.019} & {\small{}0.014$\pm$0.001}\tabularnewline
\bottomrule
\end{tabular}{\small\par}
\end{table}
{\small\par}
\par\end{center}

\begin{center}
\begin{table}[H]
\caption{CIFAR10 average training accuracy across 5 seeds for selected epochs.
Highest accuracy is bolded per selected epoch. \label{tab:CIFAR10-train-acc}}

{\small{}}{\small\par}
\centering{}{\small{}}%
\begin{tabular}{lccccccccc}
\toprule 
{\small{}Algorithm} & {\small{}1} & {\small{}10} & {\small{}20} & {\small{}40} & {\small{}60} & {\small{}70} & {\small{}80} & {\small{}90} & {\small{}100}\tabularnewline
\midrule 
{\small{}MS-SG (H)} & {\small{}0.405} & {\small{}0.886} & \textbf{\small{}0.946} & \textbf{\small{}0.982} & {\small{}0.990} & {\small{}0.992} & {\small{}0.994} & {\small{}0.995} & {\small{}0.995$\pm$0.000}\tabularnewline
{\small{}MS (H)} & {\small{}0.403} & {\small{}0.886} & \textbf{\small{}0.946} & \textbf{\small{}0.982} & {\small{}0.990} & {\small{}0.992} & {\small{}0.994} & {\small{}0.995} & {\small{}0.995$\pm$0.000}\tabularnewline
{\small{}MS-SG} & {\small{}0.317} & {\small{}0.854} & {\small{}0.932} & {\small{}0.981} & \textbf{\small{}0.992} & \textbf{\small{}0.995} & {\small{}0.995} & {\small{}0.996} & \textbf{\small{}0.997}{\small{}$\pm$0.001}\tabularnewline
{\small{}MS} & {\small{}0.306} & {\small{}0.846} & {\small{}0.928} & {\small{}0.980} & \textbf{\small{}0.992} & {\small{}0.994} & \textbf{\small{}0.996} & \textbf{\small{}0.997} & \textbf{\small{}0.997}{\small{}$\pm$0.000}\tabularnewline
{\small{}STORM+} & {\small{}0.413} & {\small{}0.879} & {\small{}0.938} & {\small{}0.979} & {\small{}0.991} & {\small{}0.993} & {\small{}0.994} & {\small{}0.996} & \textbf{\small{}0.997}{\small{}$\pm$0.000}\tabularnewline
{\small{}Adam} & {\small{}0.468} & \textbf{\small{}0.887} & \textbf{\small{}0.946} & \textbf{\small{}0.982} & {\small{}0.989} & {\small{}0.992} & {\small{}0.993} & {\small{}0.994} & {\small{}0.995$\pm$0.000}\tabularnewline
{\small{}Adagrad} & \textbf{\small{}0.504} & {\small{}0.840} & {\small{}0.894} & {\small{}0.945} & {\small{}0.969} & {\small{}0.976} & {\small{}0.983} & {\small{}0.986} & {\small{}0.988$\pm$0.001}\tabularnewline
{\small{}SGD} & {\small{}0.423} & {\small{}0.847} & {\small{}0.912} & {\small{}0.966} & {\small{}0.985} & {\small{}0.989} & {\small{}0.992} & {\small{}0.994} & {\small{}0.995$\pm$0.000}\tabularnewline
\bottomrule
\end{tabular}{\small\par}
\end{table}
\par\end{center}

\begin{center}
\begin{table}[H]
\caption{CIFAR10 average test loss across 5 seeds for selected epochs. Lowest
loss is bolded per selected epoch. \label{tab:CIFAR10-test-loss}}

\centering{}{\small{}}%
\begin{tabular}{lccccccccc}
\toprule 
{\small{}Algorithm} & {\small{}1} & {\small{}10} & {\small{}20} & {\small{}40} & {\small{}60} & {\small{}70} & {\small{}80} & {\small{}90} & {\small{}100}\tabularnewline
\midrule 
{\small{}MS-SG (H)} & {\small{}1.272} & {\small{}0.405} & {\small{}0.343} & {\small{}0.386} & {\small{}0.423} & \textbf{\small{}0.423} & {\small{}0.456} & {\small{}0.487} & {\small{}0.481$\pm$0.039}\tabularnewline
{\small{}MS (H)} & {\small{}1.250} & \textbf{\small{}0.390} & {\small{}0.337} & {\small{}0.390} & {\small{}0.431} & {\small{}0.441} & {\small{}0.444} & {\small{}0.446} & {\small{}0.460$\pm$0.021}\tabularnewline
{\small{}MS-SG} & {\small{}1.553} & {\small{}0.472} & {\small{}0.373} & {\small{}0.437} & {\small{}0.473} & {\small{}0.484} & {\small{}0.501} & {\small{}0.498} & {\small{}0.522$\pm$0.034}\tabularnewline
{\small{}MS} & {\small{}1.577} & {\small{}0.498} & {\small{}0.379} & {\small{}0.425} & {\small{}0.463} & {\small{}0.490} & {\small{}0.488} & {\small{}0.496} & {\small{}0.506$\pm$0.016}\tabularnewline
{\small{}STORM+} & {\small{}1.321} & {\small{}0.457} & {\small{}0.355} & {\small{}0.385} & {\small{}0.404} & \textbf{\small{}0.423} & {\small{}0.443} & {\small{}0.457} & {\small{}0.470$\pm$0.025}\tabularnewline
{\small{}Adam} & {\small{}1.222} & {\small{}0.412} & \textbf{\small{}0.335} & \textbf{\small{}0.384} & \textbf{\small{}0.401} & {\small{}0.432} & \textbf{\small{}0.434} & \textbf{\small{}0.441} & \textbf{\small{}0.446}{\small{}$\pm$0.025}\tabularnewline
{\small{}Adagrad} & \textbf{\small{}1.104} & {\small{}0.541} & {\small{}0.468} & {\small{}0.447} & {\small{}0.468} & {\small{}0.476} & {\small{}0.488} & {\small{}0.499} & {\small{}0.502$\pm$0.013}\tabularnewline
{\small{}SGD} & {\small{}1.315} & {\small{}0.525} & {\small{}0.446} & {\small{}0.425} & {\small{}0.450} & {\small{}0.447} & {\small{}0.471} & {\small{}0.460} & {\small{}0.487$\pm$0.017}\tabularnewline
\bottomrule
\end{tabular}{\small\par}
\end{table}
\par\end{center}

\begin{center}
\begin{table}[H]
\caption{CIFAR10 average test accuracy across 5 seeds for selected epochs.
Highest accuracy is bolded per selected epoch.\label{tab:CIFAR10-test-acc}}

\centering{}{\small{}}%
\begin{tabular}{lccccccccc}
\toprule 
{\small{}Algorithm} & {\small{}1} & {\small{}10} & {\small{}20} & {\small{}40} & {\small{}60} & {\small{}70} & {\small{}80} & {\small{}90} & {\small{}100}\tabularnewline
\midrule 
{\small{}MS-SG (H)} & {\small{}0.539} & {\small{}0.867} & {\small{}0.901} & \textbf{\small{}0.916} & \textbf{\small{}0.923} & \textbf{\small{}0.922} & \textbf{\small{}0.924} & {\small{}0.924} & {\small{}0.922$\pm$0.005}\tabularnewline
{\small{}MS (H)} & {\small{}0.546} & \textbf{\small{}0.871} & {\small{}0.901} & {\small{}0.914} & {\small{}0.918} & \textbf{\small{}0.922} & \textbf{\small{}0.924} & \textbf{\small{}0.925} & \textbf{\small{}0.927}{\small{}$\pm$0.001}\tabularnewline
{\small{}MS-SG} & {\small{}0.427} & {\small{}0.843} & {\small{}0.886} & {\small{}0.902} & {\small{}0.909} & {\small{}0.914} & {\small{}0.913} & {\small{}0.914} & {\small{}0.915$\pm$0.004}\tabularnewline
{\small{}MS} & {\small{}0.418} & {\small{}0.834} & {\small{}0.883} & {\small{}0.902} & {\small{}0.910} & {\small{}0.913} & {\small{}0.915} & {\small{}0.917} & {\small{}0.918$\pm$0.004}\tabularnewline
{\small{}STORM+} & {\small{}0.529} & {\small{}0.852} & {\small{}0.892} & {\small{}0.911} & {\small{}0.918} & {\small{}0.920} & {\small{}0.921} & {\small{}0.922} & {\small{}0.923$\pm$0.003}\tabularnewline
{\small{}Adam} & {\small{}0.574} & {\small{}0.866} & \textbf{\small{}0.902} & {\small{}0.913} & {\small{}0.921} & {\small{}0.920} & {\small{}0.923} & {\small{}0.923} & {\small{}0.925$\pm$0.002}\tabularnewline
{\small{}Adagrad} & \textbf{\small{}0.601} & {\small{}0.816} & {\small{}0.845} & {\small{}0.862} & {\small{}0.866} & {\small{}0.869} & {\small{}0.868} & {\small{}0.870} & {\small{}0.872$\pm$0.003}\tabularnewline
{\small{}SGD} & {\small{}0.522} & {\small{}0.825} & {\small{}0.860} & {\small{}0.888} & {\small{}0.897} & {\small{}0.904} & {\small{}0.905} & {\small{}0.909} & {\small{}0.907$\pm$0.003}\tabularnewline
\bottomrule
\end{tabular}{\small\par}
\end{table}
\par\end{center}

\paragraph{Discussion.}

META-STORM-SG achieves the lowest training loss and best training
accuracy (with META-STORM and STORM+ coming in close). META-STORM-SG
maintains the best training loss and accuracy for longest before the
final epoch. For test loss and test accuracy, META-STORM (H) attains
the best test accuracy (with Adam coming in close) while Adam attains
the best test loss. While META-STORM-SG and META-STORM achieve low
training loss, their generalization performance seems worse than their
heuristic counterparts. 

To further study this generalization gap among different algorithms,
Table \ref{tab:CIFAR10-generalization-gap} shows the generalization
gap of different algorithms. META-STORM with heuristics and Adam achieve
the smallest gap among all the algorithms. For our algorithms, the
version with heuristics exhibit a smaller generalization gap than
the version without the heuristics while STORM+ lies in between. Interestingly,
Adagrad and SGD exhibit larger generalization gaps. 
\begin{center}
\begin{table}[th]
\caption{CIFAR10 accuracy generalization gap (train acc - test acc) of the
last epoch's accuracy.\label{tab:CIFAR10-generalization-gap}}

\centering{}%
\begin{tabular}{@{}lrrrrrrrr@{}}
\toprule 
Algorithm & \multicolumn{1}{l}{MS-SG (H)} & \multicolumn{1}{l}{MS (H)} & \multicolumn{1}{l}{MS-SG} & \multicolumn{1}{l}{MS} & \multicolumn{1}{l}{STORM+} & \multicolumn{1}{l}{Adam} & \multicolumn{1}{l}{Adagrad} & \multicolumn{1}{l}{SGD}\tabularnewline
\midrule 
Test acc & 92.2\% & 92.7\% & 91.7\% & 91.8\% & 92.3\% & 92.5\% & 87.2\% & 90.7\%\tabularnewline
Train acc & 99.5\% & 99.5\% & 99.8\% & 99.7\% & 99.7\% & 99.5\% & 98.9\% & 99.6\%\tabularnewline
\midrule
Gen gap & 7.3\% & \textbf{6.8\%} & 8.1\% & 7.9\% & 7.4\% & 7.0\% & 11.7\% & 8.9\%\tabularnewline
\bottomrule
\end{tabular}
\end{table}
\par\end{center}

\subsubsection{IMDB: results and discussions\label{subsec:IMDB-full-results}}

\begin{figure}[b]
\begin{centering}
\subfloat{\includegraphics[width=0.49\textwidth]{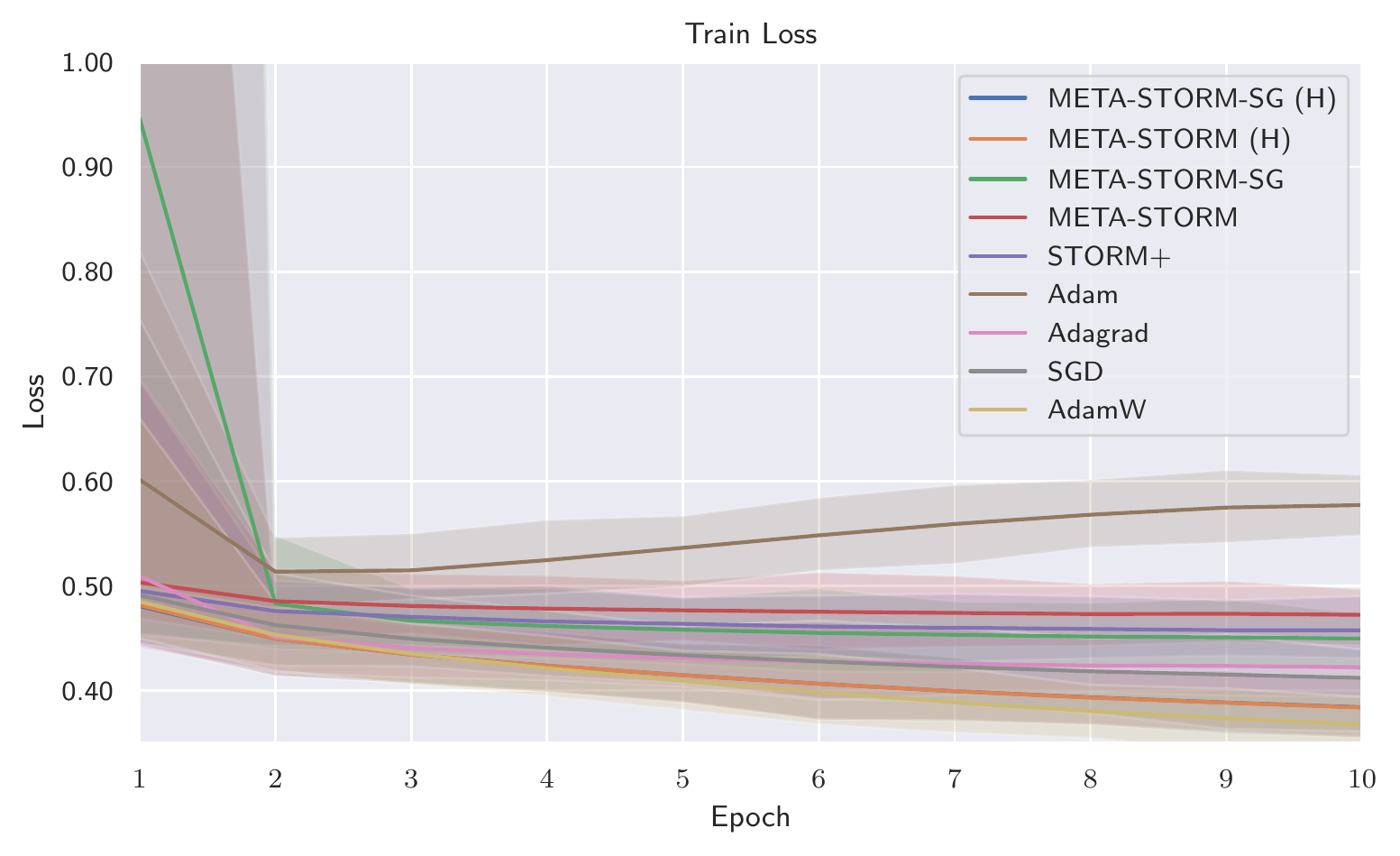}}\subfloat{\includegraphics[width=0.49\textwidth]{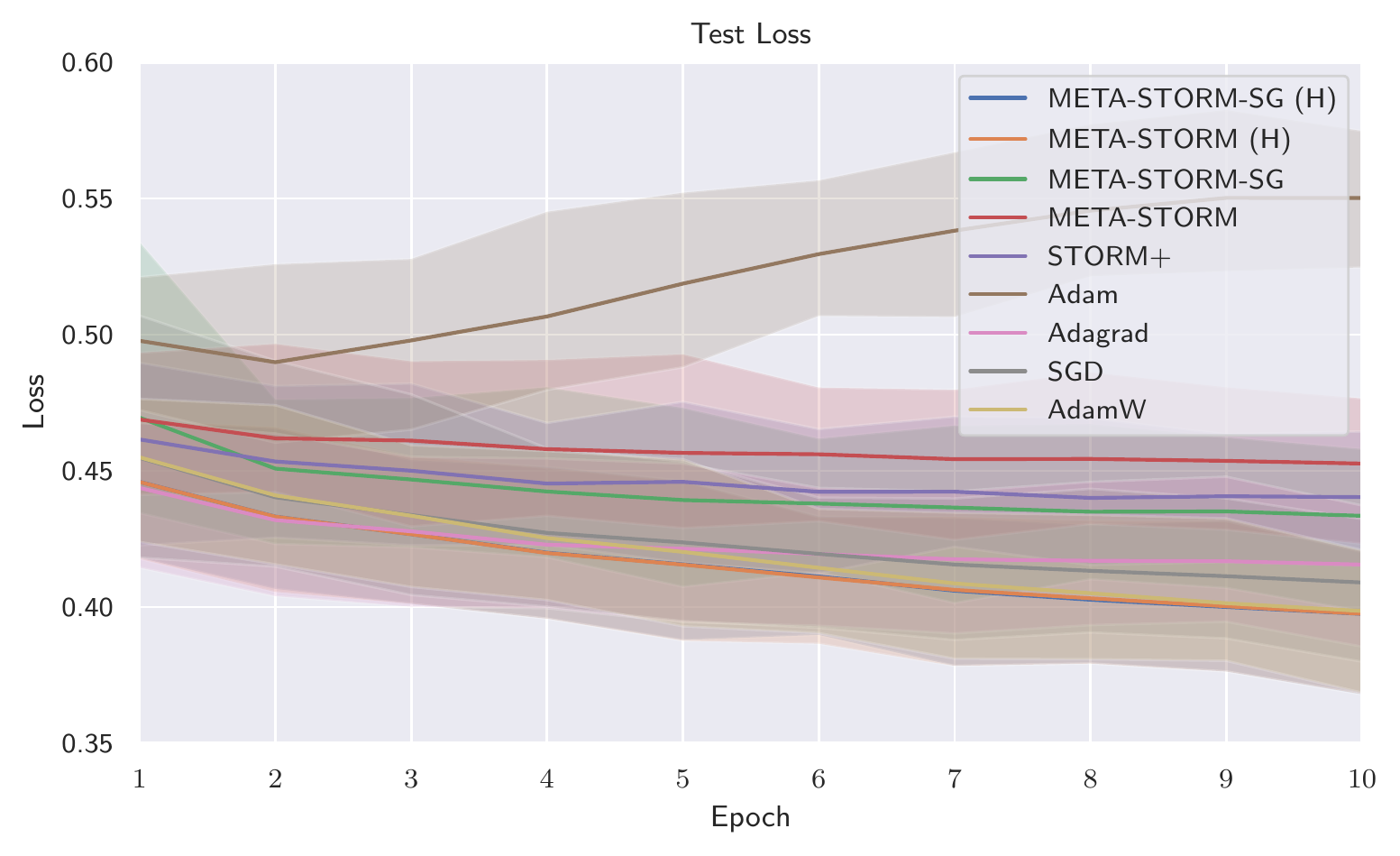}}\caption{Training loss and test loss on IMDB. (H) denotes the addition of heuristics.\label{fig:imdb-train-test-loss-error-bar}}
\par\end{centering}
\end{figure}

Figure \ref{fig:imdb-train-test-loss} from Section \ref{sec:Experiments}
shows the train and test loss of the algorithms used. We include Figure
\ref{fig:imdb-train-test-loss-error-bar} here that includes the error
bars across 5 random seeds. 

\paragraph{Tables.}

Tables \ref{tab:IMDB-train-loss} and \ref{tab:IMDB-test-loss} show
the train and test loss for our experiments. 
\begin{center}
\begin{table}[H]
\caption{IMDB average training loss across 5 seeds for selected epochs. Lowest
loss for each epoch is bolded below.\label{tab:IMDB-train-loss}}

\centering{}{\small{}}%
\begin{tabular}{lcccccccccc}
\toprule 
{\small{}Algorithm} & {\small{}1} & {\small{}2} & {\small{}3} & {\small{}4} & {\small{}5} & {\small{}6} & {\small{}7} & {\small{}8} & {\small{}9} & {\small{}10}\tabularnewline
\midrule 
{\small{}MS-SG (H)} & \textbf{\small{}0.481} & \textbf{\small{}0.450} & \textbf{\small{}0.435} & {\small{}0.424} & {\small{}0.415} & {\small{}0.407} & {\small{}0.400} & {\small{}0.394} & {\small{}0.389} & {\small{}0.384$\pm$0.011}\tabularnewline
{\small{}MS (H)} & {\small{}0.482} & \textbf{\small{}0.450} & \textbf{\small{}0.435} & {\small{}0.424} & {\small{}0.415} & {\small{}0.407} & {\small{}0.400} & {\small{}0.393} & {\small{}0.389} & {\small{}0.384$\pm$0.010}\tabularnewline
{\small{}MS-SG} & {\small{}0.947} & {\small{}0.483} & {\small{}0.467} & {\small{}0.462} & {\small{}0.458} & {\small{}0.455} & {\small{}0.453} & {\small{}0.452} & {\small{}0.451} & {\small{}0.450$\pm$0.009}\tabularnewline
{\small{}MS} & {\small{}0.503} & {\small{}0.486} & {\small{}0.481} & {\small{}0.478} & {\small{}0.477} & {\small{}0.475} & {\small{}0.474} & {\small{}0.473} & {\small{}0.473} & {\small{}0.472$\pm$0.011}\tabularnewline
{\small{}STORM+} & {\small{}0.495} & {\small{}0.476} & {\small{}0.471} & {\small{}0.466} & {\small{}0.464} & {\small{}0.461} & {\small{}0.460} & {\small{}0.459} & {\small{}0.458} & {\small{}0.458$\pm$0.007}\tabularnewline
{\small{}Adam} & {\small{}0.602} & {\small{}0.514} & {\small{}0.515} & {\small{}0.525} & {\small{}0.536} & {\small{}0.548} & {\small{}0.559} & {\small{}0.568} & {\small{}0.575} & {\small{}0.577$\pm$0.013}\tabularnewline
{\small{}Adagrad} & {\small{}0.509} & {\small{}0.451} & {\small{}0.441} & {\small{}0.435} & {\small{}0.431} & {\small{}0.428} & {\small{}0.426} & {\small{}0.424} & {\small{}0.424} & {\small{}0.422$\pm$0.009}\tabularnewline
{\small{}SGD} & {\small{}0.491} & {\small{}0.463} & {\small{}0.450} & {\small{}0.441} & {\small{}0.434} & {\small{}0.428} & {\small{}0.423} & {\small{}0.419} & {\small{}0.415} & {\small{}0.412$\pm$0.010}\tabularnewline
{\small{}AdamW} & {\small{}0.485} & {\small{}0.453} & \textbf{\small{}0.435} & \textbf{\small{}0.421} & \textbf{\small{}0.410} & \textbf{\small{}0.399} & \textbf{\small{}0.389} & \textbf{\small{}0.381} & \textbf{\small{}0.374} & \textbf{\small{}0.368}{\small{}$\pm$0.010}\tabularnewline
\bottomrule
\end{tabular}{\small\par}
\end{table}
\par\end{center}

\begin{center}
\begin{table}[H]
\caption{IMDB test loss. Lowest loss for each epoch is bolded below.\label{tab:IMDB-test-loss}}

\centering{}{\small{}}%
\begin{tabular}{lcccccccccc}
\toprule 
{\small{}Algorithm} & {\small{}1} & {\small{}2} & {\small{}3} & {\small{}4} & {\small{}5} & {\small{}6} & {\small{}7} & {\small{}8} & {\small{}9} & {\small{}10}\tabularnewline
\midrule 
{\small{}MS-SG (H)} & {\small{}0.446} & {\small{}0.433} & \textbf{\small{}0.427} & \textbf{\small{}0.420} & \textbf{\small{}0.416} & \textbf{\small{}0.411} & \textbf{\small{}0.406} & \textbf{\small{}0.403} & \textbf{\small{}0.400} & \textbf{\small{}0.397}{\small{}$\pm$0.010}\tabularnewline
{\small{}MS (H)} & {\small{}0.446} & {\small{}0.433} & \textbf{\small{}0.427} & \textbf{\small{}0.420} & \textbf{\small{}0.416} & \textbf{\small{}0.411} & \textbf{\small{}0.406} & \textbf{\small{}0.403} & \textbf{\small{}0.400} & \textbf{\small{}0.397}{\small{}$\pm$0.010}\tabularnewline
{\small{}MS-SG} & {\small{}0.470} & {\small{}0.451} & {\small{}0.447} & {\small{}0.442} & {\small{}0.439} & {\small{}0.438} & {\small{}0.436} & {\small{}0.435} & {\small{}0.435} & {\small{}0.433$\pm$0.012}\tabularnewline
{\small{}MS} & {\small{}0.469} & {\small{}0.462} & {\small{}0.461} & {\small{}0.458} & {\small{}0.457} & {\small{}0.456} & {\small{}0.454} & {\small{}0.454} & {\small{}0.454} & {\small{}0.453$\pm$0.010}\tabularnewline
{\small{}STORM+} & {\small{}0.462} & {\small{}0.453} & {\small{}0.450} & {\small{}0.445} & {\small{}0.446} & {\small{}0.442} & {\small{}0.442} & {\small{}0.440} & {\small{}0.441} & {\small{}0.440$\pm$0.012}\tabularnewline
{\small{}Adam} & {\small{}0.498} & {\small{}0.490} & {\small{}0.498} & {\small{}0.507} & {\small{}0.519} & {\small{}0.530} & {\small{}0.538} & {\small{}0.545} & {\small{}0.550} & {\small{}0.550$\pm$0.013}\tabularnewline
{\small{}Adagrad} & \textbf{\small{}0.444} & \textbf{\small{}0.432} & {\small{}0.428} & {\small{}0.423} & {\small{}0.421} & {\small{}0.419} & {\small{}0.417} & {\small{}0.417} & {\small{}0.417} & {\small{}0.416$\pm$0.010}\tabularnewline
{\small{}SGD} & {\small{}0.455} & {\small{}0.440} & {\small{}0.434} & {\small{}0.427} & {\small{}0.424} & {\small{}0.419} & {\small{}0.416} & {\small{}0.413} & {\small{}0.411} & {\small{}0.409$\pm$0.009}\tabularnewline
{\small{}AdamW} & {\small{}0.455} & {\small{}0.441} & {\small{}0.433} & {\small{}0.425} & {\small{}0.420} & {\small{}0.414} & {\small{}0.409} & {\small{}0.405} & {\small{}0.401} & {\small{}0.398$\pm$0.009}\tabularnewline
\bottomrule
\end{tabular}{\small\par}
\end{table}
\par\end{center}

\paragraph{Discussion.}

Here, AdamW achieves the best training loss with the heuristic algorithms
coming in close. For the test loss, these algorithms also have similar
performances. All META-STORM algorithms (with and without heuristics)
perform better than STORM+ in minimizing training loss. For test loss,
META-STORM-SG performs better than STORM+ but META-STORM does not.
Both the heuristic versions of META-STORM and META-STORM-SG outperform
STORM+. 

\subsubsection{SST2: full results and discussions\label{subsec:SST2:-full-results}}
\begin{center}
\begin{figure}

\begin{centering}
\subfloat{\includegraphics[width=0.49\textwidth]{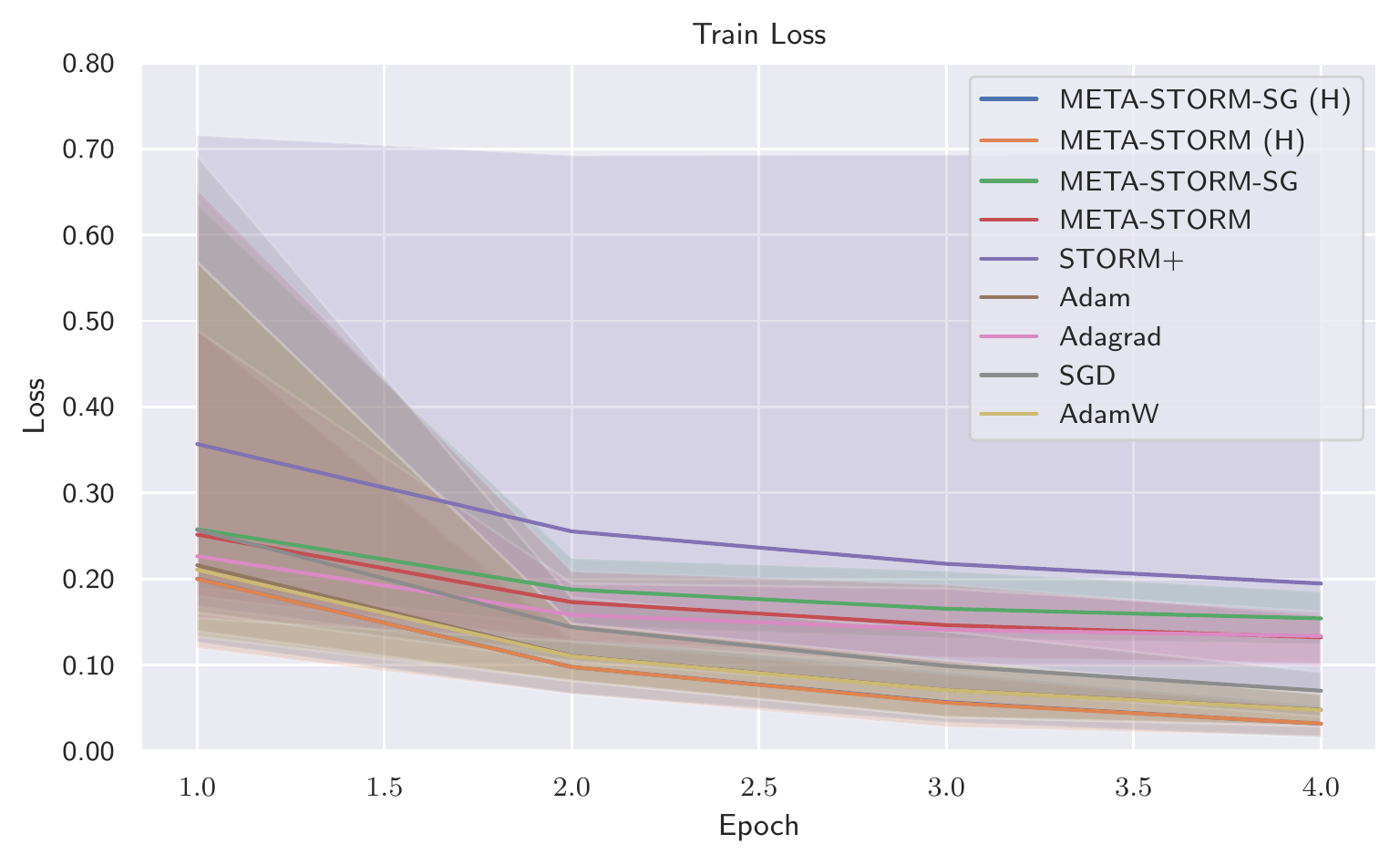}}\subfloat{\includegraphics[width=0.49\textwidth]{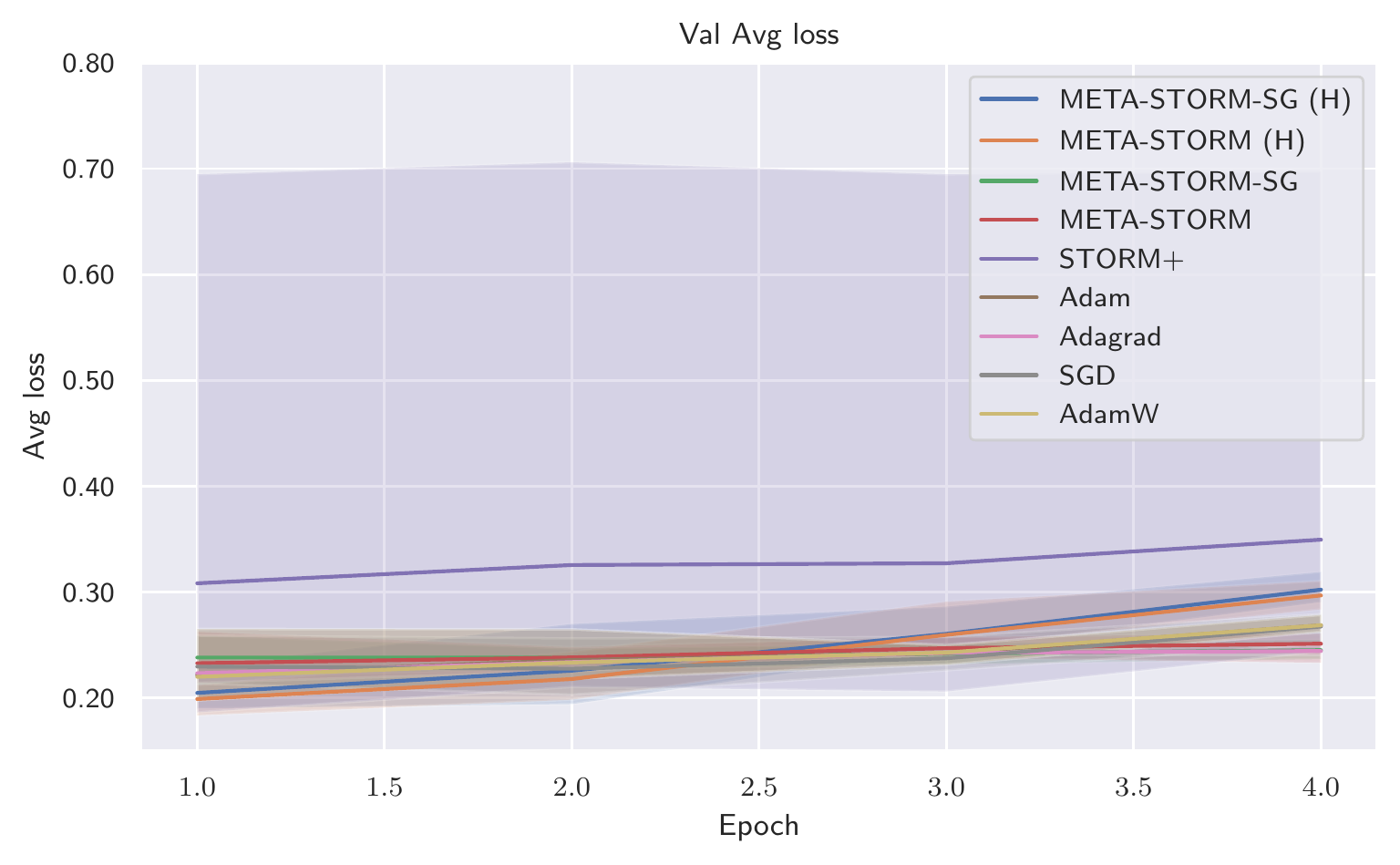}}\hfill{}
\par\end{centering}
\begin{centering}
\subfloat{\includegraphics[width=0.49\textwidth]{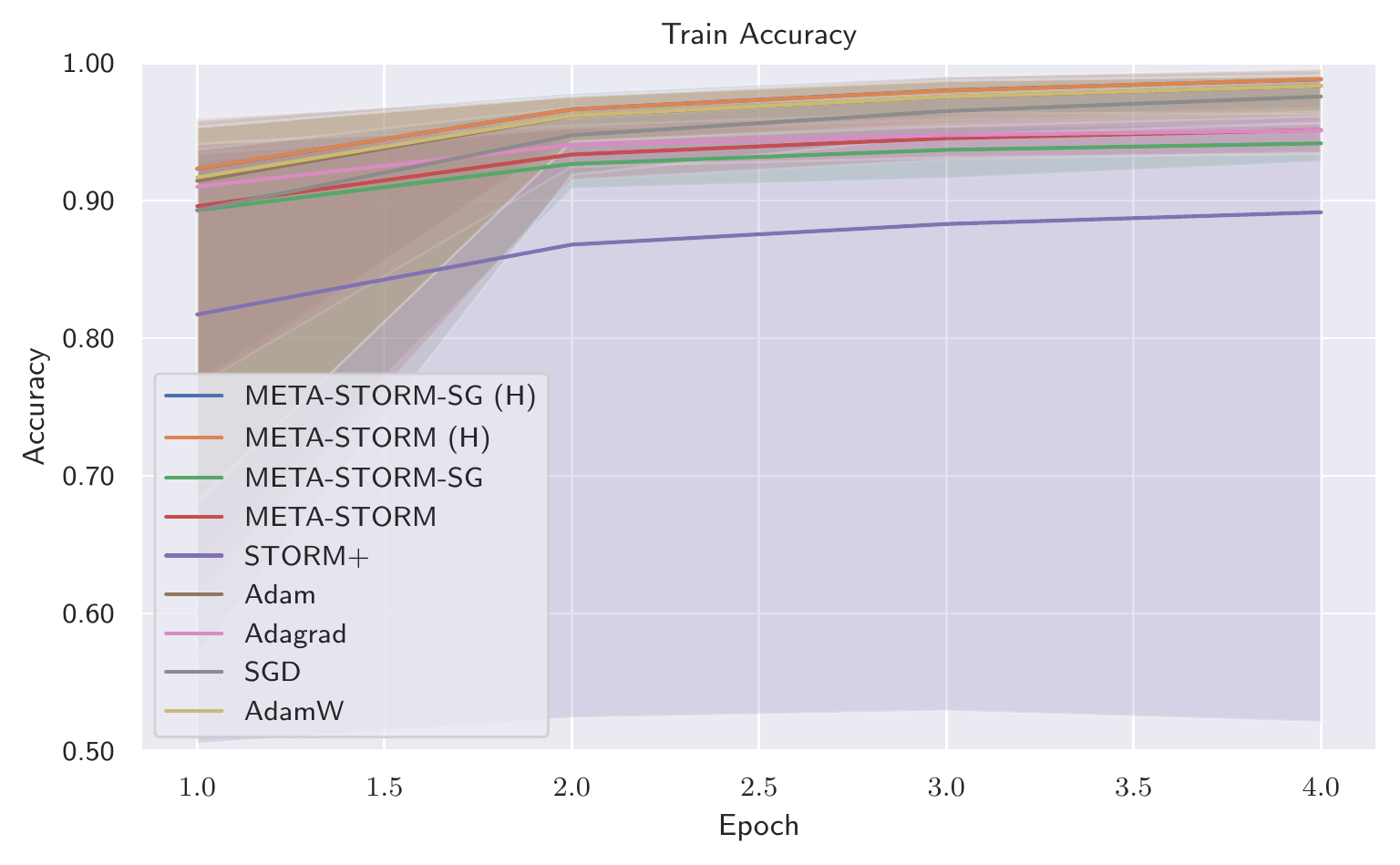}}\subfloat{\includegraphics[width=0.49\textwidth]{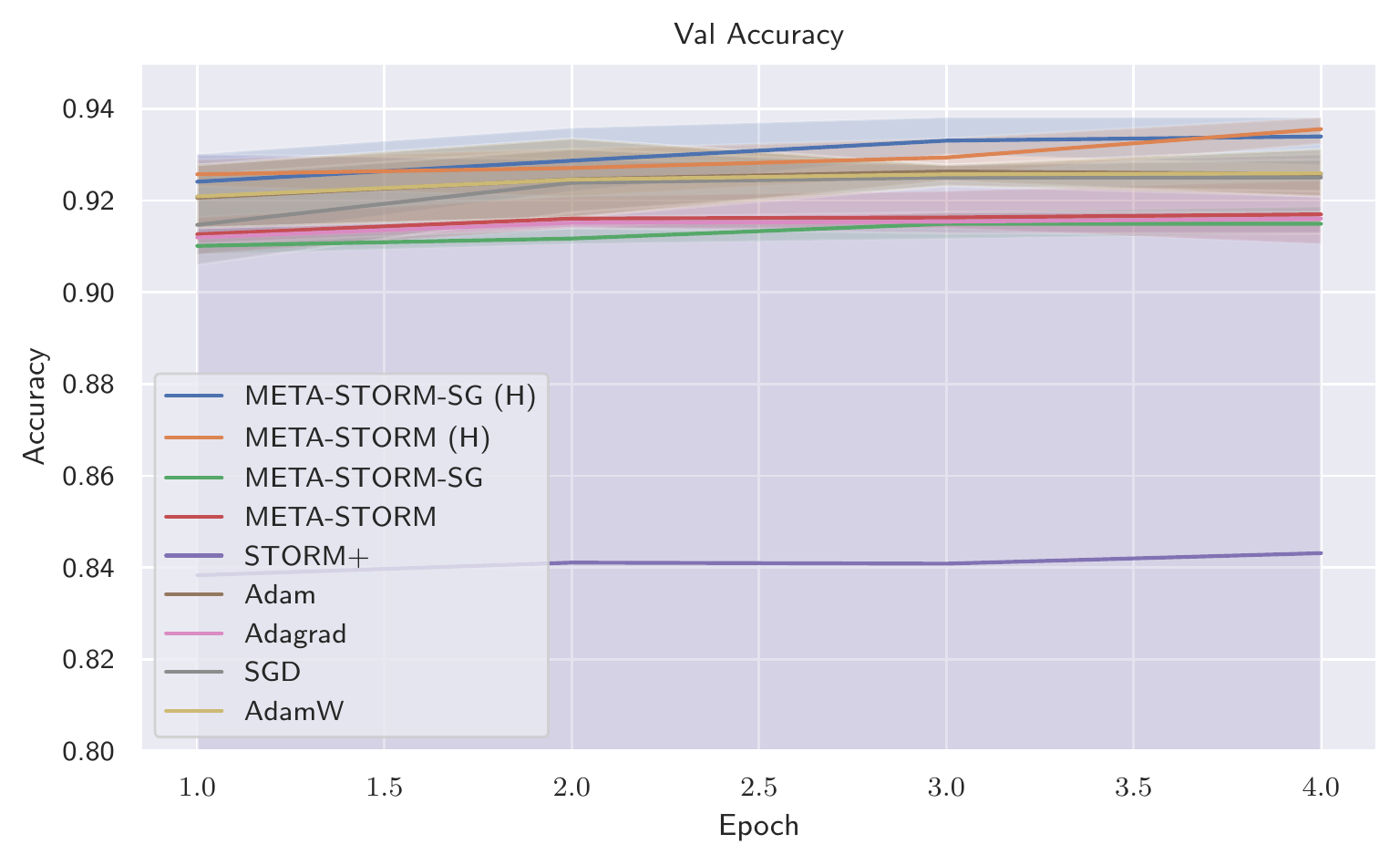}}\hfill{}\caption{Losses and accuracies on SST2.\label{fig:Accuracies-and-losses-sst2}}
\par\end{centering}
\end{figure}
\par\end{center}

Figure \ref{fig:Accuracies-and-losses-sst2} shows all 4 plots of
the main experiments for SST2. 

\paragraph{Tables.}

Tables \ref{tab:sst2-train-loss} and \ref{tab:SST2-train-accuracy.}
present the training loss and accuracy for the experiments for SST2.
Tables \ref{tab:SST2-validation-loss} and \ref{tab:SST2-validation-accuracy}
show the validation loss and accuracy for the experiments for SST2. 
\begin{center}
\begin{table}[H]
\caption{SST2 training loss. Lowest loss for each epoch is bolded below.\label{tab:sst2-train-loss}}

\centering{}{\small{}}%
\begin{tabular}{lrrrr}
\toprule 
{\small{}Algorithm} & {\small{}1} & {\small{}2} & {\small{}3} & {\small{}4}\tabularnewline
\midrule 
{\small{}META-STORM-SG (H)} & \textbf{\small{}0.200} & \textbf{\small{}0.098} & {\small{}0.057} & \textbf{\small{}0.032}{\small{}$\pm$0.005}\tabularnewline
{\small{}META-STORM (H)} & \textbf{\small{}0.200} & \textbf{\small{}0.098} & \textbf{\small{}0.056} & \textbf{\small{}0.032}{\small{}$\pm$0.005}\tabularnewline
{\small{}META-STORM-SG} & {\small{}0.258} & {\small{}0.188} & {\small{}0.165} & {\small{}0.154$\pm$0.008}\tabularnewline
{\small{}META-STORM} & {\small{}0.251} & {\small{}0.173} & {\small{}0.146} & {\small{}0.132$\pm$0.008}\tabularnewline
{\small{}STORM+} & {\small{}0.357} & {\small{}0.255} & {\small{}0.218} & {\small{}0.195$\pm$0.269}\tabularnewline
{\small{}Adam} & {\small{}0.216} & {\small{}0.111} & {\small{}0.071} & {\small{}0.048$\pm$0.006}\tabularnewline
{\small{}Adagrad} & {\small{}0.227} & {\small{}0.158} & {\small{}0.141} & {\small{}0.134$\pm$0.006}\tabularnewline
{\small{}SGD} & {\small{}0.257} & {\small{}0.144} & {\small{}0.099} & {\small{}0.070$\pm$0.005}\tabularnewline
{\small{}AdamW} & {\small{}0.211} & {\small{}0.110} & {\small{}0.071} & {\small{}0.048$\pm$0.006}\tabularnewline
\bottomrule
\end{tabular}{\small\par}
\end{table}
\par\end{center}

\begin{center}
\begin{table}[H]
\caption{SST2 training accuracy. Highest accuracy for each epoch is bolded
below. \label{tab:SST2-train-accuracy.}}

\centering{}{\small{}}%
\begin{tabular}{lrrrr}
\toprule 
{\small{}Algorithm} & {\small{}1} & {\small{}2} & {\small{}3} & {\small{}4}\tabularnewline
\midrule 
{\small{}META-STORM-SG (H)} & \textbf{\small{}0.923} & \textbf{\small{}0.966} & \textbf{\small{}0.980} & \textbf{\small{}0.988}{\small{}$\pm$0.002}\tabularnewline
{\small{}META-STORM (H)} & \textbf{\small{}0.923} & \textbf{\small{}0.966} & \textbf{\small{}0.980} & \textbf{\small{}0.988}{\small{}$\pm$0.003}\tabularnewline
{\small{}META-STORM-SG} & {\small{}0.893} & {\small{}0.927} & {\small{}0.937} & {\small{}0.941$\pm$0.003}\tabularnewline
{\small{}META-STORM} & {\small{}0.896} & {\small{}0.933} & {\small{}0.945} & {\small{}0.951$\pm$0.002}\tabularnewline
{\small{}STORM+} & {\small{}0.817} & {\small{}0.868} & {\small{}0.883} & {\small{}0.891$\pm$0.179}\tabularnewline
{\small{}Adam} & {\small{}0.914} & {\small{}0.961} & {\small{}0.975} & {\small{}0.983$\pm$0.001}\tabularnewline
{\small{}Adagrad} & {\small{}0.910} & {\small{}0.940} & {\small{}0.947} & {\small{}0.951$\pm$0.003}\tabularnewline
{\small{}SGD} & {\small{}0.893} & {\small{}0.947} & {\small{}0.965} & {\small{}0.976$\pm$0.002}\tabularnewline
{\small{}AdamW} & {\small{}0.917} & {\small{}0.961} & {\small{}0.975} & {\small{}0.983$\pm$0.002}\tabularnewline
\bottomrule
\end{tabular}{\small\par}
\end{table}
\par\end{center}

\begin{center}
\begin{table}[H]
\caption{SST2 validation loss. Lowest loss for each epoch is bolded below.
\label{tab:SST2-validation-loss}}

\centering{}{\small{}}%
\begin{tabular}{lrrrr}
\toprule 
{\small{}Algorithm} & {\small{}1} & {\small{}2} & {\small{}3} & {\small{}4}\tabularnewline
\midrule 
{\small{}META-STORM-SG (H)} & {\small{}0.205} & {\small{}0.226} & {\small{}0.261} & {\small{}0.302$\pm$0.012}\tabularnewline
{\small{}META-STORM (H)} & \textbf{\small{}0.199} & \textbf{\small{}0.218} & {\small{}0.260} & {\small{}0.297$\pm$0.010}\tabularnewline
{\small{}META-STORM-SG} & {\small{}0.238} & {\small{}0.238} & \textbf{\small{}0.242} & {\small{}0.245$\pm$0.007}\tabularnewline
{\small{}META-STORM} & {\small{}0.233} & {\small{}0.238} & {\small{}0.247} & {\small{}0.251$\pm$0.011}\tabularnewline
{\small{}STORM+} & {\small{}0.308} & {\small{}0.326} & {\small{}0.327} & {\small{}0.350$\pm$0.195}\tabularnewline
{\small{}Adam} & {\small{}0.222} & {\small{}0.236} & \textbf{\small{}0.242} & {\small{}0.269$\pm$0.007}\tabularnewline
{\small{}Adagrad} & {\small{}0.223} & {\small{}0.234} & {\small{}0.243} & \textbf{\small{}0.244}{\small{}$\pm$0.003}\tabularnewline
{\small{}SGD} & {\small{}0.230} & {\small{}0.228} & {\small{}0.238} & {\small{}0.268$\pm$0.011}\tabularnewline
{\small{}AdamW} & {\small{}0.220} & {\small{}0.234} & {\small{}0.243} & {\small{}0.269$\pm$0.006}\tabularnewline
\bottomrule
\end{tabular}{\small\par}
\end{table}
\par\end{center}

\begin{center}
\begin{table}[H]
\caption{SST2 validation accuracy. Highest accuracy for each epoch is bolded
below. \label{tab:SST2-validation-accuracy}}

\centering{}{\small{}}%
\begin{tabular}{lrrrr}
\toprule 
{\small{}Algorithm} & {\small{}1} & {\small{}2} & {\small{}3} & {\small{}4}\tabularnewline
\midrule 
{\small{}META-STORM-SG (H)} & {\small{}0.924} & \textbf{\small{}0.929} & \textbf{\small{}0.933} & {\small{}0.934$\pm$0.004}\tabularnewline
{\small{}META-STORM (H)} & \textbf{\small{}0.926} & {\small{}0.927} & {\small{}0.929} & \textbf{\small{}0.936}{\small{}$\pm$0.002}\tabularnewline
{\small{}META-STORM-SG} & {\small{}0.910} & {\small{}0.912} & {\small{}0.915} & {\small{}0.915$\pm$0.002}\tabularnewline
{\small{}META-STORM} & {\small{}0.913} & {\small{}0.916} & {\small{}0.916} & {\small{}0.917$\pm$0.005}\tabularnewline
{\small{}STORM+} & {\small{}0.838} & {\small{}0.841} & {\small{}0.841} & {\small{}0.843$\pm$0.187}\tabularnewline
{\small{}Adam} & {\small{}0.921} & {\small{}0.925} & {\small{}0.926} & {\small{}0.926$\pm$0.004}\tabularnewline
{\small{}Adagrad} & {\small{}0.912} & {\small{}0.915} & {\small{}0.915} & {\small{}0.916$\pm$0.002}\tabularnewline
{\small{}SGD} & {\small{}0.915} & {\small{}0.924} & {\small{}0.925} & {\small{}0.925$\pm$0.003}\tabularnewline
{\small{}AdamW} & {\small{}0.921} & {\small{}0.925} & {\small{}0.926} & {\small{}0.926$\pm$0.004}\tabularnewline
\bottomrule
\end{tabular}{\small\par}
\end{table}
\par\end{center}

\paragraph{Discussions.}

Similarly to CIFAR10, we examine the generalization gap of different
algorithms in Table \ref{tab:SST2-gen-gap}. Here, we see that MS-SG
attains the lowest generalization gap between training accuracy and
test accuracy while Adam suffers from the largest generalization gap
among the algorithms compared in our experiments. 
\begin{center}
\begin{table}[H]
\caption{SST2 accuracy generalization gap (train acc $-$ test acc) of the
last epoch's accuracy. \label{tab:SST2-gen-gap}}

\centering{}%
\begin{tabular}{lrrrrrrrr}
\hline 
Algorithm & \multicolumn{1}{l}{MS-SG (H)} & \multicolumn{1}{l}{MS (H)} & \multicolumn{1}{l}{MS-SG} & \multicolumn{1}{l}{MS} & \multicolumn{1}{l}{STORM+} & \multicolumn{1}{l}{Adam} & \multicolumn{1}{l}{Adagrad} & \multicolumn{1}{l}{SGD}\tabularnewline
\hline 
Train acc. & 98.8\% & 98.8\% & 94.1\% & 95.1\% & 89.1\% & 98.3\% & 95.1\% & 97.6\%\tabularnewline
Val acc. & 93.6\% & 93.6\% & 91.5\% & 91.7\% & 84.3\% & 92.6\% & 91.6\% & 92.5\%\tabularnewline
\hline 
Gen. gap & 5.2\% & 5.2\% & \textbf{2.6\%} & 3.4\% & 4.8\% & 5.7\% & 3.5\% & 5.1\%\tabularnewline
\hline 
\end{tabular}
\end{table}
\par\end{center}

\allowdisplaybreaks

\section{Assumptions and notations\label{sec:Notations-and-assumptions}}

\subsection{Assumptions}

We recall the assumptions in Section \ref{sec:preliminaries} we rely
on:

\textbf{1. Lower bounded function value}: $F^{*}\coloneqq\inf_{x\in\R^{d}}F(x)>-\infty$.

\textbf{2. Unbiased estimator with bounded variance}: We assume to
have access to $\nabla f(x,\xi)$ satisfying $\E_{\xi}\left[\nabla f(x,\xi)\right]=\nabla F(x)$,
$\E_{\xi}\left[\norm{\nabla f(x,\xi)-\nabla F(x)}^{2}\right]\le\sigma^{2}$
for some $\sigma\geq0$.

\textbf{3. Averaged $\beta$-smoothness}: $\E_{\xi}\left[\norm{\nabla f(x,\xi)-\nabla f(y,\xi)}^{2}\right]\le\beta^{2}\norm{x-y}^{2},\forall x,y\in\R^{d}$.

\textbf{4. Bounded stochastic gradients:} $\|\nabla f(x,\xi)\|\leq\widehat{G},\forall x\in\R^{d},\xi\in\mathbf{support}(\domxi)$
for some $\hG\geq0$.

\textbf{5. Bounded stochastic gradient differences}: $\norm{\nabla f(x,\xi)-\nabla f(x,\xi')}\le2\hS,\forall x\in\R^{d},\xi,\xi'\in\mathbf{support}(\domxi)$
for some $\hS\geq0$.

We remind the reader that $\sigma=O(\hS)$ and $\hS=O(\hG)$.

\subsection{Notations}

In the analysis below, we employ the following notations
\begin{align*}
\beta_{\max}\coloneqq\max\left\{ \beta,1\right\} ;\quad\bD_{t} & \coloneqq\sum_{i=1}^{t}\left\Vert d_{i}\right\Vert ^{2};\quad\bE_{t,s}\coloneqq\sum_{i=1}^{t}a_{i+1}^{s}\left\Vert \epsilon_{i}\right\Vert ^{2};\\
\bH_{t}\coloneqq\sum_{i=1}^{t}\left\Vert \nabla F(x_{i})\right\Vert ^{2};\quad\hH_{t} & \coloneqq\sum_{i=1}^{t}\left\Vert \nabla f(x_{i},\xi_{i})\right\Vert ^{2};\quad\tH_{t}\coloneqq\sum_{i=1}^{t}\left\Vert \nabla f(x_{i},\xi_{i})-\nabla f(x_{i},\xi_{i+1})\right\Vert ^{2}
\end{align*}
We will also write $\bE_{t}\coloneqq\bE_{t,0}=\sum_{i=1}^{t}\left\Vert \epsilon_{i}\right\Vert ^{2}$.
We denote $\F_{t}=\sigma\left(\xi_{i},1\leq i\leq t\right)$ as the
sigma algebra generated by the first $t$ samples. Besides, we define
$0^{0}\coloneqq1$. In Section \ref{sec:Basic-inequalities}, we will
list and prove all inequalities used in the subsequent proofs.

\section{Proof sketch for Theorem \ref{thm:Main-MS-convergence-rate}\label{sec:Proof-sketch}}

In this section, to give an overview of the proof techniques, we present
the proof sketch for Theorem \ref{thm:Main-MS-convergence-rate} for
the special case $p=\frac{1}{2}$. For simplicity, we assume $\beta\geq1$
to simplify the notation. The analysis of the fully adaptive algorithms
follows a similar approach to the non-adaptive analysis given in Section
\ref{sec:Analysis-outline}. As before, towards our final goal of
bounding $\|\nabla F(\xout)\|$, we will translate to $\bH_{T}$ and
upper bound it via $\bD_{T}$ and $\bE_{T}$.

\textbf{Bounding $\bE_{T}$: }As in existing VR algorithms, we need
to calculate how the stochastic error $\epsilon_{t}$ changes with
each iteration. By a standard calculation, we obtain
\begin{align}
a_{t+1}\|\epsilon_{t}\|^{2} & \leq\|\epsilon_{t}\|^{2}-\|\epsilon_{t+1}\|^{2}+2\|Z_{t+1}\|^{2}+2a_{t+1}^{2}\|\nabla f(x_{t+1},\xi_{t+1})-\nabla F(x_{t+1})\|^{2}+M_{t+1}\label{eq:PS-VR-inequality}
\end{align}
where
\begin{align*}
Z_{t+1} & =\nabla f(x_{t+1},\xi_{t+1})-\nabla f(x_{t},\xi_{t+1})-\nabla F(x_{t+1})+\nabla F(x_{t});\\
M_{t+1} & =2(1-a_{t+1})^{2}\langle\epsilon_{t},Z_{t+1}\rangle+2(1-a_{t+1})a_{t+1}\langle\epsilon_{t},\nabla f(x_{t+1},\xi_{t+1})-\nabla F(x_{t+1})\rangle.
\end{align*}
We note that, in $\text{\ensuremath{\algnamenew}}$, $a_{t+1}\in\F_{t+1}$,
which implies $\E[M_{t+1}\mid\F_{t}]\neq0$. This extra term $M_{t+1}$
makes our analysis more challenging compared with previous works.
Now, we highlight some challenges and point out how to solve them:

\subparagraph{Challenge 1.}

How to obtain a term as close to $\bE_{T}$ as possible with a proper
upper bound? In the L.H.S. of (\ref{eq:PS-VR-inequality}), we can
see an extra coefficient $a_{t+1}$ appear in front of $\|\epsilon_{t}\|^{2}$.
A straightforward option is to divide both sides by $a_{t+1}$ then
sum up to get $\bE_{T}.$ However, if we do so, the following problem
arises. Let us focus on the term $\|Z_{t+1}\|^{2}/a_{t+1}$ . The
averaged $\beta$-smoothness assumption gives 
\[
\E\left[\|Z_{t+1}\|^{2}\mid\F_{t}\right]\leq\eta^{2}\beta^{2}\frac{\|d_{t}\|^{2}}{b_{t}^{2}}.
\]
However, we cannot apply this result to $\|Z_{t+1}\|^{2}/a_{t+1}$
since $a_{t+1}\in\F_{t+1}$ as noted above. If we temporarily think
$a_{t+1}^{-1}\leq ca_{t}^{-1}$ for some constant $c$ (we can expect
this because the change from $a_{t}$ to $a_{t+1}$ is not too large
due to the bounded differences assumption), we will get $\E[a_{t+1}^{-1}|\|Z_{t+1}\|^{2}\mid\F_{t}]\leq\E[ca_{t}^{-1}\|Z_{t+1}\|^{2}\mid\F_{t}]\leq\eta^{2}\beta^{2}\frac{\|d_{t}\|^{2}}{a_{t}b_{t}^{2}}.$
If we plug in the update rule of $b_{t}=(b_{0}^{2}+\bD_{T})^{1/2}/a_{t}^{1/4}$,
then we obtain $\E[\|Z_{t+1}\|^{2}|/a_{t+1}\mid\F_{t}]\leq\eta^{2}\beta^{2}a_{t}^{-1/2}\frac{\|d_{t}\|^{2}}{b_{0}^{2}+\bD_{t}}$.
It can be shown that $\sum_{t=1}^{T}\frac{\|d_{t}\|^{2}}{b_{0}^{2}+\bD_{T}}$
can be upper bounded by $\log\frac{\bD_{T}}{b_{0}^{2}}$, but now
we still have the extra $a_{t}^{-1/2}$ coefficent. To remove it,
it is reasonable to divide both sides of (\ref{eq:PS-VR-inequality})
by $a_{t+1}^{1/2}$ rather than $a_{t+1}$. 

\subparagraph{Challenge 2. }

How to get rid of the term involving $M_{t+1}$? As discussed in Challenge
1, we want to divide both sides by $a_{t+1}^{1/2}$. Now we focus
on the term $a_{t+1}^{-1/2}M_{t+1}$. Again, due to $a_{t+1}\in\F_{t+1}$,
$\E[a_{t+1}^{-1/2}M_{t+1}\mid\F_{t}]\neq0$. An important observation
here is that, if we replace $a_{t+1}$ by $a_{t}$ in $M_{t+1}$,
we will have a martingale difference sequence. Formally, we define
\[
N_{t+1}=2(1-a_{t})^{2}\langle\epsilon_{t},Z_{t+1}\rangle+2(1-a_{t})a_{t}\langle\epsilon_{t},\nabla f(x_{t+1},\xi_{t+1})-\nabla F(x_{t+1})\rangle.
\]
Then $\E[N_{t+1}\mid\F_{t}]$ and $\E[a_{t}^{-1/2}N_{t+1}\mid\F_{t}]$
are both $0$. This observation tells us that, in order to bound $\E[\sum_{t=1}^{T}a_{t+1}^{-1/2}M_{t+1}]$,
it suffices to bound $\E[\sum_{t=1}^{T}a_{t+1}^{-1/2}M_{t+1}-a_{t}^{-1/2}N_{t+1}]$.
Using the Cauchy-Schwartz inequality, we show that the term $\sum_{t=1}^{T}a_{t+1}^{-1/2}M_{t+1}-a_{t}^{-1/2}N_{t+1}$
can be bounded by terms related to $\sum_{t=1}^{T}(a_{t+1}^{-1/2}-a_{t}^{-1/2})\|\epsilon_{t}\|^{2},$$\sum_{t=1}^{T}a_{t+1}^{-1/2}\|Z_{t+1}\|^{2}$
and $\sum_{t=1}^{T}a_{t}^{3/2}\|\nabla f(x_{t+1},\xi_{t+1})-\nabla F(x_{t+1})\|^{2}$.
We then bound these latter terms in turn, and eliminate the term involving
$M_{t+1}$.

After overcoming the two challenges above, we can finally show the
following inequality, where $K_{1},K_{2},K_{4}$ are constants that
depend only on $\sigma,\widehat{\sigma},\beta,a_{0},b_{0},\eta$ and
are independent of $T$.
\begin{align}
\E\left[a_{T+1}^{1/2}\bE_{T}\right]\leq\E\left[\bE_{T,1/2}\right] & \leq K_{1}+K_{2}\E\left[\log\left(1+\tH_{T}/a_{0}^{2}\right)\right]+K_{4}\E\left[\log\left(1+\bD_{T}/b_{0}^{2}\right)\right].\label{eq:PS-E-bound-main}
\end{align}

\textbf{Bounding $\bD_{T}$:} By following the standard non-adaptive
analysis via smoothness, we obtain
\begin{align}
F(x_{t+1}) & \leq F(x_{t})-\frac{\eta}{b_{t}}\langle\nabla F(x_{t}),d_{t}\rangle+\frac{\eta^{2}\beta}{2b_{t}^{2}}\|d_{t}\|^{2}.\label{eq:PS-function-value-eq}
\end{align}
Here we proceed similarly to the non-adaptive analysis from Section
\ref{subsec:Non-adaptive-algorithm}, but start to diverge from the
analysis approach used in STORM+. The STORM+ analysis proceeds by
splitting $-\langle\nabla F(x_{t}),d_{t}\rangle=-\|\nabla F(x_{t})\|^{2}-\langle\nabla F(x_{t}),\epsilon_{t}\rangle\le-\frac{1}{2}\|\nabla F(x_{t})\|^{2}+\frac{1}{2}\|\epsilon_{t}\|^{2}$,
multiplying both sides of (\ref{eq:PS-function-value-eq}) with $b_{t}/\eta$,
and summing up over all iterations. This gives the following upper
bound on $\bH_{T}$:
\begin{align*}
\bH_{T}=\sum_{t=1}^{T}\|\nabla F(x_{t})\|^{2} & \le\sum_{t=1}^{T}\frac{2}{\eta}\left(F(x_{t})-F(x_{t+1})\right)b_{t}+\sum_{t=1}^{T}\|\epsilon_{t}\|^{2}+\eta\beta\sum_{t=1}^{T}\frac{\|d_{t}\|^{2}}{b_{t}}.
\end{align*}
This analysis requires $F(x)$ to be bounded so that the sum $\sum_{t=1}^{T}\frac{2}{\eta}(F(x_{t})-F(x_{t+1}))b_{t}$
can telescope. To remove this assumption, we go back to (\ref{eq:PS-function-value-eq}),
split $-\langle\nabla F(x_{t}),d_{t}\rangle=-\|d_{t}\|^{2}+\langle\epsilon_{t},d_{t}\rangle$,
and upper bound the inner product via the Cauchy-Schwartz inequality
and the inequality $ab\leq\frac{\gamma}{2}a^{2}+\frac{1}{2\gamma}b^{2}$
which holds for any $\gamma>0$:
\begin{align*}
\langle\epsilon_{t},d_{t}\rangle & \le\|\epsilon_{t}\|\|d_{t}\|\leq\frac{\lambda a_{t+1}^{1/2}b_{t}}{2\eta\beta}\|\epsilon_{t}\|^{2}+\frac{\eta\beta}{2\lambda a_{t+1}^{1/2}b_{t}}\|d_{t}\|^{2}
\end{align*}
where $\lambda>0$ is a constant (setting $\lambda$ based on $\widehat{\sigma}$
yields the best dependence on $\widehat{\sigma}$). We note that this
choice will need a bound on $\E[\sum_{t=1}^{T}a_{t+1}^{1/2}\|\epsilon_{t}\|^{2}]$,
and $1/2$ turns out to be the smallest choice of $c$ which makes
$\E[\sum_{t=1}^{T}a_{t+1}^{c}\|\epsilon_{t}\|^{2}]$ have a constant
order. The intuition for setting $\gamma=\frac{\lambda a_{t+1}^{1/2}b_{t}}{\eta\beta}$
is that this coefficient ensures a constant split if $a_{t}$ and
$b_{t}$ correspond to the non-adaptive choices we derived in Section
\ref{subsec:Non-adaptive-algorithm}, which were set so that $a^{1/2}b=\Theta\left(\beta\right)$.
We obtain
\begin{align}
\E\left[\sum_{t=1}^{T}\frac{\|d_{t}\|^{2}}{b_{t}}\right] & \leq\frac{2}{\eta}\left(F(x_{1})-F^{*}\right)+\underbrace{\E\left[\sum_{t=1}^{T}\left(\eta\beta+\frac{\eta\beta}{a_{t+1}^{1/2}\lambda}-b_{t}\right)\frac{\|d_{t}\|^{2}}{b_{t}^{2}}\right]}_{(\star)}+\frac{\lambda}{\eta\beta}\underbrace{\E\left[\bE_{T,1/2}\right]}_{(\star\star)}.\label{eq:PS-D-bound-main}
\end{align}
The term $(\star)$ can be bounded using standard techniques used
in the analyses of adaptive algorithms. The term $(\star\star)$ has
already been bounded in the previous analysis. Now we only need to
simplify the term on the L.H.S. to $\bD_{T}$. But due to the randomness
of $b_{t}$, this is not achievable. However, the same as for the
first inequality in (\ref{eq:PS-E-bound-main}), we can bridge this
gap by aiming for a slightly weaker inequality that bounds $\bD_{T}^{1/2}$
instead of $\bD_{T}$. More precisely, we connect the left-hand side
of (\ref{eq:PS-D-bound-main}) to $\bD_{T}^{1/2}$ as follows:
\begin{align}
\sum_{t=1}^{T}\frac{\|d_{t}\|^{2}}{b_{t}} & \geq-b_{0}+\frac{b_{0}^{2}}{b_{0}}+\sum_{t=1}^{T}\frac{a_{T+1}^{1/4}\|d_{t}\|^{2}}{\left(b_{0}^{2}+\sum_{i=1}^{T}\|d_{i}\|^{2}\right)^{1/2}}\ge a_{T+1}^{1/4}\bD_{T}^{1/2}-b_{0}.\label{eq:PS-D-bound-LHS}
\end{align}
By plugging in (\ref{eq:PS-D-bound-LHS}) into (\ref{eq:PS-D-bound-main})
and setting $\lambda$ appropriately, we can finally obtain the following
upper bound:
\begin{align}
\E\left[a_{T+1}^{1/4}\bD_{T}^{1/2}\right] & \leq K_{5}+K_{6}\E\left[\log\left(1+\tH_{T}/a_{0}^{2}\right)\right]+K_{7}\E\left[\log\frac{K_{8}+K_{9}\left(1+\tH_{T}/a_{0}^{2}\right)^{1/3}}{b_{0}}\right].\label{eq:PS-D-bound-main-1}
\end{align}
where $K_{5},K_{6},K_{7},K_{8},K_{9}$ depend only on $\sigma,\widehat{\sigma},\beta,a_{0},b_{0},\eta$
and are independent of $T$.

\textbf{Combining the bounds:} The final part of the analysis is to
combine (\ref{eq:PS-E-bound-main}) and (\ref{eq:PS-D-bound-main-1}).
In contrast to the simpler non-adaptive analysis, these inequalities
bound $a_{T+1}^{1/2}\bE_{T}$ and $a_{T+1}^{1/4}\bD_{T}^{1/2}$ instead
of $\bD_{T}$ and $\bE_{T}$. In order to obtain an upper bound on
$\bH_{T}$ via the inequality $\bH_{T}\le2\bD_{T}+2\bE_{T}$, we need
to connect $a_{T+1}^{1/2}\bE_{T}$ and $a_{T+1}^{1/4}\bD_{T}^{1/2}$
and $\bD_{T}$ and $\bE_{T}$. The bounded variance assumption on
the stochastic gradients gives us a bound on $\E[a_{T+1}^{-3/2}]=\E[1+\tH_{T}/a_{0}^{2}]=O(1+\sigma^{2}T)$
(note that this $-3/2$ is the smallest $c$ to make sure we can upper
bound $\E[a_{t+1}^{c}]$). Combining this result and Holder's inequality
gives us the bound 
\begin{align*}
\E\left[\bD_{T}^{3/7}\right] & \leq\E^{6/7}\left[a_{T+1}^{1/4}\bD_{T}^{1/2}\right]\E^{1/7}\left[a_{T+1}^{-3/2}\right];\\
\E\left[\bE_{T}^{3/7}\right] & \leq\E^{3/7}\left[a_{T+1}^{1/2}\bE_{T}\right]\E^{4/7}\left[a_{T+1}^{-3/8}\right]\leq\E^{3/7}\left[a_{T+1}^{1/2}\bE_{T}\right]\E^{1/7}\left[a_{T+1}^{-3/2}\right];
\end{align*}
where $3/7$ is chosen to ensure that we finally can use the bound
on $\E[a_{T+1}^{-3/2}]$. Thus we obtain an upper bound on $\E[\bH_{T}^{3/7}]$.
Finally, applying the concavity of $x^{3/7}$ to $\E[\bH_{T}^{3/7}]$
gives Theorem \ref{thm:Main-MS-convergence-rate}.

\section{Basic analysis\label{sec:Basic-analysis}}

As discussed in Section \ref{sec:Analysis-outline}, we aim to use
$\bE_{T}$ and $\bD_{T}$ to bound $\bH_{T}$. Here, we apply this
framework to give some basic results which will be used frequently
for the full analysis of every algorithm. We first state the following
decomposition in our analysis framework. The reason we use $\hp\leq1$
here is that we can not always bound $\bH_{T}$ directly because of
the randomness of $a_{t}$ and $b_{t}$ in our algorithms.
\begin{lem}
\label{lem:decomposition}Given $\hp\leq1$, we have
\[
\E\left[\bH_{T}^{\hp}\right]\leq2^{\hp+1}\max\left\{ \E\left[\bE_{T}^{\hp}\right],\E\left[\bD_{T}^{\hp}\right]\right\} \leq4\max\left\{ \E\left[\bE_{T}^{\hp}\right],\E\left[\bD_{T}^{\hp}\right]\right\} .
\]
\end{lem}
\begin{proof}
By the definition of $\bH_{T}$, $\bE_{T}$ and $\bD_{T}$, we have
$\bH_{T}\leq2\bE_{T}+2\bD_{T}$. Hence
\begin{align*}
\bH_{T}^{\hp} & \leq\left(2\bE_{T}+2\bD_{T}\right)^{\hp}\overset{(a)}{\leq}2^{\hp}\bE_{T}^{\hp}+2^{\hp}\bD_{T}^{\hp}\\
\Rightarrow\E\left[\bH_{T}^{\hp}\right] & \leq2^{\hp}\left(\E\left[\bE_{T}^{\hp}\right]+\E\left[\bD_{T}^{\hp}\right]\right)\\
 & \leq2^{\hp+1}\max\left\{ \E\left[\bE_{T}^{\hp}\right],\E\left[\bD_{T}^{\hp}\right]\right\} \\
 & \overset{(b)}{\leq}4\max\left\{ \E\left[\bE_{T}^{\hp}\right],\E\left[\bD_{T}^{\hp}\right]\right\} 
\end{align*}
where $(a)$ and $(b)$ are both by due to $\hp\leq1$.
\end{proof}

\subsection{Variance reduction analysis for $\protect\bE_{T}$}

The same as in all existing momentum-based VR methods, we need to
analyze how the error term $\epsilon_{t}$ changes in the algorithm.
Based on our notations, we give the following two standard lemmas.
\begin{lem}
\label{lem:vr-inequality}$\forall t\geq1$, we have
\begin{align*}
a_{t+1}\|\epsilon_{t}\|^{2} & \leq\|\epsilon_{t}\|^{2}-\|\epsilon_{t+1}\|^{2}+2\|Z_{t+1}\|^{2}\\
 & \qquad+2a_{t+1}^{2}\|\nabla f(x_{t+1},\xi_{t+1})-\nabla F(x_{t+1})\|^{2}+M_{t+1},
\end{align*}
where 
\begin{align*}
Z_{t+1} & \coloneqq\nabla f(x_{t+1},\xi_{t+1})-\nabla f(x_{t},\xi_{t+1})-\nabla F(x_{t+1})+\nabla F(x_{t}),\\
M_{t+1} & \coloneqq2(1-a_{t+1})^{2}\langle\epsilon_{t},Z_{t+1}\rangle+2(1-a_{t+1})a_{t+1}\langle\epsilon_{t},\nabla f(x_{t+1},\xi_{t+1})-\nabla F(x_{t+1})\rangle.
\end{align*}
\end{lem}
\begin{proof}
Starting from the definition of $\epsilon_{t+1}$, we have
\begin{align*}
\|\epsilon_{t+1}\|^{2} & =\|d_{t+1}-\nabla F(x_{t+1})\|^{2}\\
 & =\|\nabla f(x_{t+1},\xi_{t+1})+(1-a_{t+1})(d_{t}-\nabla f(x_{t},\xi_{t+1}))-\nabla F(x_{t+1})\|^{2}\\
 & =\|(1-a_{t+1})\epsilon_{t}+(1-a_{t+1})Z_{t+1}+a_{t+1}(\nabla f(x_{t+1},\xi_{t+1})-\nabla F(x_{t+1}))\|^{2}\\
 & =(1-a_{t+1})^{2}\|\epsilon_{t}\|^{2}\\
 & \quad+\|(1-a_{t+1})Z_{t+1}+a_{t+1}(\nabla f(x_{t+1},\xi_{t+1})-\nabla F(x_{t+1}))\|^{2}+M_{t+1}\\
 & \overset{(a)}{\leq}(1-a_{t+1})^{2}\|\epsilon_{t}\|^{2}\\
 & \quad+2(1-a_{t+1})^{2}\|Z_{t+1}\|^{2}+2a_{t+1}^{2}\|\nabla f(x_{t+1},\xi_{t+1})-\nabla F(x_{t+1})\|^{2}+M_{t+1}\\
 & \overset{(b)}{\leq}(1-a_{t+1})\|\epsilon_{t}\|^{2}\\
 & \quad+2\|Z_{t+1}\|^{2}+2a_{t+1}^{2}\|\nabla f(x_{t+1},\xi_{t+1})-\nabla F(x_{t+1})\|^{2}+M_{t+1}
\end{align*}
where $(a)$ is by $(x+y)^{2}\leq2x^{2}+2y^{2}$, $(b)$ is by $0\leq1-a_{t+1}\leq1$.
Adding $a_{t+1}\|\epsilon_{t}\|^{2}-\|\epsilon_{t+1}\|^{2}$ to both
sides, we get the desired result.
\end{proof}

\begin{lem}
\label{lem:smooth-Z}$\forall t\geq1$, we have
\[
\E\left[\|Z_{t+1}\|^{2}\mid\F_{t}\right]\leq\eta^{2}\beta^{2}\frac{\|d_{t}\|^{2}}{b_{t}^{2}}.
\]
\end{lem}
\begin{proof}
From the definition of $Z_{t+1}$, we have
\begin{align*}
\E\left[\|Z_{t+1}\|^{2}\mid\F_{t}\right] & =\E\left[\|\nabla f(x_{t+1},\xi_{t+1})-\nabla f(x_{t},\xi_{t+1})-\nabla F(x_{t+1})+\nabla F(x_{t})\|^{2}|\F_{t}\right]\\
 & \overset{(a)}{\leq}\E\left[\|\nabla f(x_{t+1},\xi_{t+1})-\nabla f(x_{t},\xi_{t+1})\|^{2}|\F_{t}\right]\\
 & \overset{(b)}{\leq}\beta^{2}\|x_{t+1}-x_{t}\|^{2}\overset{(c)}{=}\eta^{2}\beta^{2}\frac{\|d_{t}\|^{2}}{b_{t}^{2}}
\end{align*}
where $(a)$ is by $\E\left[\|X-\E\left[X\right]\|^{2}\right]\leq\E\left[\|X\|^{2}\right]$,
$(b)$ is by the averaged $\beta$-smooth assumption, $(c)$ is by
the fact $x_{t+1}-x_{t}=-\frac{\eta}{b_{t}}d_{t}$.
\end{proof}

\subsection{On the way to bound $\protect\bD_{T}$}

We choose to bound the terms $\bD_{T}$ instead of starting from $\bH_{T}$
as done in AdaGradNorm or STORM+. The latter also requires the bounded
function value assumption in the analysis. 
\begin{lem}
\label{lem:f-value-analysis}For any of $\algnameold$, $\algnamenew$
or $\algnamena$, we have, for any $\lambda>0$
\begin{align*}
\E\left[a_{T+1}^{q}\bD_{T}^{1-p}\right] & \leq b_{0}^{\frac{1}{p}-1}+\frac{2}{\eta}\left(F(x_{1})-F^{*}\right)\\
 & \quad+\E\left[\sum_{t=1}^{T}\left(\eta\beta_{\max}+\frac{\eta\beta_{\max}}{a_{t+1}^{1/2}\lambda}-b_{t}\right)\frac{\|d_{t}\|^{2}}{b_{t}^{2}}\right]+\frac{\lambda\E\left[\bE_{T,1/2}\right]}{\eta\beta_{\max}}.
\end{align*}
\end{lem}
\begin{proof}
Using smoothness, the update rule $x_{t+1}=x_{t}-\frac{\eta}{b_{t}}d_{t}$
and the definition of $\epsilon_{t}=d_{t}-\nabla F(x_{t})$, we obtain
\begin{align*}
F(x_{t+1}) & \leq F(x_{t})+\langle\nabla F(x_{t}),x_{t+1}-x_{t}\rangle+\frac{\beta}{2}\|x_{t+1}-x_{t}\|^{2}\\
 & =F(x_{t})-\frac{\eta\langle\nabla F(x_{t}),d_{t}\rangle}{b_{t}}+\frac{\eta^{2}\beta}{2b_{t}^{2}}\|d_{t}\|^{2}\\
 & =F(x_{t})-\frac{\eta\|d_{t}\|^{2}}{b_{t}}+\frac{\eta\langle\epsilon_{t},d_{t}\rangle}{b_{t}}+\frac{\eta^{2}\beta}{2b_{t}^{2}}\|d_{t}\|^{2}.
\end{align*}
First we use Cauchy-Schwarz to separate the stochastic gradient and
the stochastic error terms
\begin{align*}
F(x_{t+1}) & \le F(x_{t})-\frac{\eta\|d_{t}\|^{2}}{b_{t}}+\frac{\lambda_{t}\eta\|\epsilon_{t}\|^{2}}{2b_{t}}+\frac{\eta\|d_{t}\|^{2}}{2\lambda_{t}b_{t}}+\frac{\eta^{2}\beta}{2b_{t}^{2}}\|d_{t}\|^{2}.
\end{align*}
Taking
\[
\lambda_{t}=\frac{\lambda a_{t+1}^{1/2}b_{t}}{\eta\beta_{\max}}
\]
for some $\lambda>0$. We have
\begin{align*}
\frac{\eta\|d_{t}\|^{2}}{2b_{t}} & \leq F(x_{t})-F(x_{t+1})+\left(\frac{\eta^{2}\beta}{2b_{t}^{2}}+\frac{\eta}{2\lambda_{t}b_{t}}-\frac{\eta}{2b_{t}}\right)\|d_{t}\|^{2}+\frac{\lambda_{t}\eta\|\epsilon_{t}\|^{2}}{2b_{t}}\\
 & =F(x_{t})-F(x_{t+1})+\left(\frac{\eta^{2}\beta}{2b_{t}^{2}}+\frac{\eta^{2}\beta_{\max}}{2b_{t}^{2}a_{t+1}^{1/2}\lambda}-\frac{\eta}{2b_{t}}\right)\|d_{t}\|^{2}+\frac{\lambda a_{t+1}^{1/2}\|\epsilon_{t}\|^{2}}{2\beta_{\max}}\\
 & =F(x_{t})-F(x_{t+1})+\left(\frac{\eta^{2}\beta}{2}+\frac{\eta^{2}\beta_{\max}}{2a_{t+1}^{1/2}\lambda}-\frac{\eta b_{t}}{2}\right)\frac{\|d_{t}\|^{2}}{b_{t}^{2}}+\frac{\lambda a_{t+1}^{1/2}\|\epsilon_{t}\|^{2}}{2\beta_{\max}}\\
 & \leq F(x_{t})-F(x_{t+1})+\left(\frac{\eta^{2}\beta_{\max}}{2}+\frac{\eta^{2}\beta_{\max}}{2a_{t+1}^{1/2}\lambda}-\frac{\eta b_{t}}{2}\right)\frac{\|d_{t}\|^{2}}{b_{t}^{2}}+\frac{\lambda a_{t+1}^{1/2}\|\epsilon_{t}\|^{2}}{2\beta_{\max}}\\
\Rightarrow\E\left[\sum_{t=1}^{T}\frac{\|d_{t}\|^{2}}{b_{t}}\right] & \leq\frac{2}{\eta}\left(F(x_{1})-F^{*}\right)+\E\left[\sum_{t=1}^{T}\left(\eta\beta_{\max}+\frac{\eta\beta_{\max}}{a_{t+1}^{1/2}\lambda}-b_{t}\right)\frac{\|d_{t}\|^{2}}{b_{t}^{2}}\right]+\frac{\lambda\E\left[\bE_{T,1/2}\right]}{\eta\beta_{\max}}.
\end{align*}
The final step is to relate the L.H.S. to $\bD_{T}$. Recall for $\algnameold$
and $\algnamena$, we have
\[
b_{t}=(b_{0}^{1/p}+\sum_{i=1}^{t}\|d_{i}\|^{2})^{p}/a_{t+1}^{q}.
\]
Hence
\begin{align*}
\sum_{t=1}^{T}\frac{\|d_{t}\|^{2}}{b_{t}} & =\sum_{t=1}^{T}\frac{a_{t+1}^{q}\|d_{t}\|^{2}}{(b_{0}^{1/p}+\sum_{i=1}^{t}\|d_{i}\|^{2})^{p}}\geq\sum_{t=1}^{T}\frac{a_{T+1}^{q}\|d_{t}\|^{2}}{(b_{0}^{1/p}+\sum_{i=1}^{T}\|d_{i}\|^{2})^{p}}\\
 & =a_{T+1}^{q}(b_{0}^{1/p}+\sum_{i=1}^{T}\|d_{i}\|^{2})^{1-p}-a_{T+1}^{q}\frac{b_{0}^{1/p}}{(b_{0}^{1/p}+\sum_{i=1}^{T}\|d_{i}\|^{2})^{p}}\\
 & \geq a_{T+1}^{q}(b_{0}^{1/p}+\sum_{i=1}^{T}\|d_{i}\|^{2})^{1-p}-b_{0}^{1/p-1}\\
 & \geq a_{T+1}^{q}\bD_{T}^{1-p}-b_{0}^{1/p-1}.
\end{align*}
The same result holds for $\algnamenew$ by a similar proof. By using
this bound, the proof is finished.
\end{proof}

To finish section, we prove a technical result, Lemma \ref{lem:residual-bound},
which will be very useful in the proof of every algorithm. The motivation
to prove it is because we want to bound the term inside the expectation
part in Lemma \ref{lem:f-value-analysis}.
\begin{lem}
\label{lem:residual-bound} Given $A,B\geq0$. We have
\begin{itemize}
\item for $\algnameold$ and $\algnamena$
\[
\sum_{t=1}^{T}\left(A+\frac{B}{a_{t+1}^{1/2}}-b_{t}\right)\frac{\|d_{t}\|^{2}}{b_{t}^{2}}\leq\frac{\left(A+B\right)^{\frac{1}{p}-1}}{1-p}\log\frac{A+a_{T+1}^{-1/2}B}{b_{0}}.
\]
\item for $\algnamenew$
\[
\sum_{t=1}^{T}\left(A+\frac{B}{a_{t}^{1/2}}-b_{t}\right)\frac{\|d_{t}\|^{2}}{b_{t}^{2}}\leq\frac{\left(A+B\right)^{\frac{1}{p}-1}}{1-p}\log\frac{A+a_{T+1}^{-1/2}B}{b_{0}}.
\]
\end{itemize}
\end{lem}
\begin{proof}
In $\algnameold$ and $\algnamena$, we have 
\[
b_{t}=(b_{0}^{1/p}+\sum_{i=1}^{t}\|d_{i}\|^{2})^{p}/a_{t+1}^{q}
\]
where $p+2q=1$. Define the set 
\[
S=\left\{ t\in\left[T\right]:b_{t}\leq A+\frac{B}{a_{t+1}^{1/2}}\right\} 
\]
and let $s=\max S$. We know
\begin{align*}
\sum_{t=1}^{T}\left(A+\frac{B}{a_{t+1}^{1/2}}-b_{t}\right)\frac{\|d_{t}\|^{2}}{b_{t}^{2}} & \leq\sum_{t\in S}\left(A+\frac{B}{a_{t+1}^{1/2}}-b_{t}\right)\frac{\|d_{t}\|^{2}}{b_{t}^{2}}\\
 & =\sum_{t\in S}\left(A+\frac{B}{a_{t+1}^{1/2}}-b_{t}\right)\frac{a_{t+1}^{q/p}b_{t}^{1/p}-a_{t}^{q/p}b_{t-1}^{1/p}}{b_{t}^{2}}\\
 & \overset{(a)}{\leq}\sum_{t\in S}\left(A+\frac{B}{a_{t+1}^{1/2}}-b_{t}\right)a_{t+1}^{q/p}\frac{b_{t}^{1/p}-b_{t-1}^{1/p}}{b_{t}^{2}}\\
 & =\sum_{t\in S}\left(a_{t+1}^{1/2}A+B-a_{t+1}^{1/2}b_{t}\right)a_{t+1}^{\frac{q}{p}-\frac{1}{2}}b_{t}^{\frac{1}{p}-2}\frac{b_{t}^{1/p}-b_{t-1}^{1/p}}{b_{t}^{1/p}}
\end{align*}
where $(a)$ is by $a_{t}\geq a_{t+1}$. Note that
\begin{align*}
\left(a_{t+1}^{1/2}A+B-a_{t+1}^{1/2}b_{t}\right)a_{t+1}^{\frac{q}{p}-\frac{1}{2}}b_{t}^{\frac{1}{p}-2} & \overset{(b)}{\leq}\left(A+B-a_{t+1}^{1/2}b_{t}\right)a_{t+1}^{\frac{q}{p}-\frac{1}{2}}b_{t}^{\frac{1}{p}-2}\\
 & \overset{(c)}{=}\left(A+B-a_{t+1}^{1/2}b_{t}\right)a_{t+1}^{\frac{1}{2p}-1}b_{t}^{\frac{1}{p}-2}\\
 & =\left(A+B-a_{t+1}^{1/2}b_{t}\right)\left(a_{t+1}^{1/2}b_{t}\right)^{\frac{1}{p}-2}\\
 & \overset{(d)}{\leq}\left(\frac{A+B}{\frac{1}{p}-1}\right)^{\frac{1}{p}-1}\left(\frac{1}{p}-2\right)^{\frac{1}{p}-2}\\
 & \leq\frac{p}{1-p}\left(A+B\right)^{\frac{1}{p}-1}
\end{align*}
where $(b)$ holds by $a_{t+1}\leq1$, $(c)$ is due to $\frac{q}{p}-\frac{1}{2}=\frac{2q-p}{2p}=\frac{1-2p}{2p}=\frac{1}{2p}-1$
by $p+2q=1$ and $(d)$ is by applying Lemma \ref{lem:inequality-poly-g}.
Thus we know
\begin{align*}
\sum_{t=1}^{T}\left(A+\frac{B}{a_{t+1}^{1/2}}-b_{t}\right)\frac{\|d_{t}\|^{2}}{b_{t}^{2}} & \leq\frac{p}{1-p}\left(A+B\right)^{\frac{1}{p}-1}\sum_{t\in S}\frac{b_{t}^{1/p}-b_{t-1}^{1/p}}{b_{t}^{1/p}}\\
 & \overset{(e)}{\leq}\frac{\left(A+B\right)^{\frac{1}{p}-1}}{1-p}\sum_{t\in S}\log\frac{b_{t}}{b_{t-1}}\\
 & \overset{(f)}{\leq}\frac{\left(A+B\right)^{\frac{1}{p}-1}}{1-p}\sum_{t=1}^{s}\log\frac{b_{t}}{b_{t-1}}\\
 & =\frac{\left(A+B\right)^{\frac{1}{p}-1}}{1-p}\log\frac{b_{s}}{b_{0}}\\
 & \overset{(g)}{\leq}\frac{\left(A+B\right)^{\frac{1}{p}-1}}{1-p}\log\frac{A+a_{T+1}^{-1/2}B}{b_{0}}
\end{align*}
where $(e)$ is by taking $x=\left(b_{t}/b_{t-1}\right)^{1/p}$ in
$1-\frac{1}{x}\leq\log x$, $(f)$ is because $b_{t}$ is increasing.
The reason $(g)$ is true is that $b_{s}\leq A+a_{s+1}^{-1/2}B\leq A+a_{T+1}^{-1/2}B$
where the first inequality is due to $s\in S$ and the second one
holds by that $a_{t}^{-1/2}$ is increasing. Now we finish the proof
for $\algnameold$ and $\algnamena$. The proof for $\algnamenew$
is essentially the same hence omitted here.
\end{proof}

\section{Analysis of $\protect\algnamenew$ for general $p$\label{sec:Algorithm-MS}}

In this section, we give a general analysis for our Algorithm $\algnamenew$.
We will see that $p=\frac{1}{2}$ is a special corner case. First
we recall the choices of $a_{t}$ and $b_{t}$
\begin{align*}
a_{t+1} & =(1+\sum_{i=1}^{t}\left\Vert \nabla f(x_{i},\xi_{i})-\nabla f(x_{i},\xi_{i+1})\right\Vert ^{2}/a_{0}^{2})^{-2/3},\\
b_{t} & =(b_{0}^{1/p}+\sum_{i=1}^{t}\|d_{i}\|^{2})^{p}/a_{t}^{q}
\end{align*}
where $p,q$ satisfy $p+2q=1,p\in\left[\frac{3-\sqrt{7}}{2},\frac{1}{2}\right].$
$a_{0}>0$ and $b_{0}>0$ are absolute constants. Naturally, we have
$a_{1}=1$. We will finally prove the following theorem.
\begin{thm}
\label{thm:MS-rate}Under the assumptions 1-3 and 5, by defining $\hp=\frac{3(1-p)}{4-p}\in\left[\frac{3}{7},\sqrt{7}-2\right]$,
we have
\begin{align*}
 & \E\left[\bH_{T}^{\hp}\right]\\
\leq & 4\begin{cases}
\left(\frac{2K_{1}}{K_{4}}\right)^{\frac{\hp}{1-2p}}+\left(\left(\frac{2K_{2}}{K_{4}}\right)^{\frac{\hp}{1-2p}}+\left(2K_{4}\right)^{\frac{\hp}{2p}}\right)\left(1+\frac{2\sigma^{2}T}{a_{0}^{2}}\right)^{\frac{\hp}{3}} & p\neq\frac{1}{2}\\
\left(2K_{1}+2\left(K_{2}+\frac{K_{4}}{3}\right)\log\left(1+\frac{2\sigma^{2}T}{a_{0}^{2}}\right)+\frac{2K_{4}}{\hp}\log\frac{4K_{4}}{b_{0}^{2\hp}}\right)^{\hp}\left(1+\frac{2\sigma^{2}T}{a_{0}^{2}}\right)^{\frac{\hp}{3}}+b_{0}^{2\hp} & p=\frac{1}{2}
\end{cases}\\
 & +4\left(K_{5}+\left(K_{6}+\frac{K_{7}}{3}\right)\log\left(1+\frac{2\sigma^{2}T}{a_{0}^{2}}\right)+K_{7}\log\frac{K_{8}+K_{9}}{b_{0}}\right)^{\frac{\hp}{1-p}}\left(1+\frac{2\sigma^{2}T}{a_{0}^{2}}\right)^{\frac{\hp}{3}},
\end{align*}
where $K_{i},i\in\left[9\right]$ are some constants only depending
on $a_{0},b_{0},\eta,\sigma,\hS,\beta,p,q,F(x_{1})-F^{*}$. To simplify
our final bound, we only indicate the dependency on $\beta$ and $F(x_{1})-F^{*}$.
\[
\E\left[\bH_{T}^{\hp}\right]=O\left(\left((F(x_{1})-F^{*})^{\frac{\hp}{1-p}}+\beta^{\frac{\hp}{p}}\log^{\frac{\hp}{1-p}}\beta+\beta^{\frac{\hp}{p}}\log^{\frac{\hp}{1-p}}\left(1+\sigma^{2}T\right)\right)(1+\sigma^{2}T)^{\frac{\hp}{3}}\right).
\]
\end{thm}
\begin{rem}
For all $i\in\left[9\right]$, the constant $K_{i}$ will be defined
in the proof that follows.
\end{rem}
By using the concavity of $x^{\hp}$, we state the following convergence
theorem without proof.
\begin{thm}
Under the assumptions 1-3 and 5, by defining $\hp=\frac{3(1-p)}{4-p}\in\left[\frac{3}{7},\sqrt{7}-2\right]$,
we have
\begin{align*}
 & \E\left[\|\nabla F(\xout)\|^{2\hp}\right]\\
 & \quad=O\left((F(x_{1})-F^{*})^{\frac{\hp}{1-p}}+\beta^{\frac{\hp}{p}}\log^{\frac{\hp}{1-p}}\beta+\beta^{\frac{\hp}{p}}\log^{\frac{\hp}{1-p}}\left(1+\sigma^{2}T\right)\right)\left(\frac{1}{T^{\hp}}+\frac{\sigma^{2\hp/3}}{T^{2\hp/3}}\right).
\end{align*}
\end{thm}
Here, we give a more explicit convergence dependency for $p=\frac{1}{2}$
used in Theorem \ref{thm:Main-MS-convergence-rate}
\begin{thm}
Under the assumptions 1-3 and 5, when $p=\frac{1}{2}$, by setting
$\lambda=\min\left\{ 1,(a_{0}/\hS)^{7/3}\right\} $(which is used
in $K_{5}$to $K_{9}$) we get the best dependency on $\hS$. For
simplicity, under the setting $a_{0}=b_{0}=\eta=1$, we have
\begin{align*}
\E\left[\|\nabla F(\xout)\|^{6/7}\right] & =O\left(\left(Q_{1}+Q_{2}\log^{6/7}\left(1+\sigma^{2}T\right)\right)\left(\frac{1}{T^{3/7}}+\frac{\sigma^{2/7}}{T^{2/7}}\right)\right)
\end{align*}
where $Q_{1}=O\big(\big(F(x_{1})-F^{*}\big)^{6/7}+\sigma^{12/7}+\big(\rd\sigma\big)^{6/7}+\rd^{18/7}+\big(1+\rd^{18/7}\big)\beta^{6/7}\log^{6/7}\big(\beta+\rd^{3}\beta\big)\big)$
and $Q_{2}=O\big(\big(1+\rd^{18/7}\big)\beta^{6/7}\big)$.
\end{thm}
To start with, we first state the following useful bound for $a_{t}$:
\begin{lem}
\label{lem:MS-a_t-bound}$\forall\alpha\in(0,3/2]$ and $\forall t\geq1$,
there is
\begin{align*}
\left(\frac{a_{t}}{a_{t+1}}\right)^{\alpha} & \leq1+\left(\frac{4\hS^{2}}{a_{0}^{2}}\right)^{\frac{2\alpha}{3}}a_{t}^{\alpha}.
\end{align*}
Especially, taking $\alpha\in\left\{ 1/2,1,3/2\right\} $, we have
\begin{align*}
\left(\frac{a_{t}}{a_{t+1}}\right)^{1/2} & \le1+\frac{4^{1/3}\hS^{2/3}}{a_{0}^{2/3}}a_{t}^{1/2};\\
\frac{a_{t}}{a_{t+1}} & \leq1+\frac{4^{2/3}\hS^{4/3}}{a_{0}^{2/3}}a_{t};\\
\left(\frac{a_{t}}{a_{t+1}}\right)^{3/2} & \le1+\frac{4\hS^{2}}{a_{0}^{2}}a_{t}^{3/2}.
\end{align*}
\end{lem}
\begin{proof}
Note that
\begin{align*}
\left(\frac{a_{t}}{a_{t+1}}\right)^{\alpha} & =a_{t}^{\alpha}\left(\frac{1}{a_{t}^{3/2}}+\frac{\|\nabla f(x_{t},\xi_{t})-\nabla f(x_{t},\xi_{t+1})\|^{2}}{a_{0}^{2}}\right)^{2\alpha/3}\\
 & =\left(1+\frac{\|\nabla f(x_{t},\xi_{t})-\nabla f(x_{t},\xi_{t+1})\|^{2}}{a_{0}^{2}}a_{t}^{3/2}\right)^{2\alpha/3}\\
 & \leq\left(1+\frac{4\hS^{2}}{a_{0}^{2}}a_{t}^{3/2}\right)^{2\alpha/3}\leq1+\left(\frac{4\hS^{2}}{a_{0}^{2}}\right)^{2\alpha/3}a_{t}^{\alpha}
\end{align*}
where the last inequality is because $2\alpha/3\leq1$.
\end{proof}

Lemma \ref{lem:MS-a_t-bound} allows us to obtain some other properties
of $a_{t}$.
\begin{lem}
\label{lem:MS-a_t-prop}For $t\ge1$
\begin{align*}
\frac{\left((1-a_{t+1})^{2}-(1-a_{t})^{2}\right)^{2}}{a_{t+1}} & \le\frac{4^{2/3}\hS^{4/3}}{a_{0}^{4/3}}\\
\frac{\left((1-a_{t+1})a_{t+1}-(1-a_{t})a_{t}\right)^{2}}{a_{t+1}} & \le\frac{4^{2/3}\hS^{4/3}}{a_{0}^{4/3}}a_{t}^{2}.
\end{align*}
\end{lem}
\begin{proof}
Let $a_{t+1}=x,a_{t}=y$ and note that $x\leq y\leq1$. For the first
inequality,
\begin{align*}
\frac{\left((1-a_{t+1})^{2}-(1-a_{t})^{2}\right)^{2}}{a_{t+1}} & \le\frac{(1-x)^{2}-(1-y)^{2}}{x}\\
 & =\frac{(y-x)(2-x-y)}{x}\leq(\frac{y}{x}-1)(2-y)\\
 & \le\frac{4^{2/3}\hS^{4/3}}{a_{0}^{2/3}}a_{t}\times\left(2-a_{t}\right)\qquad(\text{Lemma \ref{lem:MS-a_t-bound}})\\
 & \le\frac{4^{2/3}\hS^{4/3}}{a_{0}^{4/3}}.
\end{align*}
For the second inequality, we have
\begin{align*}
\frac{\left((1-a_{t+1})a_{t+1}-(1-a_{t})a_{t}\right)^{2}}{a_{t+1}} & =\frac{\left((1-x)x-(1-y)y\right)^{2}}{x}=\frac{(y-x)^{2}(1-x-y)^{2}}{x}\\
 & \le\frac{(y-x)^{2}}{x}\leq\left(\frac{y}{x}-1\right)y\\
 & \leq\frac{4^{2/3}\hS^{4/3}}{a_{0}^{2/3}}a_{t}\times a_{t}\qquad(\text{Lemma \ref{lem:MS-a_t-bound}})\\
 & =\frac{4^{2/3}\hS^{4/3}}{a_{0}^{4/3}}a_{t}^{2}.
\end{align*}
\end{proof}

\subsection{Analysis of $\protect\bE_{T}$}

Following a similar approach, we first define a random time $\tau$
satisfying
\[
\tau=\max\left\{ \left[T\right],a_{t}\geq K_{-1}\right\} ,
\]
where
\[
K_{-1}\coloneqq\min\left\{ 1,a_{0}^{4}/(144\hS^{4})\right\} .
\]
One thing we need to emphasize here is that, in our current choice,
$a_{t}\in\F_{t}$, which implies $\left\{ \tau+1=t\right\} =\left\{ \tau=t-1\right\} =\left\{ a_{t-1}\geq K_{-1},a_{t}<K_{-1}\right\} \in\F_{t}$.
This means $\tau+1$ is\textbf{ }a stopping time instead of $\tau$
itself. We now prove a useful proposition for $\tau$:
\begin{lem}
\label{lem:stopping-time-bound-g-d} We have
\begin{align*}
a_{t+1} & \geq K_{0},\forall t\leq\tau,\\
a_{t+1}^{-1}-a_{t}^{-1} & \leq2/9,\forall t\geq\tau+1.
\end{align*}
where
\begin{align*}
K_{0} & \coloneqq(K_{-1}^{-3/2}+4\hS^{2}/a_{0}^{2})^{-2/3}=(\max\left\{ 1,1728\hS^{6}/a_{0}^{6}\right\} +4\hS^{2}/a_{0}^{2})^{-2/3}.
\end{align*}
\end{lem}
\begin{proof}
First, by the definition of $\tau$, we know $a_{t}\geq K_{-1}\geq K_{0}$,$\forall t\le\tau$.
For time $\tau$, we have
\begin{align*}
a_{\tau+1}^{-3/2}-a_{\tau}^{-3/2} & =\|\nabla f(x_{\tau},\xi_{\tau})-\nabla f(x_{\tau},\xi_{\tau+1})\|^{2}/a_{0}^{2}\leq4\hS^{2}/a_{0}^{2}\\
\Rightarrow a_{\tau+1}^{-1} & \leq(a_{\tau}^{-3/2}+4\hS^{2}/a_{0}^{2})^{2/3}\leq(K_{-1}^{-3/2}+4\hS^{2}/a_{0}^{2})^{2/3}=K_{0}^{-1},
\end{align*}
which implies $a_{\tau+1}\geq K_{0}.$

For the second proposition, let $h(y)=y^{2/3}$. Due to the concavity
of $h$, we know $h(y_{1})-h(y_{2})\leq h'(y_{2})(y_{1}-y_{2})=\frac{2(y_{1}-y_{2})}{3y_{2}^{1/3}}$.
Now we have
\begin{align*}
a_{t+1}^{-1}-a_{t}^{-1} & =(a_{t}^{-3/2}+\|\nabla f(x_{t},\xi_{t})-\nabla f(x_{t},\xi_{t+1})\|^{2}/a_{0}^{2})^{2/3}-(a_{t}^{-3/2})^{2/3}\\
 & \leq\frac{2a_{t}^{1/2}\|\nabla f(x_{t},\xi_{t})-\nabla f(x_{t},\xi_{t+1})\|^{2}}{3a_{0}^{2}}\leq\frac{8a_{t}^{1/2}\hS^{2}}{3a_{0}^{2}}\leq\frac{2}{9}
\end{align*}
where the last step is by $a_{t}\leq a_{\tau+1}<K_{-1}\leq a_{0}^{4}/(144\hS^{4})$.
\end{proof}

\subsubsection{Bound on $\protect\E\left[\protect\bE_{\tau,3/2-2\ell}\right]$ for
$\ell\in\left[\frac{1}{4},\frac{1}{2}\right]$}

Unlike STORM+ in which they bound $\E\left[\bE_{\tau}\right]$, we
choose to bound $\E\left[\bE_{\tau,3/2-2\ell}\right]$. We first prove
the following bound on $\E\left[\bE_{\tau,3/2-2\ell}\right]$:
\begin{lem}
\label{lem:MS-E-tau}For any $\ell\in\left[\frac{1}{4},\frac{1}{2}\right]$,
we have
\begin{align*}
\E\left[\bE_{\tau,3/2-2\ell}\right] & \leq\frac{2\sigma^{2}+16\left(1+\frac{6\hS^{4/3}}{a_{0}^{4/3}}\right)\left(3a_{0}^{2}+5\hS^{2}\right)}{K_{0}^{2\ell-1/2}}+\frac{4\left(1+\frac{6\hS^{4/3}}{a_{0}^{4/3}}\right)\eta^{2}\beta^{2}}{K_{0}^{2\ell-1/2}}\E\left[\sum_{t=1}^{T}\frac{\|d_{t}\|^{2}}{b_{t}^{2}}\right].
\end{align*}
\end{lem}
\begin{proof}
We start from Lemma \ref{lem:vr-inequality}
\begin{align*}
a_{t+1}\|\epsilon_{t}\|^{2} & \leq\|\epsilon_{t}\|^{2}-\|\epsilon_{t+1}\|^{2}+2\|Z_{t+1}\|^{2}+2a_{t+1}^{2}\|\nabla f(x_{t+1},\xi_{t+1})-\nabla F(x_{t+1})\|^{2}+M_{t+1}.
\end{align*}
Summing up from $1$ to $\tau$ and taking expectations on both sides,
we will have
\begin{align}
 & \E\left[\bE_{\tau,1}\right]\nonumber \\
\le & \E\left[\sum_{t=1}^{\tau}\|\epsilon_{t}\|^{2}-\|\epsilon_{t+1}\|^{2}+2\|Z_{t+1}\|^{2}+2a_{t+1}^{2}\|\nabla f(x_{t+1},\xi_{t+1})-\nabla F(x_{t+1})\|^{2}+M_{t+1}\right]\nonumber \\
\le & \sigma^{2}+\E\left[\sum_{t=1}^{\tau}2\|Z_{t+1}\|^{2}+2a_{t+1}^{2}\|\nabla f(x_{t+1},\xi_{t+1})-\nabla F(x_{t+1})\|^{2}+M_{t+1}\right]\nonumber \\
\le & \sigma^{2}+\E\left[\sum_{t=1}^{T}2\|Z_{t+1}\|^{2}+2a_{t+1}^{2}\|\nabla f(x_{t+1},\xi_{t+1})-\nabla F(x_{t+1})\|^{2}+\sum_{t=1}^{\tau}M_{t+1}\right].\label{eq:MS-E-bound}
\end{align}

First we bound $\E\left[\sum_{t=1}^{\tau}M_{t+1}\right]$. From the
definition of $M_{t+1}$, we have
\begin{align*}
\E\left[M_{t+1}\right] & =\E\left[2(1-a_{t+1})^{2}\langle\epsilon_{t},Z_{t+1}\rangle+2(1-a_{t+1})a_{t+1}\langle\epsilon_{t},\nabla f(x_{t+1},\xi_{t+1})-\nabla F(x_{t+1})\rangle\right].
\end{align*}
Now for $t\geq1$, we define
\[
N_{t+1}\coloneqq2(1-a_{t})^{2}\langle\epsilon_{t},Z_{t+1}\rangle+2(1-a_{t})a_{t}\langle\epsilon_{t},\nabla f(x_{t+1},\xi_{t+1})-\nabla F(x_{t+1})\rangle\in\F_{t+1}
\]
with $N_{1}\coloneqq0.$ A key observation is that
\[
\E\left[\sum_{t=1}^{\tau}N_{t+1}\right]=0.
\]
This is because $\mathcal{N}_{t}\coloneqq\sum_{i=1}^{t}N_{t}$ is
a martingale and $\tau+1$ is a bounded stopping time. Then by optional
sampling theorem, we have
\begin{align*}
\E\left[\sum_{t=1}^{\tau}N_{t+1}\right] & =\E\left[\sum_{t=1}^{\tau+1}N_{t}\right]=\E\left[\mathcal{N}_{\tau+1}\right]=0.
\end{align*}
By subtracting $\E\left[\sum_{t=1}^{\tau}M_{t+1}\right]$ by $\E\left[\sum_{t=1}^{\tau}N_{t+1}\right]$,
we obtain
\begin{align}
\E\left[\sum_{t=1}^{\tau}M_{t+1}\right] & =\E\left[\sum_{t=1}^{\tau}2\left((1-a_{t+1})^{2}-(1-a_{t})^{2}\right)\langle\epsilon_{t},Z_{t+1}\rangle\right.\nonumber \\
 & \quad+\left.2\left((1-a_{t+1})a_{t+1}-(1-a_{t})a_{t}\right)\langle\epsilon_{t},\nabla f(x_{t+1},\xi_{t+1})-\nabla F(x_{t+1})\rangle\right].\label{eq:MS-M-bound}
\end{align}
Using Cauchy-Schwarz inequality for each term, we have
\begin{align*}
 & 2\left((1-a_{t+1})^{2}-(1-a_{t})^{2}\right)\langle\epsilon_{t},Z_{t+1}\rangle\\
\le & 2\left|(1-a_{t+1})^{2}-(1-a_{t})^{2}\right|\|\epsilon_{t}\|\|Z_{t+1}\|\\
\le & \frac{a_{t+1}}{4}\|\epsilon_{t}\|^{2}+\frac{4\left((1-a_{t+1})^{2}-(1-a_{t})^{2}\right)^{2}}{a_{t+1}}\|Z_{t+1}\|^{2},
\end{align*}
\begin{align*}
 & 2\left((1-a_{t+1})a_{t+1}-(1-a_{t})a_{t}\right)\langle\epsilon_{t},\nabla f(x_{t+1},\xi_{t+1})-\nabla F(x_{t+1})\rangle\\
\le & 2\left|(1-a_{t+1})a_{t+1}-(1-a_{t})a_{t}\right|\|\epsilon_{t}\|\|\nabla f(x_{t+1},\xi_{t+1})-\nabla F(x_{t+1})\|\\
\le & \frac{a_{t+1}}{4}\|\epsilon_{t}\|^{2}+\frac{4\left((1-a_{t+1})a_{t+1}-(1-a_{t})a_{t}\right)^{2}}{a_{t+1}}\|\nabla f(x_{t+1},\xi_{t+1})-\nabla F(x_{t+1})\|^{2}.
\end{align*}
Plugging the above bounds into (\ref{eq:MS-M-bound}), we obtain
\begin{align}
 & \E\left[\sum_{t=1}^{\tau}M_{t+1}\right]\nonumber \\
\le & \E\left[\sum_{t=1}^{\tau}\frac{a_{t+1}}{2}\|\epsilon_{t}\|^{2}+\underbrace{\frac{\left((1-a_{t+1})^{2}-(1-a_{t})^{2}\right)^{2}}{a_{t+1}}}_{(i)}4\|Z_{t+1}\|^{2}\right.\nonumber \\
 & \quad\left.+\sum_{t=1}^{\tau}\underbrace{\frac{\left((1-a_{t+1})a_{t+1}-(1-a_{t})a_{t}\right)^{2}}{a_{t+1}}}_{(ii)}4\|\nabla f(x_{t+1},\xi_{t+1})-\nabla F(x_{t+1})\|^{2}\right].\label{eq:MS-M-bound-2}
\end{align}
Plugging the bounds for $(i)$ and $(ii)$ from Lemma \ref{lem:MS-a_t-prop}
into (\ref{eq:MS-M-bound-2}), the following bound on $\E\left[\sum_{t=1}^{\tau}M_{t+1}\right]$
comes up
\begin{align*}
 & \E\left[\sum_{t=1}^{\tau}M_{t+1}\right]\\
\leq & \E\left[\sum_{t=1}^{\tau}\frac{a_{t+1}}{2}\|\epsilon_{t}\|^{2}+\frac{4^{5/3}\hS^{4/3}}{a_{0}^{4/3}}\|Z_{t+1}\|^{2}+\frac{4^{5/3}\hS^{4/3}}{a_{0}^{4/3}}a_{t}^{2}\|\nabla f(x_{t+1},\xi_{t+1})-\nabla F(x_{t+1})\|^{2}\right]\\
\leq & \E\left[\frac{1}{2}\bE_{\tau,1}+\frac{12\hS^{4/3}}{a_{0}^{4/3}}\|Z_{t+1}\|^{2}+\frac{12\hS^{4/3}}{a_{0}^{4/3}}a_{t}^{2}\|\nabla f(x_{t+1},\xi_{t+1})-\nabla F(x_{t+1})\|^{2}\right].
\end{align*}
Then from (\ref{eq:MS-E-bound}), we have
\begin{align*}
\E\left[\bE_{\tau,1}\right] & \leq\sigma^{2}+\E\left[\frac{1}{2}\bE_{\tau,1}\right]+\E\left[\sum_{t=1}^{T}\left(2+\frac{12\hS^{4/3}}{a_{0}^{4/3}}\right)\|Z_{t+1}\|^{2}\right]\\
 & \quad+\E\left[\sum_{t=1}^{T}\left(2a_{t+1}^{2}+\frac{12\hS^{4/3}}{a_{0}^{4/3}}a_{t}^{2}\right)\|\nabla f(x_{t+1},\xi_{t+1})-\nabla F(x_{t+1})\|^{2}\right],
\end{align*}
which will give us
\begin{align}
\E\left[\bE_{\tau,1}\right] & \le2\sigma^{2}+4\left(1+\frac{6\hS^{4/3}}{a_{0}^{4/3}}\right)\E\left[\sum_{t=1}^{T}\underbrace{\|Z_{t+1}\|^{2}}_{(iii)}\right]\nonumber \\
 & \quad+\E\left[\sum_{t=1}^{T}\underbrace{4\left(a_{t+1}^{2}+\frac{6\hS^{4/3}}{a_{0}^{4/3}}a_{t}^{2}\right)\|\nabla f(x_{t+1},\xi_{t+1})-\nabla F(x_{t+1})\|^{2}}_{(iv)}\right].\label{eq:MS-E-bound-2}
\end{align}

For term $(iii)$, Lemma \ref{lem:smooth-Z} tells us
\begin{align}
\E\left[\|Z_{t+1}\|^{2}\mid\F_{t}\right]\leq & \eta^{2}\beta^{2}\frac{\|d_{t}\|^{2}}{b_{t}^{2}}.\label{eq:MS-Z-term}
\end{align}

For term $(iv)$, we know
\begin{align}
\E\left[(iv)\right] & =\E\left[\sum_{t=1}^{T}4\left(a_{t+1}^{2}+\frac{6\hS^{4/3}}{a_{0}^{4/3}}a_{t}^{2}\right)\|\nabla f(x_{t+1},\xi_{t+1})-\E\left[\nabla f(x_{t+1},\xi_{t+2})\vert\F_{t+1}\right]\|^{2}\right]\nonumber \\
 & \leq\E\left[\sum_{t=1}^{T}4\left(a_{t+1}^{2}+\frac{6\hS^{4/3}}{a_{0}^{4/3}}a_{t}^{2}\right)\E\left[\|\nabla f(x_{t+1},\xi_{t+1})-\nabla f(x_{t+1},\xi_{t+2})\|^{2}\vert\F_{t+1}\right]\right]\nonumber \\
 & =\E\left[\sum_{t=1}^{T}4\left(a_{t+1}^{2}+\frac{6\hS^{4/3}}{a_{0}^{4/3}}a_{t}^{2}\right)\|\nabla f(x_{t+1},\xi_{t+1})-\nabla f(x_{t+1},\xi_{t+2})\|^{2}\right].\label{eq:MS-condition-expectations}
\end{align}
Note that
\[
a_{t}^{2}=\left(1+\sum_{i=1}^{t-1}\|\nabla f(x_{i},\xi_{i})-\nabla f(x_{i},\xi_{i+1})\|^{2}/a_{0}^{2}\right)^{4/3},
\]
then we have
\begin{align}
 & \sum_{t=1}^{T}4\left(a_{t+1}^{2}+\frac{6\hS^{4/3}}{a_{0}^{4/3}}a_{t}^{2}\right)\|\nabla f(x_{t+1},\xi_{t+1})-\nabla f(x_{t+1},\xi_{t+2})\|^{2}\nonumber \\
= & 4a_{0}^{2}\sum_{t=1}^{T}\frac{\|\nabla f(x_{t+1},\xi_{t+1})-\nabla f(x_{t+1},\xi_{t+2})\|^{2}/a_{0}^{2}}{\left(1+\sum_{i=1}^{t}\|\nabla f(x_{i},\xi_{i})-\nabla f(x_{i},\xi_{i+1})\|^{2}/a_{0}^{2}\right)^{4/3}}\nonumber \\
 & +24\hS^{4/3}a_{0}^{2/3}\sum_{t=1}^{T}\frac{\|\nabla f(x_{t+1},\xi_{t+1})-\nabla f(x_{t+1},\xi_{t+2})\|^{2}/a_{0}^{2}}{\left(1+\sum_{i=1}^{t-1}\|\nabla f(x_{i},\xi_{i})-\nabla f(x_{i},\xi_{i+1})\|^{2}/a_{0}^{2}\right)^{4/3}}\nonumber \\
\leq & 4a_{0}^{2}\left(12+\frac{8\hS^{2}}{a_{0}^{2}}\right)+24\hS^{4/3}a_{0}^{2/3}\left(12+\frac{20\hS^{2}}{a_{0}^{2}}\right)\nonumber \\
= & 16\left(3a_{0}^{2}+2\hS^{2}\right)+96\frac{\hS^{4/3}}{a_{0}^{4/3}}\left(3a_{0}^{2}+5\hS^{2}\right)\nonumber \\
\leq & 16\left(1+\frac{6\hS^{4/3}}{a_{0}^{4/3}}\right)\left(3a_{0}^{2}+5\hS^{2}\right),\label{eq:MS-gradient-term}
\end{align}
where, for the first inequality, we use Lemma \ref{lem:inequality-4}
and Lemma \ref{lem:inequality-5}. Plugging (\ref{eq:MS-Z-term})
and (\ref{eq:MS-gradient-term}) into (\ref{eq:MS-E-bound-2}), we
obtain
\[
\E\left[\bE_{\tau,1}\right]\leq2\sigma^{2}+16\left(1+\frac{6\hS^{4/3}}{a_{0}^{4/3}}\right)\left(3a_{0}^{2}+5\hS^{2}\right)+4\left(1+\frac{6\hS^{4/3}}{a_{0}^{4/3}}\right)\eta^{2}\beta^{2}\E\left[\sum_{t=1}^{T}\frac{\|d_{t}\|^{2}}{b_{t}^{2}}\right].
\]

Note that by Lemma \ref{lem:stopping-time-bound-g-d}, we have for
$t\leq\tau$,$a_{t+1}\geq K_{0}.$ By using this property and noticing
$2\ell-1/2\geq0$ , we can obtain
\begin{align*}
 & \E\left[K_{0}^{2\ell-1/2}\bE_{\tau,3/2-2\ell}\right]\\
= & \E\left[K_{0}^{2\ell-1/2}\sum_{t=1}^{\tau}a_{t+1}^{3/2-2\ell}\|\epsilon_{t}\|^{2}\right]\leq\E\left[\sum_{t=1}^{\tau}a_{t+1}\|\epsilon_{t}\|^{2}\right]\\
\le & 2\sigma^{2}+16\left(1+\frac{6\hS^{4/3}}{a_{0}^{4/3}}\right)\left(3a_{0}^{2}+5\hS^{2}\right)+4\left(1+\frac{6\hS^{4/3}}{a_{0}^{4/3}}\right)\eta^{2}\beta^{2}\E\left[\sum_{t=1}^{T}\frac{\|d_{t}\|^{2}}{b_{t}^{2}}\right],
\end{align*}
which will give the desired bound immediately.
\end{proof}

\subsubsection{Bound on $\protect\E\left[\protect\bE_{T,1-2\ell}\right]$ for $\ell\in\left[\frac{1}{4},\frac{1}{2}\right]$}

With the previous result on $\E\left[\bE_{\tau,3/2-2\ell}\right]$,
we can bound $\E\left[\bE_{T,1-2\ell}\right]$.
\begin{lem}
\label{lem:MS-E-T}For any $\ell\in\left[\frac{1}{4},\frac{1}{2}\right]$,
we have
\begin{align*}
\E\left[\bE_{T,1-2\ell}\right] & \leq K_{1}(\ell)+K_{2}(\ell)\begin{cases}
\E\left[\left(\tH_{T}/a_{0}^{2}\right)^{\frac{4\ell-1}{3}}\right] & \ell>\frac{1}{4}\\
\E\left[\log\left(1+\tH_{T}/a_{0}^{2}\right)\right] & \ell=\frac{1}{4}
\end{cases}\\
 & \quad+\E\left[\sum_{t=1}^{T}\left(K_{3}(\ell)a_{t}^{2\ell}+\frac{3\left(1+2\ell^{2}\right)}{\ell^{2}}\right)\eta^{2}\beta^{2}\frac{\|d_{t}\|^{2}}{a_{t}^{2\ell}b_{t}^{2}}\right],
\end{align*}
where
\begin{align*}
K_{1}(\ell) & \coloneqq3\left(\sigma^{2}+\frac{24\left(1+\ell^{2}\right)\hS^{2}}{\ell^{2}}\right)+\frac{72\hS^{2}\left(\sigma^{2}+8\left(1+\frac{6\hS^{4/3}}{a_{0}^{4/3}}\right)\left(3a_{0}^{2}+5\hS^{2}\right)\right)}{a_{0}^{2}K_{0}^{2\ell-1/2}}\\
K_{2}(\ell) & \coloneqq\begin{cases}
\frac{9\left(1+2\ell^{2}\right)a_{0}^{2}}{\ell^{2}(4\ell-1)} & \ell\neq\frac{1}{4}\\
\frac{3\left(1+2\ell^{2}\right)a_{0}^{2}}{\ell^{2}} & \ell=\frac{1}{4}
\end{cases}\\
K_{3}(\ell) & \coloneqq\frac{144\hS^{2}}{K_{0}^{2\ell-1/2}a_{0}^{2}}\left(1+\frac{6\hS^{4/3}}{a_{0}^{4/3}}\right)+\frac{3\left(1+2\ell^{2}\right)}{\ell^{2}}\left(\frac{4\hS^{2}}{a_{0}^{2}}\right)^{\frac{4\ell}{3}}
\end{align*}
\end{lem}
\begin{proof}
We use a similar strategy as in the previous proof in which we bound
$\E\left[\bE_{\tau,3/2-2\ell}\right]$. Starting from Lemma \ref{lem:vr-inequality}
\begin{align*}
a_{t+1}\|\epsilon_{t}\|^{2} & \leq\|\epsilon_{t}\|^{2}-\|\epsilon_{t+1}\|^{2}+2\|Z_{t+1}\|^{2}+2a_{t+1}^{2}\|\nabla f(x_{t+1},\xi_{t+1})-\nabla F(x_{t+1})\|^{2}+M_{t+1}.
\end{align*}
Dividing both sides by $a_{t+1}^{2\ell}$, taking the expectations
on both sides and summing up from $1$ to $T$ to get
\begin{align}
\E\left[\bE_{T,1-2\ell}\right] & \leq\E\bigg[\sum_{t=1}^{T}\frac{\|\epsilon_{t}\|^{2}}{a_{t+1}^{2\ell}}-\frac{\|\epsilon_{t+1}\|^{2}}{a_{t+1}^{2\ell}}\nonumber \\
 & \quad+\frac{2}{a_{t+1}^{2\ell}}\|Z_{t+1}\|^{2}+2a_{t+1}^{2-2\ell}\|\nabla f(x_{t+1},\xi_{t+1})-\nabla F(x_{t+1})\|^{2}+\frac{M_{t+1}}{a_{t+1}^{2\ell}}\bigg]\nonumber \\
 & \leq\sigma^{2}+\E\left[\sum_{t=1}^{T}\left(a_{t+1}^{-2\ell}-a_{t}^{-2\ell}\right)\|\epsilon_{t}\|^{2}\right.\nonumber \\
 & \quad\left.+\sum_{t=1}^{T}\frac{2}{a_{t+1}^{2\ell}}\|Z_{t+1}\|^{2}+2a_{t+1}^{2-2\ell}\|\nabla f(x_{t+1},\xi_{t+1})-\nabla F(x_{t+1})\|^{2}+\frac{M_{t+1}}{a_{t+1}^{2\ell}}\right].\label{eq:MS-E-bound-4}
\end{align}
As before, we bound $\E\left[\frac{M_{t+1}}{a_{t+1}^{2\ell}}\right]$
first. From the definition of $M_{t+1}$, we have
\begin{align*}
\E\left[\frac{M_{t+1}}{a_{t+1}^{2\ell}}\right] & =\E\left[\frac{2(1-a_{t+1})^{2}}{a_{t+1}^{2\ell}}\langle\epsilon_{t},Z_{t+1}\rangle+2(1-a_{t+1})a_{t+1}^{1-2\ell}\langle\epsilon_{t},\nabla f(x_{t+1},\xi_{t+1})-\nabla F(x_{t+1})\rangle\right].
\end{align*}
A similar key observation is that, if we replace $a_{t+1}$ by $a_{t}$,
we can find
\[
\E\left[\frac{2(1-a_{t})^{2}}{a_{t}^{2\ell}}\langle\epsilon_{t},Z_{t+1}\rangle+2(1-a_{t})a_{t}^{1-2\ell}\langle\epsilon_{t},\nabla f(x_{t+1},\xi_{t+1})-\nabla F(x_{t+1})\rangle\right]=0.
\]
By subtracting $\E\left[\frac{M_{t+1}}{a_{t+1}^{2\ell}}\right]$ by
$0$, we know
\begin{align}
\E\left[\frac{M_{t+1}}{a_{t+1}^{2\ell}}\right] & =\E\left[2\left(\frac{(1-a_{t+1})^{2}}{a_{t+1}^{2\ell}}-\frac{(1-a_{t})^{2}}{a_{t}^{2\ell}}\right)\langle\epsilon_{t},Z_{t+1}\rangle\right.\nonumber \\
 & \left.+2\left((1-a_{t+1})a_{t+1}^{1-2\ell}-(1-a_{t})a_{t}^{1-2\ell}\right)\langle\epsilon_{t},\nabla f(x_{t+1},\xi_{t+1})-\nabla F(x_{t+1})\rangle\right].\label{eq:MS-M-bound-3}
\end{align}
Using Cauchy-Schwarz for each term
\begin{align*}
 & 2\left(\frac{(1-a_{t+1})^{2}}{a_{t+1}^{2\ell}}-\frac{(1-a_{t})^{2}}{a_{t}^{2\ell}}\right)\langle\epsilon_{t},Z_{t+1}\rangle\\
\le & 2\left|\frac{(1-a_{t+1})^{2}}{a_{t+1}^{2\ell}}-\frac{(1-a_{t})^{2}}{a_{t}^{2\ell}}\right|\|\epsilon_{t}\|\|Z_{t+1}\|\\
\le & \left(a_{t+1}^{-2\ell}-a_{t}^{-2\ell}\right)\|\epsilon_{t}\|^{2}+\frac{\left(\frac{(1-a_{t+1})^{2}}{a_{t+1}^{2\ell}}-\frac{(1-a_{t})^{2}}{a_{t}^{2\ell}}\right)^{2}}{a_{t+1}^{-2\ell}-a_{t}^{-2\ell}}\|Z_{t+1}\|^{2},
\end{align*}
\begin{align*}
 & 2\left((1-a_{t+1})a_{t+1}^{1-2\ell}-(1-a_{t})a_{t}^{1-2\ell}\right)\langle\epsilon_{t},\nabla f(x_{t+1},\xi_{t+1})-\nabla F(x_{t+1})\rangle\\
\le & 2\left|(1-a_{t+1})a_{t+1}^{1-2\ell}-(1-a_{t})a_{t}^{1-2\ell}\right|\|\epsilon_{t}\|\|\nabla f(x_{t+1},\xi_{t+1})-\nabla F(x_{t+1})\|\\
\le & \left(a_{t+1}^{-2\ell}-a_{t}^{-2\ell}\right)\|\epsilon_{t}\|^{2}+\frac{\left((1-a_{t+1})a_{t+1}^{1-2\ell}-(1-a_{t})a_{t}^{1-2\ell}\right)^{2}}{a_{t+1}^{-2\ell}-a_{t}^{-2\ell}}\|\nabla f(x_{t+1},\xi_{t+1})-\nabla F(x_{t+1})\|^{2},
\end{align*}
Plugging these two bounds into (\ref{eq:MS-M-bound-3}), we obtain
\begin{align}
\E\left[\frac{M_{t+1}}{a_{t+1}^{2\ell}}\right] & \le\E\left[2\left(a_{t+1}^{-2\ell}-a_{t}^{-2\ell}\right)\|\epsilon_{t}\|^{2}\right]+\E\left[\underbrace{\frac{\left(\frac{(1-a_{t+1})^{2}}{a_{t+1}^{2\ell}}-\frac{(1-a_{t})^{2}}{a_{t}^{2\ell}}\right)^{2}}{a_{t+1}^{-2\ell}-a_{t}^{-2\ell}}}_{(i)}\|Z_{t+1}\|^{2}\right]\nonumber \\
 & +\E\left[\underbrace{\frac{\left((1-a_{t+1})a_{t+1}^{1-2\ell}-(1-a_{t})a_{t}^{1-2\ell}\right)^{2}}{a_{t+1}^{-2\ell}-a_{t}^{-2\ell}}}_{(ii)}\|\nabla f(x_{t+1},\xi_{t+1})-\nabla F(x_{t+1})\|^{2}\right].\label{eq:MS-M-bound-4}
\end{align}

To bound $(i)$ and $(ii)$, let $a_{t+1}^{\ell}=x,a_{t}^{\ell}=y$
and note that $0\leq x\leq y\leq1$. By Lemma \ref{lem:cauchy-bound-g-1},
we have for $(i)$
\begin{align}
\frac{\left(\frac{(1-a_{t+1})^{2}}{a_{t+1}^{2\ell}}-\frac{(1-a_{t})^{2}}{a_{t}^{2\ell}}\right)^{2}}{a_{t+1}^{-2\ell}-a_{t}^{-2\ell}} & =\frac{\left(\frac{(1-x^{1/\ell})^{2}}{x^{2}}-\frac{(1-y^{1/\ell})^{2}}{y^{2}}\right)^{2}x^{2}y^{2}}{y^{2}-x^{2}}\nonumber \\
 & \leq\frac{1}{\ell^{2}x^{2}}=\frac{1}{\ell^{2}a_{t+1}^{2\ell}}.\label{eq:MS-cauchy-bound-1}
\end{align}
For $(ii)$, by Lemma \ref{lem:cauchy-bound-g-2},
\begin{align}
\frac{\left((1-a_{t+1})a_{t+1}^{1-2\ell}-(1-a_{t})a_{t}^{1-2\ell}\right)^{2}}{a_{t+1}^{-2\ell}-a_{t}^{-2\ell}} & =\frac{\left((1-x^{1/\ell})x^{1/\ell-2}-(1-y^{1/\ell})y^{1/\ell-2}\right)^{2}x^{2}y^{2}}{y^{2}-x^{2}}\nonumber \\
 & \leq\frac{y^{2/\ell-2}}{\ell^{2}}=\frac{a_{t}^{2-2\ell}}{\ell^{2}}.\label{eq:MS-cauchy-bound-2}
\end{align}
Plugging (\ref{eq:MS-cauchy-bound-1}) and (\ref{eq:MS-cauchy-bound-2})
into (\ref{eq:MS-M-bound-4}), we will have
\begin{align*}
\E\left[\frac{M_{t+1}}{a_{t+1}^{2\ell}}\right] & \leq\E\bigg[2\left(a_{t+1}^{-2\ell}-a_{t}^{-2\ell}\right)\|\epsilon_{t}\|^{2}\\
 & \quad\quad+\frac{1}{\ell^{2}a_{t+1}^{2\ell}}\|Z_{t+1}\|^{2}+\frac{a_{t}^{2-2\ell}}{\ell^{2}}\|\nabla f(x_{t+1},\xi_{t+1})-\nabla F(x_{t+1})\|^{2}\bigg].
\end{align*}
Now combining this with (\ref{eq:MS-E-bound-4}), we obtain
\begin{align}
\E\left[\bE_{T,1-2\ell}\right] & \leq\sigma^{2}+\E\left[\underbrace{\sum_{t=1}^{T}3\left(a_{t+1}^{-2\ell}-a_{t}^{-2\ell}\right)\|\epsilon_{t}\|^{2}}_{(iii)}+\sum_{t=1}^{T}\underbrace{\frac{1+2\ell^{2}}{\ell^{2}a_{t+1}^{2\ell}}\|Z_{t+1}\|^{2}}_{(iv)}\right.\nonumber \\
 & \quad\left.+\underbrace{\sum_{t=1}^{T}\left(\frac{a_{t}^{2-2\ell}}{\ell^{2}}+2a_{t+1}^{2-2\ell}\right)\|\nabla f(x_{t+1},\xi_{t+1})-\nabla F(x_{t+1})\|^{2}}_{(v)}\right].\label{eq:MS-E-bound-5}
\end{align}

For $(iii)$, we split the sum according to $\tau$ then use Lemma
\ref{lem:stopping-time-bound-g-d} and Lemma \ref{lem:MS-a_t-bound},
\begin{align*}
\sum_{t=1}^{T}3\left(a_{t+1}^{-2\ell}-a_{t}^{-2\ell}\right)\|\epsilon_{t}\|^{2} & =\sum_{t=1}^{\tau}3\left(a_{t+1}^{-2\ell}-a_{t}^{-2\ell}\right)\|\epsilon_{t}\|^{2}+\sum_{t=\tau+1}^{T}3\left(a_{t+1}^{-2\ell}-a_{t}^{-2\ell}\right)\|\epsilon_{t}\|^{2}
\end{align*}
Note that $3/2-2\ell\in\left[\frac{1}{2},1\right]$, we have
\begin{align*}
a_{t+1}^{-2\ell}-a_{t}^{-2\ell} & =\left(\frac{1}{a_{t+1}^{3/2}}-\frac{1}{a_{t}^{2\ell}a_{t+1}^{3/2-2\ell}}\right)a_{t+1}^{3/2-2\ell}\le\left(a_{t+1}^{-3/2}-a_{t}^{-3/2}\right)a_{t+1}^{3/2-2\ell}\\
 & \le\frac{4\hS^{2}}{a_{0}^{2}}a_{t+1}^{3/2-2\ell},\qquad(\text{Lemma \ref{lem:MS-a_t-bound}})
\end{align*}
and we can use Lemma \ref{lem:stopping-time-bound-g-d} to bound for
$t\ge\tau+1$
\begin{align*}
a_{t+1}^{-2\ell}-a_{t}^{-2\ell} & =\left(\frac{1}{a_{t+1}}-\frac{1}{a_{t}^{2\ell}a_{t+1}^{1-2\ell}}\right)a_{t+1}^{1-2\ell}\le\left(a_{t+1}^{-1}-a_{t}^{-1}\right)a_{t+1}^{1-2\ell}\\
 & \le\frac{2}{9}a_{t+1}^{1-2\ell}.
\end{align*}
Thus
\begin{align*}
\sum_{t=1}^{T}3\left(a_{t+1}^{-2\ell}-a_{t}^{-2\ell}\right)\|\epsilon_{t}\|^{2} & \le\sum_{t=1}^{\tau}\frac{12\hS^{2}}{a_{0}^{2}}a_{t+1}^{3/2-2\ell}\|\epsilon_{t}\|^{2}+\sum_{t=\tau+1}^{T}\frac{2}{3}a_{t+1}^{1-2\ell}\|\epsilon_{t}\|^{2}\\
 & \leq\frac{12\hS^{2}}{a_{0}^{2}}\sum_{t=1}^{\tau}a_{t+1}^{3/2-2\ell}\|\epsilon_{t}\|^{2}+\sum_{t=1}^{T}\frac{2}{3}a_{t+1}^{1-2\ell}\|\epsilon_{t}\|^{2}\\
 & =\frac{12\hS^{2}}{a_{0}^{2}}\bE_{\tau,3/2-2\ell}+\frac{2}{3}\bE_{T,1-2\ell}.
\end{align*}

For $(iv)$, note that
\begin{align*}
\E\left[\frac{1+2\ell^{2}}{\ell^{2}a_{t+1}^{2\ell}}\|Z_{t+1}\|^{2}\right] & =\frac{1+2\ell^{2}}{\ell^{2}}\E\left[\frac{a_{t}^{2\ell}}{a_{t+1}^{2\ell}}\frac{\|Z_{t+1}\|^{2}}{a_{t}^{2\ell}}\right]\\
 & \leq\frac{1+2\ell^{2}}{\ell^{2}}\E\left[\left(1+\left(\frac{4\hS^{2}}{a_{0}^{2}}\right)^{\frac{4\ell}{3}}a_{t}^{2\ell}\right)\frac{\|Z_{t+1}\|^{2}}{a_{t}^{2\ell}}\right]\qquad\text{(\text{Lemma \ref{lem:MS-a_t-bound}})}\\
 & \leq\frac{1+2\ell^{2}}{\ell^{2}}\E\left[\left(1+\left(\frac{4\hS^{2}}{a_{0}^{2}}\right)^{\frac{4\ell}{3}}a_{t}^{2\ell}\right)\frac{\E\left[\|Z_{t+1}\|^{2}\mid\F_{t}\right]}{a_{t}^{2\ell}}\right]\\
 & \leq\frac{1+2\ell^{2}}{\ell^{2}}\E\left[\left(1+\left(\frac{4\hS^{2}}{a_{0}^{2}}\right)^{\frac{4\ell}{3}}a_{t}^{2\ell}\right)\eta^{2}\beta^{2}\frac{\|d_{t}\|^{2}}{a_{t}^{2\ell}b_{t}^{2}}\right],
\end{align*}
where the last step is by Lemma \ref{lem:smooth-Z}. Hence we obtain
\begin{align*}
\E\left[\sum_{t=1}^{T}\frac{1+2\ell^{2}}{\ell^{2}a_{t+1}^{2\ell}}\|Z_{t+1}\|^{2}\right] & \leq\E\left[\sum_{t=1}^{T}\frac{1+2\ell^{2}}{\ell^{2}}\left(1+\left(\frac{4\hS^{2}}{a_{0}^{2}}\right)^{\frac{4\ell}{3}}a_{t}^{2\ell}\right)\eta^{2}\beta^{2}\frac{\|d_{t}\|^{2}}{a_{t}^{2\ell}b_{t}^{2}}\right].
\end{align*}

For $(v)$, by the same argument when bounding (\ref{eq:MS-condition-expectations}),
we know
\[
\E\left[(v)\right]\leq\E\left[\sum_{t=1}^{T}\left(\frac{a_{t}^{2-2\ell}}{\ell^{2}}+2a_{t+1}^{2-2\ell}\right)\|\nabla f(x_{t+1},\xi_{t+1})-\nabla f(x_{t+1},\xi_{t+2})|\|^{2}\right].
\]
Now we use Lemma \ref{lem:inequality-2} and Lemma \ref{lem:inequality-3}
to get
\begin{align*}
 & \sum_{t=1}^{T}\left(\frac{a_{t}^{2-2\ell}}{\ell^{2}}+2a_{t+1}^{2-2\ell}\right)\|\nabla f(x_{t+1},\xi_{t+1})-\nabla f(x_{t+1},\xi_{t+2})\|^{2}\\
= & \frac{a_{0}^{2}}{\ell^{2}}\sum_{t=1}^{T}\frac{\|\nabla f(x_{t+1},\xi_{t+1})-\nabla f(x_{t+1},\xi_{t+2})\|^{2}/a_{0}^{2}}{\left(1+\sum_{i=1}^{t-1}\|\nabla f(x_{i},\xi_{i})-\nabla f(x_{i},\xi_{i+1})\|^{2}/a_{0}^{2}\right)^{4(1-\ell)/3}}\\
 & +2a_{0}^{2}\sum_{t=1}^{T}\frac{\|\nabla f(x_{t+1},\xi_{t+1})-\nabla f(x_{t+1},\xi_{t+2})\|^{2}/a_{0}^{2}}{\left(1+\sum_{i=1}^{t}\|\nabla f(x_{i},\xi_{i})-\nabla f(x_{i},\xi_{i+1})\|^{2}/a_{0}^{2}\right)^{4(1-\ell)/3}}\\
\leq & \frac{a_{0}^{2}}{\ell^{2}}\times\frac{24\hS^{2}}{a_{0}^{2}}+2a_{0}^{2}\times\frac{12\hS^{2}}{a_{0}^{2}}\\
 & \quad+\frac{\left(1+2\ell^{2}\right)a_{0}^{2}}{\ell^{2}}\begin{cases}
\frac{3}{4\ell-1}\left(\tH_{T}/a_{0}^{2}\right)^{\frac{4\ell-1}{3}} & \ell\neq\frac{1}{4}\\
\log\left(1+\tH_{T}/a_{0}^{2}\right) & \ell=\frac{1}{4}
\end{cases}\\
 & =\frac{24\left(1+\ell^{2}\right)\hS^{2}}{\ell^{2}}+\frac{\left(1+2\ell^{2}\right)a_{0}^{2}}{\ell^{2}}\begin{cases}
\frac{3}{4\ell-1}\left(\tH_{T}/a_{0}^{2}\right)^{\frac{4\ell-1}{3}} & \ell\neq\frac{1}{4}\\
\log\left(1+\tH_{T}/a_{0}^{2}\right) & \ell=\frac{1}{4}
\end{cases}.
\end{align*}

Plugging the bounds on $(iii),(iv)$ and $(v)$ into (\ref{eq:MS-E-bound-5}),
we get
\begin{align*}
\E\left[\bE_{T,1-2\ell}\right] & \leq\sigma^{2}+\frac{24\left(1+\ell^{2}\right)\hS^{2}}{\ell^{2}}+\E\left[\frac{12\hS^{2}}{a_{0}^{2}}\bE_{\tau,3/2-2\ell}+\frac{2}{3}\bE_{T,1-2\ell}\right]\\
 & \quad+\E\left[\sum_{t=1}^{T}\frac{1+2\ell^{2}}{\ell^{2}}\left(1+\left(\frac{4\hS^{2}}{a_{0}^{2}}\right)^{\frac{4\ell}{3}}a_{t}^{2\ell}\right)\eta^{2}\beta^{2}\frac{\|d_{t}\|^{2}}{a_{t}^{2\ell}b_{t}^{2}}\right]\\
 & \quad+\frac{\left(1+2\ell^{2}\right)a_{0}^{2}}{\ell^{2}}\begin{cases}
\frac{3}{4\ell-1}\E\left[\left(\tH_{T}/a_{0}^{2}\right)^{\frac{4\ell-1}{3}}\right] & \ell\neq\frac{1}{4}\\
\E\left[\log\left(1+\tH_{T}/a_{0}^{2}\right)\right] & \ell=\frac{1}{4}
\end{cases},
\end{align*}
which gives us
\begin{align*}
\E\left[\bE_{T,1-2\ell}\right] & \leq3\left(\sigma^{2}+\frac{24\left(1+\ell^{2}\right)\hS^{2}}{\ell^{2}}\right)+\frac{36\hS^{2}}{a_{0}^{2}}\E\left[\bE_{\tau,3/2-2\ell}\right]\\
 & \quad+\E\left[\sum_{t=1}^{T}\frac{3\left(1+2\ell^{2}\right)}{\ell^{2}}\left(1+\left(\frac{4\hS^{2}}{a_{0}^{2}}\right)^{\frac{4\ell}{3}}a_{t}^{2\ell}\right)\eta^{2}\beta^{2}\frac{\|d_{t}\|^{2}}{a_{t}^{2\ell}b_{t}^{2}}\right]\\
 & \quad+\frac{3\left(1+2\ell^{2}\right)a_{0}^{2}}{\ell^{2}}\begin{cases}
\frac{3}{4\ell-1}\E\left[\left(\tH_{T}/a_{0}^{2}\right)^{\frac{4\ell-1}{3}}\right] & \ell\neq\frac{1}{4}\\
\E\left[\log\left(1+\tH_{T}/a_{0}^{2}\right)\right] & \ell=\frac{1}{4}
\end{cases}.
\end{align*}

Now we plug in the bound on $\E\left[\bE_{\tau,3/2-2\ell}\right]$
in Lemma \ref{lem:MS-E-tau} to get the final result
\begin{align*}
 & \E\left[\bE_{T,1-2\ell}\right]\\
\leq & \underbrace{3\left(\sigma^{2}+\frac{24\left(1+\ell^{2}\right)\hS^{2}}{\ell^{2}}\right)+\frac{72\hS^{2}\left(\sigma^{2}+8\left(1+\frac{6\hS^{4/3}}{a_{0}^{4/3}}\right)\left(3a_{0}^{2}+5\hS^{2}\right)\right)}{a_{0}^{2}K_{0}^{2\ell-1/2}}}_{K_{1}(\ell)}\\
 & +K_{2}(\ell)\begin{cases}
\E\left[\left(\tH_{T}/a_{0}^{2}\right)^{\frac{4\ell-1}{3}}\right] & \ell\neq\frac{1}{4}\\
\E\left[\log\left(1+\tH_{T}/a_{0}^{2}\right)\right] & \ell=\frac{1}{4}
\end{cases}\\
 & +\E\left[\sum_{t=1}^{T}\left(\underbrace{\left(\frac{144\hS^{2}}{K_{0}^{2\ell-1/2}a_{0}^{2}}\left(1+\frac{6\hS^{4/3}}{a_{0}^{4/3}}\right)+\frac{3\left(1+2\ell^{2}\right)}{\ell^{2}}\left(\frac{4\hS^{2}}{a_{0}^{2}}\right)^{\frac{4\ell}{3}}\right)}_{K_{3}(\ell)}a_{t}^{2\ell}+\frac{3\left(1+2\ell^{2}\right)}{\ell^{2}}\right)\eta^{2}\beta^{2}\frac{\|d_{t}\|^{2}}{a_{t}^{2\ell}b_{t}^{2}}\right],
\end{align*}
where
\[
K_{2}(\ell)\coloneqq\begin{cases}
\frac{9\left(1+2\ell^{2}\right)a_{0}^{2}}{(4\ell-1)\ell^{2}} & \ell\neq\frac{1}{4}\\
\frac{3\left(1+2\ell^{2}\right)a_{0}^{2}}{\ell^{2}} & \ell=\frac{1}{4}
\end{cases}.
\]
\end{proof}

\subsubsection{Bound on $\protect\E\left[\protect\bE_{T,1/2}\right]$}

The following bound on $\E\left[\bE_{T,1/2}\right]$ will be useful
when we bound $\bD_{T}$.
\begin{cor}
\label{cor:MS-E-1/2}We have
\[
\E\left[\bE_{T,1/2}\right]\leq K_{1}(1/4)+K_{2}(1/4)\E\left[\log\left(1+\tH_{T}/a_{0}^{2}\right)\right]+\E\left[\sum_{t=1}^{T}\left(K_{3}(1/4)a_{t}^{1/2}+54\right)\eta^{2}\beta^{2}\frac{\|d_{t}\|^{2}}{a_{t}^{1/2}b_{t}^{2}}\right].
\]
\end{cor}
\begin{proof}
Take $\ell=\frac{1}{4}$ in Lemma \ref{lem:MS-E-T}.
\end{proof}

\subsubsection{Bound on $\protect\E\left[a_{T+1}^{1-2q}\protect\bE_{T}\right]$}

With the previous result on $\E\left[\bE_{T,1-2\ell}\right]$, we
can bound $\E\left[a_{T+1}^{1-2q}\bE_{T}\right]$ immediately.
\begin{lem}
\label{lem:MS-E-final-bound}Given $p+2q=1$,$p\in\left[\frac{3-\sqrt{7}}{2},\frac{1}{2}\right]$,
we have
\begin{align*}
\E\left[a_{T+1}^{1-2q}\bE_{T}\right] & \leq\begin{cases}
K_{1}+K_{2}\E\left[\left(\tH_{T}/a_{0}^{2}\right)^{\frac{4q-1}{3}}\right]+K_{4}\E\left[\bD_{T}^{1-2p}\right] & q>\frac{1}{4}\\
K_{1}+K_{2}\E\left[\log\left(1+\tH_{T}/a_{0}^{2}\right)\right]+K_{4}\E\left[\log\left(1+\frac{\bD_{T}}{b_{0}^{2}}\right)\right] & q=\frac{1}{4}
\end{cases}
\end{align*}
where
\begin{align*}
K_{1} & \coloneqq K_{1}(q)\\
K_{2} & \coloneqq K_{2}(q)\\
K_{4} & \coloneqq\begin{cases}
\left(K_{3}(q)+\frac{3\left(1+2q^{2}\right)}{q^{2}}\right)\frac{\eta^{2}\beta^{2}}{4q-1} & q>\frac{1}{4}\\
\left(K_{3}(q)+\frac{3\left(1+2q^{2}\right)}{q^{2}}\right)\eta^{2}\beta^{2} & q=\frac{1}{4}
\end{cases}.
\end{align*}
\end{lem}
\begin{proof}
When $q>\frac{1}{4}\Leftrightarrow p<\frac{1}{2}$, by Lemma \ref{lem:MS-E-T},
taking $\ell=q$, we know
\begin{align*}
\E\left[\bE_{T,1-2q}\right] & \leq K_{1}(q)+K_{2}(q)\E\left[\left(\tH_{T}/a_{0}^{2}\right)^{\frac{4q-1}{3}}\right]+\E\left[\sum_{t=1}^{T}\left(K_{3}(q)a_{t}^{2q}+\frac{3\left(1+2q^{2}\right)}{q^{2}}\right)\eta^{2}\beta^{2}\frac{\|d_{t}\|^{2}}{a_{t}^{2q}b_{t}^{2}}\right]\\
 & \leq K_{1}(q)+K_{2}(q)\E\left[\left(\tH_{T}/a_{0}^{2}\right)^{\frac{4q-1}{3}}\right]+\E\left[\sum_{t=1}^{T}\left(K_{3}(q)+\frac{3\left(1+2q^{2}\right)}{q^{2}}\right)\eta^{2}\beta^{2}\frac{\|d_{t}\|^{2}}{a_{t}^{2q}b_{t}^{2}}\right]\\
 & =K_{1}(q)+K_{2}(q)\E\left[\left(\tH_{T}/a_{0}^{2}\right)^{\frac{4q-1}{3}}\right]+\left(K_{3}(q)+\frac{3\left(1+2q^{2}\right)}{q^{2}}\right)\eta^{2}\beta^{2}\E\left[\sum_{t=1}^{T}\frac{\|d_{t}\|^{2}}{a_{t}^{2q}b_{t}^{2}}\right]\\
 & \overset{(a)}{=}K_{1}(q)+K_{2}(q)\E\left[\left(\tH_{T}/a_{0}^{2}\right)^{\frac{4q-1}{3}}\right]\\
 & \qquad\qquad\qquad+\left(K_{3}(q)+\frac{3\left(1+2q^{2}\right)}{q^{2}}\right)\eta^{2}\beta^{2}\E\left[\sum_{t=1}^{T}\frac{\|d_{t}\|^{2}}{\left(b_{0}^{1/p}+\sum_{i=1}^{t}\|d_{i}\|^{2}\right)^{2p}}\right]\\
 & \overset{(b)}{\leq}K_{1}(q)+K_{2}(q)\E\left[\left(\tH_{T}/a_{0}^{2}\right)^{\frac{4q-1}{3}}\right]+\left(K_{3}(q)+\frac{3\left(1+2q^{2}\right)}{q^{2}}\right)\eta^{2}\beta^{2}\E\left[\frac{\bD_{T}^{1-2p}}{1-2p}\right]\\
 & \overset{(c)}{=}K_{1}(q)+K_{2}(q)\E\left[\left(\tH_{T}/a_{0}^{2}\right)^{\frac{4q-1}{3}}\right]+\left(K_{3}(q)+\frac{3\left(1+2q^{2}\right)}{q^{2}}\right)\frac{\eta^{2}\beta^{2}}{4q-1}\E\left[\bD_{T}^{1-2p}\right],
\end{align*}
where $(a)$ is by
\begin{align*}
a_{t}^{2q}b_{t}^{2} & =a_{t}^{2q}\frac{\left(b_{0}^{1/p}+\sum_{i=1}^{t}\|d_{i}\|^{2}\right)^{2p}}{a_{t}^{2q}}=\left(b_{0}^{1/p}+\sum_{i=1}^{t}\|d_{i}\|^{2}\right)^{2p},
\end{align*}
$(b)$ is by Lemma \ref{lem:inequality-1}, $(c)$ is by $1-2p=4q-1$.

When $q=\frac{1}{4}$, by a similar argument, we have
\begin{align*}
\E\left[\bE_{T,1-2q}\right] & \leq K_{1}(q)+K_{2}(q)\E\left[\log\left(1+\tH_{T}/a_{0}^{2}\right)\right]\\
 & \qquad+\left(K_{3}(q)+\frac{3\left(1+2q^{2}\right)}{q^{2}}\right)\eta^{2}\beta^{2}\E\left[\log\left(1+\frac{\bD_{T}}{b_{0}^{2}}\right)\right].
\end{align*}

Now we can define
\[
K_{4}\coloneqq\begin{cases}
\left(K_{3}(q)+\frac{3\left(1+2q^{2}\right)}{q^{2}}\right)\frac{\eta^{2}\beta^{2}}{4q-1} & q>\frac{1}{4}\\
\left(K_{3}(q)+\frac{3\left(1+2q^{2}\right)}{q^{2}}\right)\eta^{2}\beta^{2} & q=\frac{1}{4}
\end{cases}.
\]
The final step is by noticing for $1-2q=p>0$,
\begin{align*}
\bE_{T,1-2q} & =\sum_{t=1}^{T}a_{t+1}^{1-2q}\|\epsilon_{t}\|^{2}\geq a_{T+1}^{1-2q}\sum_{t=1}^{T}\|\epsilon_{t}\|^{2}=a_{T+1}^{1-2q}\bE_{T}.
\end{align*}
\end{proof}

\subsection{Analysis of $\protect\bD_{T}$}

We will prove the following bound
\begin{lem}
\label{lem:MS-D-final-bound}Given $p+2q=1$,$p\in\left[\frac{3-\sqrt{7}}{2},\frac{1}{2}\right]$,
we have
\begin{align*}
\E\left[a_{T+1}^{q}\bD_{T}^{1-p}\right] & \leq K_{5}+K_{6}\E\left[\log\frac{a_{0}^{2}+\tH_{T}}{a_{0}^{2}}\right]+K_{7}\E\left[\log\frac{K_{8}+K_{9}\left(1+\tH_{T}/a_{0}^{2}\right)^{1/3}}{b_{0}}\right]
\end{align*}
where
\begin{align*}
K_{5} & \coloneqq b_{0}^{\frac{1}{p}-1}+\frac{2}{\eta}\left(F(x_{1})-F^{*}\right)+\frac{\lambda K_{1}(1/4)}{\eta\beta_{\max}},K_{6}\coloneqq\frac{\lambda K_{2}(1/4)}{\eta\beta_{\max}},\\
K_{7} & \coloneqq\frac{\left(K_{8}+K_{9}\right)^{\frac{1}{p}-1}}{1-p},K_{8}\coloneqq\left(1+\lambda K_{3}(1/4)\right)\eta\beta_{\max},K_{9}\coloneqq\left(\frac{1}{\lambda}+\frac{2\hS^{2/3}}{a_{0}^{2/3}\lambda}+54\lambda\right)\eta\beta_{\max},\\
\lambda & >0\text{ can be any number}.
\end{align*}
\end{lem}
\begin{proof}
We start from Lemma \ref{lem:f-value-analysis}
\begin{align*}
\E\left[a_{T+1}^{q}\bD_{T}^{1-p}\right] & \leq b_{0}^{\frac{1}{p}-1}+\frac{2}{\eta}\left(F(x_{1})-F^{*}\right)\\
 & \quad+\E\left[\sum_{t=1}^{T}\left(\eta\beta_{\max}+\frac{\eta\beta_{\max}}{a_{t+1}^{1/2}\lambda}-b_{t}\right)\frac{\|d_{t}\|^{2}}{b_{t}^{2}}\right]+\frac{\lambda\E\left[\bE_{T,1/2}\right]}{\eta\beta_{\max}}
\end{align*}
where $\lambda>0$ is used to reduce the order of $\hS$ in the final
bound. In the proof of the general case, we don't choose $\lambda$
explicitly anymore. Plugging in the bound on $\E\left[\bE_{T,1/2}\right]$
in Corollary \ref{cor:MS-E-1/2}, we have
\begin{align}
 & \E\left[a_{T+1}^{q}\bD_{T}^{1-p}\right]\nonumber \\
\le & b_{0}^{\frac{1}{p}-1}+\frac{2}{\eta}\left(F(x_{1})-F^{*}\right)+\frac{\lambda K_{1}(1/4)}{\eta\beta_{\max}}+\frac{\lambda K_{2}(1/4)}{\eta\beta_{\max}}\E\left[\log\frac{a_{0}^{2}+\tH_{T}}{a_{0}^{2}}\right]\nonumber \\
 & \quad+\E\left[\sum_{t=1}^{T}\left(\eta\beta_{\max}+\frac{\eta\beta_{\max}}{a_{t+1}^{1/2}\lambda}+\frac{K_{3}(1/4)\lambda\eta^{2}\beta^{2}}{\eta\beta_{\max}}+\frac{54\lambda\eta^{2}\beta^{2}}{a_{t}^{1/2}\eta\beta_{\max}}-b_{t}\right)\frac{\|d_{t}\|^{2}}{b_{t}^{2}}\right]\nonumber \\
\le & K_{5}+K_{6}\E\left[\log\frac{a_{0}^{2}+\tH_{T}}{a_{0}^{2}}\right]\nonumber \\
 & \quad+\E\left[\sum_{t=1}^{T}\left(\left(1+\lambda K_{3}(1/4)\right)\eta\beta_{\max}+\left(\frac{a_{t}^{1/2}}{\lambda a_{t+1}^{1/2}}+54\lambda\right)\frac{\eta\beta_{\max}}{a_{t}^{1/2}}-b_{t}\right)\frac{\|d_{t}\|^{2}}{b_{t}^{2}}\right]\nonumber \\
\le & K_{5}+K_{6}\E\left[\log\frac{a_{0}^{2}+\tH_{T}}{a_{0}^{2}}\right]\nonumber \\
 & \quad+\E\left[\underbrace{\sum_{t=1}^{T}\left(\left(1+\lambda K_{3}(1/4)\right)\eta\beta_{\max}+\left(\frac{1}{\lambda}+\frac{2\hS^{2/3}}{a_{0}^{2/3}\lambda}+54\lambda\right)\frac{\eta\beta_{\max}}{a_{t}^{1/2}}-b_{t}\right)\frac{\|d_{t}\|^{2}}{b_{t}^{2}}}_{(i)}\right]\label{eq:MS-bound-D}
\end{align}
where, in the last step, we use Lemma \ref{lem:MS-a_t-bound}. Next,
we apply Lemma \ref{lem:residual-bound} to $(i)$ to get
\begin{align*}
(i) & \leq\frac{\left(\left(1+\lambda K_{3}(1/4)+\frac{1}{\lambda}+\frac{2\hS^{2/3}}{a_{0}^{2/3}\lambda}+54\lambda\right)\eta\beta_{\max}\right)^{\frac{1}{p}-1}}{1-p}\\
 & \quad\times\log\frac{\left(1+\lambda K_{3}(1/4)\right)\eta\beta_{\max}+\left(\frac{1}{\lambda}+\frac{2\hS^{2/3}}{a_{0}^{2/3}\lambda}+54\lambda\right)\eta\beta_{\max}\left(1+\tH_{T}/a_{0}^{2}\right)^{1/3}}{b_{0}}\\
 & =K_{7}\log\frac{K_{8}+K_{9}\left(1+\tH_{T}/a_{0}^{2}\right)^{1/3}}{b_{0}}
\end{align*}
By plugging the above bound into (\ref{eq:MS-bound-D}), we get the
desired result.
\end{proof}

\subsection{Combine the bounds and the final Proof}

From Lemma \ref{lem:MS-E-final-bound}, we have
\begin{align*}
\E\left[a_{T+1}^{1-2q}\bE_{T}\right] & \leq\begin{cases}
K_{1}+K_{2}\E\left[\left(\tH_{T}/a_{0}^{2}\right)^{\frac{4q-1}{3}}\right]+K_{4}\E\left[\bD_{T}^{1-2p}\right] & q>\frac{1}{4}\\
K_{1}+K_{2}\E\left[\log\left(1+\tH_{T}/a_{0}^{2}\right)\right]+K_{4}\E\left[\log\left(1+\frac{\bD_{T}}{b_{0}^{2}}\right)\right] & q=\frac{1}{4}
\end{cases}
\end{align*}
From Lemma \ref{lem:MS-D-final-bound}, we have
\begin{align*}
\E\left[a_{T+1}^{q}\bD_{T}^{1-p}\right] & \leq K_{5}+K_{6}\E\left[\log\frac{a_{0}^{2}+\tH_{T}}{a_{0}^{2}}\right]+K_{7}\E\left[\log\frac{K_{8}+K_{9}\left(1+\tH_{T}/a_{0}^{2}\right)^{1/3}}{b_{0}}\right].
\end{align*}

Now let 
\[
\hp=\frac{3(1-p)}{4-p}\in\left[\frac{3}{7},\sqrt{7}-2\right].
\]
Apply Lemma \ref{lem:decomposition}, we can obtain
\begin{align}
\E\left[\bH_{T}^{\hp}\right] & \leq4\max\left\{ \E\left[\bE_{T}^{\hp}\right],\E\left[\bD_{T}^{\hp}\right]\right\} .\label{eq:MS-H-bound-via-E-D}
\end{align}
Now we can give the final proof of Theorem \ref{thm:MS-rate}.

\begin{proof}
First, we have
\begin{align*}
\E\left[\tH_{T}\right] & =\E\left[\sum_{i=1}^{T}\|\nabla f(x_{i},\xi_{i})-\nabla f(x_{i},\xi_{i+1})\|^{2}\right]\\
 & =2\sum_{i=1}^{T}\text{Var}\left[\nabla f(x_{i},\xi_{i})\right]\leq2\sigma^{2}T,
\end{align*}
where the second equation is by the independency of $\xi_{i}$ and
$\xi_{i+1}$. Now we consider following two cases:

\textbf{Case 1:} $\E\left[\bD_{T}^{\hp}\right]\leq\E\left[\bE_{T}^{\hp}\right]$.
In this case, we will finally prove
\begin{align*}
\E\left[\bE_{T}^{\hp}\right] & \leq\begin{cases}
\left(\frac{2K_{1}}{K_{4}}\right)^{\frac{\hp}{1-2p}}+\left(\left(\frac{2K_{2}}{K_{4}}\right)^{\frac{\hp}{1-2p}}+\left(2K_{4}\right)^{\frac{\hp}{2p}}\right)\left(1+\frac{2\sigma^{2}T}{a_{0}^{2}}\right)^{\frac{\hp}{3}} & q\neq\frac{1}{4}\\
\left(2K_{1}+2\left(K_{2}+\frac{K_{4}}{3}\right)\log\left(1+\frac{2\sigma^{2}T}{a_{0}^{2}}\right)+\frac{2K_{4}}{\hp}\log\frac{4K_{4}}{b_{0}^{2\hp}}\right)^{\hp}\left(1+\frac{2\sigma^{2}T}{a_{0}^{2}}\right)^{\frac{\hp}{3}}+b_{0}^{2\hp} & q=\frac{1}{4}
\end{cases}.
\end{align*}
Note that by Holder inequality
\begin{align*}
\E\left[\bE_{T}^{\hp}\right] & =\E\left[a_{T+1}^{(1-2q)\hp}\bE_{T}^{\hp}\times a_{T+1}^{-(1-2q)\hp}\right]\\
 & \leq\E^{\hp}\left[a_{T+1}^{1-2q}\bE_{T}\right]\E^{1-\hp}\left[a_{T+1}^{-\frac{(1-2q)\hp}{1-\hp}}\right]\\
 & =\E^{\hp}\left[a_{T+1}^{1-2q}\bE_{T}\right]\E^{1-\hp}\left[(1+\tH_{T}/a_{0}^{2})^{\frac{2(1-2q)\hp}{3\left(1-\hp\right)}}\right]\\
 & \overset{(a)}{=}\E^{\hp}\left[a_{T+1}^{1-2q}\bE_{T}\right]\E^{1-\hp}\left[(1+\tH_{T}/a_{0}^{2})^{\frac{2p\hp}{3\left(1-\hp\right)}}\right]\\
 & \overset{(b)}{\leq}\E^{\hp}\left[a_{T+1}^{1-2q}\bE_{T}\right]\E^{\frac{2p\hp}{3}}\left[1+\tH_{T}/a_{0}^{2}\right],
\end{align*}
where $(a)$ is by $1-2q=p$, $(b)$ is due to $\frac{2p\hp}{3\left(1-\hp\right)}=\frac{2p(1-p)}{2p+1}<1.$

First, if $q\neq\frac{1}{4}$, we have
\begin{align*}
\E\left[a_{T+1}^{1-2q}\bE_{T}\right] & \leq K_{1}+K_{2}\E\left[\left(\tH_{T}/a_{0}^{2}\right)^{\frac{4q-1}{3}}\right]+K_{4}\E\left[\bD_{T}^{1-2p}\right]\\
 & \overset{(c)}{\leq}K_{1}+K_{2}\E^{\frac{1-2p}{3}}\left[\tH_{T}/a_{0}^{2}\right]+K_{4}\E^{\frac{1-2p}{\hp}}\left[\bD_{T}^{\hp}\right]\\
 & \overset{(d)}{\leq}K_{1}+K_{2}\left(2\sigma^{2}T/a_{0}^{2}\right)^{\frac{1-2p}{3}}+K_{4}\E^{\frac{1-2p}{\hp}}\left[\bE_{T}^{\hp}\right]
\end{align*}
where $(c)$ is by $\frac{4q-1}{3}=\frac{1-2p}{3}\leq1$ and $p\geq\frac{3-\sqrt{7}}{2}\Rightarrow1-2p\leq\frac{3(1-p)}{4-p}=\hp$,
$(d)$ is by $\E\left[\tH_{T}\right]\leq2\sigma^{2}T$ and $\E\left[\bD_{T}^{\hp}\right]\leq\E\left[\bE_{T}^{\hp}\right]$.
Then we know
\begin{align*}
\E\left[\bE_{T}^{\hp}\right] & \leq\E^{\hp}\left[a_{T+1}^{1-2q}\bE_{T}\right]\E^{\frac{2p\hp}{3}}\left[1+\tH_{T}/a_{0}^{2}\right]\\
 & \le\left(K_{1}+K_{2}\left(2\sigma^{2}T/a_{0}^{2}\right)^{\frac{1-2p}{3}}+K_{4}\E^{\frac{1-2p}{\hp}}\left[\bE_{T}^{\hp}\right]\right)^{\hp}\left(1+2\sigma^{2}T/a_{0}^{2}\right)^{\frac{2p\hp}{3}}.
\end{align*}

If $K_{4}\E^{\frac{1-2p}{\hp}}\left[\bE_{T}^{\hp}\right]\leq K_{1}+K_{2}\left(2\sigma^{2}T/a_{0}^{2}\right)^{\frac{1-2p}{3}}$,
we know
\begin{align*}
\E^{\frac{1-2p}{\hp}}\left[\bE_{T}^{\hp}\right] & \leq\frac{K_{1}}{K_{4}}+\frac{K_{2}}{K_{4}}\left(\frac{2\sigma^{2}T}{a_{0}^{2}}\right)^{\frac{1-2p}{3}}\\
\Rightarrow\E\left[\bE_{T}^{\hp}\right] & \leq\left(\frac{K_{1}}{K_{4}}+\frac{K_{2}}{K_{4}}\left(\frac{2\sigma^{2}T}{a_{0}^{2}}\right)^{\frac{1-2p}{3}}\right)^{\frac{\hp}{1-2p}}\\
 & \le\left(\frac{2K_{1}}{K_{4}}\right)^{\frac{\hp}{1-2p}}+\left(\frac{2K_{2}}{K_{4}}\right)^{\frac{\hp}{1-2p}}\left(\frac{2\sigma^{2}T}{a_{0}^{2}}\right)^{\frac{\hp}{3}}.
\end{align*}
If $K_{4}\E^{\frac{1-2p}{\hp}}\left[\bE_{T}^{\hp}\right]\geq K_{1}+K_{2}\left(2\sigma^{2}T/a_{0}^{2}\right)^{\frac{1-2p}{3}}$,
then we know
\begin{align*}
\E\left[\bE_{T}^{\hp}\right] & \leq\left(2K_{4}\right)^{\hp}\E^{1-2p}\left[\bE_{T}^{\hp}\right]\left(1+\frac{2\sigma^{2}T}{a_{0}^{2}}\right)^{\frac{2p\hp}{3}}\\
\Rightarrow\E\left[\bE_{T}^{\hp}\right] & \leq\left(2K_{4}\right)^{\frac{\hp}{2p}}\left(1+\frac{2\sigma^{2}T}{a_{0}^{2}}\right)^{\frac{\hp}{3}}.
\end{align*}
Combining two results, we know when $q\neq\frac{1}{4}$
\begin{align*}
\E\left[\bE_{T}^{\hp}\right] & \leq\left(\frac{2K_{1}}{K_{4}}\right)^{\frac{\hp}{1-2p}}+\left(\frac{2K_{2}}{K_{4}}\right)^{\frac{\hp}{1-2p}}\left(\frac{2\sigma^{2}T}{a_{0}^{2}}\right)^{\frac{\hp}{3}}+\left(2K_{4}\right)^{\frac{\hp}{2p}}\left(1+\frac{2\sigma^{2}T}{a_{0}^{2}}\right)^{\frac{\hp}{3}}\\
 & \leq\left(\frac{2K_{1}}{K_{4}}\right)^{\frac{\hp}{1-2p}}+\left(\left(\frac{2K_{2}}{K_{4}}\right)^{\frac{\hp}{1-2p}}+\left(2K_{4}\right)^{\frac{\hp}{2p}}\right)\left(1+\frac{2\sigma^{2}T}{a_{0}^{2}}\right)^{\frac{\hp}{3}}.
\end{align*}

Following a similar approach, we can prove for $q=\frac{1}{4},$there
is
\begin{align*}
\E\left[\bE_{T}^{\hp}\right] & \leq\left(K_{1}+K_{2}\log\left(1+\frac{2\sigma^{2}T}{a_{0}^{2}}\right)+\frac{K_{4}}{\hp}\log\left(1+\frac{\E\left[\bE_{T}^{\hp}\right]}{b_{0}^{2\hp}}\right)\right)^{\hp}\left(1+\frac{2\sigma^{2}T}{a_{0}^{2}}\right)^{\frac{\hp}{3}}
\end{align*}
Now we use Lemma \ref{lem:inequality-log-bound-g} to get
\begin{align*}
\E\left[\bE_{T}^{\hp}\right] & \leq\left(2K_{1}+2K_{2}\log\left(1+\frac{2\sigma^{2}T}{a_{0}^{2}}\right)+\frac{2K_{4}}{\hp}\log\frac{4K_{4}\left(1+\frac{2\sigma^{2}T}{a_{0}^{2}}\right)^{\frac{\hp}{3}}}{b_{0}^{2\hp}}\right)^{\hp}\left(1+\frac{2\sigma^{2}T}{a_{0}^{2}}\right)^{\frac{\hp}{3}}+b_{0}^{2\hp}\\
 & =\left(2K_{1}+2\left(K_{2}+\frac{K_{4}}{3}\right)\log\left(1+\frac{2\sigma^{2}T}{a_{0}^{2}}\right)+\frac{2K_{4}}{\hp}\log\frac{4K_{4}}{b_{0}^{2\hp}}\right)^{\hp}\left(1+\frac{2\sigma^{2}T}{a_{0}^{2}}\right)^{\frac{\hp}{3}}+b_{0}^{2\hp}.
\end{align*}

Finally, we have
\begin{align*}
\E\left[\bE_{T}^{\hp}\right] & \leq\begin{cases}
\left(\frac{2K_{1}}{K_{4}}\right)^{\frac{\hp}{1-2p}}+\left(\left(\frac{2K_{2}}{K_{4}}\right)^{\frac{\hp}{1-2p}}+\left(2K_{4}\right)^{\frac{\hp}{2p}}\right)\left(1+\frac{2\sigma^{2}T}{a_{0}^{2}}\right)^{\frac{\hp}{3}} & q\neq\frac{1}{4}\\
\left(2K_{1}+2\left(K_{2}+\frac{K_{4}}{3}\right)\log\left(1+\frac{2\sigma^{2}T}{a_{0}^{2}}\right)+\frac{2K_{4}}{\hp}\log\frac{4K_{4}}{b_{0}^{2\hp}}\right)^{\hp}\left(1+\frac{2\sigma^{2}T}{a_{0}^{2}}\right)^{\frac{\hp}{3}}+b_{0}^{2\hp} & q=\frac{1}{4}
\end{cases}.
\end{align*}

\textbf{Case 2:} $\E\left[\bD_{T}^{\hp}\right]\geq\E\left[\bE_{T}^{\hp}\right]$.
In this case, we will finally prove
\begin{align*}
\E\left[\bD_{T}^{\hp}\right] & \leq\left(K_{5}+\left(K_{6}+\frac{K_{7}}{3}\right)\log\left(1+\frac{2\sigma^{2}T}{a_{0}^{2}}\right)+K_{7}\log\frac{K_{8}+K_{9}}{b_{0}}\right)^{\frac{\hp}{1-p}}\left(1+\frac{2\sigma^{2}T}{a_{0}^{2}}\right)^{\frac{\hp}{3}}
\end{align*}
Note that by Holder inequality
\begin{align*}
\E\left[\bD_{T}^{\hp}\right] & =\E\left[a_{T+1}^{\frac{\hp q}{1-p}}\bD_{T}^{\hp}\times a_{T+1}^{-\frac{\hp q}{1-p}}\right]\\
 & \leq\E^{\frac{\hp}{1-p}}\left[a_{T+1}^{q}\bD_{T}^{1-p}\right]\E^{\frac{1-p-\hp}{1-p}}\left[a_{T+1}^{-\frac{\hp q}{1-p-\hp}}\right]\\
 & =\E^{\frac{\hp}{1-p}}\left[a_{T+1}^{q}\bD_{T}^{1-p}\right]\E^{\frac{1-p-\hp}{1-p}}\left[\left(1+\tH_{T}/a_{0}^{2}\right)^{\frac{2\hp q}{3(1-p-\hp)}}\right]\\
 & =\E^{\frac{\hp}{1-p}}\left[a_{T+1}^{q}\bD_{T}^{1-p}\right]\E^{\frac{\hp}{3}}\left[1+\tH_{T}/a_{0}^{2}\right],
\end{align*}
where the last step is by $\frac{2\hp q}{3(1-p-\hp)}=\frac{(1-p)\hp}{3(1-p-\hp)}=1$.
We know
\begin{align*}
\E\left[a_{T+1}^{q}\bD_{T}^{1-p}\right] & \leq K_{5}+K_{6}\E\left[\log\frac{a_{0}^{2}+\tH_{T}}{a_{0}^{2}}\right]+K_{7}\E\left[\log\frac{K_{8}+K_{9}\left(1+\tH_{T}/a_{0}^{2}\right)^{1/3}}{b_{0}}\right]\\
 & \overset{(e)}{\leq}K_{5}+K_{6}\log\frac{a_{0}^{2}+\E\left[\tH_{T}\right]}{a_{0}^{2}}+K_{7}\log\frac{K_{8}+K_{9}\E\left[\left(1+\tH_{T}/a_{0}^{2}\right)^{1/3}\right]}{b_{0}}\\
 & \overset{(f)}{\leq}K_{5}+K_{6}\log\frac{a_{0}^{2}+\E\left[\tH_{T}\right]}{a_{0}^{2}}+K_{7}\log\frac{K_{8}+K_{9}\left(1+\E\left[\tH_{T}\right]/a_{0}^{2}\right)^{1/3}}{b_{0}}\\
 & \overset{(g)}{\leq}K_{5}+K_{6}\log\frac{a_{0}^{2}+2\sigma^{2}T}{a_{0}^{2}}+K_{7}\log\frac{K_{8}+K_{9}\left(1+2\sigma^{2}T/a_{0}^{2}\right)^{1/3}}{b_{0}},
\end{align*}
where $(e)$ is by the concavity of $\log$ function, $(f)$ holds
due to $\E\left[X^{1/3}\right]\leq\E^{1/3}\left[X\right]$ for $X\geq0$,
$(g)$ is by $\E\left[\tH_{T}\right]\leq2\sigma^{2}T$. Then we have
\begin{align*}
\E\left[\bD_{T}^{\hp}\right] & \leq\E^{\frac{\hp}{1-p}}\left[a_{T+1}^{q}\bD_{T}^{1-p}\right]\E^{\frac{\hp}{3}}\left[1+\tH_{T}/a_{0}^{2}\right]\\
 & \le\left(K_{5}+K_{6}\log\frac{a_{0}^{2}+2\sigma^{2}T}{a_{0}^{2}}+K_{7}\log\frac{K_{8}+K_{9}\left(1+2\sigma^{2}T/a_{0}^{2}\right)^{1/3}}{b_{0}}\right)^{\frac{\hp}{1-p}}\left(1+\frac{2\sigma^{2}T}{a_{0}^{2}}\right)^{\frac{\hp}{3}}\\
 & \leq\left(K_{5}+\left(K_{6}+\frac{K_{7}}{3}\right)\log\left(1+\frac{2\sigma^{2}T}{a_{0}^{2}}\right)+K_{7}\log\frac{K_{8}+K_{9}}{b_{0}}\right)^{\frac{\hp}{1-p}}\left(1+\frac{2\sigma^{2}T}{a_{0}^{2}}\right)^{\frac{\hp}{3}}.
\end{align*}

Finally, combining \textbf{Case 1} and \textbf{Case 2} and using (\ref{eq:MS-H-bound-via-E-D}),
we get the desired result and finish the proof
\begin{align*}
 & \E\left[\bH_{T}^{\hp}\right]\\
\leq & 4\max\left\{ \E\left[\bE_{T}^{\hp}\right],\E\left[\bD_{T}^{\hp}\right]\right\} \\
\leq & 4\begin{cases}
\left(\frac{2K_{1}}{K_{4}}\right)^{\frac{\hp}{1-2p}}+\left(\left(\frac{2K_{2}}{K_{4}}\right)^{\frac{\hp}{1-2p}}+\left(2K_{4}\right)^{\frac{\hp}{2p}}\right)\left(1+\frac{2\sigma^{2}T}{a_{0}^{2}}\right)^{\frac{\hp}{3}} & q\neq\frac{1}{4}\\
\left(2K_{1}+2\left(K_{2}+\frac{K_{4}}{3}\right)\log\left(1+\frac{2\sigma^{2}T}{a_{0}^{2}}\right)+\frac{2K_{4}}{\hp}\log\frac{4K_{4}}{b_{0}^{2\hp}}\right)^{\hp}\left(1+\frac{2\sigma^{2}T}{a_{0}^{2}}\right)^{\frac{\hp}{3}}+b_{0}^{2\hp} & q=\frac{1}{4}
\end{cases}\\
 & +4\left(K_{5}+\left(K_{6}+\frac{K_{7}}{3}\right)\log\left(1+\frac{2\sigma^{2}T}{a_{0}^{2}}\right)+K_{7}\log\frac{K_{8}+K_{9}}{b_{0}}\right)^{\frac{\hp}{1-p}}\left(1+\frac{2\sigma^{2}T}{a_{0}^{2}}\right)^{\frac{\hp}{3}}.
\end{align*}
\end{proof}

\section{Analysis of $\protect\algnameold$ for general $p$\label{sec:Algorithm-SG}}

In this section, we give a general analysis for our Algorithm $\algnameold$.
Readers will see $p=\frac{1}{2}$ is a very special corner case.
First we recall the choices of $a_{t}$ and $b_{t}$:
\begin{align*}
a_{t+1} & =(1+\sum_{i=1}^{t}\|\nabla f(x_{i},\xi_{i})\|/a_{0}^{2})^{-2/3},\\
b_{t} & =(b_{0}^{1/p}+\sum_{i=1}^{t}\|d_{i}\|^{2})^{p}/a_{t+1}^{q}
\end{align*}
where $p,q$ satisfy $p+2q=1,p\in\left[\frac{1}{4},\frac{1}{2}\right].$
$a_{0}>0$ and $b_{0}>0$ are absolute constants. Naturally, we have
$a_{1}=1$. We will finally prove the following theorem.
\begin{thm}
\label{thm:SG-rate}Under the assumptions 1-4, by defining $\hp=\frac{2(1-p)}{3}\in\left[\frac{1}{3},\frac{1}{2}\right]$,
we have
\begin{align*}
 & \E\left[\bH_{T}^{\hp}\right]\\
\leq & 4C_{9}\mathds{1}\left[\left(2\sigma^{2}T\right)^{\hp}\leq4C_{9}\right]+4C_{10}\mathds{1}\left[\left(2\sigma^{2}T\right)^{\hp}\leq4C_{10}\right]\\
 & +4\begin{cases}
\left(\frac{2C_{1}}{C_{3}}\right)^{\frac{\hp}{1-2p}}+\left(\left(\frac{2C_{2}}{C_{3}}\right)^{\frac{\hp}{1-2p}}+\left(2C_{3}\right)^{\frac{\hp}{2p}}\right)\left(1+\frac{2\left(2\sigma^{2}T\right)^{\hp}}{a_{0}^{2\hp}}\right)^{\frac{1}{3}} & p\neq\frac{1}{2}\\
\left(C_{1}+\left(\frac{C_{2}}{\hp}+\frac{C_{3}}{\hp}\right)\log\left(1+\frac{\left(2\sigma^{2}T\right)^{\hp}}{\min\left\{ a_{0}^{2\hp}/2,4b_{0}^{2\hp}\right\} }\right)\right)^{\hp}\left(1+\frac{2\left(2\sigma^{2}T\right)^{\hp}}{a_{0}^{2\hp}}\right)^{\frac{1}{3}} & p=\frac{1}{2}
\end{cases}.\\
 & +4\left(C_{4}+\left(3C_{5}+C_{6}\right)\log\frac{a_{0}^{2/3}+2\left(2\sigma^{2}T\right)^{1/3}}{a_{0}^{2/3}}+C_{6}\log\frac{2C_{7}+2C_{8}}{b_{0}}\right)^{\frac{\hp}{1-p}}\\
 & \quad\times\left(1+\frac{2\left(2\sigma^{2}T\right)^{\hp}}{a_{0}^{2\hp}}\right)^{1/3}
\end{align*}
where $C_{i},i\in\left[10\right]$ are some constants only depending
on $a_{0},b_{0},\sigma,\hG,\beta,p,q,F(x_{1})-F^{*}$. To simplify
our final bound, we only indicate the dependency on $\beta$ and $F(x_{1})-F^{*}$
when $\sigma\neq0$ and $T$ is big enough to eliminate $C_{9}$ and
$C_{10}$
\[
\E\left[\bH_{T}^{\hp}\right]=O\left(\left((F(x_{1})-F^{*})^{\frac{\hp}{1-p}}+\beta^{\frac{\hp}{p}}\log^{\frac{\hp}{1-p}}\beta+\beta^{\frac{\hp}{p}}\log^{\frac{\hp}{1-p}}\left(1+\sigma^{2}T\right)\right)(1+\sigma^{2}T)^{\frac{\hp}{3}}\right).
\]
\end{thm}
\begin{rem}
For all $i\in\left[10\right]$, the constant $C_{i}$ will be defined
in the proof that follows.
\end{rem}
Again, by the concavity of $x^{\hp}$, we have the following convergence
theorem, of which the proof is omitted.
\begin{thm}
Under the assumptions 1-4 by defining $\hp=\frac{2(1-p)}{3}\in\left[\frac{1}{3},\frac{1}{2}\right]$,
when $\sigma\neq0$ and $T$ is big enough, we have
\[
\E\left[\|\nabla F(\xout)\|^{2\hp}\right]=O\left((F(x_{1})-F^{*})^{\frac{\hp}{1-p}}+\beta^{\frac{\hp}{p}}\log^{\frac{\hp}{1-p}}\beta+\beta^{\frac{\hp}{p}}\log^{\frac{\hp}{1-p}}\left(1+\sigma^{2}T\right)\right)\left(\frac{1}{T^{\hp}}+\frac{\sigma^{2\hp/3}}{T^{2\hp/3}}\right).
\]
\end{thm}
Here, we give a more explicit convergence dependency for $p=\frac{1}{2}$
used in Theorem \ref{thm:Main-SG-convergence-rate}.
\begin{thm}
Under the assumptions 1-4, when $p=\frac{1}{2}$, by setting $\lambda=\min\left\{ 1,(a_{0}/\hG)^{2}\right\} $(which
is used in $C_{4}$ to $C_{8}$ and $C_{10}$) we get the best dependency
on $\hG$. For simplicity, under the setting $a_{0}=b_{0}=\eta=1$,
we have
\begin{align*}
\E\left[\|\nabla F(\xout)\|^{2/3}\right] & =O\left(\frac{W_{1}\mathds{1}\left[\left(\sigma^{2}T\right)^{1/3}\leq W_{1}\right]}{T^{1/3}}+\left(W_{2}+W_{3}\log^{2/3}\left(1+\sigma^{2}T\right)\right)\left(\frac{1}{T^{1/3}}+\frac{\sigma^{2/9}}{T^{2/9}}\right)\right)
\end{align*}
where $W_{1}=O\big(F(x_{1})-F^{*}+\sigma^{2}+\rn^{2}+\beta\big(1+\rn^{2}\big)\log\big(\beta+\rn^{2}\beta\big)\big)$,
$W_{2}=O\big((F(x_{1})-F^{*})^{2/3}+\sigma^{4/3}+\rn^{4/3}+(1+\rn^{4/3})\beta^{2/3}\log^{2/3}\big(\beta+\rn^{2}\beta\big)\big)$
and $W_{3}=O\big((1+\rn^{4/3})\beta^{2/3}\big)$.
\end{thm}
To start with, we first state the following useful bound for $a_{t}$:
\begin{lem}
\label{lem:SG-a-bound}$\forall t\geq1$, there is
\begin{align*}
a_{t+1}^{-3/2}-a_{t}^{-3/2} & \leq(\hG/a_{0})^{2}.
\end{align*}
\end{lem}
\begin{proof}
\begin{align*}
a_{t+1}^{-3/2}-a_{t}^{-3/2} & =\|\nabla f(x_{t},\xi_{t})\|^{2}/a_{0}^{2}\leq(\hG/a_{0})^{2}.
\end{align*}
\end{proof}

\subsection{Analysis of $\protect\bE_{T}$}

Following a similar approach, we define a random time $\tau$ satisfying

\[
\tau=\max\left\{ \left[T\right],a_{t}\geq C_{0}\right\} 
\]
where 
\[
C_{0}\coloneqq\min\left\{ 1,(a_{0}/\hG)^{4}\right\} .
\]
Note that $\left\{ \tau=t\right\} =\left\{ a_{t}\geq C_{0},a_{t+1}<C_{0}\right\} \in\F_{t}$,
this means $\tau$ is a stopping time. We now prove a useful proposition
of $\tau$:
\begin{lem}
\label{lem:SG-stopping-time}$\forall t\geq\tau+1$, we have
\begin{align*}
a_{t+1}^{-1}-a_{t}^{-1}\leq & 2/3.
\end{align*}
\end{lem}
\begin{proof}
Let $h(y)=y^{2/3}$. Due to the concavity, we know $h(y_{1})-h(y_{2})\leq h'(y_{2})(y_{1}-y_{2})=\frac{2(y_{1}-y_{2})}{3y_{2}^{1/3}}$.
Now we have
\begin{align*}
a_{t+1}^{-1}-a_{t}^{-1} & =(a_{t}^{-3/2}+\|\nabla f(x_{t},\xi_{t})\|^{2}/a_{0}^{2})^{2/3}-(a_{t}^{-3/2})^{2/3}\\
 & \leq\frac{2a_{t}^{1/2}\|\nabla f(x_{t},\xi_{t})\|^{2}}{3a_{0}^{2}}\leq\frac{2a_{t}^{1/2}\hG^{2}}{3a_{0}^{2}}\leq\frac{2}{3}
\end{align*}
where the last step is by $a_{t}\leq a_{\tau+1}<C_{0}\leq(a_{0}/\hG)^{4}$.
\end{proof}

\subsubsection{Bound on $\protect\E\left[\protect\bE_{\tau,3/2-2\ell}\right]$ for
$\ell\in\left[\frac{1}{4},\frac{1}{2}\right]$}

Similar to the analysis of $\algnamenew$, we choose to bound $\E\left[\bE_{\tau,3/2-2\ell}\right]$.
We first prove the following bound on $\E\left[\bE_{\tau,3/2-2\ell}\right]$:
\begin{lem}
\label{lem:SG-E-tau}For any $\ell\in\left[\frac{1}{4},\frac{1}{2}\right]$,
we have
\begin{align*}
\E\left[\bE_{\tau,3/2-2\ell}\right] & \leq\frac{\sigma^{2}+24a_{0}^{2}+4\hG^{2}}{C_{0}^{2\ell-1/2}}+\frac{2\eta^{2}\beta^{2}}{C_{0}^{2\ell-1/2}}\E\left[\sum_{t=1}^{T}\frac{\|d_{t}\|^{2}}{b_{t}^{2}}\right].
\end{align*}
\end{lem}
\begin{proof}
We start from Lemma \ref{lem:vr-inequality},
\begin{align*}
a_{t+1}\|\epsilon_{t}\|^{2} & \leq\|\epsilon_{t}\|^{2}-\|\epsilon_{t+1}\|^{2}+2\|Z_{t+1}\|^{2}\\
 & \quad+2a_{t+1}^{2}\|\nabla f(x_{t+1},\xi_{t+1})-\nabla F(x_{t+1})\|^{2}+M_{t+1}.
\end{align*}
Summing up from $1$ to $\tau-1$ and taking the expectations on both
sides, we obtain
\begin{align*}
\E\left[\bE_{\tau-1,1}\right] & \le\E\bigg[\sum_{t=1}^{\tau-1}\|\epsilon_{t}\|^{2}-\|\epsilon_{t+1}\|^{2}+2\|Z_{t+1}\|^{2}\\
 & \qquad+2a_{t+1}^{2}\|\nabla f(x_{t+1},\xi_{t+1})-\nabla F(x_{t+1})\|^{2}+M_{t+1}\bigg]\\
 & =\E\bigg[\|\epsilon_{1}\|^{2}-\|\epsilon_{\tau}\|^{2}+\sum_{t=1}^{\tau-1}2\|Z_{t+1}\|^{2}\\
 & \qquad+2a_{t+1}^{2}\|\nabla f(x_{t+1},\xi_{t+1})-\nabla F(x_{t+1})\|^{2}+M_{t+1}\bigg]\\
 & \leq\E\bigg[\|\epsilon_{1}\|^{2}-\|\epsilon_{\tau}\|^{2}+\sum_{t=1}^{T}2\|Z_{t+1}\|^{2}\\
 & \qquad+2a_{t+1}^{2}\|\nabla f(x_{t+1},\xi_{t+1})-\nabla F(x_{t+1})\|^{2}+\sum_{t=1}^{\tau-1}M_{t+1}\bigg]
\end{align*}
\begin{align*}
\Rightarrow\E\left[\bE_{\tau-1,1}+\|\epsilon_{\tau}\|^{2}\right] & \leq\sigma^{2}+\E\bigg[\sum_{t=1}^{T}2\|Z_{t+1}\|^{2}\\
 & \qquad+2a_{t+1}^{2}\|\nabla f(x_{t+1},\xi_{t+1})-\nabla F(x_{t+1})\|^{2}+\sum_{t=1}^{\tau-1}M_{t+1}\bigg]
\end{align*}
Because $C_{0}\leq1$, $a_{\tau+1}\leq1$, $2\ell-1/2\geq0$ and $3/2-2\ell\geq0$,
so we have
\begin{align*}
C_{0}^{2\ell-1/2}a_{\tau+1}^{3/2-2\ell} & \leq1.
\end{align*}
Besides, for $t\leq\tau-1$, by the definition of $\tau$, we have
$C_{0}\leq a_{t+1}$, then we know
\[
C_{0}^{2\ell-1/2}a_{t+1}^{3/2-2\ell}\leq a_{t+1}^{2\ell-1/2}a_{t+1}^{3/2-2\ell}=a_{t+1}.
\]
These two results give us
\begin{align*}
C_{0}^{2\ell-1/2}\bE_{\tau,3/2-2\ell} & =C_{0}^{2\ell-1/2}\sum_{t=1}^{\tau}a_{t+1}^{3/2-2\ell}\|\epsilon_{t}\|^{2}\leq\sum_{t=1}^{\tau-1}a_{t+1}\|\epsilon_{t}\|^{2}+\|\epsilon_{\tau}\|^{2}\\
 & =\bE_{\tau-1,1}+\|\epsilon_{\tau}\|^{2},
\end{align*}
which implies
\begin{align*}
\E\left[C_{0}^{2\ell-1/2}\bE_{\tau,3/2-2\ell}\right] & \leq\sigma^{2}+\E\bigg[\sum_{t=1}^{T}2\|Z_{t+1}\|^{2}\\
 & \qquad+2a_{t+1}^{2}\|\nabla f(x_{t+1},\xi_{t+1})-\nabla F(x_{t+1})\|^{2}+\sum_{t=1}^{\tau-1}M_{t+1}\bigg]
\end{align*}
Let$\mathcal{M}_{t}\coloneqq\sum_{i=1}^{t}M_{i}\in\F_{t}$ with $M_{1}=0$.
For $s\leq t$, we know $\E\left[M_{t}|\F_{s}\right]=0$, hence $\mathcal{M}_{t}$
is a martingale. Note that $\tau$ is a bounded stopping time, hence
by optional sampling theorem
\[
\E\left[\sum_{t=1}^{\tau-1}M_{t+1}\right]=\E\left[\mathcal{M}_{\tau}\right]=0.
\]
Now we have
\[
\E\left[C_{0}^{2\ell-1/2}\bE_{\tau,3/2-2\ell}\right]\leq\sigma^{2}+\E\left[\sum_{t=1}^{T}2\|Z_{t+1}\|^{2}+2a_{t+1}^{2}\|\nabla f(x_{t+1},\xi_{t+1})-\nabla F(x_{t+1})\|^{2}\right].
\]
By Lemma \ref{lem:smooth-Z}
\begin{align*}
\E\left[\|Z_{t+1}\|^{2}\mid\F_{t}\right] & \leq\eta^{2}\beta^{2}\frac{\|d_{t}\|^{2}}{b_{t}^{2}}.
\end{align*}
Besides, under our current choice, $a_{t+1}\in\F_{t}$,
\begin{align*}
 & \E\left[a_{t+1}^{2}\|\nabla f(x_{t+1},\xi_{t+1})-\nabla F(x_{t+1})\|^{2}|\F_{t}\right]\\
= & a_{t+1}^{2}\E\left[\|\nabla f(x_{t+1},\xi_{t+1})-\nabla F(x_{t+1})\|^{2}|\F_{t}\right]\\
\leq & a_{t+1}^{2}\E\left[\|\nabla f(x_{t+1},\xi_{t+1})\|^{2}|\F_{t}\right].
\end{align*}
Using these two bounds, we have
\begin{align*}
\E\left[C_{0}^{2\ell-1/2}\bE_{\tau,3/2-2\ell}\right] & \leq\sigma^{2}+\E\left[\sum_{t=1}^{T}2\eta^{2}\beta^{2}\frac{\|d_{t}\|^{2}}{b_{t}^{2}}+2a_{t+1}^{2}\|\nabla f(x_{t+1},\xi_{t+1})\|^{2}\right]\\
 & =\sigma^{2}+\E\left[\sum_{t=1}^{T}2\eta^{2}\beta^{2}\frac{\|d_{t}\|^{2}}{b_{t}^{2}}+2a_{0}^{2}\times\frac{\|\nabla f(x_{t+1},\xi_{t+1})\|^{2}/a_{0}^{2}}{(1+\sum_{i=1}^{t}\|\nabla f(x_{i},\xi_{i})\|^{2}/a_{0}^{2})^{4/3}}\right]\\
 & \leq\sigma^{2}+24a_{0}^{2}+4\widehat{G}^{2}+2\eta^{2}\beta^{2}\E\left[\sum_{t=1}^{T}\frac{\|d_{t}\|^{2}}{b_{t}^{2}}\right],
\end{align*}
where the last inequality holds by Lemma \ref{lem:inequality-4}.
Dividing both sides by $C_{0}^{2\ell-1/2}$ , we get the desired bound
immediately
\begin{align*}
\E\left[\bE_{\tau,3/2-2\ell}\right] & \leq\frac{\sigma^{2}+24a_{0}^{2}+4\widehat{G}^{2}}{C_{0}^{2\ell-1/2}}+\frac{2\eta^{2}\beta^{2}}{C_{0}^{2\ell-1/2}}\E\left[\sum_{t=1}^{T}\frac{\|d_{t}\|^{2}}{b_{t}^{2}}\right].
\end{align*}
\end{proof}

\subsubsection{Bound on $\protect\E\left[\protect\bE_{T,1-2\ell}\right]$ for $\ell\in\left[\frac{1}{4},\frac{1}{2}\right]$}

With the previous result on $\E\left[\bE_{\tau,3/2-2\ell}\right]$,
we can bound $\E\left[\bE_{T,1-2\ell}\right]$.
\begin{lem}
\label{lem:SG-E-T}For any $\ell\in\left[\frac{1}{4},\frac{1}{2}\right]$,
we have
\begin{align*}
\E\left[\bE_{T,1-2\ell}\right] & \leq C_{1}(\ell)+C_{2}(\ell)\begin{cases}
\E\left[\left(\hH_{T}/a_{0}^{2}\right)^{\frac{4\ell-1}{3}}\right] & \ell>\frac{1}{4}\\
\E\left[\log\left(1+\hH_{T}/a_{0}^{2}\right)\right] & \ell=\frac{1}{4}
\end{cases}\\
 & \quad+\E\left[\sum_{t=1}^{T}\left(\frac{\hG^{2}}{a_{0}^{2}C_{0}^{2\ell-1/2}}a_{t+1}^{2\ell}+1\right)6\eta^{2}\beta^{2}\frac{\|d_{t}\|^{2}}{a_{t+1}^{2\ell}b_{t}^{2}}\right],
\end{align*}
where
\begin{align*}
C_{1}(\ell) & \coloneqq3\left(\sigma^{2}+6\hG^{2}+\frac{\hG^{2}\left(\sigma^{2}+24a_{0}^{2}+4\hG^{2}\right)}{a_{0}^{2}C_{0}^{2\ell-1/2}}\right)\\
C_{2}(\ell) & \coloneqq\begin{cases}
\frac{18a_{0}^{2}}{4\ell-1} & \ell>\frac{1}{4}\\
6a_{0}^{2} & \ell=\frac{1}{4}
\end{cases}.
\end{align*}
\end{lem}
\begin{proof}
Starting from Lemma \ref{lem:vr-inequality} as well
\begin{align*}
a_{t+1}\|\epsilon_{t}\|^{2} & \leq\|\epsilon_{t}\|^{2}-\|\epsilon_{t+1}\|^{2}+2\|Z_{t+1}\|^{2}+2a_{t+1}^{2}\|\nabla f(x_{t+1},\xi_{t+1})-\nabla F(x_{t+1})\|^{2}+M_{t+1}.
\end{align*}
Dividing both sides by $a_{t+1}^{2\ell}$ and taking expectations,
we have
\begin{align}
\E\left[a_{t+1}^{1-2\ell}\|\epsilon_{t}\|^{2}\right] & \leq\E\bigg[\frac{\|\epsilon_{t}\|^{2}}{a_{t+1}^{2\ell}}-\frac{\|\epsilon_{t+1}\|^{2}}{a_{t+1}^{2\ell}}+\frac{2}{a_{t+1}^{2\ell}}\|Z_{t+1}\|^{2}\nonumber \\
 & \quad+2a_{t+1}^{2-2\ell}\|\nabla f(x_{t+1},\xi_{t+1})-\nabla F(x_{t+1})\|^{2}+\frac{M_{t+1}}{a_{t+1}^{2\ell}}\bigg].\label{eq:SG-E-bound-1}
\end{align}
Note that under our current choice, $a_{t+1}\in\F_{t}$, hence we
have
\begin{align*}
\E\left[\frac{M_{t+1}}{a_{t+1}^{2\ell}}\right] & =\E\left[\frac{\E\left[M_{t+1}|\F_{t}\right]}{a_{t+1}^{2\ell}}\right]=0;\\
\E\left[\frac{\|Z_{t+1}\|^{2}}{a_{t+1}^{2\ell}}\right] & =\E\left[\frac{\E\left[\|Z_{t+1}\|^{2}|\F_{t}\right]}{a_{t+1}^{2\ell}}\right]\leq\E\left[\eta^{2}\beta^{2}\frac{\|d_{t}\|^{2}}{a_{t+1}^{2\ell}b_{t}^{2}}\right];\\
\E\left[a_{t+1}^{2-2\ell}\|\nabla f(x_{t+1},\xi_{t+1})-\nabla F(x_{t+1})\|^{2}\right] & =\E\left[a_{t+1}^{2-2\ell}\E\left[\|\nabla f(x_{t+1},\xi_{t+1})-\nabla F(x_{t+1})\|^{2}|\F_{t}\right]\right]\\
 & \leq\E\left[a_{t+1}^{2-2\ell}\|\nabla f(x_{t+1},\xi_{t+1})\|^{2}\right],
\end{align*}
where the second bound holds by Lemma \ref{lem:smooth-Z}. Plugging
these three bounds into (\ref{eq:SG-E-bound-1}), we know
\begin{align*}
\E\left[a_{t+1}^{1-2\ell}\|\epsilon_{t}\|^{2}\right] & \leq\E\left[\frac{\|\epsilon_{t}\|^{2}}{a_{t+1}^{2\ell}}-\frac{\|\epsilon_{t+1}\|^{2}}{a_{t+1}^{2\ell}}+2\eta^{2}\beta^{2}\frac{\|d_{t}\|^{2}}{a_{t+1}^{2\ell}b_{t}^{2}}+2a_{t+1}^{2-2\ell}\|\nabla f(x_{t+1},\xi_{t+1})\|^{2}\right].
\end{align*}
Now sum up from $1$ to $T$ to get
\begin{align}
 & \E\left[\bE_{T,1-2\ell}\right]\nonumber \\
\leq & \E\left[\sum_{t=1}^{T}\frac{\|\epsilon_{t}\|^{2}}{a_{t+1}^{2\ell}}-\frac{\|\epsilon_{t+1}\|^{2}}{a_{t+1}^{2\ell}}+2\eta^{2}\beta^{2}\frac{\|d_{t}\|^{2}}{a_{t+1}^{2\ell}b_{t}^{2}}+2a_{t+1}^{2-2\ell}\|\nabla f(x_{t+1},\xi_{t+1})\|^{2}\right]\nonumber \\
\leq & \sigma^{2}+\E\left[\underbrace{\sum_{t=1}^{T}\left(a_{t+1}^{-2\ell}-a_{t}^{-2\ell}\right)\|\epsilon_{t}\|^{2}}_{(i)}+2\eta^{2}\beta^{2}\sum_{t=1}^{T}\frac{\|d_{t}\|^{2}}{a_{t+1}^{2\ell}b_{t}^{2}}+\underbrace{\sum_{t=1}^{T}2a_{t+1}^{2-2\ell}\|\nabla f(x_{t+1},\xi_{t+1})\|^{2}}_{(ii)}\right].\label{eq:SG-E-bound-2}
\end{align}

For $(i)$, we split the time by $\tau$
\begin{align*}
\sum_{t=1}^{T}\left(a_{t+1}^{-2\ell}-a_{t}^{-2\ell}\right)\|\epsilon_{t}\|^{2} & =\sum_{t=1}^{\tau}\left(a_{t+1}^{-2\ell}-a_{t}^{-2\ell}\right)\|\epsilon_{t}\|^{2}+\sum_{t=\tau+1}^{T}\left(a_{t+1}^{-2\ell}-a_{t}^{-2\ell}\right)\|\epsilon_{t}\|^{2}\\
 & \leq\sum_{t=1}^{\tau}\left(a_{t+1}^{-3/2}-a_{t}^{-3/2}\right)a_{t+1}^{3/2-2\ell}\|\epsilon_{t}\|^{2}+\sum_{t=\tau+1}^{T}\left(a_{t+1}^{-1}-a_{t}^{-1}\right)a_{t+1}^{1-2\ell}\|\epsilon_{t}\|^{2}\\
 & \leq\frac{\hG^{2}}{a_{0}^{2}}\sum_{t=1}^{\tau}a_{t+1}^{3/2-2\ell}\|\epsilon_{t}\|^{2}+\sum_{t=\tau+1}^{T}\frac{2}{3}a_{t+1}^{1-2\ell}\|\epsilon_{t}\|^{2}\\
 & \leq\frac{\hG^{2}}{a_{0}^{2}}\sum_{t=1}^{\tau}a_{t+1}^{3/2-2\ell}\|\epsilon_{t}\|^{2}+\sum_{t=1}^{T}\frac{2}{3}a_{t+1}^{1-2\ell}\|\epsilon_{t}\|^{2}\\
 & =\frac{\hG^{2}}{a_{0}^{2}}\bE_{\tau,3/2-2\ell}+\frac{2}{3}\bE_{T,1-2\ell},
\end{align*}
where the second inequality is by Lemma \ref{lem:SG-a-bound} and
Lemma \ref{lem:SG-stopping-time}.

Next, for $(ii)$, we use Lemma \ref{lem:inequality-2} to get
\begin{align*}
 & \sum_{t=1}^{T}2a_{t+1}^{2-2\ell}\|\nabla f(x_{t+1},\xi_{t+1})\|^{2}\\
= & 2a_{0}^{2}\sum_{t=1}^{T}\frac{\|\nabla f(x_{t+1},\xi_{t+1})\|^{2}/a_{0}^{2}}{\left(1+\sum_{i=1}^{t}\|\nabla f(x_{i},\xi_{i})\|^{2}/a_{0}^{2}\right)^{4(1-\ell)/3}}\\
\leq & 2a_{0}^{2}\times\left(\frac{3\hG^{2}}{a_{0}^{2}}+\begin{cases}
\frac{1}{1-4(1-\ell)/3}\left(\frac{\sum_{i=1}^{T}\|\nabla f(x_{i},\xi_{i})\|^{2}}{a_{0}^{2}}\right)^{1-4(1-\ell)/3} & 4(1-\ell)/3<1\\
\log\left(1+\frac{\sum_{i=1}^{T}\|\nabla f(x_{i},\xi_{i})\|^{2}}{a_{0}^{2}}\right) & 4(1-\ell)/3=1
\end{cases}\right)\\
= & 6\hG^{2}+\begin{cases}
\frac{6a_{0}^{2}}{4\ell-1}\left(\hH_{T}/a_{0}^{2}\right)^{\frac{4\ell-1}{3}} & \ell>\frac{1}{4}\\
2a_{0}^{2}\log\left(1+\hH_{T}/a_{0}^{2}\right) & \ell=\frac{1}{4}
\end{cases}.
\end{align*}

Plugging these two bounds into (\ref{eq:SG-E-bound-2}), we have
\begin{align*}
\E\left[\bE_{T,1-2\ell}\right] & \leq\sigma^{2}+6\hG^{2}+\E\left[\frac{\widehat{G}^{2}}{a_{0}^{2}}\bE_{\tau,3/2-2\ell}+\frac{2}{3}\bE_{T,1-2\ell}+2\eta^{2}\beta^{2}\sum_{t=1}^{T}\frac{\|d_{t}\|^{2}}{a_{t+1}^{2\ell}b_{t}^{2}}\right]\\
 & \quad+\begin{cases}
\frac{6a_{0}^{2}}{4\ell-1}\left(\hH_{T}/a_{0}^{2}\right)^{\frac{4\ell-1}{3}} & \ell>\frac{1}{4}\\
2a_{0}^{2}\log\left(1+\hH_{T}/a_{0}^{2}\right) & \ell=\frac{1}{4}
\end{cases}.
\end{align*}
Thus
\begin{align*}
\E\left[\bE_{T,1-2\ell}\right] & \le3\left(\sigma^{2}+6\widehat{G}^{2}\right)+\frac{3\widehat{G}^{2}}{a_{0}^{2}}\E\left[\bE_{\tau,3/2-2\ell}\right]+\begin{cases}
\frac{18a_{0}^{2}}{4\ell-1}\E\left[\left(\hH_{T}/a_{0}^{2}\right)^{\frac{4\ell-1}{3}}\right] & \ell>\frac{1}{4}\\
6a_{0}^{2}\E\left[\log\left(1+\hH_{T}/a_{0}^{2}\right)\right] & \ell=\frac{1}{4}
\end{cases}\\
 & \quad+6\eta^{2}\beta^{2}\E\left[\sum_{t=1}^{T}\frac{\|d_{t}\|^{2}}{a_{t+1}^{2\ell}b_{t}^{2}}\right]
\end{align*}
Plugging the bound on $\E\left[\bE_{\tau,3/2-2\ell}\right]$ in Lemma
\ref{lem:SG-E-tau}, we finally get
\begin{align*}
\E\left[\bE_{T,1-2\ell}\right] & \leq\underbrace{3\left(\sigma^{2}+6\widehat{G}^{2}+\frac{\widehat{G}^{2}\left(\sigma^{2}+24a_{0}^{2}+4\widehat{G}^{2}\right)}{a_{0}^{2}C_{0}^{2\ell-1/2}}\right)}_{C_{1}(\ell)}+C_{2}(\ell)\begin{cases}
\E\left[\left(\hH_{T}/a_{0}^{2}\right)^{\frac{4\ell-1}{3}}\right] & \ell>\frac{1}{4}\\
\E\left[\log\left(1+\hH_{T}/a_{0}^{2}\right)\right] & \ell=\frac{1}{4}
\end{cases}\\
 & \quad+\E\left[\sum_{t=1}^{T}\left(\frac{\widehat{G}^{2}}{a_{0}^{2}C_{0}^{2\ell-1/2}}a_{t+1}^{2\ell}+1\right)6\eta^{2}\beta^{2}\frac{\|d_{t}\|^{2}}{a_{t+1}^{2\ell}b_{t}^{2}}\right],
\end{align*}
where
\[
C_{2}(\ell)\coloneqq\begin{cases}
\frac{18a_{0}^{2}}{4\ell-1} & \ell>\frac{1}{4}\\
6a_{0}^{2} & \ell=\frac{1}{4}
\end{cases}.
\]
\end{proof}

\subsubsection{Bound on $\protect\E\left[\protect\bE_{T,1/2}\right]$}

The following bound on $\E\left[\bE_{T,1/2}\right]$ will be useful
when we bound $\bD_{T}$.
\begin{cor}
\label{cor:SG-E-1/2}We have
\begin{align*}
\E\left[\bE_{T,1/2}\right] & \leq C_{1}\left(1/4\right)+C_{2}\left(1/4\right)\E\left[\log\left(1+\hH_{T}/a_{0}^{2}\right)\right]\\
 & \qquad+\E\left[\sum_{t=1}^{T}\left(\frac{\widehat{G}^{2}}{a_{0}^{2}}a_{t+1}^{1/2}+1\right)6\eta^{2}\beta^{2}\frac{\|d_{t}\|^{2}}{a_{t+1}^{1/2}b_{t}^{2}}\right].
\end{align*}
\end{cor}
\begin{proof}
Take $\ell=\frac{1}{4}$ in Lemma \ref{lem:SG-E-T}.
\end{proof}

\subsubsection{Bound on $\protect\E\left[a_{T+1}^{1-2q}\protect\bE_{T}\right]$}
\begin{lem}
\label{lem:SG-E-final-bound}Given $p+2q=1$,$p\in\left[\frac{1}{4},\frac{1}{2}\right]$,
we have
\begin{align*}
\E\left[a_{T+1}^{1-2q}\bE_{T}\right] & \leq\begin{cases}
C_{1}+C_{2}\E\left[\left(\hH_{T}/a_{0}^{2}\right)^{\frac{4q-1}{3}}\right]+C_{3}\E\left[\bD_{T}^{1-2p}\right] & q>\frac{1}{4}\\
C_{1}+C_{2}\E\left[\log\left(1+\hH_{T}/a_{0}^{2}\right)\right]+C_{3}\E\left[\log\left(1+\frac{\bD_{T}}{b_{0}^{2}}\right)\right] & q=\frac{1}{4}
\end{cases},
\end{align*}
where
\begin{align*}
C_{1} & \coloneqq C_{1}(q)\\
C_{2} & \coloneqq C_{2}(q)\\
C_{3} & \coloneqq\begin{cases}
\left(\frac{\widehat{G}^{2}}{a_{0}^{2}C_{0}^{2q-1/2}}+1\right)\frac{6\eta^{2}\beta^{2}}{4q-1} & q>\frac{1}{4}\\
\left(\frac{\widehat{G}^{2}}{a_{0}^{2}}+1\right)6\eta^{2}\beta^{2} & q=\frac{1}{4}
\end{cases}.
\end{align*}
\end{lem}
\begin{proof}
When $p\neq\frac{1}{2}\Leftrightarrow q>\frac{1}{4}$, by Lemma \ref{lem:SG-E-T},
taking $\ell=q$, we know
\begin{align*}
\E\left[\bE_{T,1-2q}\right] & \leq C_{1}(q)+C_{2}(q)\E\left[\left(\hH_{T}/a_{0}^{2}\right)^{\frac{4q-1}{3}}\right]+\E\left[\sum_{t=1}^{T}\left(\frac{\widehat{G}^{2}}{a_{0}^{2}C_{0}^{2q-1/2}}a_{t+1}^{2q}+1\right)6\eta^{2}\beta^{2}\frac{\|d_{t}\|^{2}}{a_{t+1}^{2q}b_{t}^{2}}\right]\\
 & \leq C_{1}(q)+C_{2}(q)\E\left[\left(\hH_{T}/a_{0}^{2}\right)^{\frac{4q-1}{3}}\right]+\left(\frac{\widehat{G}^{2}}{a_{0}^{2}C_{0}^{2q-1/2}}+1\right)6\eta^{2}\beta^{2}\E\left[\sum_{t=1}^{T}\frac{\|d_{t}\|^{2}}{a_{t+1}^{2q}b_{t}^{2}}\right]\\
 & \overset{(a)}{=}C_{1}(q)+C_{2}(q)\E\left[\left(\hH_{T}/a_{0}^{2}\right)^{\frac{4q-1}{3}}\right]\\
 & \qquad\qquad\qquad\qquad+\left(\frac{\widehat{G}^{2}}{a_{0}^{2}C_{0}^{2q-1/2}}+1\right)6\eta^{2}\beta^{2}\times\E\left[\sum_{t=1}^{T}\frac{\|d_{t}\|^{2}}{\left(b_{0}^{1/p}+\sum_{i=1}^{t}\|d_{i}\|^{2}\right)^{2p}}\right]\\
 & \overset{(b)}{\leq}C_{1}(q)+C_{2}(q)\E\left[\left(\hH_{T}/a_{0}^{2}\right)^{\frac{4q-1}{3}}\right]+\left(\frac{\widehat{G}^{2}}{a_{0}^{2}C_{0}^{2q-1/2}}+1\right)6\eta^{2}\beta^{2}\E\left[\frac{\bD_{T}^{1-2p}}{1-2p}\right]\\
 & \overset{(c)}{=}C_{1}(q)+C_{2}(q)\E\left[\left(\hH_{T}/a_{0}^{2}\right)^{\frac{4q-1}{3}}\right]+\left(\frac{\widehat{G}^{2}}{a_{0}^{2}C_{0}^{2q-1/2}}+1\right)\frac{6\eta^{2}\beta^{2}}{4q-1}\E\left[\bD_{T}^{1-2p}\right],
\end{align*}
where $(a)$ is by
\begin{align*}
a_{t+1}^{2q}b_{t}^{2} & =a_{t+1}^{2q}\frac{\left(b_{0}^{1/p}+\sum_{i=1}^{t}\|d_{i}\|^{2}\right)^{2p}}{a_{t+1}^{2q}}=\left(b_{0}^{1/p}+\sum_{i=1}^{t}\|d_{i}\|^{2}\right)^{2p},
\end{align*}
$(b)$ is by Lemma \ref{lem:inequality-1}, $(c)$ is by $1-2p=4q-1$.

When $p=\frac{1}{2}\Leftrightarrow q=\frac{1}{4}$, by a similar argument,
we have
\[
\E\left[\bE_{T,1-2q}\right]\leq C_{1}(q)+C_{2}(q)\E\left[\log\left(1+\hH_{T}/a_{0}^{2}\right)\right]+\left(\frac{\widehat{G}^{2}}{a_{0}^{2}}+1\right)6\eta^{2}\beta^{2}\E\left[\log\left(1+\frac{\bD_{T}}{b_{0}^{2}}\right)\right].
\]
Now we can define
\[
C_{3}\coloneqq\begin{cases}
\left(\frac{\widehat{G}^{2}}{a_{0}^{2}C_{0}^{2q-1/2}}+1\right)\frac{6\eta^{2}\beta^{2}}{4q-1} & q>\frac{1}{4}\\
\left(\frac{\widehat{G}^{2}}{a_{0}^{2}}+1\right)6\eta^{2}\beta^{2} & q=\frac{1}{4}
\end{cases}.
\]
The final step is by noticing for $1-2q=p>0$
\begin{align*}
\bE_{T,1-2q} & =\sum_{t=1}^{T}a_{t+1}^{1-2q}\|\epsilon_{t}\|^{2}\geq a_{T+1}^{1-2q}\sum_{t=1}^{T}\|\epsilon_{t}\|^{2}=a_{T+1}^{1-2q}\bE_{T}.
\end{align*}
\end{proof}

\subsection{Analysis of $\protect\bD_{T}$}

We will prove the following bound
\begin{lem}
\label{lem:D-final-bound-g-n}Given $p+2q=1$,$p\in\left[\frac{1}{4},\frac{1}{2}\right]$,
we have
\begin{align*}
\E\left[a_{T+1}^{q}\bD_{T}^{1-p}\right] & \leq C_{4}+C_{5}\E\left[\log\frac{a_{0}^{2}+\hH_{T}}{a_{0}^{2}}\right]+C_{6}\E\left[\log\frac{C_{7}+C_{8}\left(1+\hH_{T}/a_{0}^{2}\right)^{1/3}}{b_{0}}\right]
\end{align*}
where
\begin{align*}
C_{4} & \coloneqq b_{0}^{\frac{1}{p}-1}+\frac{2}{\eta}\left(F(x_{1})-F^{*}\right)+\frac{\lambda C_{1}\left(1/4\right)}{\eta\beta_{\max}},C_{5}\coloneqq\frac{\lambda C_{2}\left(1/4\right)}{\eta\beta_{\max}},\\
C_{6} & \coloneqq\frac{\left(C_{7}+C_{8}\right)^{\frac{1}{p}-1}}{1-p},C_{7}\coloneqq\left(1+\frac{6\lambda\widehat{G}^{2}}{a_{0}^{2}}\right)\eta\beta_{\max},C_{8}\coloneqq\left(\frac{1}{\lambda}+6\lambda\right)\eta\beta_{\max},\\
\lambda & >0\text{ can be any number}.
\end{align*}
\end{lem}
\begin{proof}
The same as before, we start from Lemma \ref{lem:f-value-analysis}
\[
\E\left[a_{T+1}^{q}\bD_{T}^{1-p}\right]\leq b_{0}^{\frac{1}{p}-1}+\frac{2}{\eta}\left(F(x_{1})-F^{*}\right)+\E\left[\sum_{t=1}^{T}\left(\eta\beta_{\max}+\frac{\eta\beta_{\max}}{a_{t+1}^{1/2}\lambda}-b_{t}\right)\frac{\|d_{t}\|^{2}}{b_{t}^{2}}\right]+\frac{\lambda\E\left[\bE_{T,1/2}\right]}{\eta\beta_{\max}}
\]
where $\lambda>0$ is used to reduce the order of $\hG$ in the final
bound. In the proof of the general case , we don't choose $\lambda$
explicitly anymore. Plugging in the bound on $\E\left[\bE_{T,1/2}\right]$
in Corollary \ref{cor:SG-E-1/2}, we know
\begin{align}
\E\left[a_{T+1}^{q}\bD_{T}^{1-p}\right] & \leq b_{0}^{\frac{1}{p}-1}+\frac{2}{\eta}\left(F(x_{1})-F^{*}\right)+\frac{\lambda C_{1}\left(1/4\right)}{\eta\beta_{\max}}+\frac{\lambda C_{2}\left(1/4\right)}{\eta\beta_{\max}}\E\left[\log\frac{a_{0}^{2}+\hH_{T}}{a_{0}^{2}}\right]\nonumber \\
 & \quad+\E\left[\sum_{t=1}^{T}\left(\left(1+\frac{6\lambda\widehat{G}^{2}}{a_{0}^{2}}\right)\eta\beta_{\max}+\left(\frac{1}{\lambda}+6\lambda\right)\frac{\eta\beta_{\max}}{a_{t+1}^{1/2}}-b_{t}\right)\frac{\|d_{t}\|^{2}}{b_{t}^{2}}\right]\nonumber \\
 & =C_{4}+C_{5}\E\left[\log\frac{a_{0}^{2}+\hH_{T}}{a_{0}^{2}}\right]\nonumber \\
 & \quad+\E\left[\underbrace{\sum_{t=1}^{T}\left(\left(1+\frac{6\lambda\widehat{G}^{2}}{a_{0}^{2}}\right)\eta\beta_{\max}+\left(\frac{1}{\lambda}+6\lambda\right)\frac{\eta\beta_{\max}}{a_{t+1}^{1/2}}-b_{t}\right)\frac{\|d_{t}\|^{2}}{b_{t}^{2}}}_{(i)}\right].\label{eq:SG-bound-D}
\end{align}
Applying Lemma \ref{lem:residual-bound} to $(i)$, we get
\begin{align*}
(i) & \leq\frac{\left(\left(1+\frac{6\lambda\widehat{G}^{2}}{a_{0}^{2}}+\frac{1}{\lambda}+6\lambda\right)\eta\beta_{\max}\right)^{\frac{1}{p}-1}}{1-p}\\
 & \qquad\times\log\frac{\left(1+\frac{6\lambda\widehat{G}^{2}}{a_{0}^{2}}\right)\eta\beta_{\max}+\left(\frac{1}{\lambda}+6\lambda\right)\eta\beta_{\max}\left(1+\hH_{T}/a_{0}^{2}\right)^{1/3}}{b_{0}}\\
 & =C_{6}\log\frac{C_{7}+C_{8}\left(1+\hH_{T}/a_{0}^{2}\right)^{1/3}}{b_{0}}
\end{align*}
By using this bound to (\ref{eq:SG-bound-D}), the proof is completed.
\end{proof}

\subsection{Combine the Bounds and the Final Proof.}

From Lemma \ref{lem:SG-E-final-bound}, we have
\begin{align*}
\E\left[a_{T+1}^{1-2q}\bE_{T}\right] & \leq\begin{cases}
C_{1}+C_{2}\E\left[\left(\hH_{T}/a_{0}^{2}\right)^{\frac{4q-1}{3}}\right]+C_{3}\E\left[\bD_{T}^{1-2p}\right] & q>\frac{1}{4}\\
C_{1}+C_{2}\E\left[\log\left(1+\hH_{T}/a_{0}^{2}\right)\right]+C_{3}\E\left[\log\left(1+\frac{\bD_{T}}{b_{0}^{2}}\right)\right] & q=\frac{1}{4}
\end{cases}
\end{align*}
From Lemma \ref{lem:D-final-bound-g-n}, we have
\begin{align*}
\E\left[a_{T+1}^{q}\bD_{T}^{1-p}\right] & \leq C_{4}+C_{5}\E\left[\log\frac{a_{0}^{2}+\hH_{T}}{a_{0}^{2}}\right]+C_{6}\E\left[\log\frac{C_{7}+C_{8}\left(1+\hH_{T}/a_{0}^{2}\right)^{1/3}}{b_{0}}\right]
\end{align*}

Now let 
\[
\hp=\frac{2(1-p)}{3}\in\left[\frac{1}{3},\frac{1}{2}\right].
\]
Apply Lemma \ref{lem:decomposition}, we have
\begin{align}
\E\left[\bH_{T}^{\hp}\right] & \leq2^{\hp+1}\max\left\{ \E\left[\bE_{T}^{\hp}\right],\E\left[\bD_{T}^{\hp}\right]\right\} \leq4\max\left\{ \E\left[\bE_{T}^{\hp}\right],\E\left[\bD_{T}^{\hp}\right]\right\} ,\label{eq:SG-H-bound-via-E-D}
\end{align}
Now we can give the final proof of Theorem \ref{thm:SG-rate}.

\begin{proof}
First, we have
\begin{align}
\E\left[\hH_{T}^{\hp}\right] & =\E\left[\left(\sum_{i=1}^{T}\|\nabla f(x_{i},\xi_{i})\|^{2}\right)^{\hp}\right]\nonumber \\
 & \leq\E\left[\left(\sum_{i=1}^{T}2\|\nabla F(x_{i})\|^{2}+2\|\nabla f(x_{i},\xi_{i})-\nabla F(x_{i})\|^{2}\right)^{\hp}\right]\nonumber \\
 & =\E\left[\left(2\bH_{T}+2\sum_{i=1}^{T}\|\nabla f(x_{i},\xi_{i})-\nabla F(x_{i})\|^{2}\right)^{\hp}\right]\nonumber \\
 & \leq\E\left[2^{\hp}\bH_{T}^{\hp}+\left(2\sum_{i=1}^{T}\|\nabla f(x_{i},\xi_{i})-\nabla F(x_{i})\|^{2}\right)^{\hp}\right]\nonumber \\
 & =2^{\hp}\E\left[\bH_{T}^{\hp}\right]+\E\left[\left(2\sum_{i=1}^{T}\|\nabla f(x_{i},\xi_{i})-\nabla F(x_{i})\|^{2}\right)^{\hp}\right]\nonumber \\
 & \leq2^{\hp}\E\left[\bH_{T}^{\hp}\right]+\E^{\hp}\left[\left(2\sum_{i=1}^{T}\|\nabla f(x_{i},\xi_{i})-\nabla F(x_{i})\|^{2}\right)\right]\nonumber \\
 & \leq2^{\hp}\E\left[\bH_{T}^{\hp}\right]+\left(2\sigma^{2}T\right)^{\hp}\leq2^{2\hp+1}\max\left\{ \E\left[\bE_{T}^{\hp}\right],\E\left[\bD_{T}^{\hp}\right]\right\} +\left(2\sigma^{2}T\right)^{\hp}\nonumber \\
 & \leq4\max\left\{ \E\left[\bE_{T}^{\hp}\right],\E\left[\bD_{T}^{\hp}\right]\right\} +\left(2\sigma^{2}T\right)^{\hp}.\label{eq:SG-hH-bound}
\end{align}
Now we consider following two cases:

\textbf{Case 1:} $\E\left[\bE_{T}^{\hp}\right]\geq\E\left[\bD_{T}^{\hp}\right]$.
In this case, we will finally prove
\begin{align*}
\E\left[\bE_{T}^{\hp}\right] & \leq\begin{cases}
\left(\frac{2C_{1}}{C_{3}}\right)^{\frac{\hp}{1-2p}}+\left(\left(\frac{2C_{2}}{C_{3}}\right)^{\frac{\hp}{1-2p}}+\left(2C_{3}\right)^{\frac{\hp}{2p}}\right)\left(1+\frac{2\left(2\sigma^{2}T\right)^{\hp}}{a_{0}^{2\hp}}\right)^{\frac{1}{3}}\\
\qquad\qquad\qquad\qquad\qquad\qquad\qquad+C_{9}\mathds{1}\left[\left(2\sigma^{2}T\right)^{\hp}\leq4C_{9}\right] & q\neq\frac{1}{4}\\
\left(C_{1}+\left(\frac{C_{2}}{\hp}+\frac{C_{3}}{\hp}\right)\log\left(1+\frac{\left(2\sigma^{2}T\right)^{\hp}}{\min\left\{ a_{0}^{2\hp}/2,4b_{0}^{2\hp}\right\} }\right)\right)^{\hp}\left(1+\frac{2\left(2\sigma^{2}T\right)^{\hp}}{a_{0}^{2\hp}}\right)^{\frac{1}{3}}\\
\qquad\qquad\qquad\qquad\qquad\qquad\qquad+C_{9}\mathds{1}\left[\left(2\sigma^{2}T\right)^{\hp}\leq4C_{9}\right] & q=\frac{1}{4}
\end{cases}.
\end{align*}
where $C_{9}$ is a constant. Note that by Holder inequality
\begin{align*}
\E\left[\bE_{T}^{\hp}\right] & =\E\left[a_{T+1}^{(1-2q)\hp}\bE_{T}^{\hp}\times a_{T+1}^{-(1-2q)\hp}\right]\\
 & \leq\E^{\hp}\left[a_{T+1}^{1-2q}\bE_{T}\right]\E^{1-\hp}\left[a_{T+1}^{\frac{-(1-2q)\hp}{1-\hp}}\right]\\
 & =\E^{\hp}\left[a_{T+1}^{1-2q}\bE_{T}\right]\E^{1-\hp}\left[(1+\hH_{T}/a_{0}^{2})^{\frac{2(1-2q)\hp}{3(1-\hp)}}\right]\\
 & \overset{(a)}{=}\E^{\hp}\left[a_{T+1}^{1-2q}\bE_{T}\right]\E^{1-\hp}\left[(1+\hH_{T}/a_{0}^{2})^{\frac{2p\hp}{3(1-\hp)}}\right]\\
 & \overset{(b)}{\leq}\E^{\hp}\left[a_{T+1}^{1-2q}\bE_{T}\right]\E^{\frac{2p}{3}}\left[(1+\hH_{T}/a_{0}^{2})^{\hp}\right]\\
 & \leq\E^{\hp}\left[a_{T+1}^{1-2q}\bE_{T}\right]\E^{\frac{2p}{3}}\left[1+\left(\hH_{T}/a_{0}^{2}\right)^{\hp}\right]
\end{align*}
where $(a)$ is by $1-2q=p$, $(b)$ is due to $\frac{2p}{3(1-\hp)}=\frac{2p}{1+2p}<1$.

First, if $q\neq\frac{1}{4}$, we have 
\begin{align*}
\E\left[a_{T+1}^{1-2q}\bE_{T}\right] & \leq C_{1}+C_{2}\E\left[\left(\hH_{T}/a_{0}^{2}\right)^{\frac{4q-1}{3}}\right]+C_{3}\E\left[\bD_{T}^{1-2p}\right]\\
 & \overset{(c)}{\leq}C_{1}+C_{2}\E^{\frac{1-2p}{3\hp}}\left[\left(\hH_{T}/a_{0}^{2}\right)^{\hp}\right]+C_{3}\E^{\frac{1-2p}{\hp}}\left[\bD_{T}^{\hp}\right]\\
 & \overset{(d)}{\leq}C_{1}+C_{2}\left(\frac{4\E\left[\bE_{T}^{\hp}\right]+\left(2\sigma^{2}T\right)^{\hp}}{a_{0}^{2\hp}}\right)^{\frac{1-2p}{3\hp}}+C_{3}\E^{\frac{1-2p}{\hp}}\left[\bE_{T}^{\hp}\right],
\end{align*}
where $(c)$ is by $\frac{4q-1}{3}=\frac{1-2p}{3}\leq\frac{2-2p}{3}=\hp$
and $p\geq\frac{1}{4}\Rightarrow1-2p\leq\frac{2-2p}{3}=\hp$, $(d)$
is by (\ref{eq:SG-hH-bound}) and $\E\left[\bD_{T}^{\hp}\right]\leq\E\left[\bE_{T}^{\hp}\right]$.
Then we know
\begin{align*}
\E\left[\bE_{T}^{\hp}\right] & \leq\E^{\hp}\left[a_{T+1}^{1-2q}\bE_{T}\right]\E^{\frac{2p}{3}}\left[1+\left(\hH_{T}/a_{0}^{2}\right)^{\hp}\right]\\
 & \le\left(C_{1}+C_{2}\left(\frac{4\E\left[\bE_{T}^{\hp}\right]+\left(2\sigma^{2}T\right)^{\hp}}{a_{0}^{2\hp}}\right)^{\frac{1-2p}{3\hp}}+C_{3}\E^{\frac{1-2p}{\hp}}\left[\bE_{T}^{\hp}\right]\right)^{\hp}\\
 & \quad\times\left(1+\frac{4\E\left[\bE_{T}^{\hp}\right]+\left(2\sigma^{2}T\right)^{\hp}}{a_{0}^{2\hp}}\right)^{\frac{2p}{3}}.
\end{align*}

If $4\E\left[\bE_{T}^{\hp}\right]\leq\left(2\sigma^{2}T\right)^{\hp}$,
we will get
\begin{align*}
\E\left[\bE_{T}^{\hp}\right] & \leq\left(C_{1}+C_{2}\left(\frac{2\left(2\sigma^{2}T\right)^{\hp}}{a_{0}^{2\hp}}\right)^{\frac{1-2p}{3\hp}}+C_{3}\E^{\frac{1-2p}{\hp}}\left[\bE_{T}^{\hp}\right]\right)^{\hp}\left(1+\frac{2\left(2\sigma^{2}T\right)^{\hp}}{a_{0}^{2\hp}}\right)^{\frac{2p}{3}}.
\end{align*}
If $C_{3}\E^{\frac{1-2p}{\hp}}\left[\bE_{T}^{\hp}\right]\leq C_{1}+C_{2}\left(\frac{2\left(2\sigma^{2}T\right)^{\hp}}{a_{0}^{2\hp}}\right)^{\frac{1-2p}{3\hp}}$,
we have
\begin{align*}
\E^{\frac{1-2p}{\hp}}\left[\bE_{T}^{\hp}\right] & \leq\frac{C_{1}}{C_{3}}+\frac{C_{2}}{C_{3}}\left(\frac{2\left(2\sigma^{2}T\right)^{\hp}}{a_{0}^{2\hp}}\right)^{\frac{1-2p}{3\hp}}\\
\Rightarrow\E\left[\bE_{T}^{\hp}\right] & \leq\left(\frac{C_{1}}{C_{3}}+\frac{C_{2}}{C_{3}}\left(\frac{2\left(2\sigma^{2}T\right)^{\hp}}{a_{0}^{2\hp}}\right)^{\frac{1-2p}{3\hp}}\right)^{\frac{\hp}{1-2p}}\\
 & \leq\left(\frac{2C_{1}}{C_{3}}\right)^{\frac{\hp}{1-2p}}+\left(\frac{2C_{2}}{C_{3}}\right)^{\frac{\hp}{1-2p}}\left(\frac{2\left(2\sigma^{2}T\right)^{\hp}}{a_{0}^{2\hp}}\right)^{\frac{1}{3}}.
\end{align*}
If $C_{3}\E^{\frac{1-2p}{\hp}}\left[\bE_{T}^{\hp}\right]\geq C_{1}+C_{2}\left(\frac{2\left(2\sigma^{2}T\right)^{\hp}}{a_{0}^{2\hp}}\right)^{\frac{1-2p}{3\hp}}$,
we have
\begin{align*}
\E\left[\bE_{T}^{\hp}\right] & \leq\left(2C_{3}\E^{\frac{1-2p}{\hp}}\left[\bE_{T}^{\hp}\right]\right)^{\hp}\left(1+\frac{2\left(2\sigma^{2}T\right)^{\hp}}{a_{0}^{2\hp}}\right)^{\frac{2p}{3}}\\
 & =\left(2C_{3}\right)^{\hp}\E^{1-2p}\left[\bE_{T}^{\frac{2(1-p)}{3}}\right]\left(1+\frac{2\left(2\sigma^{2}T\right)^{\hp}}{a_{0}^{2\hp}}\right)^{\frac{2p}{3}}\\
\Rightarrow\E\left[\bE_{T}^{\hp}\right] & \leq\left(2C_{3}\right)^{\frac{\hp}{2p}}\left(1+\frac{2\left(2\sigma^{2}T\right)^{\hp}}{a_{0}^{2\hp}}\right)^{\frac{1}{3}}.
\end{align*}
Combining two cases, we know under $4\E\left[\bE_{T}^{\hp}\right]\leq\left(2\sigma^{2}T\right)^{\hp}$
\begin{align*}
\E\left[\bE_{T}^{\hp}\right] & \leq\left(\frac{2C_{1}}{C_{3}}\right)^{\frac{\hp}{1-2p}}+\left(\frac{2C_{2}}{C_{3}}\right)^{\frac{\hp}{1-2p}}\left(\frac{2\left(2\sigma^{2}T\right)^{\hp}}{a_{0}^{2\hp}}\right)^{\frac{1}{3}}+\left(2C_{3}\right)^{\frac{\hp}{2p}}\left(1+\frac{2\left(2\sigma^{2}T\right)^{\hp}}{a_{0}^{2\hp}}\right)^{\frac{1}{3}}\\
 & \leq\left(\frac{2C_{1}}{C_{3}}\right)^{\frac{\hp}{1-2p}}+\left(\left(\frac{2C_{2}}{C_{3}}\right)^{\frac{\hp}{1-2p}}+\left(2C_{3}\right)^{\frac{\hp}{2p}}\right)\left(1+\frac{2\left(2\sigma^{2}T\right)^{\hp}}{a_{0}^{2\hp}}\right)^{\frac{1}{3}}.
\end{align*}

Now if $4\E\left[\bE_{T}^{\hp}\right]\geq\left(2\sigma^{2}T\right)^{\hp}$,
then we have
\begin{align}
\E\left[\bE_{T}^{\hp}\right] & \le\left(C_{1}+C_{2}\left(\frac{8\E\left[\bE_{T}^{\hp}\right]}{a_{0}^{2\hp}}\right)^{\frac{1-2p}{3\hp}}+C_{3}\E^{\frac{1-2p}{\hp}}\left[\bE_{T}^{\hp}\right]\right)^{\hp}\left(1+\frac{8\E\left[\bE_{T}^{\hp}\right]}{a_{0}^{2\hp}}\right)^{\frac{2p}{3}}\nonumber \\
 & \leq\left(C_{1}^{\hp}+C_{2}^{\hp}\left(\frac{8\E\left[\bE_{T}^{\hp}\right]}{a_{0}^{2\hp}}\right)^{\frac{1-2p}{3}}+C_{3}^{\hp}\E^{1-2p}\left[\bE_{T}^{\hp}\right]\right)\left(1+\frac{8\E\left[\bE_{T}^{\hp}\right]}{a_{0}^{2\hp}}\right)^{\frac{2p}{3}}.\label{eq:C_9-g-n-1}
\end{align}
We claim there is a constant $C_{9}$ such that $\E\left[\bE_{T}^{\hp}\right]\leq C_{9}$
because the highest order of $\E\left[\bE_{T}^{\hp}\right]$ is only
$1-2p+\frac{2p}{3}=1-\frac{4p}{3}<1$. Here we give the order of $C_{9}$
directly without proof
\[
C_{9}=O\left(a_{0}^{2\hp}+\left(\frac{C_{1}}{C_{3}}\right)^{\frac{\hp}{1-2p}}+\left(C_{2}^{\frac{3\hp}{2}}+C_{3}^{\frac{3\hp}{4p}}\right)\frac{1}{a_{0}^{\hp}}\right).
\]

Hence, when $q\neq\frac{1}{4}$, we finally have
\[
\E\left[\bE_{T}^{\hp}\right]\leq\left(\frac{2C_{1}}{C_{3}}\right)^{\frac{\hp}{1-2p}}+\left(\left(\frac{2C_{2}}{C_{3}}\right)^{\frac{\hp}{1-2p}}+\left(2C_{3}\right)^{\frac{\hp}{2p}}\right)\left(1+\frac{2\left(2\sigma^{2}T\right)^{\hp}}{a_{0}^{2\hp}}\right)^{\frac{1}{3}}+C_{9}\mathds{1}\left[\left(2\sigma^{2}T\right)^{\hp}\leq4C_{9}\right].
\]
Following a similar approach, we can prove for $q=\frac{1}{4},$there
is
\[
\E\left[\bE_{T}^{\hp}\right]\leq\left(C_{1}+\left(\frac{C_{2}}{\hp}+\frac{C_{3}}{\hp}\right)\log\left(1+\frac{\left(2\sigma^{2}T\right)^{\hp}}{\min\left\{ a_{0}^{2\hp}/2,4b_{0}^{2\hp}\right\} }\right)\right)^{\hp}\left(1+\frac{2\left(2\sigma^{2}T\right)^{\hp}}{a_{0}^{2\hp}}\right)^{\frac{1}{3}}+C_{9},
\]
where
\[
C_{9}=O\left(C_{1}^{1/2}+\left(C_{2}^{1/2}+C_{3}^{1/2}\right)\log^{1/2}\frac{C_{2}+C_{3}}{a_{0}^{2\hp}b_{0}^{\hp}}+a_{0}^{2\hp}+a_{0}^{3\hp}+a_{0}^{\hp}b_{0}^{2\hp}\right).
\]

Finally, we have
\begin{align*}
\E\left[\bE_{T}^{\hp}\right] & \leq\begin{cases}
\left(\frac{2C_{1}}{C_{3}}\right)^{\frac{\hp}{1-2p}}+\left(\left(\frac{2C_{2}}{C_{3}}\right)^{\frac{\hp}{1-2p}}+\left(2C_{3}\right)^{\frac{\hp}{2p}}\right)\left(1+\frac{2\left(2\sigma^{2}T\right)^{\hp}}{a_{0}^{2\hp}}\right)^{\frac{1}{3}}\\
\qquad\qquad\qquad\qquad\qquad\qquad\qquad+C_{9}\mathds{1}\left[\left(2\sigma^{2}T\right)^{\hp}\leq4C_{9}\right] & q\neq\frac{1}{4}\\
\left(C_{1}+\left(\frac{C_{2}}{\hp}+\frac{C_{3}}{\hp}\right)\log\left(1+\frac{\left(2\sigma^{2}T\right)^{\hp}}{\min\left\{ a_{0}^{2\hp}/2,4b_{0}^{2\hp}\right\} }\right)\right)^{\hp}\left(1+\frac{2\left(2\sigma^{2}T\right)^{\hp}}{a_{0}^{2\hp}}\right)^{\frac{1}{3}}\\
\qquad\qquad\qquad\qquad\qquad\qquad\qquad+C_{9}\mathds{1}\left[\left(2\sigma^{2}T\right)^{\hp}\leq4C_{9}\right] & q=\frac{1}{4}
\end{cases}.
\end{align*}

\textbf{Case 2:} $\E\left[\bE_{T}^{\hp}\right]\leq\E\left[\bD_{T}^{\hp}\right]$.
In this case, we will finally prove
\begin{align*}
\E\left[\bD_{T}^{\hp}\right] & \leq\left(C_{4}+\left(3C_{5}+C_{6}\right)\log\frac{a_{0}^{2/3}+2\left(2\sigma^{2}T\right)^{1/3}}{a_{0}^{2/3}}+C_{6}\log\frac{2C_{7}+2C_{8}}{b_{0}}\right)^{\frac{\hp}{1-p}}\\
 & \quad\times\left(1+\frac{2\left(2\sigma^{2}T\right)^{\hp}}{a_{0}^{2\hp}}\right)^{1/3}+C_{10}.
\end{align*}
where $C_{10}$ is a constant. Note that by Holder inequality
\begin{align*}
\E\left[\bD_{T}^{\hp}\right] & =\E\left[a_{T+1}^{\frac{q\hp}{1-p}}\bD_{T}^{\hp}\times a_{T+1}^{-\frac{q\hp}{1-p}}\right]\\
 & \leq\E^{\frac{\hp}{1-p}}\left[a_{T+1}^{q}\bD_{T}^{1-p}\right]\E^{\frac{1-p-\hp}{1-p}}\left[a_{T+1}^{-\frac{q\hp}{1-p-\hp}}\right]\\
 & =\E^{\frac{\hp}{1-p}}\left[a_{T+1}^{q}\bD_{T}^{1-p}\right]\E^{\frac{1-p-\hp}{1-p}}\left[\left(1+\hH_{T}/a_{0}^{2}\right)^{\frac{2q\hp}{3(1-p-\hp)}}\right]\\
 & \overset{(e)}{\leq}\E^{\frac{\hp}{1-p}}\left[a_{T+1}^{q}\bD_{T}^{1-p}\right]\E^{\frac{1}{3}}\left[\left(1+\hH_{T}/a_{0}^{2}\right)^{\hp}\right]\\
 & \leq\E^{\frac{\hp}{1-p}}\left[a_{T+1}^{q}\bD_{T}^{1-p}\right]\E^{\frac{1}{3}}\left[1+\left(\hH_{T}/a_{0}^{2}\right)^{\hp}\right]
\end{align*}
where $(e)$ is by $\frac{2q}{3(1-p-\hp)}=\frac{1-p}{3(1-p-\hp)}=1$.
We know
\begin{align*}
 & \E\left[a_{T+1}^{q}\bD_{T}^{1-p}\right]\\
\leq & C_{4}+C_{5}\E\left[\log\frac{a_{0}^{2}+\hH_{T}}{a_{0}^{2}}\right]+C_{6}\E\left[\log\frac{C_{7}+C_{8}\left(1+\hH_{T}/a_{0}^{2}\right)^{1/3}}{b_{0}}\right]\\
 & =C_{4}+\frac{C_{5}}{\hp}\E\left[\log\left(\frac{a_{0}^{2}+\hH_{T}}{a_{0}^{2}}\right)^{\hp}\right]+\frac{C_{6}}{3\hp}\E\left[\log\left(\frac{C_{7}+C_{8}\left(1+\hH_{T}/a_{0}^{2}\right)^{1/3}}{b_{0}}\right)^{3\hp}\right]\\
\overset{(f)}{\leq} & C_{4}+\frac{C_{5}}{\hp}\E\left[\log\frac{a_{0}^{2\hp}+\hH_{T}^{\hp}}{a_{0}^{2\hp}}\right]+\frac{C_{6}}{3\hp}\E\left[\log\frac{\left(2C_{7}\right)^{3\hp}+\left(2C_{8}\right)^{3\hp}\left(1+\left(\hH_{T}/a_{0}^{2}\right)^{\hp}\right)}{b_{0}^{3\hp}}\right]\\
\overset{(g)}{\leq} & C_{4}+\frac{C_{5}}{\hp}\log\frac{a_{0}^{2\hp}+\E\left[\hH_{T}^{\hp}\right]}{a_{0}^{2\hp}}+\frac{C_{6}}{3\hp}\log\frac{\left(2C_{7}\right)^{3\hp}+\left(2C_{8}\right)^{3\hp}\left(1+\E\left[\hH_{T}^{\hp}\right]/a_{0}^{2\hp}\right)}{b_{0}^{3\hp}}\\
\overset{(h)}{\leq} & C_{4}+\frac{C_{5}}{\hp}\log\frac{a_{0}^{2\hp}+4\E\left[\bD_{T}^{\hp}\right]+\left(2\sigma^{2}T\right)^{\hp}}{a_{0}^{2\hp}}\\
 & \quad+\frac{C_{6}}{3\hp}\log\frac{\left(2C_{7}\right)^{3\hp}+\left(2C_{8}\right)^{3\hp}\left(1+\frac{4\E\left[\bD_{T}^{\hp}\right]+\left(2\sigma^{2}T\right)^{\hp}}{a_{0}^{2\hp}}\right)}{b_{0}^{3\hp}}
\end{align*}
where $(f)$ is by $(x+y)^{p}\leq x^{p}+y^{p}$,$\left(x+y\right)^{q}\leq(2x)^{q}+(2y)^{q}$
for $0\leq x,y,0\leq p\leq1,q\geq0$, $(g)$ holds by the concavity
of $\log$ function, $(h)$ is due to (\ref{eq:SG-hH-bound}) and
$\E\left[\bE_{T}^{\hp}\right]\leq\E\left[\bD_{T}^{\hp}\right]$. Then
we know
\begin{align*}
\E\left[\bD_{T}^{\hp}\right] & \leq\E^{\frac{\hp}{1-p}}\left[a_{T+1}^{q}\bD_{T}^{1-p}\right]\E^{\frac{1}{3}}\left[1+\left(\hH_{T}/a_{0}^{2}\right)^{\hp}\right]\\
 & \le\left(C_{4}+\frac{C_{5}}{\hp}\log\frac{a_{0}^{2\hp}+4\E\left[\bD_{T}^{\hp}\right]+\left(2\sigma^{2}T\right)^{\hp}}{a_{0}^{2\hp}}\right.\\
 & \quad\left.+\frac{C_{6}}{3\hp}\log\frac{\left(2C_{7}\right)^{3\hp}+\left(2C_{8}\right)^{3\hp}\left(1+\frac{4\E\left[\bD_{T}^{\hp}\right]+\left(2\sigma^{2}T\right)^{\hp}}{a_{0}^{2\hp}}\right)}{b_{0}^{3\hp}}\right)^{\frac{\hp}{1-p}}\\
 & \quad\times\left(1+\frac{4\E\left[\bD_{T}^{\hp}\right]+\left(2\sigma^{2}T\right)^{\hp}}{a_{0}^{2\hp}}\right)^{1/3}.
\end{align*}

If $4\E\left[\bD_{T}^{\hp}\right]\leq\left(2\sigma^{2}T\right)^{\hp}$,
we will get
\begin{align*}
\E\left[\bD_{T}^{\hp}\right] & \leq\left(C_{4}+\frac{C_{5}}{\hp}\log\frac{a_{0}^{2\hp}+2\left(2\sigma^{2}T\right)^{\hp}}{a_{0}^{2\hp}}+\frac{C_{6}}{3\hp}\log\frac{\left(2C_{7}\right)^{3\hp}+\left(2C_{8}\right)^{3\hp}\left(1+\frac{2\left(2\sigma^{2}T\right)^{\hp}}{a_{0}^{2\hp}}\right)}{b_{0}^{3\hp}}\right)^{\frac{\hp}{1-p}}\\
 & \quad\times\left(1+\frac{2\left(2\sigma^{2}T\right)^{\hp}}{a_{0}^{2\hp}}\right)^{1/3}\\
 & \leq\left(C_{4}+\left(\frac{C_{5}}{\hp}+\frac{C_{6}}{3\hp}\right)\log\frac{a_{0}^{2\hp}+2\left(2\sigma^{2}T\right)^{\hp}}{a_{0}^{2\hp}}+\frac{C_{6}}{3\hp}\log\frac{\left(2C_{7}\right)^{3\hp}+\left(2C_{8}\right)^{3\hp}}{b_{0}^{3\hp}}\right)^{\frac{\hp}{1-p}}\\
 & \quad\times\left(1+\frac{2\left(2\sigma^{2}T\right)^{\hp}}{a_{0}^{2\hp}}\right)^{1/3}\\
 & \leq\left(C_{4}+\left(3C_{5}+C_{6}\right)\log\frac{a_{0}^{2/3}+2\left(2\sigma^{2}T\right)^{1/3}}{a_{0}^{2/3}}+C_{6}\log\frac{2C_{7}+2C_{8}}{b_{0}}\right)^{\frac{\hp}{1-p}}\\
 & \quad\times\left(1+\frac{2\left(2\sigma^{2}T\right)^{\hp}}{a_{0}^{2\hp}}\right)^{1/3}.
\end{align*}

If $4\E\left[\bD_{T}^{\hp}\right]\geq\left(2\sigma^{2}T\right)^{\hp}$,
we have
\begin{align}
\E\left[\bD_{T}^{\hp}\right] & \leq\left(C_{4}+\frac{C_{5}}{\hp}\log\frac{a_{0}^{2\hp}+8\E\left[\bD_{T}^{\hp}\right]}{a_{0}^{2\hp}}+\frac{C_{6}}{3\hp}\log\frac{\left(2C_{7}\right)^{3\hp}+\left(2C_{8}\right)^{3\hp}\left(1+\frac{8\E\left[\bD_{T}^{\hp}\right]}{a_{0}^{2\hp}}\right)}{b_{0}^{3\hp}}\right)^{\frac{\hp}{1-p}}\nonumber \\
 & \quad\times\left(1+\frac{8\E\left[\bD_{T}^{\hp}\right]}{a_{0}^{2\hp}}\right)^{1/3}.\label{eq:C_10-g-n}
\end{align}
which implies there is a constant $C_{10}$ such that $\E\left[\bD_{T}^{\hp}\right]\leq C_{10}$.
Here we give the order of $C_{10}$ directly without proof
\[
C_{10}=O\left(a_{0}^{2\hp}+a_{0}^{3\hp}+C_{4}+C_{6}\log\frac{C_{7}+C_{8}}{b_{0}}+(C_{5}+C_{6})\log\frac{C_{5}+C_{6}}{a_{0}^{3\hp}}\right)
\]

Combining these two results, we know
\begin{align*}
\E\left[\bD_{T}^{\hp}\right] & \leq\left(C_{4}+\left(3C_{5}+C_{6}\right)\log\frac{a_{0}^{2/3}+2\left(2\sigma^{2}T\right)^{1/3}}{a_{0}^{2/3}}+C_{6}\log\frac{2C_{7}+2C_{8}}{b_{0}}\right)^{\frac{\hp}{1-p}}\\
 & \quad\times\left(1+\frac{2\left(2\sigma^{2}T\right)^{\hp}}{a_{0}^{2\hp}}\right)^{1/3}+C_{10}\mathds{1}\left[\left(2\sigma^{2}T\right)^{\hp}\leq4C_{10}\right].
\end{align*}

Finally, combining \textbf{Case 1} and \textbf{Case 2} and using \ref{eq:SG-H-bound-via-E-D},
we get the desired result and the finish the proof 
\begin{align*}
 & \E\left[\bH_{T}^{\hp}\right]\\
\leq & 4\max\left\{ \E\left[\bE_{T}^{\hp}\right],\E\left[\bD_{T}^{\hp}\right]\right\} \\
\leq & 4C_{9}\mathds{1}\left[\left(2\sigma^{2}T\right)^{\hp}\leq4C_{9}\right]+4C_{10}\mathds{1}\left[\left(2\sigma^{2}T\right)^{\hp}\leq4C_{10}\right]\\
 & +4\begin{cases}
\left(\frac{2C_{1}}{C_{3}}\right)^{\frac{\hp}{1-2p}}+\left(\left(\frac{2C_{2}}{C_{3}}\right)^{\frac{\hp}{1-2p}}+\left(2C_{3}\right)^{\frac{\hp}{2p}}\right)\left(1+\frac{2\left(2\sigma^{2}T\right)^{\hp}}{a_{0}^{2\hp}}\right)^{\frac{1}{3}} & q\neq\frac{1}{4}\\
\left(C_{1}+\left(\frac{C_{2}}{\hp}+\frac{C_{3}}{\hp}\right)\log\left(1+\frac{\left(2\sigma^{2}T\right)^{\hp}}{\min\left\{ a_{0}^{2\hp}/2,4b_{0}^{2\hp}\right\} }\right)\right)^{\hp}\left(1+\frac{2\left(2\sigma^{2}T\right)^{\hp}}{a_{0}^{2\hp}}\right)^{\frac{1}{3}} & q=\frac{1}{4}
\end{cases}.\\
 & +4\left(C_{4}+\left(3C_{5}+C_{6}\right)\log\frac{a_{0}^{2/3}+2\left(2\sigma^{2}T\right)^{1/3}}{a_{0}^{2/3}}+C_{6}\log\frac{2C_{7}+2C_{8}}{b_{0}}\right)^{\frac{\hp}{1-p}}\\
 & \quad\times\left(1+\frac{2\left(2\sigma^{2}T\right)^{\hp}}{a_{0}^{2\hp}}\right)^{1/3}
\end{align*}
\end{proof}

\section{Algorithm $\protect\algnamena$ and its analysis for general $p$
\label{sec:Algorithm-NA}}

Algorithm $\algnamena$ is shown in Algorithm \ref{alg:fully-adaptive-na}.
To highlight the differences with $\algnameold$ and $\algnamenew$,
we set $a_{t}$ only based on the time round $t$, not using the stochastic
gradients. This is the reason that the convergence of this algorithm
does not depend on bounded stochastic gradients or bounded stochastic
gradients differences assumptions. Moreover, the requirement of $p\in\left(0,\frac{1}{2}\right]$
is also more relaxed compared with our previous algorithms.

\begin{algorithm}[h]
\caption{$\protect\algnamena$}
\label{alg:fully-adaptive-na}

\textbf{Input:} Initial point $x_{1}\in\R^{d}$

\textbf{Parameters: }$a_{0}>\sqrt{\frac{2}{3}},b_{0},\eta,p\in\left(0,\frac{1}{2}\right],p+2q=1$

Sample $\xi_{1}\sim\domxi,d_{1}=\nabla f(x_{1},\xi_{1})$

\textbf{for} $t=1,\cdots,T$ \textbf{do:}

$\quad$$a_{t+1}=\left(1+t/a_{0}^{2}\right)^{-\frac{2}{3}}$

$\quad$$b_{t}=(b_{0}^{1/p}+\sum_{i=1}^{t}\|d_{i}\|^{2})^{p}/a_{t+1}^{q}$

$\quad$$x_{t+1}=x_{t}-\frac{\eta}{b_{t}}d_{t}$

$\quad$Sample $\xi_{t+1}\sim\domxi$

$\quad$$d_{t+1}=\nabla f(x_{t+1},\xi_{t+1})+(1-a_{t+1})(d_{t}-\nabla f(x_{t},\xi_{t+1}))$

\textbf{end for}

\textbf{Output $\xout=x_{t}$} where $t\sim\mathrm{Uniform}\left(\left[T\right]\right)$.
\end{algorithm}

Now we give the main convergence result, Theorem \ref{thm:NA-rate},
of $\algnamena$. As we discussed before, it can achieve the rate
$\widetilde{O}(1/T^{3})$ under the weakest assumptions 1-3, however,
with losing the adaptivity to the variance parameter $\sigma$ as
a tradeoff.
\begin{thm}
\label{thm:NA-rate}Under the assumptions 1-3, by defining $\hp=1-p\in\left[\frac{1}{2},1\right)$,
we have (omitting the dependency on $\eta,a_{0}\text{ and }b_{0}$)
\begin{align*}
\E\left[\bH_{T}^{\hp}\right] & =O\left(\left(F(x_{1})-F^{*}+\beta^{\frac{\hp}{p}}\log\left(\beta T\right)+\sigma^{2}\log T+\sigma^{2\hp}\right)T^{\frac{\hp}{3}}\right).
\end{align*}
\end{thm}
By combining the above theorem with the concavity of $x^{\hp}$,
we give the following convergence guarantee omitting the proof:
\begin{thm}
\label{thm:NA-rate-2}There is
\[
\E\left[\|\nabla F(\xout)\|^{2\hp}\right]=O\left(\frac{F(x_{1})-F^{*}+\beta^{\frac{\hp}{p}}\log\left(\beta T\right)+\sigma^{2}\log T+\sigma^{2\hp}}{T^{\frac{2\hp}{3}}}\right).
\]
\end{thm}
Note that $2\hp\geq1$, hence the criterion, $\E\left[\|\nabla F(\xout)\|^{2\hp}\right]$,
used in Theorem \ref{thm:NA-rate-2} is strictly stronger than $\E\left[\|\nabla F(\xout)\|\right]$.
In the following sections, we will give a proof of Theorem \ref{thm:NA-rate}.

\subsection{Bound on $\protect\E\left[\protect\bE_{T,1/2}\right]$}
\begin{lem}
\label{lem:NA-E-T-1/2}Given $p+2q=1$, $p\in\left(0,\frac{1}{2}\right]$,
we have
\[
\E\left[\bE_{T,1/2}\right]\leq\frac{\sigma^{2}\left(1+2a_{0}^{2}\log\left(1+T/a_{0}^{2}\right)\right)+2\eta^{2}\beta^{2}\E\left[\sum_{t=1}^{T}\frac{\|d_{t}\|^{2}}{a_{t+1}^{1/2}b_{t}^{2}}\right]}{1-2/(3a_{0}^{2})}.
\]
\end{lem}
\begin{proof}
We start from Lemma \ref{lem:vr-inequality},
\begin{align*}
a_{t+1}\|\epsilon_{t}\|^{2} & \leq\|\epsilon_{t}\|^{2}-\|\epsilon_{t+1}\|^{2}+2\|Z_{t+1}\|^{2}+2a_{t+1}^{2}\|\nabla f(x_{t+1},\xi_{t+1})-\nabla F(x_{t+1})\|^{2}+M_{t+1}.
\end{align*}
Dividing both sides by $a_{t+1}^{1/2}$, summing up from $1$ to $T$
and taking the expectations on both sides, we obtain
\begin{align*}
 & \E\left[\bE_{T,1/2}\right]\\
\le & \E\left[\sum_{t=1}^{T}\frac{\|\epsilon_{t}\|^{2}-\|\epsilon_{t+1}\|^{2}+2\|Z_{t+1}\|^{2}+2a_{t+1}^{2}\|\nabla f(x_{t+1},\xi_{t+1})-\nabla F(x_{t+1})\|^{2}+M_{t+1}}{a_{t+1}^{1/2}}\right]\\
\le & \sigma^{2}+\E\bigg[\sum_{t=1}^{T}\left(a_{t+1}^{-1}-a_{t}^{-1}\right)a_{t+1}^{1/2}\|\epsilon_{t}\|^{2}+\frac{2}{a_{t+1}^{1/2}}\|Z_{t+1}\|^{2}\\
 & \qquad\qquad+2a_{t+1}^{3/2}\|\nabla f(x_{t+1},\xi_{t+1})-\nabla F(x_{t+1})\|^{2}+\frac{M_{t+1}}{a_{t+1}^{1/2}}\bigg]
\end{align*}
Because $a_{t+1}$is not random, we know
\begin{align*}
\E\left[\frac{2}{a_{t+1}^{1/2}}\|Z_{t+1}\|^{2}\right] & \leq\E\left[\frac{2\eta^{2}\beta^{2}\|d_{t}\|^{2}}{a_{t+1}^{1/2}b_{t}^{2}}\right],\\
\E\left[2a_{t+1}^{3/2}\|\nabla f(x_{t+1},\xi_{t+1})-\nabla F(x_{t+1})\|^{2}\right] & \leq2a_{t+1}^{3/2}\sigma^{2},\\
\E\left[\frac{M_{t+1}}{a_{t+1}^{1/2}}\right] & =0,
\end{align*}
where the first inequality is by Lemma \ref{lem:smooth-Z}. Besides,
by the concavity of $x^{2/3}$ and $a_{0}>\sqrt{\frac{2}{3}}$, we
know
\begin{align*}
a_{t+1}^{-1}-a_{t}^{-1} & =\left(1+t/a_{0}^{2}\right)^{2/3}-\left(1+\left(t-1\right)/a_{0}^{2}\right)^{2/3}\\
 & \leq\frac{2}{3a_{0}^{2}\left(1+\left(t-1\right)/a_{0}^{2}\right)^{1/3}}\leq\frac{2}{3a_{0}^{2}}<1.
\end{align*}
Then we have
\begin{align*}
\E\left[\bE_{T,1/2}\right] & \leq\sigma^{2}+\E\left[\frac{2}{3a_{0}^{2}}\bE_{T,1/2}+\sum_{t=1}^{T}\frac{2\eta^{2}\beta^{2}\|d_{t}\|^{2}}{a_{t+1}^{1/2}b_{t}^{2}}+2a_{t+1}^{3/2}\sigma^{2}\right]\\
\Rightarrow\E\left[\bE_{T,1/2}\right] & \leq\frac{\sigma^{2}\left(1+2\sum_{t=1}^{T}a_{t+1}^{3/2}\right)+2\eta^{2}\beta^{2}\E\left[\sum_{t=1}^{T}\frac{\|d_{t}\|^{2}}{a_{t+1}^{1/2}b_{t}^{2}}\right]}{1-2/(3a_{0}^{2})}.
\end{align*}
Note that
\[
\sum_{t=1}^{T}a_{t+1}^{3/2}=\sum_{t=1}^{T}\frac{1}{1+t/a_{0}^{2}}\leq a_{0}^{2}\log\left(1+T/a_{0}^{2}\right).
\]
So we know
\[
\E\left[\bE_{T,1/2}\right]\leq\frac{\sigma^{2}\left(1+2a_{0}^{2}\log\left(1+T/a_{0}^{2}\right)\right)+2\eta^{2}\beta^{2}\E\left[\sum_{t=1}^{T}\frac{\|d_{t}\|^{2}}{a_{t+1}^{1/2}b_{t}^{2}}\right]}{1-2/(3a_{0}^{2})}.
\]
\end{proof}

\subsection{Bound on $\protect\E\left[\protect\bE_{T}\right]$}
\begin{lem}
\label{lem:NA-E-final-bound}Given $p+2q=1$, $p\in\left(0,\frac{1}{2}\right]$,
we have
\[
\E\left[\bE_{T}\right]\leq\frac{6a_{0}^{2}\sigma^{2}\left(1+T/a_{0}^{2}\right)^{1/3}}{1-2/(3a_{0}^{2})}+\frac{2\eta^{2}\beta^{2}(1+T/a_{0}^{2})^{\frac{2p}{3}}}{1-2/(3a_{0}^{2})}\begin{cases}
\frac{\E\left[\bD_{T}^{1-2p}\right]}{1-2p} & p\neq\frac{1}{2}\\
\E\left[\log\left(1+\frac{\bD_{T}}{b_{0}^{2}}\right)\right] & p=\frac{1}{2}
\end{cases}.
\]
\end{lem}
\begin{proof}
We start from Lemma \ref{lem:vr-inequality},
\begin{align*}
a_{t+1}\|\epsilon_{t}\|^{2} & \leq\|\epsilon_{t}\|^{2}-\|\epsilon_{t+1}\|^{2}+2\|Z_{t+1}\|^{2}+2a_{t+1}^{2}\|\nabla f(x_{t+1},\xi_{t+1})-\nabla F(x_{t+1})\|^{2}+M_{t+1}.
\end{align*}
Dividing both sides by $a_{t+1}$, summing up from $1$ to $T$ and
taking the expectations on both sides, we obtain
\begin{align*}
 & \E\left[\bE_{T}\right]\\
\le & \E\left[\sum_{t=1}^{T}\frac{\|\epsilon_{t}\|^{2}-\|\epsilon_{t+1}\|^{2}+2\|Z_{t+1}\|^{2}+2a_{t+1}^{2}\|\nabla f(x_{t+1},\xi_{t+1})-\nabla F(x_{t+1})\|^{2}+M_{t+1}}{a_{t+1}}\right]\\
\le & \sigma^{2}+\E\bigg[\sum_{t=1}^{T}\underbrace{\left(a_{t+1}^{-1}-a_{t}^{-1}\right)}_{\leq2/(3a_{0}^{2})}\|\epsilon_{t}\|^{2}+\frac{2}{a_{t+1}}\|Z_{t+1}\|^{2}\\
 & \qquad\qquad+2a_{t+1}\|\nabla f(x_{t+1},\xi_{t+1})-\nabla F(x_{t+1})\|^{2}+\frac{M_{t+1}}{a_{t+1}}\bigg]\\
\leq & \sigma^{2}+\E\bigg[\frac{2}{3a_{0}^{2}}\bE_{T}+\sum_{t=1}^{T}\frac{2}{a_{t+1}}\|Z_{t+1}\|^{2}\\
 & \qquad\qquad+2a_{t+1}\|\nabla f(x_{t+1},\xi_{t+1})-\nabla F(x_{t+1})\|^{2}+\frac{M_{t+1}}{a_{t+1}}\bigg].
\end{align*}
Because $a_{t+1}$is not random, we know
\begin{align*}
\E\left[\frac{2}{a_{t+1}}\|Z_{t+1}\|^{2}\right] & \leq\E\left[\frac{2\eta^{2}\beta^{2}\|d_{t}\|^{2}}{a_{t+1}b_{t}^{2}}\right],\\
\E\left[2a_{t+1}\|\nabla f(x_{t+1},\xi_{t+1})-\nabla F(x_{t+1})\|^{2}\right] & \leq2a_{t+1}\sigma^{2},\\
\E\left[\frac{M_{t+1}}{a_{t+1}}\right] & =0,
\end{align*}
where the first inequality is by Lemma \ref{lem:smooth-Z}. Then we
know
\begin{align*}
\E\left[\bE_{T}\right] & \leq\sigma^{2}+\E\left[\frac{2}{3a_{0}^{2}}\bE_{T}+\sum_{t=1}^{T}\frac{2\eta^{2}\beta^{2}\|d_{t}\|^{2}}{a_{t+1}b_{t}^{2}}+2a_{t+1}\sigma^{2}\right]\\
\Rightarrow\E\left[\bE_{T}\right] & \leq\frac{\sigma^{2}\left(1+2\sum_{t=1}^{T}a_{t+1}\right)+2\eta^{2}\beta^{2}\E\left[\sum_{t=1}^{T}\frac{\|d_{t}\|^{2}}{a_{t+1}b_{t}^{2}}\right]}{1-2/(3a_{0}^{2})}.
\end{align*}
Note that there is
\begin{align*}
\sum_{t=1}^{T}\frac{\|d_{t}\|^{2}}{a_{t+1}b_{t}^{2}} & =\sum_{t=1}^{T}\frac{\|d_{t}\|^{2}}{a_{t+1}^{1-2q}\left(b_{0}^{1/p}+\bD_{t}\right)^{2p}}\overset{(a)}{=}\sum_{t=1}^{T}\frac{\|d_{t}\|^{2}}{a_{t+1}^{p}\left(b_{0}^{1/p}+\bD_{t}\right)^{2p}}\\
 & \leq(1+T/a_{0}^{2})^{\frac{2p}{3}}\sum_{t=1}^{T}\frac{\|d_{t}\|^{2}}{\left(b_{0}^{1/p}+\bD_{t}\right)^{2p}}\\
 & \overset{(b)}{\leq}(1+T/a_{0}^{2})^{\frac{2p}{3}}\begin{cases}
\frac{\bD_{T}^{1-2p}}{1-2p} & p\neq\frac{1}{2}\\
\log\left(1+\frac{\bD_{T}}{b_{0}^{2}}\right) & p=\frac{1}{2}
\end{cases},
\end{align*}
where $(a)$ is by $1-2q=p$, $(b)$ is by Lemma \ref{lem:inequality-1}.
Besides
\[
\sum_{t=1}^{T}a_{t+1}=\sum_{t=1}^{T}\frac{1}{\left(1+t/a_{0}^{2}\right)^{2/3}}\leq3a_{0}^{2}\left(1+T/a_{0}^{2}\right)^{1/3}-3a_{0}^{2}<3a_{0}^{2}\left(1+T/a_{0}^{2}\right)^{1/3}-2.
\]
So we know
\[
\E\left[\bE_{T}\right]\leq\frac{6a_{0}^{2}\sigma^{2}\left(1+T/a_{0}^{2}\right)^{1/3}}{1-2/(3a_{0}^{2})}+\frac{2\eta^{2}\beta^{2}(1+T/a_{0}^{2})^{\frac{2p}{3}}}{1-2/(3a_{0}^{2})}\begin{cases}
\frac{\E\left[\bD_{T}^{1-2p}\right]}{1-2p} & p\neq\frac{1}{2}\\
\E\left[\log\left(1+\frac{\bD_{T}}{b_{0}^{2}}\right)\right] & p=\frac{1}{2}
\end{cases}.
\]
\end{proof}

\subsection{Bound on $\protect\E\left[\protect\bD_{T}^{1-p}\right]$}
\begin{lem}
\label{lem:NA-D-final-bound}Given $p+2q=1$, $p\in\left(0,\frac{1}{2}\right]$,
we have
\begin{align*}
\E\left[\bD_{T}^{1-p}\right] & \leq\left(1+T/a_{0}^{2}\right)^{\frac{1-p}{3}}\left(b_{0}^{\frac{1}{p}-1}+\frac{2}{\eta}\left(F(x_{1})-F^{*}\right)+\frac{\sigma^{2}\left(1+2a_{0}^{2}\log\left(1+T/a_{0}^{2}\right)\right)}{\eta\beta_{\max}\left(1-2/(3a_{0}^{2})\right)}\right)\\
 & \quad+\frac{\left(1+T/a_{0}^{2}\right)^{\frac{1-p}{3}}}{1-p}\left(\frac{3a_{0}^{2}-1}{3a_{0}^{2}-2}4\eta\beta_{\max}\right)^{\frac{1}{p}-1}\log\frac{\left(1+\frac{9a_{0}^{2}-2}{3a_{0}^{2}-2}\left(1+T/a_{0}^{2}\right)^{1/3}\right)\eta\beta_{\max}}{b_{0}}.
\end{align*}
\end{lem}
\begin{proof}
The same as before, we start from Lemma \ref{lem:f-value-analysis}
\begin{align*}
\E\left[a_{T+1}^{q}\bD_{T}^{1-p}\right] & \leq b_{0}^{\frac{1}{p}-1}+\frac{2}{\eta}\left(F(x_{1})-F^{*}\right)\\
 & \quad+\E\left[\sum_{t=1}^{T}\left(\eta\beta_{\max}+\frac{\eta\beta_{\max}}{a_{t+1}^{1/2}\lambda}-b_{t}\right)\frac{\|d_{t}\|^{2}}{b_{t}^{2}}\right]+\frac{\lambda\E\left[\bE_{T,1/2}\right]}{\eta\beta_{\max}}.
\end{align*}
Now we simply take $\lambda=1$ and use Lemma \ref{lem:NA-E-T-1/2}
to get 
\begin{align}
\E\left[a_{T+1}^{q}\bD_{T}^{1-p}\right] & \leq b_{0}^{\frac{1}{p}-1}+\frac{2}{\eta}\left(F(x_{1})-F^{*}\right)+\frac{\sigma^{2}\left(1+2a_{0}^{2}\log\left(1+T/a_{0}^{2}\right)\right)}{\eta\beta_{\max}\left(1-2/(3a_{0}^{2})\right)}\nonumber \\
 & \quad+\E\left[\underbrace{\sum_{t=1}^{T}\left(\left(1+\frac{9a_{0}^{2}-2}{\left(3a_{0}^{2}-2\right)a_{t+1}^{1/2}}\right)\eta\beta_{\max}-b_{t}\right)\frac{\|d_{t}\|^{2}}{b_{t}^{2}}}_{(i)}\right].\label{eq:NA-bound-D}
\end{align}
Applying Lemma \ref{lem:residual-bound} to $(i)$, we get
\[
(i)\leq\frac{\left(\frac{3a_{0}^{2}-1}{3a_{0}^{2}-2}4\eta\beta_{\max}\right)^{\frac{1}{p}-1}}{1-p}\log\frac{\left(1+\frac{9a_{0}^{2}-2}{3a_{0}^{2}-2}\left(1+T/a_{0}^{2}\right)^{1/3}\right)\eta\beta_{\max}}{b_{0}}.
\]
Note that $a_{T+1}^{q}=a_{T+1}^{\frac{1-p}{2}}=\left(1+T/a_{0}^{2}\right)^{-\frac{1-p}{3}}$
is deterministic, by multiplying both sides of (\ref{eq:NA-bound-D})
by $\left(1+T/a_{0}^{2}\right)^{\frac{1-p}{3}}$ , we get the desired
result.
\end{proof}

\subsection{Combine the bounds and the final proof.}

From Lemma \ref{lem:NA-E-final-bound}, we have
\begin{align*}
\E\left[\bE_{T}\right] & \leq\frac{6a_{0}^{2}\sigma^{2}\left(1+T/a_{0}^{2}\right)^{1/3}}{1-2/(3a_{0}^{2})}+\frac{2\eta^{2}\beta^{2}(1+T/a_{0}^{2})^{\frac{2p}{3}}}{1-2/(3a_{0}^{2})}\begin{cases}
\frac{\E\left[\bD_{T}^{1-2p}\right]}{1-2p} & p\neq\frac{1}{2}\\
\E\left[\log\left(1+\frac{\bD_{T}}{b_{0}^{2}}\right)\right] & p=\frac{1}{2}
\end{cases}.
\end{align*}
From Lemma \ref{lem:NA-D-final-bound}, we have
\begin{align*}
\E\left[\bD_{T}^{1-p}\right] & \leq\left(1+T/a_{0}^{2}\right)^{\frac{1-p}{3}}\left(b_{0}^{\frac{1}{p}-1}+\frac{2}{\eta}\left(F(x_{1})-F^{*}\right)+\frac{\sigma^{2}\left(1+2a_{0}^{2}\log\left(1+T/a_{0}^{2}\right)\right)}{\eta\beta_{\max}\left(1-2/(3a_{0}^{2})\right)}\right)\\
 & \quad+\frac{\left(1+T/a_{0}^{2}\right)^{\frac{1-p}{3}}}{1-p}\left(\frac{3a_{0}^{2}-1}{3a_{0}^{2}-2}4\eta\beta_{\max}\right)^{\frac{1}{p}-1}\log\frac{\left(1+\frac{9a_{0}^{2}-2}{3a_{0}^{2}-2}\left(1+T/a_{0}^{2}\right)^{1/3}\right)\eta\beta_{\max}}{b_{0}}.
\end{align*}

Now let 
\[
\hp=1-p\in\left[\frac{1}{2},1\right).
\]
Apply Lemma \ref{lem:decomposition}, we know
\begin{equation}
\E\left[\bH_{T}^{\hp}\right]\leq4\max\left\{ \E\left[\bE_{T}^{\hp}\right],\E\left[\bD_{T}^{\hp}\right]\right\} .\label{eq:NA-H-bound-via-E-D}
\end{equation}
Now we can give the final proof of Theorem \ref{thm:NA-rate}.

\begin{proof}
Now we consider following two cases:

\textbf{Case 1:} $p\neq\frac{1}{2}$. Note that by Holder inequality
\begin{align*}
\E\left[\bE_{T}^{\hp}\right] & =\E^{\hp}\left[\bE_{T}\right],\\
\E\left[\bD_{T}^{1-2p}\right] & \leq\E^{\frac{1-2p}{\hp}}\left[\bD_{T}^{\hp}\right].
\end{align*}
So we know
\begin{align*}
\E\left[\bE_{T}^{\hp}\right] & \leq\left(\frac{6a_{0}^{2}\sigma^{2}\left(1+T/a_{0}^{2}\right)^{1/3}}{1-2/(3a_{0}^{2})}+\frac{2\eta^{2}\beta^{2}(1+T/a_{0}^{2})^{\frac{2p}{3}}}{(1-2/(3a_{0}^{2}))(1-2p)}\E\left[\bD_{T}^{1-2p}\right]\right)^{\hp}\\
 & \leq\left(\frac{6a_{0}^{2}\sigma^{2}\left(1+T/a_{0}^{2}\right)^{1/3}}{1-2/(3a_{0}^{2})}+\frac{2\eta^{2}\beta^{2}(1+T/a_{0}^{2})^{\frac{2p}{3}}}{(1-2/(3a_{0}^{2}))(1-2p)}\E^{\frac{1-2p}{\hp}}\left[\bD_{T}^{\hp}\right]\right)^{\hp}.
\end{align*}
Now if $\E\left[\bE_{T}^{\hp}\right]\geq\E\left[\bD_{T}^{\hp}\right]$,
we know
\begin{align*}
\E\left[\bE_{T}^{\hp}\right] & \leq\left(\frac{6a_{0}^{2}\sigma^{2}\left(1+T/a_{0}^{2}\right)^{1/3}}{1-2/(3a_{0}^{2})}+\frac{2\eta^{2}\beta^{2}(1+T/a_{0}^{2})^{\frac{2p}{3}}}{(1-2/(3a_{0}^{2}))(1-2p)}\E^{\frac{1-2p}{\hp}}\left[\bE_{T}^{\hp}\right]\right)^{\hp}\\
 & \leq\left(\frac{6a_{0}^{2}\sigma^{2}\left(1+T/a_{0}^{2}\right)^{1/3}}{1-2/(3a_{0}^{2})}\right)^{\hp}+\left(\frac{2\eta^{2}\beta^{2}(1+T/a_{0}^{2})^{\frac{2p}{3}}}{(1-2/(3a_{0}^{2}))(1-2p)}\right)^{\hp}\E^{1-2p}\left[\bE_{T}^{\hp}\right]
\end{align*}
Then if $\left(\frac{2\eta^{2}\beta^{2}(1+T/a_{0}^{2})^{\frac{2p}{3}}}{(1-2/(3a_{0}^{2}))(1-2p)}\right)^{\hp}\E^{1-2p}\left[\bE_{T}^{\hp}\right]\leq\left(\frac{6a_{0}^{2}\sigma^{2}\left(1+T/a_{0}^{2}\right)^{1/3}}{1-2/(3a_{0}^{2})}\right)^{\hp}$,
we know
\begin{align*}
\E\left[\bE_{T}^{\hp}\right] & \leq2\left(\frac{6a_{0}^{2}\sigma^{2}\left(1+T/a_{0}^{2}\right)^{1/3}}{1-2/(3a_{0}^{2})}\right)^{\hp}=\left(\frac{2^{\frac{1}{\hp}}6a_{0}^{2}\sigma^{2}}{1-2/(3a_{0}^{2})}\right)^{\hp}\left(1+T/a_{0}^{2}\right)^{\frac{\hp}{3}}
\end{align*}
If $\left(\frac{2\eta^{2}\beta^{2}(1+T/a_{0}^{2})^{\frac{2p}{3}}}{(1-2/(3a_{0}^{2}))(1-2p)}\right)^{\hp}\E^{1-2p}\left[\bE_{T}^{\hp}\right]\geq\left(\frac{6a_{0}^{2}\sigma^{2}\left(1+T/a_{0}^{2}\right)^{1/3}}{1-2/(3a_{0}^{2})}\right)^{\hp}$,
we know
\begin{align*}
\E\left[\bE_{T}^{\hp}\right] & \leq2\left(\frac{2\eta^{2}\beta^{2}(1+T/a_{0}^{2})^{\frac{2p}{3}}}{(1-2/(3a_{0}^{2}))(1-2p)}\right)^{\hp}\E^{1-2p}\left[\bE_{T}^{\hp}\right]\\
\Rightarrow\E\left[\bE_{T}^{\hp}\right] & \leq\left(\frac{2^{\frac{1}{\hp}}2\eta^{2}\beta^{2}}{(1-2/(3a_{0}^{2}))(1-2p)}\right)^{\frac{\hp}{2p}}\left(1+T/a_{0}^{2}\right)^{\frac{\hp}{3}}.
\end{align*}
Hence under $\E\left[\bE_{T}^{\hp}\right]\geq\E\left[\bD_{T}^{\hp}\right]$,
we get
\[
\E\left[\bE_{T}^{\hp}\right]\leq\left(\left(\frac{2^{\frac{1}{\hp}}2\eta^{2}\beta^{2}}{(1-2/(3a_{0}^{2}))(1-2p)}\right)^{\frac{\hp}{2p}}+\left(\frac{2^{\frac{1}{\hp}}6a_{0}^{2}\sigma^{2}}{1-2/(3a_{0}^{2})}\right)^{\hp}\right)\left(1+T/a_{0}^{2}\right)^{\frac{\hp}{3}}.
\]
Then by using (\ref{eq:NA-H-bound-via-E-D}), we know
\begin{align*}
 & \E\left[\bH_{T}^{\hp}\right]\\
\leq & 4\max\left\{ \E\left[\bE_{T}^{\hp}\right],\E\left[\bD_{T}^{\hp}\right]\right\} \\
\leq & 4\left(1+T/a_{0}^{2}\right)^{\frac{\hp}{3}}\\
 & \times\left[\left(\frac{2^{\frac{1}{\hp}}2\eta^{2}\beta^{2}}{(1-2/(3a_{0}^{2}))(1-2p)}\right)^{\frac{\hp}{2p}}+\left(\frac{2^{\frac{1}{\hp}}6a_{0}^{2}\sigma^{2}}{1-2/(3a_{0}^{2})}\right)^{\hp}\right.\\
 & \left.+b_{0}^{\frac{1}{p}-1}+\frac{2}{\eta}\left(F(x_{1})-F^{*}\right)+\frac{\sigma^{2}\left(1+2a_{0}^{2}\log\left(1+T/a_{0}^{2}\right)\right)}{\eta\beta_{\max}\left(1-2/(3a_{0}^{2})\right)}\right.\\
 & \left.+\frac{\left(\frac{3a_{0}^{2}-1}{3a_{0}^{2}-2}4\eta\beta_{\max}\right)^{\frac{1}{p}-1}}{1-p}\log\frac{\left(1+\frac{9a_{0}^{2}-2}{3a_{0}^{2}-2}\left(1+T/a_{0}^{2}\right)^{1/3}\right)\eta\beta_{\max}}{b_{0}}\right]\\
= & O\left(\left(F(x_{1})-F^{*}+\beta^{\frac{\hp}{p}}\log\left(\beta T\right)+\sigma^{2}\log T+\sigma^{2\hp}\right)T^{\frac{\hp}{3}}\right).
\end{align*}

\textbf{Case 2:} $p=\frac{1}{2}$. By a similar proof, we still have
\begin{align*}
\E\left[\bH_{T}^{\hp}\right]\leq & O\left(\left(F(x_{1})-F^{*}+\beta^{\frac{\hp}{p}}\log\left(\beta T\right)+\sigma^{2}\log T+\sigma^{2\hp}\right)T^{\frac{\hp}{3}}\right)
\end{align*}
\end{proof}

\section{Basic inequalities\label{sec:Basic-inequalities}}

In this section, we prove some technical lemmas used in our proof.
\begin{lem}
\label{lem:inequality-1}For $c_{0}>0$, $c_{i\geq1}\geq0$, $p\in(0,1]$,
we have
\[
\sum_{t=1}^{T}\frac{c_{t}}{(c_{0}+\sum_{i=1}^{t}c_{i})^{p}}\leq\begin{cases}
\frac{1}{1-p}\left(\sum_{i=1}^{T}c_{i}\right)^{1-p} & p\neq1\\
\log\left(1+\frac{\sum_{i=1}^{T}c_{i}}{c_{0}}\right) & p=1
\end{cases}.
\]
\end{lem}
\begin{proof}
We first prove the case $p\neq1$. From Lemma 3 in \cite{levy2021storm+},
for $b_{1}>0,b_{i\geq2}\geq0$, $p\in(0,1)$, we have
\[
\sum_{t=1}^{T}\frac{b_{t}}{(\sum_{i=1}^{t}b_{i})^{p}}\leq\frac{1}{1-p}\left(\sum_{i=1}^{T}b_{i}\right)^{1-p}.
\]
Now we define
\[
T_{0}=\min\left\{ t\in\left[T\right]:c_{t}>0\right\} .
\]
By the definition of $T_{0}$, we know for any $1\leq t\leq T_{0}-1$,
$c_{t}=0.$ Then we have
\begin{align*}
\sum_{t=1}^{T}\frac{c_{t}}{(c_{0}+\sum_{i=1}^{t}c_{i})^{p}} & =\sum_{t=1}^{T_{0}-1}\frac{c_{t}}{(c_{0}+\sum_{i=1}^{t}c_{i})^{p}}+\sum_{t=T_{0}}^{T}\frac{c_{t}}{(c_{0}+\sum_{i=1}^{T_{0}-1}c_{i}+\sum_{i=T_{0}}^{t}c_{i})^{p}}\\
 & =\sum_{t=T_{0}}^{T}\frac{c_{t}}{(c_{0}+\sum_{i=T_{0}}^{t}c_{i})^{p}}\leq\sum_{t=T_{0}}^{T}\frac{c_{t}}{(\sum_{i=T_{0}}^{t}c_{i})^{p}}\\
 & \leq\frac{1}{1-p}\left(\sum_{i=T_{0}}^{T}c_{i}\right)^{1-p}=\frac{1}{1-p}\left(\sum_{i=1}^{T}c_{i}\right)^{1-p}.
\end{align*}

For $p=1$, we know
\begin{align*}
\sum_{t=1}^{T}\frac{c_{t}}{c_{0}+\sum_{i=1}^{t}c_{i}} & =\sum_{t=1}^{T}1-\frac{c_{0}+\sum_{i=1}^{t-1}c_{i}}{c_{0}+\sum_{i=T_{0}}^{t}c_{i}}\\
 & \leq\sum_{t=1}^{T}\log\frac{c_{0}+\sum_{i=T_{0}}^{t}c_{i}}{c_{0}+\sum_{i=1}^{t-1}c_{i}}\\
 & =\log\left(1+\frac{\sum_{i=1}^{T}c_{i}}{c_{0}}\right),
\end{align*}
where the inequality holds by $1-\frac{1}{x}\leq\log x$.
\end{proof}

\begin{lem}
\label{lem:inequality-2}For $c_{0}>0$, $c_{i\geq1}\in(0,c]$, $p\in(0,1]$,
we have
\[
\sum_{t=1}^{T}\frac{c_{t+1}}{(c_{0}+\sum_{i=1}^{t}c_{i})^{p}}\leq\frac{3c}{c_{0}^{p}}+\begin{cases}
\frac{1}{1-p}\left(\sum_{i=1}^{T}c_{i}\right)^{1-p} & p\neq1\\
\log\left(1+\frac{\sum_{i=1}^{T}c_{i}}{c_{0}}\right) & p=1
\end{cases}.
\]
\end{lem}
\begin{proof}
Define
\[
T_{0}=\min\left\{ t\in\left[T\right],\sum_{i=1}^{t}c_{i}\geq c\right\} ,
\]
then we know
\begin{align*}
\sum_{t=1}^{T}\frac{c_{t+1}}{(c_{0}+\sum_{i=1}^{t}c_{i})^{p}} & \leq\sum_{t=1}^{T-1}\frac{c_{t+1}}{(c_{0}+\sum_{i=1}^{t}c_{i})^{p}}+\frac{c}{c_{0}^{p}}\\
 & =\frac{c}{c_{0}^{p}}+\sum_{t=1}^{T_{0}-1}\frac{c_{t+1}}{(c_{0}+\sum_{i=1}^{t}c_{i})^{p}}+\sum_{t=T_{0}}^{T-1}\frac{c_{t+1}}{(c_{0}+\sum_{i=1}^{T_{0}}c_{i}+\sum_{i=T_{0}+1}^{t}c_{i})^{p}}\\
 & \leq\frac{c}{c_{0}^{p}}+\sum_{t=1}^{T_{0}-1}\frac{c_{t+1}}{c_{0}^{p}}+\sum_{t=T_{0}}^{T-1}\frac{c_{t+1}}{(c_{0}+c+\sum_{i=T_{0}+1}^{t}c_{i})^{p}}\\
 & \leq\frac{3c}{c_{0}^{p}}+\sum_{t=T_{0}}^{T-1}\frac{c_{t+1}}{(c_{0}+\sum_{i=T_{0}+1}^{t+1}c_{i})^{p}}\\
 & \overset{(a)}{\leq}\frac{3c}{c_{0}^{p}}+\begin{cases}
\frac{1}{1-p}\left(\sum_{i=T_{0}+1}^{T}c_{i}\right)^{1-p} & p\neq1\\
\log\left(1+\frac{\sum_{i=T_{0}+1}^{T}c_{i}}{c_{0}}\right) & p=1
\end{cases}\\
 & \leq\frac{3c}{c_{0}^{p}}+\begin{cases}
\frac{1}{1-p}\left(\sum_{i=1}^{T}c_{i}\right)^{1-p} & p\neq1\\
\log\left(1+\frac{\sum_{i=1}^{T}c_{i}}{c_{0}}\right) & p=1
\end{cases}
\end{align*}
where $(a)$ is by Lemma \ref{lem:inequality-1}.
\end{proof}

\begin{lem}
\label{lem:inequality-3}For $c_{0}>0$, $c_{i\geq1}\in(0,c]$, $p\in(0,1]$,
we have
\[
\sum_{t=1}^{T}\frac{c_{t+1}}{(c_{0}+\sum_{i=1}^{t-1}c_{i})^{p}}\leq\frac{6c}{c_{0}^{p}}+\begin{cases}
\frac{1}{1-p}\left(\sum_{i=1}^{T}c_{i}\right)^{1-p} & p\neq1\\
\log\left(1+\frac{\sum_{i=1}^{T}c_{i}}{c_{0}}\right) & p=1
\end{cases}.
\]
\end{lem}
\begin{proof}
Define 
\[
T_{0}=\min\left\{ t\in\left[T\right]:\sum_{i=1}^{t-1}c_{i}\geq c\right\} .
\]
Then we know
\begin{align*}
\sum_{t=1}^{T}\frac{c_{t+1}}{(c_{0}+\sum_{i=1}^{t-1}c_{i})^{p}} & =\sum_{t=1}^{T_{0}-1}\frac{c_{t+1}}{(c_{0}+\sum_{i=1}^{t}c_{i})^{p}}+\sum_{t=T_{0}}^{T}\frac{c_{t+1}}{(c_{0}+\sum_{i=1}^{T_{0}-1}c_{i}+\sum_{i=T_{0}}^{t-1}c_{i})^{p}}\\
 & \leq\sum_{t=1}^{T_{0}-1}\frac{c_{t+1}}{c_{0}^{p}}+\sum_{t=T_{0}}^{T}\frac{c_{t+1}}{(c_{0}+c+\sum_{i=T_{0}}^{t-1}c_{i})^{p}}\\
 & \leq\frac{3c}{c_{0}^{p}}+\sum_{t=T_{0}}^{T}\frac{c_{t+1}}{(c_{0}+\sum_{i=T_{0}}^{t}c_{i})^{p}}\\
 & \overset{(a)}{\leq}\frac{6c}{c_{0}^{p}}+\begin{cases}
\frac{1}{1-p}\left(\sum_{i=T_{0}}^{T}c_{i}\right)^{1-p} & p\neq1\\
\log\left(1+\frac{\sum_{i=T_{0}}^{T}c_{i}}{c_{0}}\right) & p=1
\end{cases}\\
 & \leq\frac{6c}{c_{0}^{p}}+\begin{cases}
\frac{1}{1-p}\left(\sum_{i=1}^{T}c_{i}\right)^{1-p} & p\neq1\\
\log\left(1+\frac{\sum_{i=1}^{T}c_{i}}{c_{0}}\right) & p=1
\end{cases}
\end{align*}
where $(a)$ is by Lemma \ref{lem:inequality-2}.
\end{proof}

\begin{lem}
\label{lem:inequality-4}(Lemma 6 in \cite{levy2021storm+}), for
$c_{i\geq1}\in(0,c]$, we have
\[
\sum_{t=1}^{T}\frac{c_{t}}{(1+\sum_{i=1}^{t-1}c_{i})^{4/3}}\leq12+2c.
\]
\end{lem}
\begin{lem}
\label{lem:inequality-5}For $c_{i\geq1}\in(0,c]$, we have, we have
\[
\sum_{t=1}^{T}\frac{c_{t+1}}{(1+\sum_{i=1}^{t-1}c_{i})^{4/3}}\leq12+5c.
\]
\end{lem}
\begin{proof}
Define
\[
T_{0}=\min\left\{ t\in\left[T\right]:\sum_{i=1}^{t-1}c_{i}\geq c\right\} .
\]
Then we know
\begin{align*}
\sum_{t=1}^{T}\frac{c_{t+1}}{(1+\sum_{i=1}^{t-1}c_{i})^{4/3}} & =\sum_{t=1}^{T_{0}-1}\frac{c_{t+1}}{(1+\sum_{i=1}^{t-1}c_{i})^{4/3}}+\sum_{t=T_{0}}^{T}\frac{c_{t+1}}{(1+\sum_{i=1}^{t-1}c_{i})^{4/3}}\\
 & \leq\sum_{t=1}^{T_{0}-1}c_{t+1}+\sum_{t=T_{0}}^{T}\frac{c_{t+1}}{(1+\sum_{i=1}^{T_{0}-1}c_{i}+\sum_{i=T_{0}}^{t-1}c_{i})^{4/3}}\\
 & \leq3c+\sum_{t=T_{0}}^{T}\frac{c_{t+1}}{(1+c+\sum_{i=T_{0}}^{t-1}c_{i})^{4/3}}\\
 & \leq3c+\sum_{t=T_{0}}^{T}\frac{c_{t+1}}{(1+\sum_{i=T_{0}}^{t}c_{i})^{4/3}}\\
 & \leq12+5c,
\end{align*}
where the last inequality is by Lemma \ref{lem:inequality-4}.
\end{proof}

\begin{lem}
\label{lem:cauchy-bound-g-1}Given $0\leq x\leq y\leq1$, $0<\ell\leq1$,
we have
\[
\left(\frac{(1-x^{1/\ell})^{2}}{x^{2}}-\frac{(1-y^{1/\ell})^{2}}{y^{2}}\right)^{2}\leq\frac{y^{2}-x^{2}}{\ell^{2}x^{4}y^{2}}.
\]
\end{lem}
\begin{proof}
Note that
\begin{align*}
\left(\frac{(1-x^{1/\ell})^{2}}{x^{2}}-\frac{(1-y^{1/\ell})^{2}}{y^{2}}\right)^{2} & =\left(\frac{1-x^{1/\ell}}{x}+\frac{1-y^{1/\ell}}{y}\right)^{2}\left(\frac{1-x^{1/\ell}}{x}-\frac{1-y^{1/\ell}}{y}\right)^{2}\\
 & \leq\left(\frac{1}{x}+\frac{1}{y}\right)^{2}\left(\frac{1-x^{1/\ell}}{x}-\frac{1-y^{1/\ell}}{y}\right)^{2},
\end{align*}
now let $h(x)=\frac{1-x^{1/l}}{x}$, we can find $h'(x)=-\frac{(1-\ell)x^{1/\ell}+\ell}{\ell x^{2}}\leq0$.
Hence
\[
\frac{1-x^{1/\ell}}{x}-\frac{1-y^{1/\ell}}{y}=h(x)-h(y)\geq0.
\]
Besides, let $g(x)=h(x)-\frac{1}{\ell x}$, we can find that
\[
g'(x)=\frac{\left(1-\ell\right)\left(1-x^{1/\ell}\right)}{\ell x^{2}}\geq0.
\]
This means
\[
h(x)-\frac{1}{\ell x}-h(y)+\frac{1}{\ell y}=g(x)-g(y)\leq0,
\]
which implies
\[
0\leq h(x)-h(y)\leq\frac{1}{\ell x}-\frac{1}{\ell y}.
\]
Thus we finally have
\begin{align*}
\left(\frac{(1-x^{1/\ell})^{2}}{x^{2}}-\frac{(1-y^{1/\ell})^{2}}{y^{2}}\right)^{2} & \leq\left(\frac{1}{x}+\frac{1}{y}\right)^{2}\left(h(x)-h(y)\right)^{2}\\
 & \leq\left(\frac{1}{x}+\frac{1}{y}\right)^{2}\left(\frac{1}{\ell x}-\frac{1}{\ell y}\right)^{2}\\
 & =\frac{\left(y^{2}-x^{2}\right)^{2}}{\ell^{2}x^{4}y^{4}}\\
 & \leq\frac{y^{2}-x^{2}}{\ell^{2}x^{4}y^{2}}.
\end{align*}
\end{proof}

\begin{lem}
\label{lem:cauchy-bound-g-2}Given $0\leq x\leq y\leq1$, $0<\ell\leq\frac{1}{2}$,
we have
\[
\left((1-x^{1/\ell})x^{1/\ell-2}-(1-y^{1/\ell})y^{1/\ell-2}\right)^{2}\leq\frac{y^{2}-x^{2}}{\ell^{2}x^{2}}y^{2/\ell-4}.
\]
\end{lem}
\begin{proof}
If $\ell=\frac{1}{2}$, then we know
\begin{align*}
 & \left((1-x^{1/\ell})x^{1/\ell-2}-(1-y^{1/\ell})y^{1/\ell-2}\right)^{2}\\
= & \left(y^{2}-x^{2}\right)^{2}\leq\left(y^{2}-x^{2}\right)y^{2}\leq\frac{4\left(y^{2}-x^{2}\right)}{x^{2}y^{2}}y^{2}\\
= & \frac{y^{2}-x^{2}}{\ell^{2}x^{2}y^{2}}y^{2/\ell-4}.
\end{align*}
If $\ell\neq\frac{1}{2}$, let $h(x)$ denote $(1-x^{1/\ell})x^{1/\ell-2}$,
then we know $h'(x)=x^{1/\ell-3}\frac{2\left(\ell-1\right)x^{1/\ell}-2\ell+1}{\ell}$.
By Taylor's expansion, there exists $x\leq z\leq y$, such that
\begin{align*}
h(x)-h(y) & =h'(z)(x-y)\\
 & =z^{1/\ell-3}\frac{2\left(\ell-1\right)z^{1/\ell}-2\ell+1}{\ell}(x-y).
\end{align*}
This will give us
\begin{align*}
\left((1-x^{1/\ell})x^{1/\ell-2}-(1-y^{1/\ell})y^{1/\ell-2}\right)^{2} & =(h(x)-h(y))^{2}\\
 & =z^{2/\ell-6}\times\frac{\left(2\left(\ell-1\right)z^{1/\ell}-2\ell+1\right)^{2}}{\ell^{2}}\times\left(y-x\right)^{2}\\
 & \leq\frac{y^{2/\ell-4}}{x^{2}}\times\frac{1}{\ell^{2}}\times\left(y^{2}-x^{2}\right)\\
 & =\frac{y^{2}-x^{2}}{\ell^{2}x^{2}}y^{2/\ell-4}.
\end{align*}
\end{proof}

\begin{lem}
\label{lem:inequality-poly-g}Given $m,n\geq0$, For $0\leq x\leq m$,
we have
\[
(m-x)x^{n}\leq\left(\frac{m}{n+1}\right)^{n+1}n^{n}.
\]
\end{lem}
\begin{proof}
Note that
\begin{align*}
\log\left(\left(m-x\right)x^{n}\right) & =\log\left(m-x\right)+n\log x=\log\left(m-x\right)+n\log\frac{x}{n}+n\log n\\
 & \overset{(a)}{\leq}\left(n+1\right)\log\left(\frac{m-x}{n+1}+\frac{n}{n+1}\times\frac{x}{n}\right)+n\log n\\
 & =\left(n+1\right)\log\frac{m}{n+1}+n\log n=\log\left(\left(\frac{m}{n+1}\right)^{n+1}n^{n}\right)
\end{align*}
where $(a)$ is by the concavity of $\log$ function. Then we know
$\left(m-x\right)x^{n}\leq\left(\frac{m}{n+1}\right)^{n+1}n^{n}.$
\end{proof}

\begin{lem}
\label{lem:inequality-log-bound-g}Given $X,A,B\geq0,C>0,D\geq0,0\leq u\leq1$,if
we have
\[
X\leq\left(A+B\log\left(1+\frac{X}{C}\right)\right)^{u}D,
\]
then there is
\[
X\leq\left(2A+2B\log\frac{4uBD}{C}+\left(\frac{C}{D}\right)^{1/u}\right)^{u}D.
\]
Especially, when $D\geq1$, we know
\[
X\leq\left(2A+2B\log\frac{4uBD}{C}+C^{1/u}\right)^{u}D.
\]
\end{lem}
\begin{proof}
Let $Y=(X/D)^{1/u}$, then we know
\begin{align*}
Y & \leq A+B\log\left(1+\frac{DY^{u}}{C}\right)\\
 & =A+uB\log\left(1+\frac{DY^{u}}{C}\right)^{1/u}\\
 & \overset{(a)}{\leq}A+uB\log\left(2^{1/u}+\left(\frac{2D}{C}\right)^{1/u}Y\right)\\
 & =A+uB\log2^{1/u}+uB\log\left(1+\left(\frac{D}{C}\right)^{1/u}Y\right)\\
 & =A+B\log2+uB\log\frac{1+\left(\frac{D}{C}\right)^{1/u}Y}{2uB\left(\frac{D}{C}\right)^{1/u}}+uB\log\left(2uB\left(\frac{D}{C}\right)^{1/u}\right)\\
 & \overset{(b)}{\leq}A+B\log2+\frac{\left(C/D\right)^{1/u}}{2}+\frac{Y}{2}+uB\log2uB+B\log\frac{D}{C}\\
 & \leq\frac{Y}{2}+A+B\log\frac{4uBD}{C}+\frac{\left(C/D\right)^{1/u}}{2},
\end{align*}
where $(a)$ is by $\left(x+y\right)^{p}\leq\left(2x\right)^{p}+\left(2y\right)^{p}$,
for $x,y\geq0,p\geq1$. $(b)$ is by $\log x\leq x-1\leq x$. Then
we know
\begin{align*}
Y & \leq2A+2B\log\frac{4uBD}{C}+\left(\frac{C}{D}\right)^{1/u}\\
\Rightarrow X & \leq\left(2A+2B\log\frac{4uBD}{C}+\left(\frac{C}{D}\right)^{1/u}\right)^{u}D.
\end{align*}
\end{proof}

\end{document}